\documentclass[]{fairmeta}

\usepackage{amsmath,amsthm,amssymb,bm}
\usepackage{color}
\usepackage{algorithm}
\usepackage[noend]{algpseudocode}
\usepackage[ruled,algo2e]{algorithm2e}%
\usepackage{derivative}
\usepackage{longtable}
\usepackage{pifont}
\usepackage{mdframed}
\usepackage{enumitem}
\usepackage{mathrsfs}
\usepackage{nicefrac}
\usepackage{subcaption}
\usepackage{float}
\newfloat{algorithm}{t}{lop}

\usepackage[normalem]{ulem}

\usepackage{siunitx}
\usepackage{tocloft}

\let\oldaddcontentsline\addcontentsline
\newcommand{\stoptocentries}{\renewcommand{\addcontentsline}[3]{}}
\newcommand{\starttocentries}{\let\addcontentsline\oldaddcontentsline}

\usepackage{tikz}
\usetikzlibrary{positioning, fit, calc}

\usepackage{booktabs}

\usepackage{arydshln} % Required for dashed lines

\usepackage{graphicx}
\usepackage{mathtools}
\usepackage{amsfonts}
\usepackage{color}
\usepackage{dsfont}
\usepackage{thmtools} 
\newcommand{\R}{\mathbb{R}}

\newcommand{\norm}[1]{\lVert#1\rVert}

\newcommand{\E}{\mathbb{E}}

\renewcommand{\phi}{\varphi}

\graphicspath {{figures/}}

\usepackage{hyperref}
\usepackage{caption}
\usepackage[super]{nth}
\newtheorem{lemma}{Lemma}
\newtheorem{definition}{Definition}
\newtheorem{theorem}{Theorem}

\newtheorem{proposition}{Proposition}

\usepackage{subcaption}

\DeclareMathOperator*{\argmin}{argmin}

\newenvironment{talign*}
 {\csname align*\endcsname}
 {\endalign}
\newenvironment{talign}
 {\csname align\endcsname}
 {\endalign}



\makeatletter
\DeclareRobustCommand{\cev}[1]{%
  {\mathpalette\do@cev{#1}}%
}
\newcommand{\do@cev}[2]{%
  \vbox{\offinterlineskip
    \sbox\z@{$\m@th#1 x$}%
    \ialign{##\cr
      \hidewidth\reflectbox{$\m@th#1\vec{}\mkern4mu$}\hidewidth\cr
      \noalign{\kern-\ht\z@}
      $\m@th#1#2$\cr
    }%
  }%
}
\makeatother

\definecolor{mygray}{gray}{0.95}
\newcommand{\greybox}[1]{
\vspace{-0.9em}
\begin{center}			% Centering minipage
\vspace{-0.5em}
\colorbox{mygray} {		% Set's the color of minipage
\begin{minipage}{0.987\linewidth} 	% Starts minipage
\centering
\vspace{-0.8em}
{#1}
\end{minipage}}			% End minipage
\end{center}
\vspace{-0.5em}
}

\newcommand{\fX}{\bm{X}}
\usepackage{xspace}
\newcommand*{\eg}{{\it e.g.}\@\xspace}
\newcommand*{\ie}{{\it i.e.}\@\xspace}
\DeclarePairedDelimiterX{\infdivx}[2]{(}{)}{%
  #1\;\delimsize\|\;#2%
}
\newcommand{\infdiv}{D_\text{KL}\infdivx}

\newcommand*{\tran}{^{\mkern-1.5mu\mathsf{T}}}

\definecolor{mygray}{gray}{0.95}
\newcommand{\graybox}[1]{%
\vspace{-1em} 
\begin{center}			% Centering minipage
\colorbox{mygray} {		% Set's the color of minipage
\begin{minipage}{0.987\linewidth} 	% Starts minipage
\centering
\vspace{-1em}   
{#1}    
\end{minipage}}			% End minipage
\end{center}
}

\title{Adjoint Matching: Fine-tuning Flow and Diffusion Generative Models with Memoryless Stochastic Optimal Control}

\author[1]{Carles Domingo-Enrich}
\author[1]{Michal Drozdzal}
\author[1]{Brian Karrer}
\author[1]{Ricky T. Q. Chen}
\affiliation[1]{FAIR, Meta}

\abstract{
    Dynamical generative models that produce samples through an iterative process, such as Flow Matching and denoising diffusion models, have seen widespread use, but there have not been many theoretically-sound methods for improving these models with reward fine-tuning.
    In this work, we cast reward fine-tuning as stochastic optimal control (SOC). 
    Critically, we prove that a very specific \emph{memoryless} noise schedule must be enforced during fine-tuning, in order to account for the dependency between the noise variable and the generated samples.
    We also propose a new algorithm named \emph{Adjoint Matching} which outperforms existing SOC algorithms, by casting SOC problems as a regression problem. 
    We find that our approach significantly improves over existing methods for reward fine-tuning, achieving better consistency, realism, and generalization to unseen human preference reward models, while retaining sample diversity.
}

\stoptocentries

\correspondence{Carles Domingo-Enrich at \email{cd2754@nyu.edu}}

\begin{document}

\maketitle

\section{Introduction}
\label{sec:intro}

\begin{figure}[b!]
    \vspace{-1.5em}
    \centering
    \begin{subfigure}[t]{0.495\linewidth}
        \centering
        \caption*{Base model (Flow Matching) w/ Guidance}
        \includegraphics[width=0.32\linewidth]{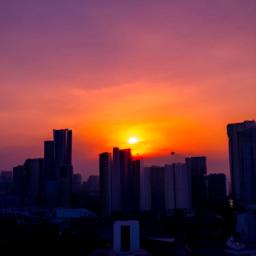}\;%
        \includegraphics[width=0.32\linewidth]{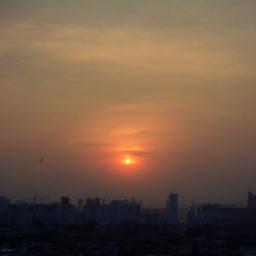}\;%
        \includegraphics[width=0.32\linewidth]{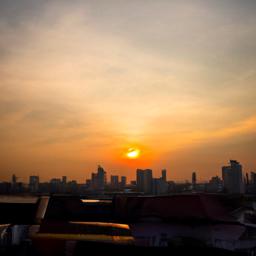}\\
        \includegraphics[width=0.32\linewidth]{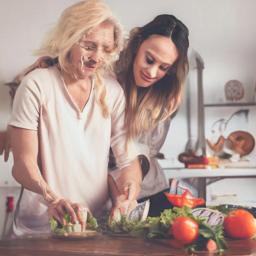}\;%
        \includegraphics[width=0.32\linewidth]{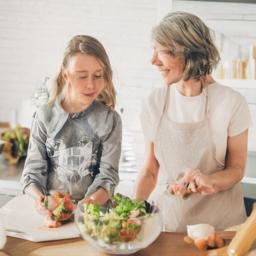}\;%
        \includegraphics[width=0.32\linewidth]{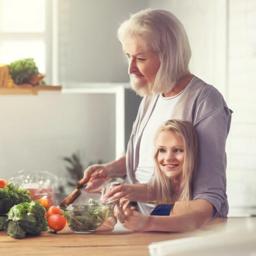}
    \end{subfigure}
    \begin{subfigure}[t]{0.495\linewidth}
        \centering
        \caption*{\textbf{Adjoint Matching (Ours)}}
        \includegraphics[width=0.32\linewidth]{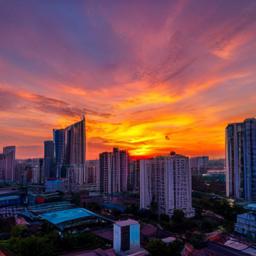}\;%
        \includegraphics[width=0.32\linewidth]{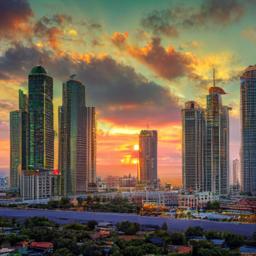}\;%
        \includegraphics[width=0.32\linewidth]{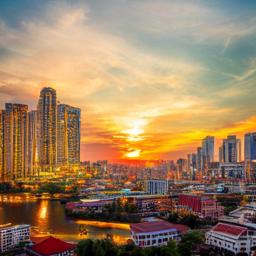}\\
        \includegraphics[width=0.32\linewidth]{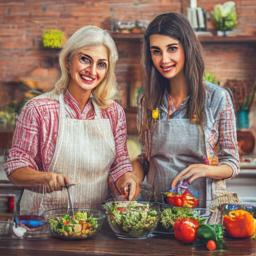}\;%
        \includegraphics[width=0.32\linewidth]{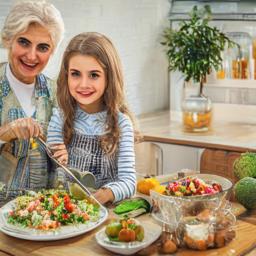}\;%
        \includegraphics[width=0.32\linewidth]{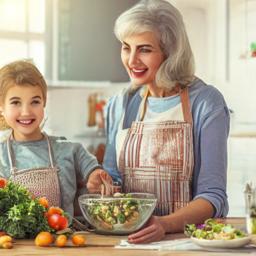}
    \end{subfigure}\\
    \begin{subfigure}[t]{0.495\linewidth}
        \centering
        \includegraphics[width=0.32\linewidth]{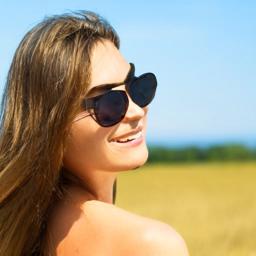}\;%
        \includegraphics[width=0.32\linewidth]{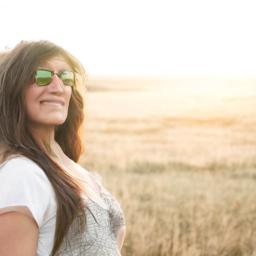}\;%
        \includegraphics[width=0.32\linewidth]{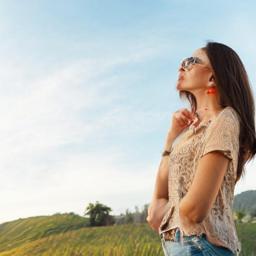}\\
        \includegraphics[width=0.32\linewidth]{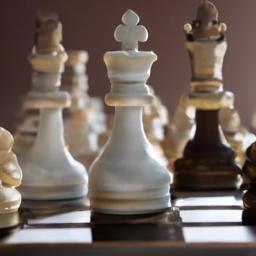}\;%
        \includegraphics[width=0.32\linewidth]{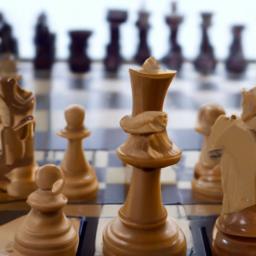}\;%
        \includegraphics[width=0.32\linewidth]{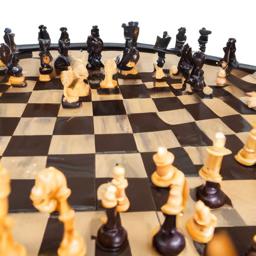}
    \end{subfigure}
    \begin{subfigure}[t]{0.495\linewidth}
        \centering
        \includegraphics[width=0.32\linewidth]{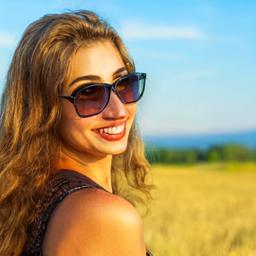}\;%
        \includegraphics[width=0.32\linewidth]{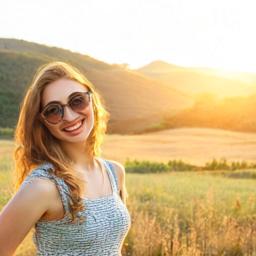}\;%
        \includegraphics[width=0.32\linewidth]{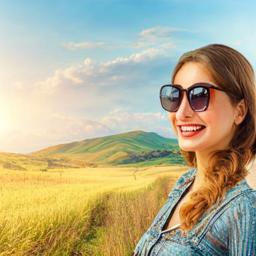}\\
        \includegraphics[width=0.32\linewidth]{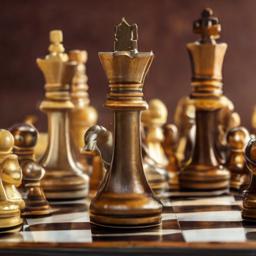}\;%
        \includegraphics[width=0.32\linewidth]{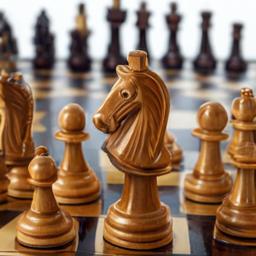}\;%
        \includegraphics[width=0.32\linewidth]{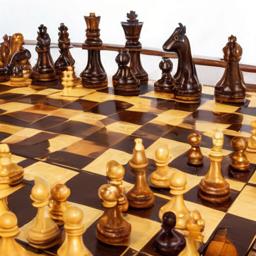}
    \end{subfigure}\\
    \caption{We introduce Adjoint Matching, a theoretically-driven yet simple algorithm for reward fine-tuning that works for a large family of dynamical generative models, including for the first time, Flow Matching models. 
    Text prompts: ``\textit{Beautiful colorful sunset midst of building in Bangkok Thailand}'', ``\textit{Beautiful grandma and granddaughter are mixing salad and smiling while cooking in kitchen}'', ``\textit{The beautiful young woman in sunglasses is standing at the background of field and hill. She is smiling and looking over shoulder}'', ``\textit{Chess, intellectual games, figure horse, chess board}''. 
    }
    \vspace{-6em}
    \label{fig:fig1}
\end{figure}

Flow Matching \citep{lipman2023flow,albergo2023building,liu2023flow} and denoising diffusion
\citep{song2019generative,ho2020denoising,song2021scorebased,kingma2021ondensity}
models are being used for many generative modeling applications, including text-to-image \citep{rombach2022high,esser2024scaling}, text-to-video \citep{singer2022make}, and text-to-audio \citep{le2024voicebox,vyas2023audiobox}. 
In most cases, the base generative model does not achieve the desired sample quality.
To improve the generated samples, it is common to resort to techniques such as classifier-free guidance \citep{ho2022classifier,zheng2023guided} to get better text-to-sample alignment, or to fine-tune using human preference reward models to improve sample quality and realism \citep{wallace2023diffusion,clark2024directly}.

In the adjacent field of large language models, the behavior of the model is aligned to human preferences through fine-tuning with reinforcement learning from human feedback (RLHF).  
Either explicitly or implicitly, RLHF methods \citep{ziegler2020finetuning,stiennon2020learning,ouyang2022training,bai2022training} assume a reward model $r(x)$ that captures human preferences, with the goal of modifying the base generative model such that it generates the following \emph{tilted distribution}: 
\begin{talign} \label{eq:p_star_info}
p^*(x) \propto p^{\mathrm{base}}(x) \exp(r(x)),
\end{talign}
where $p_{\mathrm{base}}$ is the base generative model's sample distribution. 

Inspired by this,
fine-tuning methods have been developed to improve denoising diffusion models based on human preference data; either using a reward-based approach \citep{fan2023optimizing,black2024training,fan2023dpok,xu2023imagereward,clark2024directly,uehara2024understanding,uehara2024finetuning},
or direct preference optimization \citep{wallace2023diffusion}. 
However, unlike the fine-tuning methods designed for large language models,
most of the existing methods to a large degree ignore $p^\text{base}$ and focus solely on the reward model.
Reward models can range from standard evaluation metrics such as ClipScore \citep{hessel2021clipscore,kirstain2023pickapic} to specialized models that have been trained on human preferences \citep{schuhmann2022laion,xu2023imagereward,wu2023humanpreferencescorev2}. As these are parameterized by neural networks, they fall pray to adversarial examples which lead to the generation of undesirable artifacts \citep{goodfellow2014explaining,mordvintsev2015inceptionism}. This has led some works to consider adding regularization during fine-tuning \citep{fan2024reinforcement,uehara2024finetuning} to incentivize staying close to the base model distribution; however, there does not yet exist a \emph{simple} approach which actually provably generates from the tilted distribution \eqref{eq:p_star_info}. 

The main contributions of our paper are as follows:
\begin{enumerate}[label=(\roman*)]
    \item We present a stochastic optimal control (SOC) formulation for reward fine-tuning of dynamical generative models. Importantly, we prove that the na\"ive approach considered by prior works lead to a \emph{value function bias} problem that biases the fine-tuned model away from the tilted distribution \eqref{eq:p_star_info}. 
    This problem has also been observed by \citet{uehara2024finetuning} but they propose a more complicated solution which involves training a separate generative model for the optimal noise distribution.
    \item Instead, we propose a very simple solution: the \emph{memoryless noise schedule}. This is a unique noise schedule that completely removes the dependency between noise variables and the generated samples, resulting in provable convergence to the tilted distribution. 
    This allows us to fine-tune dynamical generative models in full generality, including being the first to fine-tune noiseless Flow Matching models.
    \item We also propose a new method for solving SOC problems, called \emph{Adjoint Matching}, which combines the scalability of gradient-based methods and the simplicity of a least-squares regression objective. This is orthogonal to the reward fine-tuning application and can be applied to general SOC problems.
    \item We perform extensive comparisons to baseline approaches, and analyze them from multiple perspectives such as realism, consistency, and diversity. We find that our proposed method provides generalization to unseen human preference reward models, better text-to-sample consistency, and retains good diversity.
\end{enumerate}
In the following, sections are broken down as follows: \Cref{sec:prelim_generative_models} summarizes the algorithms used for sampling from pre-trained Flow Matching and diffusion models, while \Cref{sec:common_perspective} provides a common notation that we will use throughout. \Cref{sec:memoryless_SOC,sec:adjoint_matching} form the core of our contributions. \Cref{sec:memoryless_SOC} details the value function bias problem and our proposed solution via the memoryless noise schedule. \Cref{sec:adjoint_matching} details the new Adjoint Matching algorithm for solving SOC problems. 
%Within each section, we first present preliminaries that summarize the existing methods and their relation to our contributions. 

\section{Preliminaries on dynamical generative models}\label{sec:prelim_generative_models}

We are interested in fine-tuning base generative models $p^{\mathrm{base}}(X_1)$ where samples are generated through the simulation of a stochastic process. That is, these models transform noise variables into a sample through an iterative process. In particular, we discuss the specific constructions and sampling processes of Flow Matching~\citep{lipman2023flow,liu2023flow,liu2022rectified,albergo2023building} and Denoising Diffusion Models~\citep{ho2020denoising,song2021scorebased,song2021denoising}.
The goal of this section is to provide background information on these methods, which we will later unify into a single consistent notation in \Cref{sec:common_perspective}.

Given random variables from an initial distribution $\bar{X}_0 \sim p_0 = \mathcal{N}(0, I)$, and $\bar{X}_1$ which are distributed according to some data distribution, we define the reference flow $\bm{\bar{X}} = (\bar{X}_t)_{t\in[0,1]}$ where
\begin{talign} \label{eq:reference_flow}
\bar{X}_t = \beta_t \bar{X}_0 + \alpha_t \bar{X}_1,
\end{talign}
where $(\alpha_t)_{t \in [0,1]}, (\beta_t)_{t \in [0,1]}$ are functions such that $\alpha_0 = \beta_1 = 0$ and $\alpha_1 = \beta_0 = 1$. 
Diffusion models and Flow Matching construct generative Markov processes $X_t$ with initial distribution $X_0 \sim \mathcal{N}(0, I)$ that result in flows $\bm{X} = (X_t)_{t \in [0,1]}$ with the same time marginals as the reference flow $\bm{\bar{X}}$, \ie, the random variables $X_t$ and $\bar{X}_t$ have identical distribution for all times $t \in [0, 1]$. 
This implies $X_1$ has the same distribution as the data distribution, so simulating the Markov process from random noise $X_0$ is a way to generate artificial samples\footnote{In our derivations, we will simply assume the base model has been trained perfectly during the pre-training phase.}.

\subsection{Flow Matching}

In its simplest form, the generative Markov process of a Flow Matching model is an ordinary differential equation (ODE) of the form:
\begin{talign} \label{eq:FM_ode}
    \mathrm{d}X_t = v
    %^\text{base}
    (X_t,t) \, \mathrm{d}t, \qquad X_0 \sim \mathcal{N}(0,I).
\end{talign}
where $v(X_t, t)$ is a parametric velocity that is optimized to match the derivative of the reference flow, \ie,
%\begin{talign}
$v
% ^\text{base}
(X_t, t) = \argmin_{\hat{v}} \mathbb{E} \big\| \hat{v}(\bar{X}_t,t) - \frac{\mathrm{d}}{\mathrm{d}t}\bar{X}_t \big\|^2$ (see \eg \citet{lipman2023flow} for details on pre-training Flow Matching models).
%\end{talign}
It can then be proven that the solution of the generative process \eqref{eq:FM_ode} has the same time marginals as the reference flow \citep{lipman2023flow,liu2022rectified,albergo2023building}, and a commonly used choice is $\alpha_t = t$ and $\beta_t = 1-t$.
One can also consider a family of stochastic differential equations (SDEs) with an arbitrary state-independent diffusion coefficient%
\footnote{We use the common short-hand ``over-dot'' notation to denote the time derivative, \ie, $\dot{x}_t = \frac{\mathrm{d}}{\mathrm{d}t} x_t$.}%
:
\begin{talign} \label{eq:FM_general_diffusion_coeff}
    \mathrm{d}X_t = \left( v
    % ^\text{base}
    (X_t,t) + \frac{\sigma(t)^2}{2\beta_{t}(\frac{\dot{\alpha}_{t}}{\alpha_{t}} \beta_{t} -\dot{\beta}_{t})} \left( v
    % ^\text{base}
    (X_t,t) - \frac{\dot{\alpha}_{t}}{\alpha_{t}} X_t \right) \right) \, \mathrm{d}t + \sigma(t) \, \mathrm{d}B_t, \qquad X_0 \sim \mathcal{N}(0,I),
\end{talign}
where $(B_t)_{t\geq 0}$ is a Brownian motion.
The generative processes in \eqref{eq:FM_ode} and \eqref{eq:FM_general_diffusion_coeff} have the same time marginals. This can be seen by writing down the Fokker-Planck equations for \eqref{eq:FM_ode} and \eqref{eq:FM_general_diffusion_coeff}, and observing that they are the same up to a cancellation of terms~\citep{maoutsa2020interacting}.
The diffusion coefficient $\sigma(t)$ in \eqref{eq:FM_general_diffusion_coeff} is compensated by the second term in the drift which scales proportionally as $\sigma(t)^2$. 

% To simulate the SDE \eqref{eq:FM_general_diffusion_coeff}, a good approach is to consider a $K$-step discretization: for $k\in \{0,\dots,K\}$,
% \begin{talign} \label{eq:FM_discretization}
%     X_{k+1} = X_k + \frac{1}{K} v
%     % ^\text{base}
%     (X_k,\frac{k}{K}) + \frac{K \sigma_k^2}{2\beta_{k}\left(\frac{\alpha_{k+1} - \alpha_k}{\alpha_{k}} \beta_{k} - \beta_{k+1} + \beta_{k}\right)} \left( v
%     % ^\text{base}
%     (X_k,\frac{k}{K}) - \frac{\alpha_{k+1} - \alpha_k}{K\alpha_{k}} X_k \right) +
%     \frac{1}{\sqrt{K}} \sigma_{k} \varepsilon_{k},
%     \qquad \substack{\varepsilon_k \sim \mathcal{N}(0,I), \\ X_0 \sim \mathcal{N}(0,I),}
% \end{talign}
% where we've identified the continuous-time process $(X_t)_{t\in[0,1]}$ with a discrete-time process $(X_k)_{k\in\{0,\dots,K\}}$, where $\alpha_k = \alpha(\nicefrac{k}{K})$, $\beta_k = \beta(\nicefrac{k}{K})$, and $\sigma_k = \sigma(\nicefrac{k}{K})$, applied a first-order approximation.

\subsection{Denoising Diffusion Models}

We next discuss diffusion models, in particular the sampling scheme proposed by Denoising Diffusion Implicit Model (DDIM; \cite{song2021denoising}) which we will later relate to Denoising Diffusion Probabilistic Models (DDPM; \cite{ho2020denoising}) as a particular case of the former. 
For sampling from a diffusion model, the DDIM update rule%
\footnote{We slightly depart from the notation in \citet{song2021denoising} by flipping the direction of time and using $\bar{\alpha}_k$ which corresponds to the $\alpha_k$ in \citet{song2021denoising} while it corresponds to the $\bar{\alpha}_k$ in \citet{ho2020denoising}.}~(\citet{song2021denoising}, Eq. 12)%
, typically stated in discrete time with $k\in \{0,\dots,K\}$, is:
\begin{talign} \label{eq:DDIM_original_main}
    X_{k+1} = \sqrt{\bar{\alpha}_{k+1}}
    \big( \frac{X_{k} - \sqrt{1-\bar{\alpha}_k} \epsilon(X_k,k)}{\sqrt{\bar{\alpha}_k}} \big)
    + \sqrt{1-\bar{\alpha}_{k+1} - \sigma_{k}^2} \epsilon(X_k,k) +
    % \frac{1}{\sqrt{K}} 
    \sigma_{k} \varepsilon_{k},
    \qquad \varepsilon_k \sim \mathcal{N}(0,I), \ X_0 \sim \mathcal{N}(0,I),
\end{talign}
where $\bar{\alpha}_k$ is an increasing sequence such that $\bar{\alpha}_0 = 0$, $\bar{\alpha}_K = 1$, and the sequence $\sigma_k$ is arbitrary. That is, one samples an initial Gaussian random variable $x_0$, and applies the stochastic update \eqref{eq:DDIM_original_main} iteratively $K$ times in order to obtain an artificial sample $X_K$. Updates can be interpreted as progressively denoising the iterate: $x_0$ is completely noisy and $x_K$ is fully denoised. The noise predictor model $\epsilon(x_k,k)$ is trained to predict the noise of $x_k$ (see \eg \citet{ho2020denoising} for details on pre-training denoising diffusion models). 

\section{Flow Matching and diffusion models from a common perspective}\label{sec:common_perspective}

We formulate Flow Matching and diffusion models in a unified framework, which we will later use throughout the paper.
Firstly, to simplify notation, we will be using continuous-time formulations. This will also directly enable fine-tuning methods inspired by the continuous-time paradigm, which we find tends to perform better than discrete-time counterparts in our empirical validations. Secondly, by consolidating notation, we will be able to discuss fine-tuning of dynamical generative models that follow the same time marginals as the reference flow \eqref{eq:reference_flow}, pre-trained with either the Denoising Diffusion or Flow Matching framework, in full generality.

To convert DDIM to a continuous-time stochastic process, we can show that the DDIM update rule \eqref{eq:DDIM_original_main}, up to a first-order approximation, is equivalent to the Euler-Maruyama discretization of the following SDE:
\begin{talign} \label{eq:euler_maruyama_DDIM}
    \mathrm{d}X_t &= \big( \frac{\dot{\bar{\alpha}}_{t}}{2\bar{\alpha}_{t}} X_t - \big( \frac{\dot{\bar{\alpha}}_{t}}{2\bar{\alpha}_{t}} + \frac{\sigma(t)^2}{2} \big) \frac{\epsilon^\text{base}(X_{t},t)}{\sqrt{1-\bar{\alpha}_{t}}} \big) \mathrm{d}t + \sigma(t) \mathrm{d}B_t, \qquad X_{0} \sim \mathcal{N}(0,I).
\end{talign}
See \Cref{subsec:continuous_DDIM} for the full derivation.
To go from \eqref{eq:DDIM_original_main} to \eqref{eq:euler_maruyama_DDIM}, we assumed a uniform discretization of time, \ie $t=\tfrac{k}{K}$. 
This results in identifying the discrete-time process $(X_{k})_{k\in \{0,\dots,K\}}$ with a continuous-time process $(X_{t})_{t\in[0, 1]}$, where $\bar{\alpha}_k := \bar{\alpha}_{t}$, $\sigma_k := \frac{1}{\sqrt{K}} \sigma(t)$, and $\epsilon(X_k, k)$ with $\epsilon^\text{base}(X_k, t)$. 
In relation to the reference flow \eqref{eq:reference_flow}, the generative process in \eqref{eq:euler_maruyama_DDIM} has the same time marginals when $\alpha_t = \sqrt{\bar{\alpha}_t}$ and $\beta_t = \sqrt{1 - \bar{\alpha}_t}$~\citep{ho2020denoising}.

Furthermore, when viewed up to first order approximations, the DDPM sampling scheme~(\citet{ho2020denoising}; Algorithm 2) can be seen as special instance of the DDIM sampling scheme when $\sigma(t) = \sqrt{\nicefrac{\dot{\bar{\alpha}}_t}{\bar{\alpha}_t}}$. This results in the following generative process:
\begin{talign}\label{eq:euler_maruyama_DDPM}
    \mathrm{d}X_t &= \big( \frac{\dot{\bar{\alpha}}_{t}}{2\bar{\alpha}_{t}} X_t - \frac{\dot{\bar{\alpha}}_{t}}{\bar{\alpha}_{t}} \frac{\epsilon^\text{base}(X_{t},t)}{\sqrt{1-\bar{\alpha}_{t}}} \big) \mathrm{d}t + \sqrt{\frac{\dot{\bar{\alpha}}_{t}}{\bar{\alpha}_{t}}} \mathrm{d}B_t, \qquad X_{0} \sim \mathcal{N}(0,I),
\end{talign}
% We defer the derivation of the continuous-time limit of DDPM to \Cref{eq:euler_maruyama_DDIM}.

We can further consolidate notation by converting all quantities to the score function $\mathfrak{s}(x,t)$---defined as the gradient of the log density of the random variable $X_t$---which is possible when $X_0$ is Normal-distributed and under the affine reference flow \eqref{eq:reference_flow}. In particular, 
the velocity $v^\text{base}$ from Flow Matching can be expressed in terms of the score function (see \Cref{subsec:v_score}):
\begin{talign}
    v^\text{base}(x,t) = \frac{\dot{\alpha}_t}{\alpha_t} x + \beta_t(\frac{\dot{\alpha}_t}{\alpha_t} \beta_t - \dot{\beta}_t) \mathfrak{s}(x,t).
\end{talign}
And the noise predictor $\epsilon^\text{base}$ also admits an expression in terms of the score function (see \Cref{subsec:hat_epsilon_score}):
\begin{talign}
    \mathfrak{s}(x,t) = - \frac{\epsilon^\text{base}(x,t)}{\sqrt{1-\bar{\alpha}_t}}.
\end{talign}
Plugging these two equations into \eqref{eq:FM_general_diffusion_coeff} and \eqref{eq:euler_maruyama_DDIM}, respectively, and rewriting them in terms of only the $\alpha_t$ and $\beta_t$ in \eqref{eq:reference_flow}, we can unify both the Flow Matching and continuous-time DDIM generative processes as:
\begin{talign} \label{eq:gen_process_1}
    \mathrm{d}X_t &= b(X_t,t) \, \mathrm{d}t + \sigma(t) \, \mathrm{d}B_t, \qquad X_0 \sim \mathcal{N}(0,I), \\
    \text{where} \ b(x,t) &= \kappa_t x + \big(\frac{\sigma(t)^2}{2} + \eta_t\big) \mathfrak{s}(x,t), \quad \kappa_t = \frac{\dot{\alpha}_t}{\alpha_t}, \quad 
    \eta_t = \beta_t (\frac{\dot{\alpha}_t}{\alpha_t} \beta_t - \dot{\beta}_t)
    \label{eq:gen_process_2}
\end{talign}
where $(\alpha_t, \beta_t)$ are coefficients of the reference flow \eqref{eq:reference_flow}. We have hence expressed the generative process of a base model, whether it is a Flow Matching or a diffusion model, as an SDE of the form \eqref{eq:gen_process_1}-\eqref{eq:gen_process_2}, unified by the choice of reference flow. This expression has been written before for DDIM, e.g. \cite{bartosh2024neural,bartosh2024neural2}.
%Cref{tab:coefficients} summarizes the coefficients for each pre-training framework.

\section{Fine-tuning as ``memoryless'' stochastic optimal control}
\label{sec:memoryless_SOC}

We now discuss the crux of the problem: how to produce a fine-tuned generative model that produces samples $X_1$ which follow the tilted distribution involving a reward model \eqref{eq:p_star_info}. 
An obvious direction is to construct a \emph{fine-tuning objective} involving both the base generative model and the reward model, where the optimal solution results in a fine-tuned generative model for the tilted distribution. 
However, as we will explain, this turns out to be non-trivial, because a na\"ive formulation will introduce bias into the solution.

In \Cref{sec:SOC_formulation}, we discuss the problem formulation of stochastic optimal control, a general framework for optimizing SDEs, and its relation to the maximum entropy reinforcement learning framework commonly used for RLHF fine-tuning. 
Next, in \Cref{sec:value_function_bias_problem}, we discuss the \emph{initial value function bias} problem which plagues existing approaches and so far has seen no simple solution.
Finally, in \Cref{sec:memoryless_schedule}, we propose a novel simple solution that circumvents the bias problem, by enforcing a particular diffusion coefficient, the \emph{memoryless noise schedule}, to be used during fine-tuning. This results in an extremely simple fine-tuning objective that provably converges to a model which generates the tilted distribution \eqref{eq:p_star_info} without any statistical bias.

% However, this turns out to not be trivial. As a na\"ive formulation will add bias into the solution, or simply shift the problem into optimizing the base distribution of the generative model.

\subsection{Preliminaries on the stochastic optimal control problem formulation} \label{sec:SOC_formulation}

Stochastic optimal control (SOC; \cite{bellman1957,fleming2012deterministic,sethi2018optimal}) considers general optimization problems over stochastic differential equations, but we only need to consider a common instantiation, the quadratic cost control-affine problem formulation:
\graybox{
\begin{talign} \label{eq:control_problem_def}
    &\min\limits_{u \in \mathcal{U}} \mathbb{E} \big[ \int_0^1 
    \big(\frac{1}{2} \|u(X^u_t,t)\|^2 + f(X^u_t,t) \big) \, \mathrm{d}t + 
    g(X^u_1) \big], \\
    \begin{split}
    \text{s.t.}~ \mathrm{d}X^u_t =  \left( b(X^u_t,t) + \sigma(t) u(X^u_t,t) \right) \, \mathrm{d}t + 
    \sigma(t) \mathrm{d}B_t, \qquad X^u_0 \sim p_0
    \end{split} 
    \label{eq:controlled_SDE}
\end{talign}
}
where in \eqref{eq:controlled_SDE}, $X_t^u \in \R^d$ is the state of the stochastic process, $u : \R^d \times [0,1] \to \R^d$ is commonly referred to as the control vector field, $b : \R^d \times [0,1] \to \R^d$ is a base drift, and $\sigma : [0,1] \to \R^{d \times d}$ is the diffusion coefficient. These jointly define the \emph{controlled process} $\fX^u \sim p^u$ that we are interested in optimizing; often both $b$ and $\sigma$ are fixed and we only optimize over the control $u$. 

As part of the objective functional \eqref{eq:control_problem_def}, we have an affine control cost $\frac{1}{2} \|u(X^u_t,t)\|^2$, a running state cost $f : \R^d \times [0,1] \to \R$ and a terminal state cost $g : \R^d \to \R$. 
%For the application of RLHF to dynamical generative models, we only need to consider the case of $f=0$; however, we will keep $f$ in the discussion for full generality as our proposed Adjoint Matching algorithm (\Cref{sec:adjoint_matching}) solves any control-affine SOC problem formulation.

The stochastic optimal control (SOC) objective \eqref{eq:control_problem_def} can be decomposed recursively from the final time value. It is common to define the \emph{cost functional} which is the expected future cost starting from state $x$ at time $t$:
\begin{talign} \label{eq:cost_functional}
J(u;x,t) := \mathbb{E}_{\fX \sim p^u} \left[ \int_t^1 
\left(\frac{1}{2} \|u(X_s,s)\|^2  +  f(X_s,s) \right) \, \mathrm{d}s  +  
g(X_1) \;\big|\; X_t = x \right].
\end{talign}
From here, the \emph{value function} is the optimal value of the cost functional%
\footnote{Note that there is a slight difference in terminology between SOC and reinforcement learning, where our cost functional is referred to as the state value function and our value function is the optimal state value function in RL.}
:
\begin{talign}\label{eq:value_fn_defn}
V(x,t) := \min_{u\in \mathcal{U}} J(u;x,t) = J(u^*;x,t),
\end{talign}
where $u^*$ is the \emph{optimal control}, \ie, minimizer of \eqref{eq:control_problem_def}. Furthermore, a classical result is that the value function can be expressed in terms of the \emph{uncontrolled} base process $p^\text{base}$ (\cite{kappen2005path}, see \citealt[Eq.~8,~App.~B]{domingoenrich2023stochastic} for a self-contained proof):
\begin{talign}\label{eq:value_fn_from_uncontrolled}
    V(x, t) = - \log \E_{\fX \sim p^\text{base}} \left[ \exp( - \int_t^1 f(X_s, s) \mathrm{d}s - g(X_1) ) \;\big|\; X_t = x  \right]. 
\end{talign}
A useful expression for the optimal control (which we will make use of in deriving the Adjoint Matching objective in \Cref{sec:adjoint_matching}) is that it is related to the gradient of the value function:
\begin{talign} \label{eq:optimal_control}
u^*(x,t) = - \sigma(t)^{\top} \nabla_x V(x,t) = - \sigma(t)^{\top} \nabla_x J(u^*, x,t).
\end{talign}
\paragraph{Relation to MaxEnt RL.} Stochastic optimal control with the control-affine formulation \eqref{eq:control_problem_def} is the continuous-time equivalence of maximum entropy reinforcement learning (MaxEnt RL; \citet{todorov2006linearly,ziebart2008maximum}) with a KL regularization instead of only an entropy regularization. In particular, by the Girsanov theorem (\Cref{cor:girsanov_sdes}), 
% \ricky{cite}, 
the affine control cost is equivalent to a Kullback–Leibler (KL) divergence between the base process $p^\text{base}$, when $u=0$, and the controlled process $p^u$, when conditioned on the same initial state $X_0$ (see \Cref{subsec:proof_eq_cond_kl}):
\begin{talign}\label{eq:cond_kl}
    \infdiv*{p^u(\fX | X_0)}{p^{base}(\fX | X_0)} = \mathbb{E}_{\fX^u \sim p^u} \left[ \int_{0}^1 \frac{1}{2} \|u(X^u_t,t)\|^2 \mathrm{d}t \right],
\end{talign}
resulting in the KL-regularized RL interpretation of \eqref{eq:control_problem_def}:
\begin{talign}\label{eq:kl_regularized_interpretation}
    &\max\limits_{u \in \mathcal{U}}\;
    \E_{X_0 \sim p_0} \left[
    \mathbb{E}_{\fX \sim p^u(\cdot | X_0)} \big[ \int_0^1  -f(X_t^u , t) \mathrm{d} t - g(X_1^u) \big] - \infdiv{p^u(\fX | X_0)}{p^{base}(\fX | X_0)}
    \right],
\end{talign}
where the negative state costs correspond to intermediate and terminal rewards in the RL interpretation. The KL divergence incentivizes the optimal solution to stay close to the distribution of the base process. 
%This KL regularization is a common objective for RLHF of large language models \citep{ouyang2022training} but has seen seldom use in fine-tuning diffusion models, particularly due to a problem that we will discuss next.

\subsection{The initial value function bias problem} \label{sec:value_function_bias_problem}

We next discuss why na\"ively adding a KL regularization does not lead to the tilted distribution \eqref{eq:p_star_info}.
From \eqref{eq:kl_regularized_interpretation}, we can also show that the optimal distribution conditioned on $X_0$ is%
\footnote{Note \eqref{eq:cond_optimal_distribution_SOC} is informal because densities over continuous-time processes are ill-defined; the formal statement is $\frac{\mathrm{d}\mathbb{P}^{*}}{\mathrm{d}\mathbb{P}^{\mathrm{base}}}(\fX | X_0) = \exp ( - \int_0^1 f(X_t,t) \, \mathrm{d}t - g(X_1))$, where $\frac{\mathrm{d}\mathbb{P}^{*}}{\mathrm{d}\mathbb{P}^{\mathrm{base}}}$ denotes the Radon-Nikodym derivative. We treat this formally in the proofs.}
\begin{talign}\label{eq:cond_optimal_distribution_SOC}
    p^{*}(\fX | X_0) \propto 
    p^{\mathrm{base}}(\fX | X_0) \exp \big( - \int_0^1 f(X_t,t) \, \mathrm{d}t - g(X_1) \big).
\end{talign}
This is analogous to the exponentiated reward distribution in MaxEnt RL \citep{rawlik2013stochastic}, but since we generalize the entropy regularization to a KL regularization, $p^\text{base}$ acts as a prior distribution.

% Importantly, 
In order to relate this to the tilted distribution \eqref{eq:p_star_info} that we want to achieve for fine-tuning, 
% we need to marginalize all time values and check the distribution of $p^*(X_1)$. In order to do this, 
first notice that the normalization constant of the right-hand side (RHS) of \eqref{eq:cond_optimal_distribution_SOC} is exactly the value function at $t=0$:
\begin{talign}\label{eq:normalization_constant}
    \E_{\fX \sim p^\text{base}(\fX | X_0)} \left[ \exp \big( - \int_0^1 f(X_t,t) \, \mathrm{d}t - g(X_1) \big) \right] = \exp \left( - V(X_0, 0) \right),
\end{talign}
where the equality is due to \eqref{eq:value_fn_from_uncontrolled}.  
% Therefore, we see that this normalization constant depends on $X_0$. 
Dividing the RHS of \eqref{eq:cond_optimal_distribution_SOC} by \eqref{eq:normalization_constant} and multiplying by $p_0(X_0)$, we obtain the normalized distribution over the full path $\fX$,
\begin{talign} \label{eq:optimal_distribution_SOC} 
    p^{*}(\bm{X}) = 
    p^{\mathrm{base}}(\bm{X}) \exp \big( - \int_0^1 f(X_t,t) \, \mathrm{d}t - g(X_1) + V(X_0,0) \big).
\end{talign}
Setting $f=0$ and $g = -r$, we arrive at an expression for the optimal distribution
\begin{talign} \label{eq:optimal_distribution_SOC_RLHF} 
    p^{*}(X_0, X_1) = 
    p^{\mathrm{base}}(X_0, X_1) \exp \big( r(X_1) + V(X_0,0) \big).
\end{talign}
This unfortunately does not lead to the tilted distribution \eqref{eq:p_star_info} because we have a bias in the optimal distribution that is due to the value function of the initial distribution $V(X_0, 0)$. That is to say, na\"ively adding a KL regularization \eqref{eq:cond_kl} to the fine-tuning objective in the sense of \eqref{eq:kl_regularized_interpretation} leads to a biased distribution \eqref{eq:optimal_distribution_SOC} after fine-tuning and is \textit{not} equivalent to the tilted distribution \eqref{eq:p_star_info}. For instance, when the sampling procedure is noiseless, \ie, $\sigma(t) = 0$, fine-tuning na\"ively will not have any effect because $X_0$ completely determines $X_1$.

This is unlike the situation for large language models \citep{ouyang2022training,rafailov2023direct}, where there is no dynamical process that samples $X_1$ iteratively and hence no dependence on the initial noise variable $X_0$. 
Although this KL regularization is a common objective for RLHF of large language models, it has seen seldom use in fine-tuning diffusion models, likely due to this issue of the initial value function bias.

In the context of diffusion models, KL regularization \eqref{eq:kl_regularized_interpretation} has been explored in prior works \citep{fan2024reinforcement}, but its behavior was not well-understood and they did not relate the fine-tuned model to the tilted distribution \eqref{eq:p_star_info}. Another direction that has been proposed is to learn the initial distribution $p_0$ to cancel out the bias \citep{uehara2024finetuning,tang2024finetuning} but this simply shifts the work into tilting the initial distribution and requires an auxiliary model for parameterizing the optimal initial distribution. 
In contrast, we show in the next section that it is possible to remove the value function bias by simply choosing a very particular noise schedule during the fine-tuning procedure. 
% In Figure \Cref{fig:memorylessness_illustration}, we show that this value function bias can introduce arbitrary bias into the fine-tuned distribution depending on different choices of noise schedules.

\subsection{The memoryless noise schedule for fine-tuning dynamical generative models} \label{sec:memoryless_schedule}

\begin{table}
\centering
\begin{tabular}{lcccc}
    \toprule
     & $\kappa_t$ & $\eta_t$ & Diffusion coefficient $\sigma(t)$ & Memoryless $X_t$ \\
    \midrule
    \addlinespace
    Flow Matching \eqref{eq:FM_ode} & $\frac{\dot{\alpha}_t}{\alpha_t}$ & $\beta_t \big(\frac{\dot{\alpha}_t}{\alpha_t} \beta_t - \dot{\beta}_t\big)$ & General (commonly $0$) & No \\
    \addlinespace
    Memoryless Flow Matching \eqref{eq:FM_general_diffusion_coeff} & $\frac{\dot{\alpha}_t}{\alpha_t}$ & $\beta_t \big(\frac{\dot{\alpha}_t}{\alpha_t} \beta_t - \dot{\beta}_t\big)$ & $\sqrt{2\eta_t}$ & Yes \\
    \addlinespace
    DDIM \eqref{eq:euler_maruyama_DDIM} & 
    $\frac{\dot{\bar{\alpha}}_{t}}{2\bar{\alpha}_{t}}$
    & 
    $\frac{\dot{\bar{\alpha}}_{t}}{2\bar{\alpha}_{t}}$
    & General (commonly $0$) & No \\
    \addlinespace
    DDPM \eqref{eq:euler_maruyama_DDPM} & 
    $\frac{\dot{\bar{\alpha}}_{t}}{2\bar{\alpha}_{t}}$
    & 
    $\frac{\dot{\bar{\alpha}}_{t}}{2\bar{\alpha}_{t}}$
    & $\sqrt{2\eta_t}$ & Yes \\
    \bottomrule
\end{tabular}
\caption{Diffusion coefficient $\sigma(t)$ and the factors $\kappa_t$, $\eta_t$ for the Flow Matching, Memoryless Flow Matching, DDIM, and DDPM generative processes. When the diffusion coefficient is $\sigma(t) = \sqrt{2\eta_t}$, the generative process is memoryless, \ie, samples $X_1$ will be independent of the initial noise $X_0$.}
\label{tab:coefficients}
\end{table} 

In this section, we propose a very simple method of turning \eqref{eq:optimal_distribution_SOC_RLHF} into the tilted distribution \eqref{eq:p_star_info} through the use of a particular \emph{memoryless} noise schedule.
Throughout, we provide an intuitive explanation of why this noise schedule is sufficient for fine-tuning while discussing the full theoretical result where we show that the memoryless noise schedule is actually not only sufficient but also necessary.

Intuitively, the main reason we cannot arrive at the tilted distribution from \eqref{eq:optimal_distribution_SOC_RLHF} is due to the $p^{\text{base}}(X_0, X_1)$ distribution not factoring into $X_0$ and $X_1$. Hence, we define a memoryless generative process as follows:
\begin{definition}[Memoryless generative process]
A generative process of the form \eqref{eq:gen_process_1}-\eqref{eq:gen_process_2} is memoryless if $X_0$ and $X_1$ are independent, \ie, $p^\text{base}(X_0, X_1) = p^\text{base}(X_0) p^\text{base}(X_1)$.
\end{definition}
When the base generative process is memoryless, this implies:
\begin{talign}
    p^{*}(X_1) 
    = \int p^{\text{base}}(X_0) p^{\text{base}}(X_1) \exp( r(X_1) + V(X_0, 0)) \mathrm{d} X_0
    \propto p^\text{base}(X_1) \exp(r(X_1)).
\end{talign}
That is, solving the SOC problem \eqref{eq:control_problem_def}-\eqref{eq:controlled_SDE} with a memoryless base model will result in a fine-tuned model that generates samples $p^*(X_1)$ according to the tilted distribution \eqref{eq:p_star_info}. 
This memoryless property is not satisfied generally by the family of generative processes captured by \eqref{eq:control_problem_def}-\eqref{eq:controlled_SDE}.
For instance, the Flow Matching and DDIM generative processes with zero diffusion coefficient (\ie, $\sigma(t) = 0$) are definitely not memoryless due to $X_0$ and $X_1$ being theoretically invertible.
Below, we provide the sufficient and neccessary condition for the noise schedule in order to have a memoryless generative process.
%\textcolor{red}{Needs to be adapted to new prop. statement.}
%
\begin{proposition}[Memoryless noise schedules] \label{prop:memorylessness_noise_schedule}
    Within the family of generative processes \eqref{eq:gen_process_1}-\eqref{eq:gen_process_2}, a generative process is memoryless if and only if the noise schedule is chosen as: 
    \begin{talign} \label{eq:chi_condition}
        \sigma(t)^2 = 2 \eta_t + \chi(t), \text{ where } \chi : [0,1] \to \R \text{ is s.t. } \forall t \in (0,1], \quad \lim_{t' \to 0^{+}} \alpha_{t'} \exp \big( - \int_{t'}^t \frac{\chi(s)}{2 \beta_{s}^2} \, \mathrm{d}s \big) = 0.
    \end{talign}
    where $\eta_t$ is the coefficient defined in \eqref{eq:gen_process_2} (see also \autoref{tab:coefficients}).
    In particular, we refer to $\sigma(t) = \sqrt{2\eta_t}$ as the memoryless noise schedule.
\end{proposition}
Due to the endpoint constraints of $(\alpha_t, \beta_t)$ for the reference flow \eqref{eq:reference_flow}, the memoryless noise schedule $\sigma(t)$ is infinite at $t=0$ and approaches zero at $t=1$. This provides a way for the generative process to mix when close to noise $X_0$ while stay steadying when close to the sample $X_1$. Hence, the sample will have no information about $X_0$ due to the enormous amount of mixing with a large diffusion coefficient.
% We note that \Cref{prop:memorylessness_noise_schedule} is stronger than this simple explanation portrays: the memoryless noise schedule is the \textit{only} noise schedule that allows $X_0$ and $X_1$ to be independent. 
% Intuitively, when the diffusion coefficient is too large, the drift $b(X_t, t)$ must compensate and this can actually result in stronger dependency between states at different time values. 
Furthermore, while we have intuitively justified the memoryless noise schedule through its independence property, our theoretical result is actually even stronger: all generative models of the form \eqref{eq:gen_process_1}-\eqref{eq:gen_process_2} \textit{must} be fine-tuned using the memoryless noise schedule.
We formalize this in the following theorem, which we prove in \Cref{subsec:proof_prop_diff_finetuning}:
\begin{theorem}[Fine-tuning recipe for general noise schedule sampling] \label{thm:general_fine-tuning}
Within the family of generative processes \eqref{eq:gen_process_1}-\eqref{eq:gen_process_2}, in order to allow the use of arbitrary noise schedules and still generate samples according to the tilted distribution \eqref{eq:p_star_info}, the fine-tuning problem \eqref{eq:control_problem_def}-\eqref{eq:controlled_SDE} with $f=0$ and $g=-r$ must be done with the memoryless noise schedule $\sigma(t) = \sqrt{2\eta_t}$.
\end{theorem}
\Cref{thm:general_fine-tuning} states that we \textit{need} to use the memoryless noise schedule for fine-tuning with the SOC objective---or equivalently, the KL regularized reward objective \eqref{eq:kl_regularized_interpretation}.
This is the only noise schedule that retains the relationship between the velocity and score function, allowing the conversion to arbitrary noise schedules (\eg, $\sigma(t) = 0$) after fine-tuning.
It is worth noting that when using the memoryless noise schedule for DDIM, this recovers what we derived as the continuous-time limit of the DDPM generative process \eqref{eq:euler_maruyama_DDPM}. However, the DDPM sampler \citep{ho2020denoising} is not commonly used while the DDIM sampler \citep{song2021denoising} and Flow Matching models typically generate samples using $\sigma(t) = 0$, so an explicit conversion to the memoryless noise schedule is necessary for fine-tuning. 
To the best of our knowledge, we are not aware of any existing works that have proposed a time-varying diffusion coefficient with theoretical guarantees.
\Cref{tab:coefficients} summarizes the memoryless schedule for diffusion and Flow Matching models, which we refer to as Memoryless Flow Matching. In \Cref{fig:memorylessness_illustration}, we visualize fine-tuning a 1D model, where we see that constant $\sigma(t)$ leads to biased distributions whereas the memoryless noise schedule perfectly converges to the tilted distribution \eqref{eq:p_star_info}. 

\begin{figure}
    \centering
    \begin{subfigure}[t]{\linewidth}
        \centering
        \includegraphics[width=0.7\linewidth]{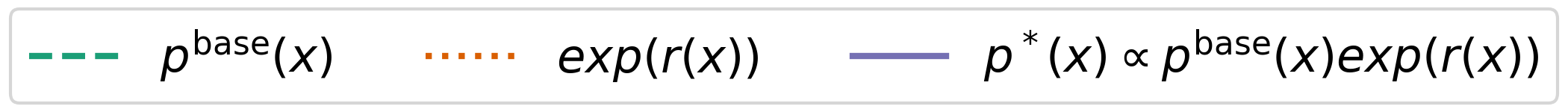}
    \end{subfigure}\\
    \begin{subfigure}[t]{0.25\linewidth}
        \centering
        \includegraphics[width=\linewidth]{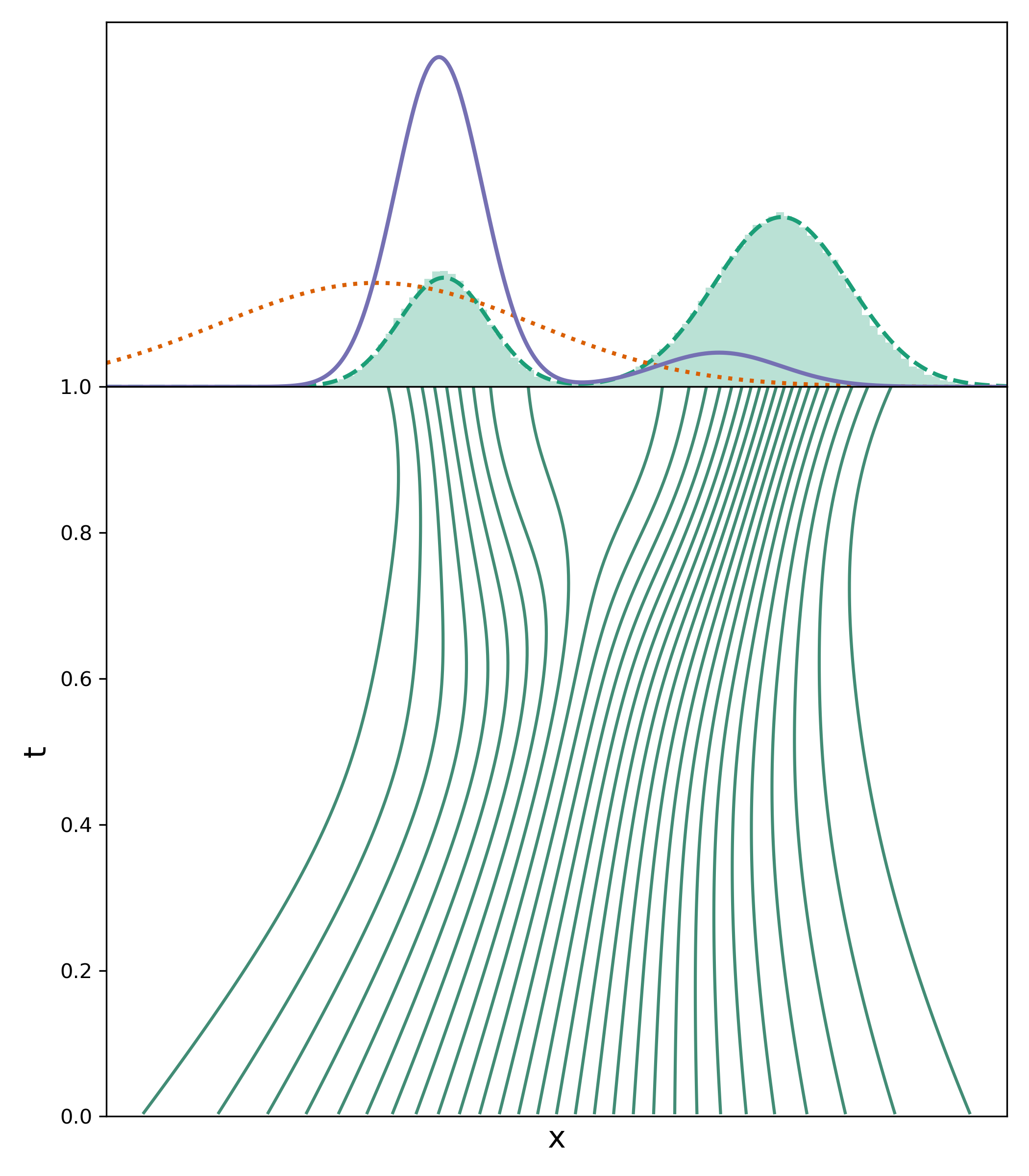}
        \caption{Pre-trained FM $v^\mathrm{base}$}
    \end{subfigure}%
    \begin{subfigure}[t]{0.25\linewidth}
        \centering
        \includegraphics[width=\linewidth]{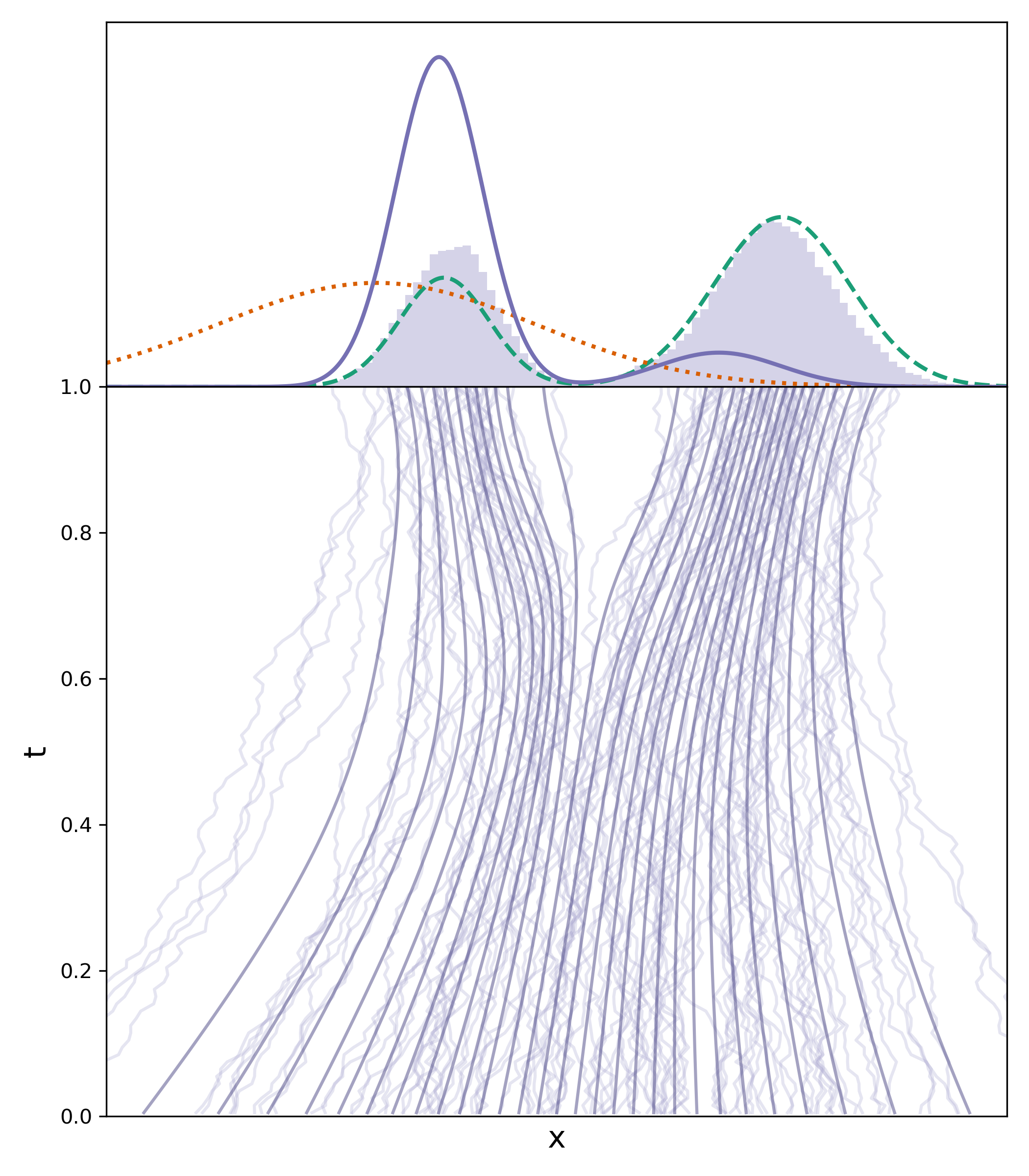}
        \caption{Fine-tuned FM $v^\mathrm{finetune}$\\ with $\sigma(t) = 0.2$}
    \end{subfigure}%
    \begin{subfigure}[t]{0.25\linewidth}
        \centering
        \includegraphics[width=\linewidth]{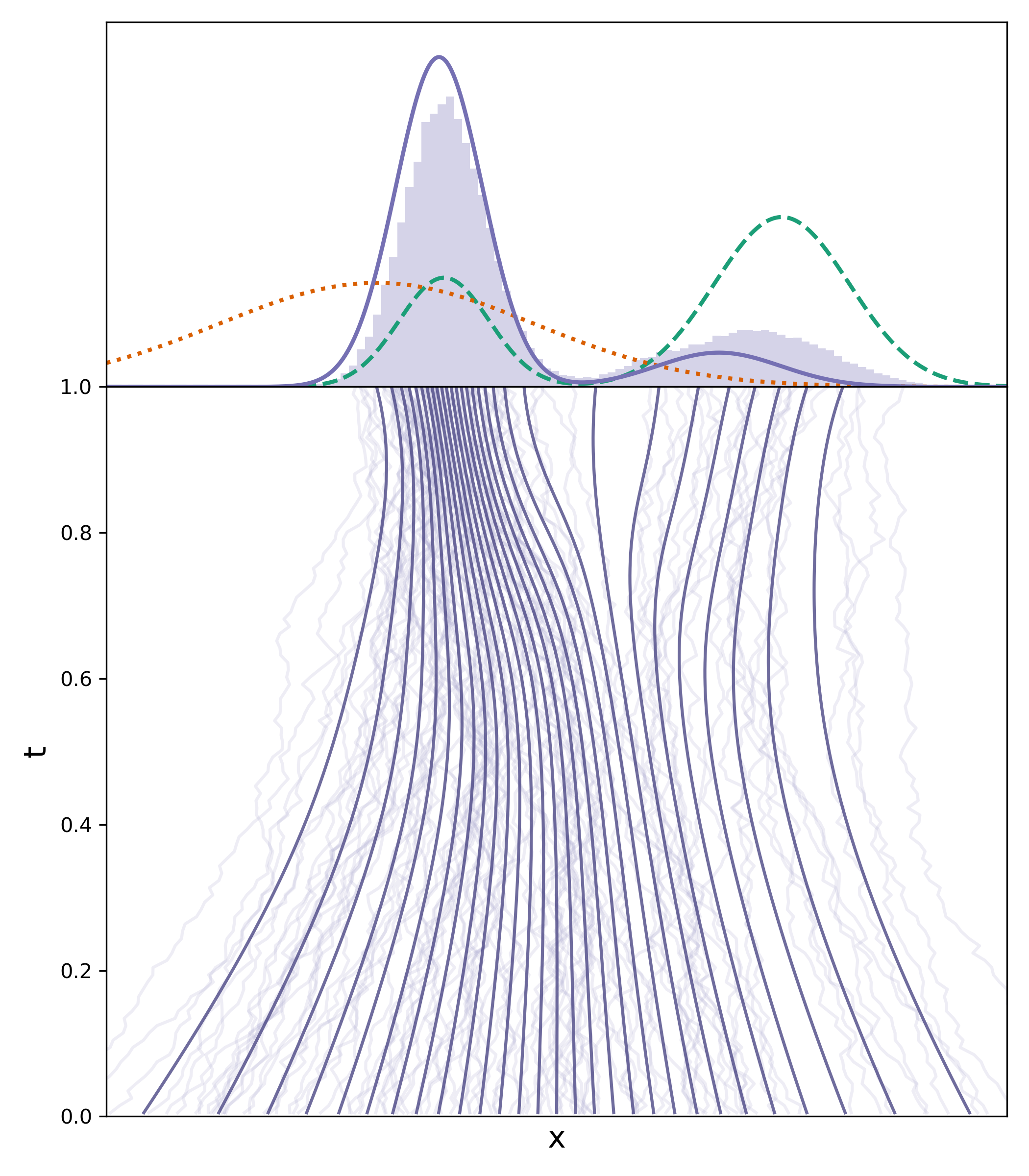}
        \caption{Fine-tuned FM $v^\mathrm{finetune}$\\ with $\sigma(t) = 1.0$}
    \end{subfigure}%
    \begin{subfigure}[t]{0.25\linewidth}
        \centering
        \includegraphics[width=\linewidth]{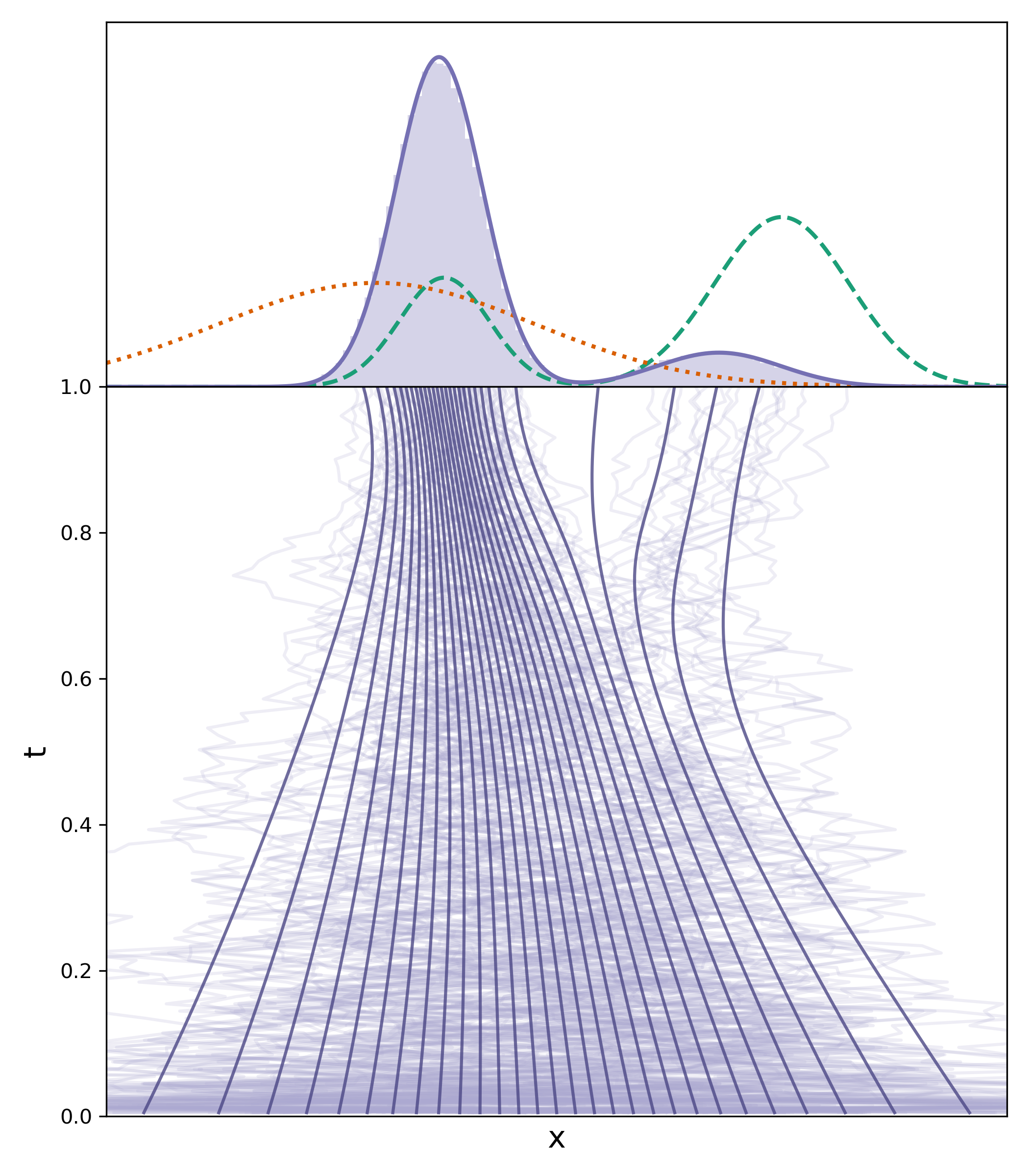}
        \caption{Fine-tuned FM $v^\mathrm{finetune}$\\ with memoryless $\sigma(t) = \sqrt{2\eta_t}$}
    \end{subfigure}
    \caption{Visualization of \Cref{thm:general_fine-tuning} showing that fine-tuning must be done with the memoryless noise schedule to ensure convergence to the tilted distribution \eqref{eq:p_star_info}. (a) Shows the base Flow Matching model. (b, c) Fine-tuning using a constant $\sigma(t)$ leads to biased distributions. (d) Fine-tuning using the memoryless noise schedule leads to the correct tilted distribution. Note that sample generation can use any noise schedule after fine-tuning, including $\sigma(t) = 0$.}
    \label{fig:memorylessness_illustration}
\end{figure}

For convenience, we plug the memoryless noise schedule into the controlled process for fine-tuning \eqref{eq:controlled_SDE}, and express them in terms of each respective framework.
Let $\epsilon^{\mathrm{base}}$, $v^{\mathrm{base}}$ denote the pre-trained vector fields and $\epsilon^{\mathrm{finetune}}$, $v^{\mathrm{finetune}}$ the fine-tuned vector fields. Then we have the following expressions for the full drift $b(x,t) + \sigma(t) u(x,t)$ and control $u(x,t)$ when $\sigma(t) = \sqrt{2\eta_t}$:

\quad \textit{DDIM / DDPM}:
\greybox{
\begin{talign} \label{eq:conversion_DDPM}
b(x,t) + \sigma(t) u(x,t) = \frac{\dot{\bar{\alpha}}_{t}}{2\bar{\alpha}_{t}} x - \frac{\dot{\bar{\alpha}}_{t}}{\bar{\alpha}_{t}} \frac{\epsilon^{\mathrm{finetune}}(x,t)}{\sqrt{1-\bar{\alpha}_{t}}} 
,\quad\quad
u(x,t) = 
- \sqrt{\frac{\dot{\bar{\alpha}}_t}{\bar{\alpha}_t(1-\bar{\alpha}_t)}} (\epsilon^{\mathrm{finetune}}(x,t) - \epsilon^{\mathrm{base}}(x,t)).
\end{talign}
}
\vspace{0.4em}
\quad \textit{Memoryless Flow Matching}:
\greybox{
\begin{talign} \label{eq:conversion_MFM}
b(x,t) + \sigma(t) u(x,t) = 2v^{\mathrm{finetune}}(x,t) - \frac{\dot{\alpha}_{t}}{\alpha_{t}} x 
,\quad\quad
u(x,t) = 
\sqrt{\frac{2}{\beta_{t}(\frac{\dot{\alpha}_{t}}{\alpha_{t}} \beta_{t} - \dot{\beta}_{t})}} (v^{\mathrm{finetune}}(x,t) - v^{\mathrm{base}}(x,t)).
\end{talign}
}
Thus, to solve the SOC problem \eqref{eq:control_problem_def}-\eqref{eq:controlled_SDE} in practice, we parameterize the control $u$ in terms of $\epsilon^{\mathrm{finetune}}$ or $v^{\mathrm{finetune}}$ and optimize these vector fields instead. After plugging in \eqref{eq:conversion_DDPM}-\eqref{eq:conversion_MFM}, the SOC problem \eqref{eq:control_problem_def}-\eqref{eq:controlled_SDE} can then be solved using any SOC algorithm in order to perform fine-tuning, and we proposed an especially effective algorithm next in \Cref{sec:adjoint_matching}. After fine-tuning, $\epsilon^{\mathrm{finetune}}$ and $v^{\mathrm{finetune}}$ can simply be plugged back into their respective generative processes \eqref{eq:FM_ode}-\eqref{eq:euler_maruyama_DDPM} to sample from the tilted distribution \eqref{eq:p_star_info} using any choice of diffusion coefficient. 

\section{Adjoint Matching for control-affine stochastic optimal control} \label{sec:adjoint_matching}

We discuss existing methods and also propose a new method for optimizing control-affine SOC problems. The new Adjoint Matching method is a combination of the time-tested continuous adjoint method \citep{pontryagin1962mathematical} with recent developments on constructing least-squares objectives for solving SOC problems \citep{domingoenrich2023stochastic}. In this section, we briefly discuss preliminaries on existing methods, their pros and cons, then detail the Adjoint Matching algorithm and its surprising connections to the prior methods. For numerical optimization, we now assume that the control $u$ is a parametric model with parameters $\theta$.

\subsection{Existing methods for stochastic optimal control}
\subsubsection{The adjoint method}\label{sec:adjoint_method}

The most basic method of optimizing the simulation of an SDE is to directly differentiate through the simulation using gradients from the SOC objective function \citep{han2016deep}. The adjoint method simply uses the objective:
\begin{talign} \label{eq:L_RE}
    \mathcal{L}(u ; \fX) := \int_0^1 \big(\frac{1}{2} \|u(X_t,t)\|^2 \! + \! f(X_t,t) \big) \, \mathrm{d}t \! + \! g(X_1), \qquad \fX \sim p^u.
\end{talign}
This is a stochastic estimate of the control objective in \eqref{eq:control_problem_def}, and the goal is to take compute the gradient of $\mathcal{L}(u ; \fX)$ with respect to the parameters $\theta$ of the control $u$.
Due to the continuous-time nature of SDEs, there are two main approaches to implementing this numerically. Firstly, the \emph{Discrete Adjoint} method uses a ``discretize-then-differentiate'' approach, where the numerical solver for simulating the SDE is simply stored in memory then differentiated through, and it has been studied extensively (\eg, \citet{bierkens2014explicit,gomez2014policy,hartmann2012efficient,kappen2012optimal,rawlik2013stochastic,haber2017stable}).
%The gradient of this loss with respect to the parameters $\theta$ of the control $u$ can be computed by keeping the computational graph throughout the computation of the trajectories, and then backpropagating, leveraging the power of automatic differentiation frameworks \ricky{cite}. 
This approach, however, uses an extremely large amount of memory as the full computational graph of the numerical solver must be stored in memory and implementations often must rely on gradient checkpointing \citep{chen2016training} to reduce memory usage.

Secondly, the \emph{Continuous Adjoint} method exploits the continuous-time nature of SDEs and uses an analytical expression for the gradient of the control objective with respect to the intermediate states $X_t$, expressed as an adjoint ODE, and then applies a numerical method to simulate this gradient itself, hence it is referred to as a ``differentiate-then-discretize'' approach \citep{pontryagin1962mathematical,chen2018neural,li2020scalable}. We first define the \emph{adjoint state} as: 
% the gradient of \eqref{eq:L_RE} 
% \brian{nit: $J$ is defined as the cost functional?  Is this second equality obvious / apparent?} 
% with respect to the state $X_t$:
\begin{talign}
\begin{split} \label{eq:adjoint_state_defn}
&a(t ; \fX, u) := \nabla_{X_t} 
\big(\int_t^1 \big(\frac{1}{2} \|u(X_{t'},t')\|^2 \! + \! f(X_{t'},t') \big) \, \mathrm{d}t' \! + \! g(X_1) \big), \\
&\text{where } \fX \text{ solves } \mathrm{d}X_t =  \left( b(X_t,t) + \sigma(t) u(X_t,t) \right) \, \mathrm{d}t + 
    \sigma(t) \mathrm{d}B_t.
% \\ &\text{implying }\;\; \mathbb{E}_{\fX \sim p^u} \left[ a(t ; \fX, u) \;|\; X_t = x \right] = \nabla_{x} J(u; x, t),
\end{split}
\end{talign}
This implies that $\mathbb{E}_{\fX \sim p^u} \left[ a(t ; \fX, u) \;|\; X_t = x \right] = \nabla_{x} J(u; x, t)$,
where $J$ denotes the cost functional defined in \eqref{eq:cost_functional}.
It can then be shown that this adjoint state satisfies 
% \brian{Will we show this in an appendix?}\ricky{TODO: appendix}%
\footnote{Note we use the convention that a Jacobian matrix $J = \nabla_x v(x)$ is defined as $J_{ij} = \frac{\partial v_i(x)}{\partial x_j}$.}%
:
\begin{talign} 
\begin{split} \label{eq:cont_adjoint_1}
    \frac{\mathrm{d}}{\mathrm{d}t} a(t;\fX,u)  &=  - \left[ a(t;\fX,u)\tran{} \left(\nabla_{X_t} (b (X_t,t) + \sigma(t) u(X_t,t))\right) 
    + \nabla_{X_t} \left( f(X_t,t) + \frac{1}{2}\|u(X_t,t)\|^2 \right) \right],
\end{split}
    \\ a(1;\fX,u) &= \nabla g(X_1). \label{eq:cont_adjoint_2}
\end{talign}
The adjoint state is solved backwards in time, starting from the terminal condition \eqref{eq:cont_adjoint_2}. Computation of \eqref{eq:cont_adjoint_1} can be efficiently done as a vector-Jacobian product on automatic differentiation software \citep{paszke2019pytorch}. Once the adjoint state has been solved for $t \in [0, 1]$, then the gradient of $\mathcal{L}(u ; \fX)$ with respect to the parameters $\theta$ can be obtained by integrating over the entire time interval: 
% \brian{Is this obvious / apparent?}
\begin{talign}\label{eq:continuous_adjoint_grads}
    \frac{\mathrm{d} \mathcal{L}}{\mathrm{d} \theta} =  
    \frac{1}{2}\int_0^1 \frac{\partial}{\partial \theta} \norm{u(X_t, t)}^2 \mathrm{d} t
    +\int_0^1 \frac{\partial u(X_t, t)}{\partial \theta}\tran{} \sigma(t)\tran{} a(t; \fX, u) \mathrm{d} t,
\end{talign}
where the first term is the partial derivative of $\mathcal{L}$ w.r.t. $\theta$ and the second term is the partial derivative through the sample trajectory $\fX$. See \Cref{prop:cont_adjoint_method} in \Cref{subsec:derivation_cont_adj_method} for a statement and proof of this result. 
The discrete and continuous adjoint methods converge to the same gradient as the step size of the numerical solvers go to zero. 
Both are scalable to high dimensions and have seen their fair share of usage in optimizing neural ODE/SDEs \citep{chen2018neural,chen2021learning,li2020scalable}. As the adjoint methods are essentially gradient-based optimization algorithms applied on a highly non-convex problem, many have also reported they can be unstable empirically \citep{mohamed2020monte,suh2022differentiable,domingoenrich2023stochastic}.

\subsubsection{Importance-weighted matching objectives for regressing onto the optimal control}
\label{sec:socm}

An alternative is to consider regressing onto the optimal control $u^*$, which is the approach of the cross-entropy method~\citep{rubinstein2013cross,zhang2014applications} and stochastic optimal control matching (SOCM; \citet{domingoenrich2023stochastic}). These methods make use of path integral theory \citep{kappen2005path} to express the optimal control through importance sampling, resulting in an \emph{importance-weighted} least-squares objective function
\begin{talign}\label{eq:socm_objective}
    \mathcal{L}_{\text{SOCM}}(u; \fX) := \int_0^1 \norm{u(X_t, t) - \hat{u}^*(X_t, t)}^2 \mathrm{d} t \times \omega(u, \fX), \qquad \fX \sim p^u,
\end{talign}
where $\omega$ is an importance weighting that approximates sampling from the optimal distribution $p^*$, and $\hat{u}^*$ is a stochastic estimator of the optimal control relying on having sampled from the optimal process. We defer to \citet{domingoenrich2023stochastic} for the exact details. The functional landscape of this objective is convex, which is argued to help yield stable training. However, the need for importance sampling renders this impractical for high dimensional applications: the variance of the importance weighting $\omega$ grows exponentially with dimension of the stochastic process, leading to catastrophic failure. This unfortunately means that such importance-weighted matching objectives are impractical for fine-tuning dynamical generative models; however, a least-squares objective is greatly coveted as it can lead to stable training and simple interpretations.

\subsection{Adjoint Matching} \label{subsec:adjoint_matching}

We make two important observations which lead to our proposed method: (\textit{i}) it is possible to construct a matching objective without any importance weighting, and (\textit{ii}) there are unnecessary terms in the adjoint differential equation \eqref{eq:cont_adjoint_1} that can lead to higher variance at convergence.

Firstly, we notice that we can simply match the gradient of the cost functional under the \textit{current} control. That is, while SOCM carefully constructs an importance-weighted estimator of the \textit{optimal} control $u^* = - \sigma(t)\tran{} \nabla J(u^*; x, t)$ \eqref{eq:optimal_control}, we claim that we can actually just regress onto the target vector field $- \sigma(t)\tran{} \nabla J(u; x, t)$ where $u$ is the current control, and furthermore, this results in a gradient equal in expectation to the continuous adjoint method. 
We formalize this in the following proposition, proven in \Cref{subsec:derivation_continuous}: 
\begin{proposition} \label{prop:continuous_adjoint_loss_main}
    Let us define, for now, the basic Adjoint Matching objective as:
    \begin{talign}
    \begin{split} \label{eq:cont_adjoint}
        \mathcal{L}_{\mathrm{Basic-Adj-Match}}(u; \fX) &:= \frac{1}{2} \int_0^{1} \big\| u(X_t,t)
        + \sigma(t)\tran{} a(t;\bm{X},\bar{u}) \big\|^2 \, \mathrm{d}t, \qquad \fX \sim p^{\bar{u}}, \quad \bar{u} = \texttt{stopgrad}(u),
    \end{split}
    \end{talign}
    where $\bar{u} = \texttt{stopgrad}(u)$ means that the gradients of $\bar{u}$ with respect to the parameters $\theta$ of the control $u$ are artificially set to zero. The gradient of $\mathcal{L}_{\mathrm{Basic-Adj-Match}}(u; \fX)$ with respect to $\theta$ is equal to the gradient $\frac{\mathrm{d} \mathcal{L}}{\mathrm{d} \theta}$ in equation \eqref{eq:continuous_adjoint_grads}.
    Importantly, the only critical point of $\E \left[\mathcal{L}_{\mathrm{Basic-Adj-Match}} \right]$ is the optimal control $u^*$.
\end{proposition}

Critical points of $\mathcal{L}$ are controls $u$ such that $\frac{\delta}{\delta u} \mathcal{L}(u) = 0$, where $\frac{\delta}{\delta u} \mathcal{L}$ denotes the first variation of the functional $\mathcal{L}$.
In other words, \Cref{prop:continuous_adjoint_loss_main} states that the only control that satisfies the first-order optimality condition for the basic Adjoint Matching objective is the optimal control, which provides theoretical grounding for gradient-based optimization algorithms. 

An intuitive way to understand the basic Adjoint Matching objective is that it is a \emph{consistency loss}. The Adjoint Matching objective is based off of the observation that the optimal control $u^*(x, t)$ is the unique fixed-point of the relation $u(x,t) = -\sigma(t)\tran{} \nabla_x J(u; x, t)$ (see \Cref{eq:lemma_cost_functional} in \Cref{subsec:derivation_continuous})
%\brian{Citation?  Is this obvious?  Or part of the proof of proposition 2?}, 
and so we are directly optimizing for a control that fits this relation, while using the adjoint state as a stochastic estimator of $\nabla_x J(u; x, t)$ \eqref{eq:adjoint_state_defn}. 
% This uniqueness property has previously been used to justify the adjoint method by showing that the optimal control is the unique critical point of the adjoint method \ricky{cite}; however, to the best of our knowledge, it has not been used to directly construct a consistency loss like the basic Adjoint Matching objective.

The basic Adjoint Matching objective in \Cref{prop:continuous_adjoint_loss_main} does not yet yield a novel algorithm for stochastic optimal control, because it produces the same gradient as the continuous adjoint method. This can be seen by taking the gradient w.r.t. $\theta$ after expanding the square in \eqref{eq:cont_adjoint} and removing terms that do not depend on $\theta$ to arrive exactly at the continuous adjoint method~\eqref{eq:continuous_adjoint_grads}. 
However, it provides the means of deriving a simpler \textit{leaner} objective function.

\paragraph{The ``Lean'' Adjoint.} The minimizer of a least-squares objective is the conditional expectation of the regression target, so for the Adjoint Matching objective, at the optimum we have that
\begin{talign}
    u^*(x, t) = \E_{\fX \sim p^*} \left[ -\sigma(t)\tran{} a(t; \fX, u^*) | X_t = x\right].
\end{talign}
Multiplying both sides by the Jacobian $\nabla_x u^*(x, t)$ and re-arranging, we get the relation
\begin{talign}\label{eq:cancellation_terms}
    \E_{\fX \sim p^*} \left[ u^*(x, t)\tran{} \nabla_x u^*(x, t) + a(t; \fX, u^*) \tran{} \sigma(t) \nabla_x u^*(x, t) \;|\; X_t = x\right] = 0.
\end{talign}
Notice that the terms inside the expectation in \eqref{eq:cancellation_terms} show up as part of the adjoint differential equation \eqref{eq:cont_adjoint_1}, which we have now shown to have expectation zero at the optimal solution. 
% Furthermore, the variance of the terms inside \eqref{eq:cancellation_terms} is non-zero even at the optimal solution, so the basic Adjoint Matching \eqref{eq:cont_adjoint} and hence \emph{the continuous adjoint method will also have non-vanishing gradients even when $u=u^*$}.
Therefore, we motivate the definition of a \emph{lean adjoint state} $\tilde{a}$ with the terms in \eqref{eq:cancellation_terms} removed. Plugging this lean adjoint back into the least-squares objective, we obtain our final proposed Adjoint Matching objective: 
\graybox{
\begin{talign}\label{eq:lean_adjoint_matching}
\mathcal{L}_{\mathrm{Adj-Match}}(u; \fX) 
:= \frac{1}{2} \int_0^{1} \big\| & u(X_t %^{\bar{u}}
,t)
+ \sigma(t)\tran{} \tilde{a}(t;\bm{X} %^{\bar{u}}
) \big\|^2 \, \mathrm{d}t, 
\qquad \fX \sim p^{\bar{u}}, \quad \bar{u} = \texttt{stopgrad}(u), \\
\label{eq:lean_adjoint_1}
\text{where }\quad \frac{\mathrm{d}}{\mathrm{d}t} \tilde{a}(t;\bm{X}) 
&= - (\tilde{a}(t;\bm{X})^{\top} \nabla_x b (X_t,t) + \nabla_x f(X_t,t)), \\ 
\label{eq:lean_adjoint_2}
\tilde{a}(1;\bm{X}) &= \nabla_x g(X_1).
\end{talign}
}
Equations \eqref{eq:lean_adjoint_1}-\eqref{eq:lean_adjoint_2} define the \emph{lean adjoint state}, and \eqref{eq:lean_adjoint_matching} is the complete Adjoint Matching objective.
\textit{The unique critical point of $\mathbb{E}[\mathcal{L}_{\mathrm{Adj-Match}}]$ is the optimal control}, which we prove relying on \Cref{prop:continuous_adjoint_loss_main} and equation \eqref{eq:cancellation_terms} (see \Cref{prop:lean_adjoint} in \Cref{subsec:proof_lean_adjoint}). 

Compared to the importance sampling methods (\Cref{sec:socm}), Adjoint Matching is a simple least-squares regression objective and has no importance weighting. This allows it to avoid the pitfalls of high variance importance weights and makes it as scalable as the adjoint methods while retaining the interpretation of matching a target vector field.

Compared to the adjoint method (\Cref{sec:adjoint_method}), Adjoint Matching produces a \emph{different gradient in expectation than the continuous adjoint}. This is because the lean adjoint state is not related to the gradient of the cost functional anymore, \ie, \eqref{eq:adjoint_state_defn} is not true, except at the optimum when $u=u^*$.
Even at the optimal solution, since Adjoint Matching removes terms that have expectation zero, it can potentially exhibit better convergence and lower variance than the continuous adjoint method. 
Additionally, computation of the lean adjoint state \eqref{eq:lean_adjoint_1} also exhibits a smaller computational cost due to the removal of the extra terms (no longer need the Jacobian of the control $\nabla_x u$).
We provide a rigorous derivation of Adjoint Matching and the above claims in \Cref{subsec:proof_lean_adjoint}.

Adjoint Matching can be applied to reward fine-tuning of dynamical generative models through the memoryless SOC formulation discussed in \Cref{sec:memoryless_SOC}. We provide pseudo-code for this in \Cref{alg:adjoint_matching_finetuning_FM} for Flow Matching models and in \Cref{alg:adjoint_matching_finetuning_DDIM} in \Cref{subsec:pseudocode_DDIM} for denoising diffusion models.

\begin{algorithm}
\SetAlgoNoLine % Disable line numbering
\SetAlgoNlRelativeSize{0} %Set number line font to zero
\small{
\KwIn{Pre-trained FM velocity field $v^{\mathrm{base}}$, step size $h$, number of fine-tuning iterations $N$.}

Initialize fine-tuned vector fields: $v^{\mathrm{finetune}} = v^{\mathrm{base}}$ with parameters $\theta$.

  \For{$n \in \{0,\dots,N-1\}$}{
    Sample $m$ trajectories $\bm{X} = (X_t)_{t\in\{0, \dots, 1\}}$ with memoryless noise schedule $\sigma(t) = \sqrt{2 \beta_t (\frac{\dot{\alpha}_t}{\alpha_t} \beta_t - \dot{\beta}_t)}$, \eg :
    \begin{talign} \label{eq:EM_update_box}
    X_{t+h} = X_{t} + h \left(2v_\theta^{\mathrm{finetune}}(X_t, t) - \frac{\dot{\alpha}_t}{\alpha_t} X_t \right) + \sqrt{h} \sigma(t) \varepsilon_t, \quad\quad \varepsilon_t \sim \mathcal{N}(0, I), \quad\quad X_0 \sim \mathcal{N}(0, I).
    \end{talign}

    For each trajectory, solve the \textit{lean adjoint ODE} \eqref{eq:lean_adjoint_1}-\eqref{eq:lean_adjoint_2} backwards in time from $t={1}$ to $0$, \eg:
    \begin{talign} \label{eq:Euler_lean_adjoint}
        \tilde{a}_{t-h} 
        = \tilde{a}_{t} + h \tilde{a}_t\tran{} \nabla_{X_t} \left(2v^{\mathrm{base}}(X_t, t) - \frac{\dot{\alpha}_t}{\alpha_t} X_t \right), \qquad 
        \tilde{a}_1 = - \nabla_{X_1} r(X_1).
    \end{talign}

    Note that $X_t$ and $\tilde{a}_t$ should be computed without gradients, \ie, $X_t = \texttt{stopgrad}(X_t)$, $\tilde{a}_t = \texttt{stopgrad}(\tilde{a}_t)$. \vspace{0.5em}

    For each trajectory, compute the Adjoint Matching objective \eqref{eq:lean_adjoint_matching}: 
    \begin{talign} \label{eq:adj_matching_algorithm_box}
    \mathcal{L}_{\mathrm{Adj-Match}}(\theta) =
        \sum_{t\in\{0, \dots, 1 - h\}} \big\|\frac{2}{\sigma(t)} \big(v^{\mathrm{finetune}}_{\theta}(X_t, t) - v^{\mathrm{base}}(X_t, t) \big) + \sigma(t) \tilde{a}_t \big\|^2.
    \end{talign}

    Compute the gradient $\nabla_{\theta} \mathcal{L}(\theta)$ and update $\theta$ using favorite gradient descent algorithm.
  }
\KwOut{Fine-tuned vector field $v^{\mathrm{finetune}}$}}
\caption{Adjoint Matching for fine-tuning Flow Matching models}
\label{alg:adjoint_matching_finetuning_FM}
\end{algorithm}

\section{Related work} 

\paragraph{Fine-tuning from human feedback.}
There are two main overarching approaches to RLHF: the \textit{reward-based} approach \citep{ziegler2020finetuning,stiennon2020learning,ouyang2022training,bai2022training} and \textit{direct preference optimization} (DPO; \cite{rafailov2023direct}).
The reward-based approach \citep{ziegler2020finetuning,stiennon2020learning,ouyang2022training,bai2022training} consists in learning the reward model $r(x)$ from human preference data, and then solving a maximum entropy RL problem with rewards produced by $r(x)$. 
DPO merges the two previous steps into one: there is no need to learn $r(x)$ as human preference data is directly used to fine-tune the model. 
However, DPO is typically only applied with a filtered dataset, and does not work explicitly with a reward model.
Furthermore, for flow and diffusion models specifically, it is possible to differentiate the reward function, so there is a larger emphasis on reward-based approaches.

\paragraph{Fine-tuning for diffusion models.}
Among existing reward-based diffusion fine-tuning methods, \citet{fan2023optimizing} interpret the denoising process as a multi-step decision-making task and use policy gradient algorithms to fine-tune diffusion samplers. 
%They restrict their attention to rewards that are integral probability metrics between a target distribution and the generated distribution. 
\citet{black2024training} makes use of proximal policy gradients for fine-tuning but this does not make use of the differentiability of the reward model. 
\citet{fan2023dpok} also consider KL-regularized rewards \eqref{eq:kl_regularized_interpretation} but do not make the critical connection to the tilted distribution \eqref{eq:p_star_info} that we flesh out in \Cref{sec:value_function_bias_problem}. 
The fine-tuning algorithms of \cite{xu2023imagereward,clark2024directly} directly take gradients of the reward model and use heuristics to try to stay close to the original base generative model, but their behavior is not well understood and unrelated to the tilted distribution: \cite{xu2023imagereward} takes gradients of the reward applied on the denoised sample at different points in time, and \cite{clark2024directly} backpropagates the reward function through all or part of the diffusion trajectory. 
Finally, \cite{uehara2024finetuning} also fine-tune diffusion models with the goal of sampling from the tilted distribution \eqref{eq:p_star_info}, but their approach is much more involved than ours as it requires learning a value function, and solving two stochastic optimal control problems. Additional reward fine-tuning works include \cite{bruna2024posterior}, that provide theoretical guarantees to sample from the tilted distribution when the reward is a quadratic function, and \cite{zhang2024improving}, that propose a reward fine-tuning algorithm for the GFlowNet architecture. 

\paragraph{Inference-time optimization methods.} Some have proposed methods that do not update the base model but instead modify the generation process directly. One approach is to add a guidance term to the velocity \citep{chung2022diffusion,song2023pseudoinverse,pokle2023training}; however, this is a heuristic and it is not well-understood what particular distribution is being generated. Another approach is to directly optimize the initial noise distribution \citep{li2021differentiable,wallace2023endtoend,benhamu2024dflow}; this is taking an opposite approach to the inital value bias problem than us by moving all of the work into optimizing the initial distribution. 
A more computationally intensive approach is to perform online estimation of the optimal control, for the purpose of heuristically solving an optimal control problem within the sampling process~\citep{huang2024symbolic,rout2024rb}; these approaches aim to solve a separate control problem for each generated sample, instead of performing amortization \citep{amos2023tutorial} to learn a fine-tuned generative model.

\paragraph{Optimal control in generative modeling.} Methods from optimal control have been used to train dynamical generative models parameterized by ODEs \citep{chen2018neural}, SDEs \citep{li2020scalable}, and jump processes \citep{chen2021learning}, enabled through the adjoint method. 
They can be used to train arbitrary generative processes, but for simplified constructions these have fallen in favor of simulation-free matching objectives such as denoising score matching \citep{vincent2011connection} and Flow Matching \citep{lipman2023flow}. 
The optimal control formalism also has significance in sampling from un-normalized distributions \citep{zhang2022path,berner2023optimal,vargas2023denoising,vargas2022bayesian,richter2024improved,tzen2019theoretical}. 
The inclusion of a state cost has been used to solve transport problems where intermediate path distributions are of importance \citep{liu2023generalized,pooladian2024neural}.
These collective advances naturally lead to the consideration of the optimal control formalism for reward fine-tuning.

\paragraph{Conditional sampling in inverse problems.} \cite{denker2024deft} and \cite{wu2023practical} independently consider a pre-trained diffusion model $p(x)$, and an observation $y$ on the generated sample $x$, as well as the analytic likelihood $p(y|x)$. Their aim is to sample from the posterior $p(x) p(y|x)$, and their applications include inpainting, class-conditional generation, super-resolution, phase retrieval, non-linear deblurring, computed tomography, and protein design. Their setting reduces to a particular case of our reward fine-tuning framework by setting $r(x) = \log p(y|x)$. \cite{denker2024deft} formulate an SOC problem, and they solve it via the log-variance loss (\cite{richter2020vargrad,nüsken2023solving}), and the moment loss \citep{nüsken2023solving}\footnote{See also \cite{domingoenrich2024taxonomy} for a comparison among SOC losses.}, which they refer to as the trajectory balance loss \citep{malkin2023trajectory}. \cite{wu2023practical} propose Twisted Diffusion Sampler, an algorithm based on Sequential Monte Carlo that uses increased inference-time compute to reduce bias.
%can be understood as a poor man's SOC solver combined with importance reweighting to reduce the bias introduced by the suboptimal control. 
A third work that also tackles the conditional sampling problem is \cite{du2024doobs}, which use a Lagrangian formulation that they solve approximately using Gaussian paths. 

\section{Experiments} \label{sec:diff_finetuning_exp}

\begin{table}[t]
\centering
% \resizebox{\textwidth}{!}{%
\small
\begin{tabular}{llccccccc}
    \toprule
    & Fine-tuning & Fine-tuning & Sampling & \multirow{2}{*}{ClipScore$\, \uparrow$} & \multirow{2}{*}{PickScore$\, \uparrow$} & \multirow{2}{*}{HPS v2$\, \uparrow$} & DreamSim 
    % & Total time (s) / 
    \\
    & Method & $\sigma(t)$ & $\sigma(t)$ &  &  &  & Diversity$\, \uparrow$ 
    % & \# iterations 
    \\
    \midrule
    & None & \multirow{2}{*}{\color{gray}N/A} & $\sqrt{2 \eta_t}$ & 24.15{\tiny$\pm$0.26} & 17.25{\tiny$\pm$0.06} & 16.19{\tiny$\pm$0.17} & 53.60{\tiny$\pm$1.37} 
    % & \multirow{2}{*}{\color{gray}N/A} 
    \\ 
    & (Base model)
                                 &                     & 0                 & 28.32{\tiny$\pm$0.22} & 18.15{\tiny$\pm$0.07} & 17.89{\tiny$\pm$0.16} & \textbf{56.53{\tiny$\pm$1.52}} 
                                 % &  
                                 \\
    \midrule % \addlinespace
    \parbox[t]{2mm}{\multirow{8}{*}{\rotatebox[origin=c]{90}{Baselines \;\;\;}}}
    & \multirow{2}{*}{DRaFT-1}           & $\sqrt{2 \eta_t}$ & $\sqrt{2 \eta_t}$ & 30.18{\tiny$\pm$0.24} & 19.38{\tiny$\pm$0.08} & 24.61{\tiny$\pm$0.17} & 25.54{\tiny$\pm$0.99} 
    % & 140k{\tiny$\pm$5.9k} 
    \\
    &                                   & 0                 & 0                 & 30.95{\tiny$\pm$0.28} & 19.37{\tiny$\pm$0.06} & 24.37{\tiny$\pm$0.17} & 27.39{\tiny$\pm$1.14} 
    % & / 4000 
    \\
    \addlinespace
    & \multirow{2}{*}{DRaFT-40}          & $\sqrt{2 \eta_t}$ & $\sqrt{2 \eta_t}$ & 26.94{\tiny$\pm$0.28} & 18.34{\tiny$\pm$0.19} & 19.98{\tiny$\pm$1.02} & 41.98{\tiny$\pm$2.14} 
    % & 148k{\tiny$\pm$4.2k} 
    \\
    &                                   & 0                 & 0                 & 30.07{\tiny$\pm$0.39} & 19.45{\tiny$\pm$0.08} & 24.06{\tiny$\pm$0.24} & 36.53{\tiny$\pm$1.69} 
    % & / 1500 
    \\
    \addlinespace
    & \multirow{2}{*}{DPO}          & $\sqrt{2 \eta_t}$ & $\sqrt{2 \eta_t}$ & 24.11{\tiny$\pm$0.22} & 17.24{\tiny$\pm$0.06} & 16.15{\tiny$\pm$0.14} & 53.27{\tiny$\pm$1.36} 
    % & 118k{\tiny$\pm$0.6k} 
    \\
    &                                   & 0                 & 0                 & 27.77{\tiny$\pm$0.18} & 17.92{\tiny$\pm$0.07} & 17.30{\tiny$\pm$0.20} & 54.11{\tiny$\pm$1.50} 
    % & / 1000 
    \\
    \addlinespace
    & \multirow{2}{*}{ReFL}              & $\sqrt{2 \eta_t}$ & $\sqrt{2 \eta_t}$ & 28.59{\tiny$\pm$0.31} & 18.68{\tiny$\pm$0.10} & 22.24{\tiny$\pm$0.46} & 32.71{\tiny$\pm$2.76} 
    % & 173k{\tiny$\pm$10.9k} 
    \\
    &                                   & 0                 & 0                 & 30.06{\tiny$\pm$0.63} & 19.07{\tiny$\pm$0.21} & 23.06{\tiny$\pm$0.41} & 32.69{\tiny$\pm$1.28} 
    % & / 6000 
    \\
    \midrule % \addlinespace
    \parbox[t]{2mm}{\multirow{10}{*}{\rotatebox[origin=c]{90}{Memoryless SOC \;\;\;\; }}} 
    & Cont. Adjoint & \multirow{2}{*}{$\sqrt{2 \eta_t}$} & $\sqrt{2 \eta_t}$ & 26.99{\tiny$\pm$0.43} & 18.33{\tiny$\pm$0.16} & 20.83{\tiny$\pm$0.63} & 46.59{\tiny$\pm$1.40} 
    % & 153k{\tiny$\pm$0.9k} 
    \\
    & $\lambda = 12500$                     &                                    & 0                 & 29.49{\tiny$\pm$0.32} & 18.98{\tiny$\pm$0.16} & 21.34{\tiny$\pm$0.53} & 48.41{\tiny$\pm$1.44} 
    % & / 750 
    \\
    \addlinespace
    & Disc. Adjoint & \multirow{2}{*}{$\sqrt{2 \eta_t}$} & $\sqrt{2 \eta_t}$ & 28.04{\tiny$\pm$0.57} & 18.44{\tiny$\pm$0.21} & 20.04{\tiny$\pm$0.39} & 54.90{\tiny$\pm$2.03} 
    % & 152k{\tiny$\pm$1.5k} 
    \\
    & $\lambda = 12500$                    &                                    & 0                 & 29.28{\tiny$\pm$0.17} & 18.82{\tiny$\pm$0.14} & 19.73{\tiny$\pm$0.17} & 53.36{\tiny$\pm$2.48} 
    % & / 1000 
    \\
    \addlinespace
    \cline{2-9}\addlinespace
    & Adj.-Matching  & \multirow{2}{*}{$\sqrt{2 \eta_t}$} & $\sqrt{2 \eta_t}$ & 30.36{\tiny$\pm$0.22} & 19.29{\tiny$\pm$0.08} & 24.12{\tiny$\pm$0.17} & 40.89{\tiny$\pm$1.50} 
    % & 
    \\
    & $\lambda = 1000$                     &                                    & 0                 & 31.41{\tiny$\pm$0.22} & 19.57{\tiny$\pm$0.09} & 23.29{\tiny$\pm$0.18} & 43.10{\tiny$\pm$1.76} 
    % &  
    \\
    \addlinespace
    & Adj.-Matching & \multirow{2}{*}{$\sqrt{2 \eta_t}$} & $\sqrt{2 \eta_t}$ & 30.59{\tiny$\pm$0.40} & 19.49{\tiny$\pm$0.10} & 24.85{\tiny$\pm$0.23} & 37.07{\tiny$\pm$1.47} 
    % & 156k{\tiny$\pm$1.9k} 
    \\
    & $\lambda = 2500$                     &                                    & 0                 & 31.64{\tiny$\pm$0.21} & 19.71{\tiny$\pm$0.09} & 24.12{\tiny$\pm$0.27} & 39.88{\tiny$\pm$1.59} 
    % & / 1000 
    \\
    \addlinespace
    & Adj.-Matching  & \multirow{2}{*}{$\sqrt{2 \eta_t}$} & $\sqrt{2 \eta_t}$ & 30.62{\tiny$\pm$0.30} & 19.50{\tiny$\pm$0.09} & \textbf{24.95{\tiny$\pm$0.28}} & 34.50{\tiny$\pm$1.33}
    % &  
    \\
    & $\lambda = 12500$                    &                                    & 0                 & \textbf{31.65{\tiny$\pm$0.19}} & \textbf{19.76{\tiny$\pm$0.08}} & 24.49{\tiny$\pm$0.27} & 37.24{\tiny$\pm$1.57} 
    % &  
    \\
    \bottomrule
\end{tabular}
% }
\caption{Evaluation metrics of different fine-tuning methods for text-to-image generation. 
The second and third columns show the noise schedules $\sigma(t)$ used for fine-tuning and for sampling: $\sigma(t) = \sqrt{2\eta_t}$ corresponds to Memoryless Flow Matching, and $\sigma(t) = 0$ to the Flow Matching ODE \eqref{eq:FM_ode}. 
We report standard errors estimated over 3 runs of the fine-tuning algorithm on random sets of 40000 training prompts, each evaluated over a random set of 1000 test prompts. 
% \carles{Put wall clock time in the appendix table}
}
\label{tab:evaluation_metrics}
\end{table}

\begin{figure}[t!]
    \centering
    \begin{subfigure}[t]{0.49\linewidth}
        \centering
        \rotatebox{90}{\;\; $\lambda=1000$}\,%
        \includegraphics[width=0.24\linewidth]{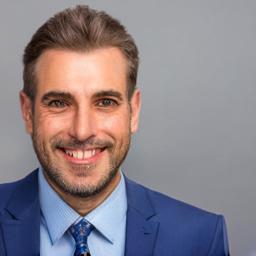}\,%
        \includegraphics[width=0.24\linewidth]{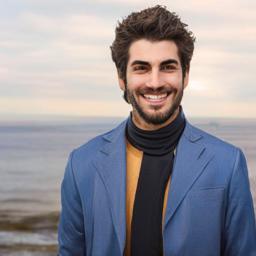}\,%
        \includegraphics[width=0.24\linewidth]{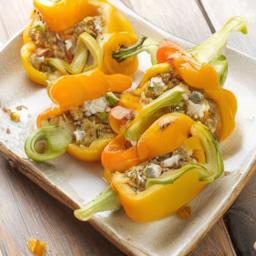}\,%
        \includegraphics[width=0.24\linewidth]{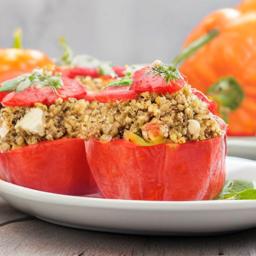}\\
        \rotatebox{90}{\;\; $\lambda=2500$}\,%
        \includegraphics[width=0.24\linewidth]{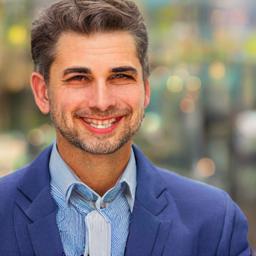}\,%
        \includegraphics[width=0.24\linewidth]{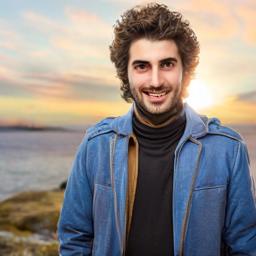}\,%
        \includegraphics[width=0.24\linewidth]{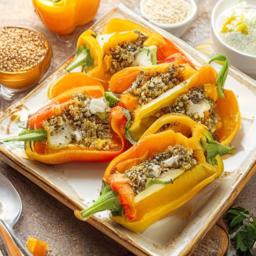}\,%
        \includegraphics[width=0.24\linewidth]{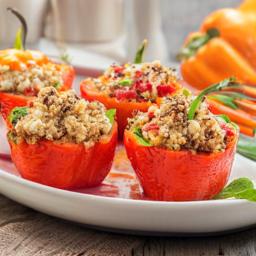}\\
        \rotatebox{90}{\; $\lambda=12500$}\,%
        \includegraphics[width=0.24\linewidth]{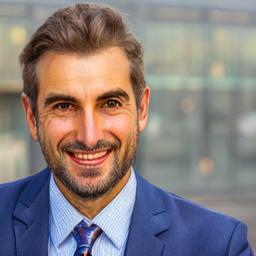}\,%
        \includegraphics[width=0.24\linewidth]{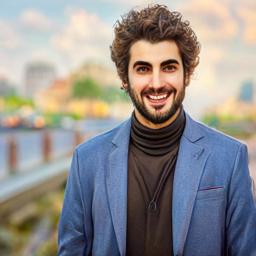}\,%
        \includegraphics[width=0.24\linewidth]{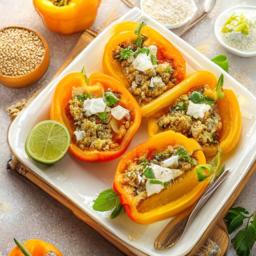}\,%
        \includegraphics[width=0.24\linewidth]{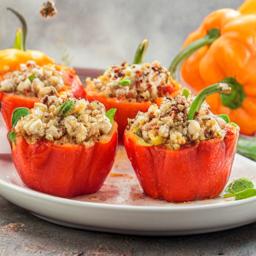}
        \caption*{Adjoint Matching (Ours)}
    \end{subfigure}\hfill
    \begin{subfigure}[t]{0.49\linewidth}
        \centering
        \includegraphics[width=0.24\linewidth]{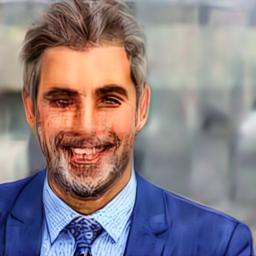}\,%
        \includegraphics[width=0.24\linewidth]{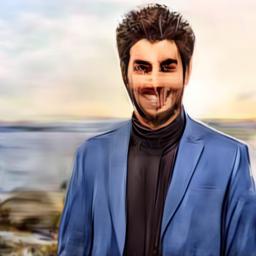}\,%
        \includegraphics[width=0.24\linewidth]{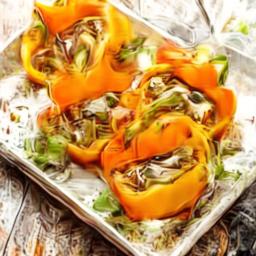}\,%
        \includegraphics[width=0.24\linewidth]{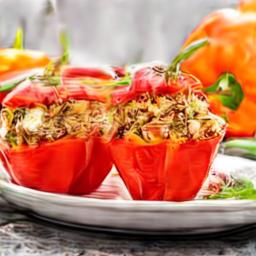}\,%
        \rotatebox[origin=r]{270}{$1000$ itrs. }\\
        \includegraphics[width=0.24\linewidth]{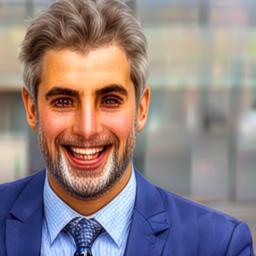}\,%
        \includegraphics[width=0.24\linewidth]{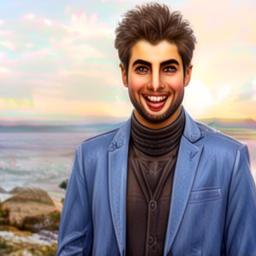}\,%
        \includegraphics[width=0.24\linewidth]{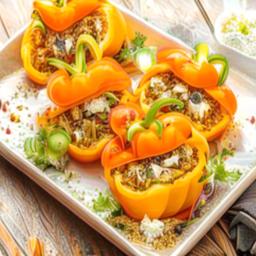}\,%
        \includegraphics[width=0.24\linewidth]{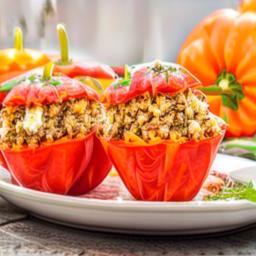}\,%
        \rotatebox[origin=r]{270}{$2000$ itrs. }\\
        \includegraphics[width=0.24\linewidth]{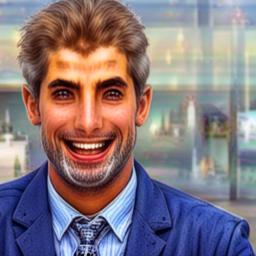}\,%
        \includegraphics[width=0.24\linewidth]{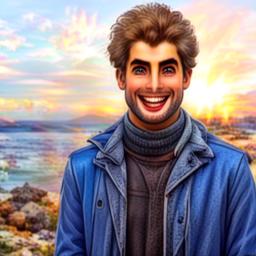}\,%
        \includegraphics[width=0.24\linewidth]{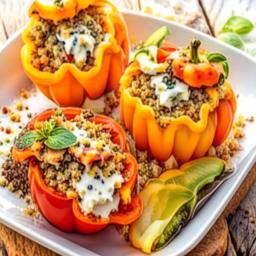}\,%
        \includegraphics[width=0.24\linewidth]{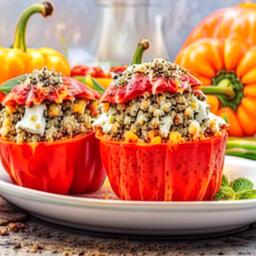}\,%
        \rotatebox[origin=r]{270}{$4000$ itrs. }
        \caption*{DRaFT-1}
    \end{subfigure}
    \caption{Our proposed Adjoint Matching using the memoryless SOC formulation introduces a much more principled way of trading off how close to stay to the base model while optimizing the reward model. In contrast, baseline methods such as DRaFT-1 only optimize the reward model and must rely on early stopping to perform this trade off, resulting in a much more sensitive hyperparameter. Samples are produced using $\sigma(t)=0$ with the same noise sample. Text prompts: ``\textit{Handsome Smiling man in blue jacket portrait}'' and ``\textit{Quinoa and Feta Stuffed Baby Bell Peppers}''.}
    \label{fig:ablation_tradeoff_lambda}
\end{figure}

\begin{figure}[t!]
    \centering
    \begin{subfigure}[t]{0.49\linewidth}
        \centering
        \rotatebox{90}{\;\;\; $w=0.0$}\,%
        \includegraphics[width=0.24\linewidth]{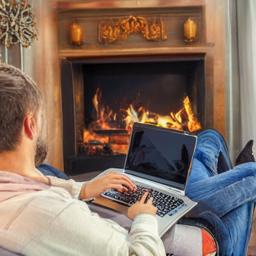}\,%
        \includegraphics[width=0.24\linewidth]{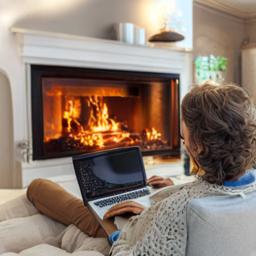}\,%
        \includegraphics[width=0.24\linewidth]{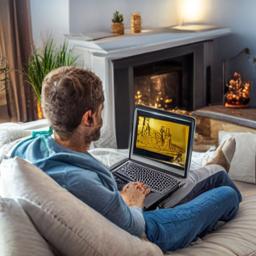}\,%
        \includegraphics[width=0.24\linewidth]{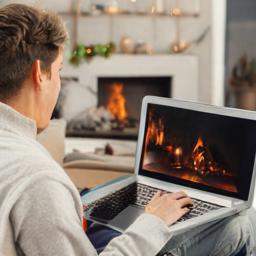}\\
        \rotatebox{90}{\;\;\; $w=1.0$}\,%
        \includegraphics[width=0.24\linewidth]{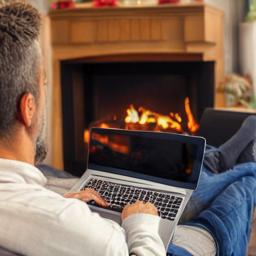}\,%
        \includegraphics[width=0.24\linewidth]{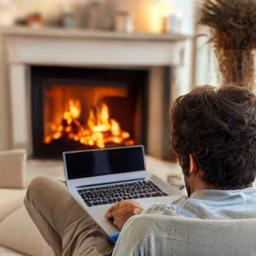}\,%
        \includegraphics[width=0.24\linewidth]{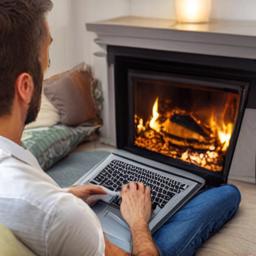}\,%
        \includegraphics[width=0.24\linewidth]{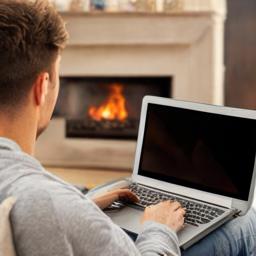}\\
        \rotatebox{90}{\;\;\; $w=4.0$}\,%
        \includegraphics[width=0.24\linewidth]
        {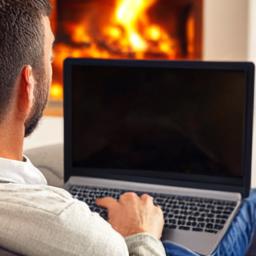}\,%
        \includegraphics[width=0.24\linewidth]{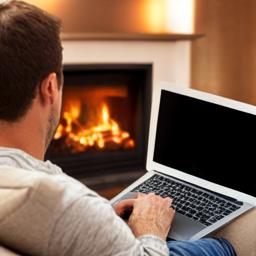}\,%
        \includegraphics[width=0.24\linewidth]{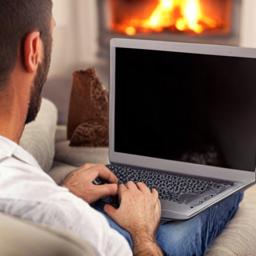}\,%
        \includegraphics[width=0.24\linewidth]{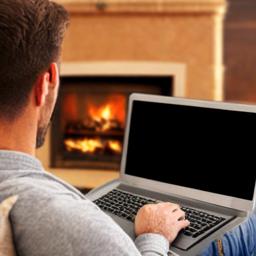}
        \caption*{Text prompt: ``\textit{Man sitting on sofa at home in front of fireplace and using laptop computer, rear view}''}
    \end{subfigure}\hfill
    \begin{subfigure}[t]{0.49\linewidth}
        \centering
        \includegraphics[width=0.24\linewidth]{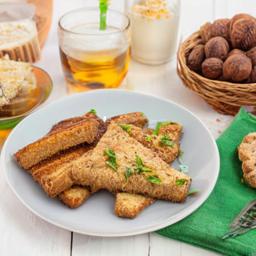}\,%
        \includegraphics[width=0.24\linewidth]{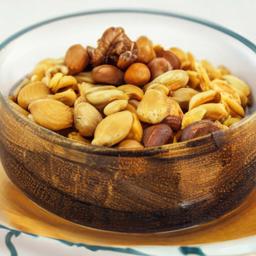}\,%
        \includegraphics[width=0.24\linewidth]{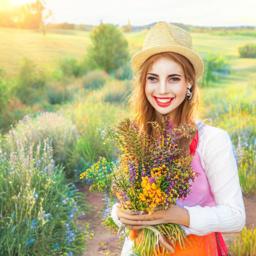}\,%
        \includegraphics[width=0.24\linewidth]{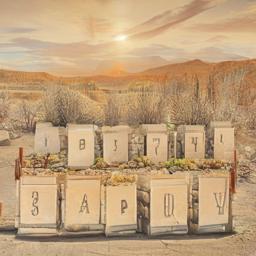}\\
        \includegraphics[width=0.24\linewidth]{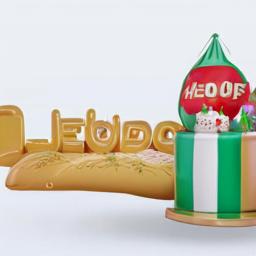}\,%
        \includegraphics[width=0.24\linewidth]{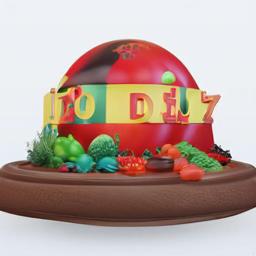}\,%
        \includegraphics[width=0.24\linewidth]{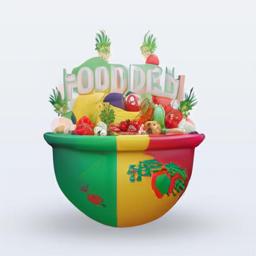}\,%
        \includegraphics[width=0.24\linewidth]{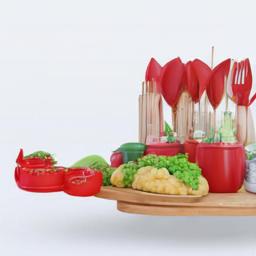}\\
        \includegraphics[width=0.24\linewidth]{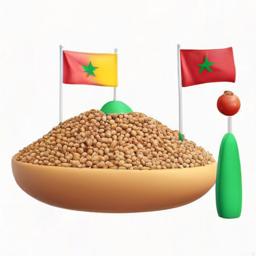}\,%
        \includegraphics[width=0.24\linewidth]{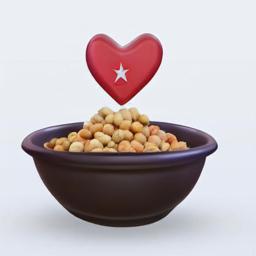}\,%
        \includegraphics[width=0.24\linewidth]{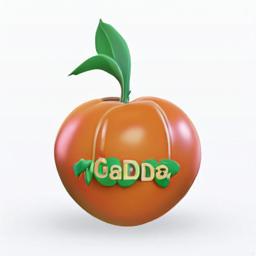}\,%
        \includegraphics[width=0.24\linewidth]{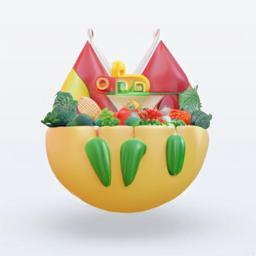}
        \caption*{Text prompt: ``\textit{3D World Food Day Morocco}''}
    \end{subfigure}
    \caption{
    Generated samples from varying classifier-free guidance weight $w$, from an Adjoint Matching fine-tuned model. 
    Higher guidance increases text-to-image consistency but loses diversity and has use cases for generating highly structured images such as 3D renderings.
    Corresponding samples from the base model can be found in \Cref{fig:ablation_tradeoff_cfg_base}. 
    }
    \label{fig:ablation_tradeoff_cfg}
\end{figure}

\begin{figure}
    \centering
    \includegraphics[width=0.65\linewidth]{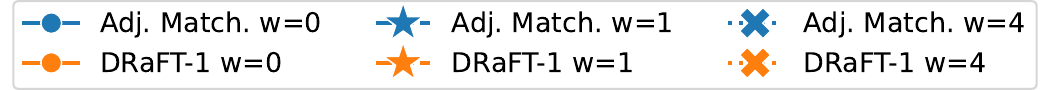}\\
    \includegraphics[width=0.325\linewidth]{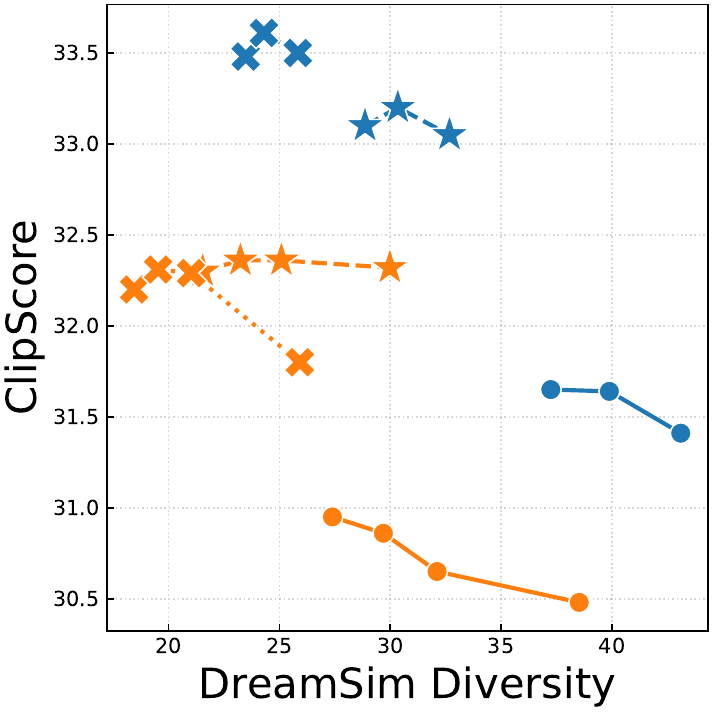}
    \includegraphics[width=0.325\linewidth]{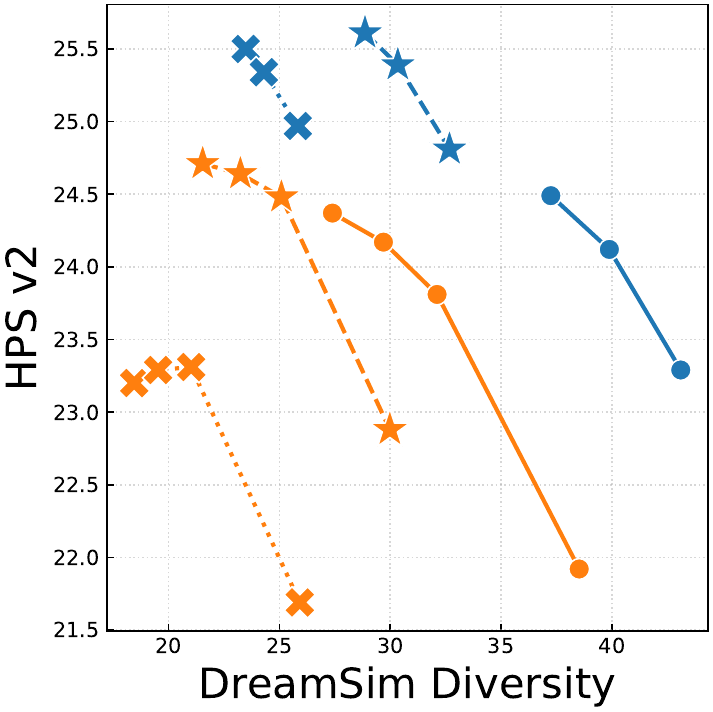}
    \includegraphics[width=0.325\linewidth]{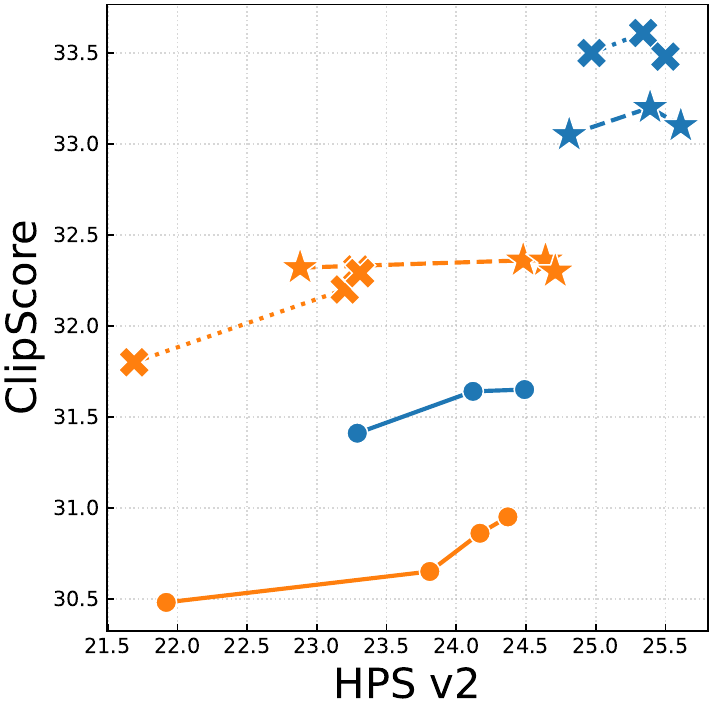}
    \caption{Tradeoffs between different aspects of generative models: text-to-image consistency (ClipScore), sample diversity for each prompt (DreamSim Diversity), and generalization to unseen human preferences (HPS v2). Different points are obtained from varying values of $\lambda$ for Adjoint Matching and varying number of fine-tuning iterations for the DRaFT-1 baseline. Overall, we find our proposed method Adjoint Matching has the best Pareto fronts.}
    \label{fig:tradeoffs}
\end{figure}

We experimentally validate our proposed method on reward fine-tuning a Flow Matching base model \citep{lipman2023flow}. In particular, we use the usual setup of pre-training an autoencoder for 512$\times$512 resolution images, then training a text-conditional Flow Matching model on the latent variables with a U-net architecture \citep{long2015fully}, similar to the setup in \citet{rombach2022high}. We pre-trained our base model using a dataset of licensed text and image pairs. Then for fine-tuning, we consider the reward function:
\begin{talign}
    r(x) := \lambda \times \texttt{RewardModel}(x)
\end{talign}
corresponding to a scaled version of the reward model, which we take to be ImageReward~\citep{xu2023imagereward}. Different values of $\lambda$ provide different tradeoffs between the KL regularization and the reward model \eqref{eq:kl_regularized_interpretation}.

For evaluation and benchmarking purposes, we report metrics that separately quantify text-to-image consistency, human preference, and sample diversity, capturing the tradeoff between each aspect of generative models~\citep{astolfi2024consistency}. For consistency, we make use of the standard ClipScore \citep{hessel2021clipscore} and PickScore \citep{kirstain2023pickapic}; for generalization to unseen human preferences, we use the HPSv2 model \citep{wu2023human}; and for diversity, we compute averages of pairwise distances of the DreamSim features \citep{fu2023learning}. More details are provided in \Cref{subsec:evaluation_metrics}.
% \ricky{TODO appendix}.

As our baselines, we consider the DPO \citep{wallace2023diffusion}, ReFL \citep{xu2023imagereward}, and DRaFT-K algorithms \citep{clark2024directly}. DPO does not use gradients from the reward function, while ReFL and DRaFT make use of heuristic gradient stopping approaches to stay close to the base generative model. 
Out of these baseline methods, we find that DRaFT-1 performs the best, so we perform additional ablation experiments comparing to this method.
Within the same SOC formulation as our method, we also consider the discrete and continuous adjoint methods. We provide full experimental details in \Cref{sec:experimental_details}; an important implementation detail is that we slightly offset $\sigma(t)$ in order to avoid division by zero.

\paragraph{Evaluation results.} In \Cref{tab:evaluation_metrics} we report the evaluation metrics for the baselines as well as our proposed Adjoint Matching approach. We compare each method at roughly the same wall clock time (see the times and number of iterations in \Cref{table:metrics_multiprompt_alternative}, and comments in \Cref{subsec:remarks_computational_cost}).
We find that across all metrics, our proposed memoryless SOC formulation outperforms existing baseline methods. The choice of SOC algorithms also obviously favors Adjoint Matching over continuous and discrete adjoint methods, which result in poorer consistency and human preference metrics.

\paragraph{Ablation: base model vs. reward tradeoff.} We note that the scaling in front of the reward model $\lambda$ determines how strongly the we should prefer the reward model over the base model. As such, we see a natural tradeoff curve: higher $\lambda$ results in better consistency and human preference, but lower diversity in the generated samples. Overall, we find that Adjoint Matching performs stably across all values of $\lambda$. Our method of regularizing the fine-tuning procedure through memoryless SOC works much better than baseline methods which often must employ early stopping. We show the qualitative effect of varying $\lambda$ in \Cref{fig:ablation_tradeoff_lambda}, while for the DRaFT-1 baseline we show the effect of varying the number of fine-tuning iterations.

\paragraph{Ablation: classifier-free guidance.} We note that it is possible to apply classifier-free guidance (CFG; \citet{ho2022classifier,zheng2023guided}) after fine-tuning. We use the formula $(1+w) v(x, t | y) - w v(x, t)$, where $w$ is the guidance weight, $v(x, t | y)$ is a fine-tuned text-to-image model while $v(x, t)$ is an unconditional image model.
This is not principled as only the conditional model is fine-tuned, but generally it is unclear what distribution guided models sample from anyhow. 
In \Cref{fig:tradeoffs} we show the evaluation metrics with classifier-free guidance applied. Comparing three different guidance weight values, we see a higher weight does improve text-to-image consistency, and to some extent, human preference, but this comes at the cost of being worse in terms of diversity. We show qualitative differences in \Cref{fig:ablation_tradeoff_cfg}.

\section{Conclusion}
\label{sec:conclusion}

We investigate the problem of fine-tuning dynamical generative models such as Flow Matching and propose the use of a stochastic optimal control (SOC) formulation with a memoryless noise schedule. 
This ensures we converge to the same tilted distribution that the large language modeling literature uses for learning from human feedback. 
In particular, the memoryless noise schedule corresponds to DDPM sampling for diffusion models and a new Memoryless Flow Matching generative process for flow models. 
In conjunction, we propose a novel training algorithm for solving stochastic optimal control problems, by casting SOC as a regression problem, which we call the Adjoint Matching objective. 
Empirically, we find that our memoryless SOC formulation works better than multiple existing works on fine-tuning diffusion models, and our Adjoint Matching algorithm outperforms related gradient-based methods.
In summary, we are the first to provide a theoretically-driven algorithm for fine-tuning Flow Matching models, and we find that our approach significantly outperforms baseline methods across multiple axes of evaluation---text-to-image consistency, generalization to unseen human preference, and sample diversity---on large-scale text-to-image generation.

\bibliographystyle{assets/plainnat}
\bibliography{biblio}

\clearpage
\newpage
\beginappendix

\tableofcontents

\starttocentries

\clearpage
\newpage

\section{Additional Figures \& Tables}
\label{sec:additional_figures_tables}

\begin{figure}[h!]
    \centering
    \includegraphics[width=0.46\linewidth]{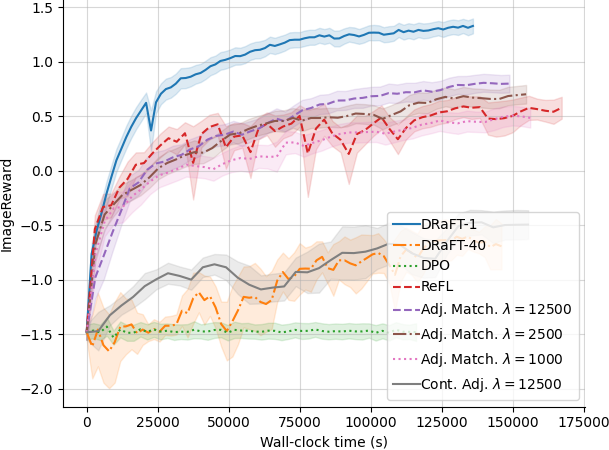}
    \includegraphics[width=0.46\linewidth]{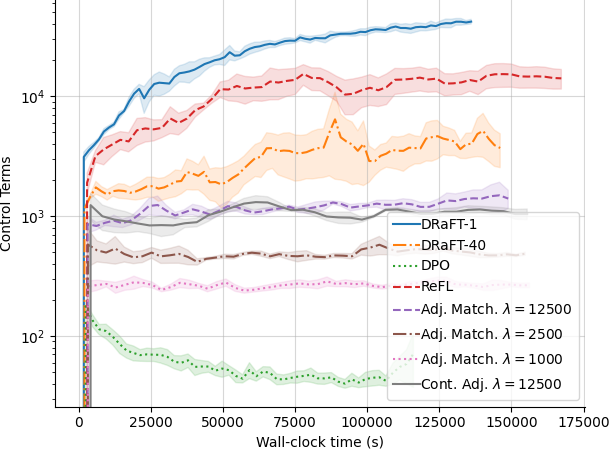}
    \includegraphics[width=0.46\linewidth]{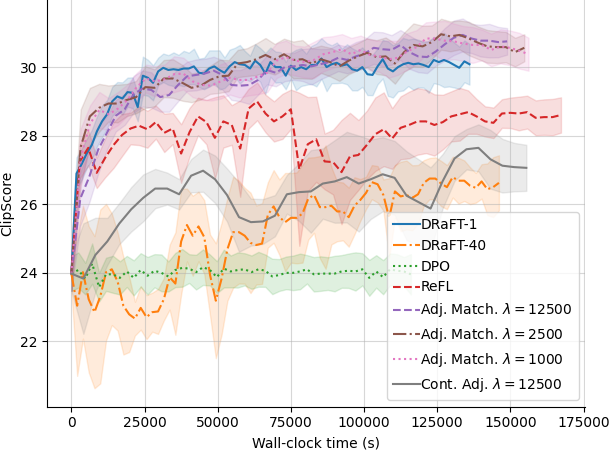}
    \caption{Average values of ImageReward (reward function), control cost ($\int_0^t \frac{1}{2} \|u(X^u_t,t)\|^2 \, \mathrm{d}t$), and ClipScore vs. wall-clock time for Adjoint Matching and our baselines. Lines show averages over three fine-tuning runs, evaluating on separate test datasets of size 200. Confidence intervals show standard errors of estimates.}
    \label{fig:training_figures}
\end{figure}

\begin{table}[h!]
\centering
{\small
\begin{tabular}{lcccccc}
    \toprule
    Fine-tuning & Fine-tuning & Sampling & \multirow{2}{*}{ImageReward$\, \uparrow$} & ClipScore & PickScore & Total time (s) /
    \\
    loss & $\sigma(t)$ & $\sigma(t)$ &  & diversity$\, \uparrow$ & diversity$\, \uparrow$ & \# iterations \\
    \midrule
    None & \multirow{2}{*}{N/A} & $\sqrt{2 \eta_t}$ & $-$1.384{\tiny$\pm$0.040} & 28.07{\tiny$\pm$1.40} & 1.63{\tiny$\pm$0.08} & \multirow{2}{*}{N/A} 
    \\
    ($\mathrm{CFG}=1.0$)                    &                     & 0                 & $-$0.920{\tiny$\pm$0.042} & 30.29{\tiny$\pm$1.53} & 1.82{\tiny$\pm$0.09} \\
    \midrule
    \multirow{2}{*}{DRaFT-1}           & $\sqrt{2 \eta_t}$ & $\sqrt{2 \eta_t}$ & 1.357{\tiny$\pm$0.039} & 16.86{\tiny$\pm$0.98} & 1.21{\tiny$\pm$0.07} 
    & 140k{\tiny$\pm$5.9k}
    \\
                                       & 0                 & 0                 & 1.251{\tiny$\pm$0.040} & 16.76{\tiny$\pm$1.06} & 1.27{\tiny$\pm$0.07} 
    & / 4000 
    \\
    \addlinespace
    \multirow{2}{*}{DRaFT-40}          & $\sqrt{2 \eta_t}$ & $\sqrt{2 \eta_t}$ & $-$0.560{\tiny$\pm$0.138} & 24.07{\tiny$\pm$1.37} & 1.64{\tiny$\pm$0.12} 
    & 148k{\tiny$\pm$4.2k} 
    \\
                                       & 0                 & 0                 & 0.424{\tiny$\pm$0.042} & 20.99{\tiny$\pm$1.54} & 1.67{\tiny$\pm$0.08} 
    & / 1500 
    \\
    \addlinespace
    \multirow{2}{*}{DPO}          & $\sqrt{2 \eta_t}$ & $\sqrt{2 \eta_t}$ & $-$1.386{\tiny$\pm$0.033} & 27.80{\tiny$\pm$1.40} & 1.62{\tiny$\pm$0.08} 
    & 118k{\tiny$\pm$0.6k} 
    \\
                                       & 0                 & 0                 & $-$0.957{\tiny$\pm$0.040} & 29.81{\tiny$\pm$1.43} & 1.68{\tiny$\pm$0.10} 
    & / 1000
    \\
    \addlinespace
    \multirow{2}{*}{ReFL}              & $\sqrt{2 \eta_t}$ & $\sqrt{2 \eta_t}$ & 0.687{\tiny$\pm$0.085} & 19.49{\tiny$\pm$1.76} & 1.22{\tiny$\pm$0.08} 
    & 173k{\tiny$\pm$10.9k}
    \\
                                       & 0                 & 0                 & 0.709{\tiny$\pm$0.080} & 18.39{\tiny$\pm$1.11} & 1.31{\tiny$\pm$0.10} 
    & / 6000 
    \\
    \midrule % \addlinespace
    Cont. Adjoint & \multirow{2}{*}{$\sqrt{2 \eta_t}$} & $\sqrt{2 \eta_t}$ & $-$0.448{\tiny$\pm$0.135} & 26.97{\tiny$\pm$1.37} & 1.82{\tiny$\pm$0.09} 
    & 153k{\tiny$\pm$0.9k}
    \\
    $\lambda = 12500$                     &                                    & 0                 & $-$0.249{\tiny$\pm$0.116} & 26.25{\tiny$\pm$1.30} & 1.90{\tiny$\pm$0.10} 
    & / 750  
    \\
    \addlinespace
    Disc. Adjoint & \multirow{2}{*}{$\sqrt{2 \eta_t}$} & $\sqrt{2 \eta_t}$ & $-$0.557{\tiny$\pm$0.113} & 30.40{\tiny$\pm$2.39} & 1.91{\tiny$\pm$0.09} 
    & 152k{\tiny$\pm$1.5k}
    \\
    $\lambda = 12500$                    &                                    & 0                 & $-$0.552{\tiny$\pm$0.041} & 28.37{\tiny$\pm$2.26} & 1.97{\tiny$\pm$0.09} 
    & / 1000 
    \\
    \midrule % \addlinespace
    Adj.-Matching  & \multirow{2}{*}{$\sqrt{2 \eta_t}$} & $\sqrt{2 \eta_t}$ & 0.550{\tiny$\pm$0.043} & 23.00{\tiny$\pm$1.27} & 1.65{\tiny$\pm$0.08} 
    &  
    \\
    $\lambda = 1000$                     &                                    & 0                 & 0.454{\tiny$\pm$0.055} & 22.76{\tiny$\pm$1.40} & 1.73{\tiny$\pm$0.09} 
    % &  
    \\
    \addlinespace
    Adj.-Matching & \multirow{2}{*}{$\sqrt{2 \eta_t}$} & $\sqrt{2 \eta_t}$ & 0.755{\tiny$\pm$0.040} & 21.33{\tiny$\pm$1.71} & 1.55{\tiny$\pm$0.08} 
    & 156k{\tiny$\pm$1.9k}
    \\
    $\lambda = 2500$                     &                                    & 0                 & 0.671{\tiny$\pm$0.047} & 21.42{\tiny$\pm$1.54} & 1.64{\tiny$\pm$0.08}
    & / 1000 
    \\
    \addlinespace
    Adj.-Matching  & \multirow{2}{*}{$\sqrt{2 \eta_t}$} & $\sqrt{2 \eta_t}$ & 0.882{\tiny$\pm$0.058} & 20.49{\tiny$\pm$1.48} & 1.50{\tiny$\pm$0.09} 
    & 
    \\
    $\lambda = 12500$                    &                                    & 0                 & 0.778{\tiny$\pm$0.050} & 20.34{\tiny$\pm$1.49} & 1.57{\tiny$\pm$0.09}
    &  
    \\
    \bottomrule
\end{tabular}
}
\caption{Metrics for various fine-tuning methods for text-to-image generation. The second and third columns show the noise schedules $\sigma(t)$ used for fine-tuning and for inference: $\sigma(t) = \sqrt{2\eta_t}$ corresponds to Memoryless Flow Matching, %\eqref{eq:memoryless_FM_sde}, 
and $\sigma(t) = 0$ to the Flow Matching ODE \eqref{eq:FM_ode}. Confidence intervals show standard errors of estimates; computed over 3 runs of the fine-tuning algorithm on separate fine-tuning prompt datasets of size 40000 each. Test prompt sets are of size 1000, and also different for each run.}
\label{table:metrics_multiprompt_diversity}
\end{table}

\begin{table}[h!]
\centering
{\footnotesize
\begin{tabular}{lcccccccc}
    \toprule
    Fine-tun. & Fine-tun. & Generat. & \multirow{2}{*}{ImageReward$\, \uparrow$} & \multirow{2}{*}{ClipScore$\, \uparrow$} & \multirow{2}{*}{PickScore$\, \uparrow$} & \multirow{2}{*}{HPS v2$\, \uparrow$} & DreamSim & Runtime/ \\
    loss & $\sigma(t)$ & $\sigma(t)$ &  & &  &  & diversity$\, \uparrow$ & $\#$iter. \\
    % \midrule
    % \multirow{2}{*}{None (Base)} & \multirow{2}{*}{N/A} & $\sqrt{2 \eta_t}$ & $-\! \pm \!-$ & $-\! \pm \!-$ & $-\! \pm \!-$ & $-\! \pm \!-$ & \multirow{2}{*}{N/A} \\
    %                              &                     & 0                 & $-\! \pm \!-$ & $-\! \pm \!-$ & $-\! \pm \!-$ & $-\! \pm \!-$ &  \\
    \midrule % \addlinespace
    \multirow{2}{*}{ReFL}              & $\sqrt{2 \eta_t}$ & $\sqrt{2 \eta_t}$ & 0.459{\tiny$\pm$0.096} & 28.46{\tiny$\pm$0.25} & 18.77{\tiny$\pm$0.09} & 22.54{\tiny$\pm$0.17} & 37.51{\tiny$\pm$3.50} & 43k{\tiny$\pm$2.7k} \\
                                       & 0                 & 0                 & 0.330{\tiny$\pm$0.114} & 29.63{\tiny$\pm$0.61} & 19.08{\tiny$\pm$0.18} & 22.46{\tiny$\pm$0.77} & 39.51{\tiny$\pm$1.30} & / 1500 \\
    \addlinespace
    \multirow{2}{*}{DRaFT-1}           & $\sqrt{2 \eta_t}$ & $\sqrt{2 \eta_t}$ & 0.913{\tiny$\pm$0.068} & 29.80{\tiny$\pm$0.22} & 19.16{\tiny$\pm$0.06} & 23.63{\tiny$\pm$0.16} & 35.21{\tiny$\pm$1.93} & 35k{\tiny$\pm$1.5k} \\
                                       & 0                 & 0                 & 0.626{\tiny$\pm$0.195} & 30.48{\tiny$\pm$0.32} & 18.91{\tiny$\pm$0.34} & 21.92{\tiny$\pm$1.63} & 38.52{\tiny$\pm$2.01} & / 1000 \\
    \addlinespace
    \multirow{2}{*}{Draft-40}          & $\sqrt{2 \eta_t}$ & $\sqrt{2 \eta_t}$ & $-$1.427{\tiny$\pm$0.267} & 23.39{\tiny$\pm$1.72} & 17.24{\tiny$\pm$0.45} & 15.72{\tiny$\pm$1.80} & 41.98{\tiny$\pm$2.14} & 49k{\tiny$\pm$1.4k} \\
                                       & 0                 & 0                 & $-$0.097{\tiny$\pm$0.052} & 29.12{\tiny$\pm$0.41} & 18.97{\tiny$\pm$0.14} & 21.93{\tiny$\pm$0.20} & 46.35{\tiny$\pm$1.34} & / 500 \\
    \midrule % \addlinespace
    Adj.-Match.  & \multirow{2}{*}{$\sqrt{2 \eta_t}$} & $\sqrt{2 \eta_t}$ & 0.107{\tiny$\pm$0.046} & 29.37{\tiny$\pm$0.25} & 19.05{\tiny$\pm$0.07} & 22.79{\tiny$\pm$0.20} & 46.38{\tiny$\pm$1.36} &  \\
    $\lambda = 1000$                     &                                    & 0                 & 0.051{\tiny$\pm$0.044} & 30.58{\tiny$\pm$0.17} & 19.31{\tiny$\pm$0.07} & 21.93{\tiny$\pm$0.23} & 48.12{\tiny$\pm$1.56} &  \\
    \addlinespace
    Adj.-Match. & \multirow{2}{*}{$\sqrt{2 \eta_t}$} & $\sqrt{2 \eta_t}$ & 0.199{\tiny$\pm$0.068} & 29.27{\tiny$\pm$0.21} & 19.07{\tiny$\pm$0.10} & 22.98{\tiny$\pm$0.30} & 45.03{\tiny$\pm$1.61} & 39k{\tiny$\pm$0.5k} \\
    $\lambda = 2500$                     &                                    & 0                 & 0.106{\tiny$\pm$0.067} & 30.43{\tiny$\pm$0.24} & 19.32{\tiny$\pm$0.11} & 22.16{\tiny$\pm$0.33} & 47.61{\tiny$\pm$1.49} & / 250 \\
    \addlinespace
    Adj.-Match.  & \multirow{2}{*}{$\sqrt{2 \eta_t}$} & $\sqrt{2 \eta_t}$ & 0.299{\tiny$\pm$0.095} & 29.61{\tiny$\pm$0.37} & 19.26{\tiny$\pm$0.14} & 23.67{\tiny$\pm$0.27} & 43.36{\tiny$\pm$1.93} &  \\
    $\lambda = 12500$                    &                                    & 0                 & 0.224{\tiny$\pm$0.051} & 30.70{\tiny$\pm$0.23} & 19.52{\tiny$\pm$0.11} & 22.93{\tiny$\pm$0.21} & 44.62{\tiny$\pm$1.79} &  \\
    \midrule % \addlinespace
    Cont. Adj. & \multirow{2}{*}{$\sqrt{2 \eta_t}$} & $\sqrt{2 \eta_t}$ & $-$0.910{\tiny$\pm$0.116} & 26.29{\tiny$\pm$0.44} & 18.06{\tiny$\pm$0.16} & 18.86{\tiny$\pm$0.88} & 51.60{\tiny$\pm$1.97} & 51k{\tiny$\pm$0.3k} \\
    $\lambda = 12500$                     &                                    & 0                 & $-$0.681{\tiny$\pm$0.051} & 28.50{\tiny$\pm$0.19} & 18.69{\tiny$\pm$0.11} & 19.90{\tiny$\pm$0.50} & 50.87{\tiny$\pm$1.52} & / 250 \\
    \addlinespace
    Disc. Adj. & \multirow{2}{*}{$\sqrt{2 \eta_t}$} & $\sqrt{2 \eta_t}$ & $-$0.978{\tiny$\pm$0.123} & 26.68{\tiny$\pm$0.76} & 18.51{\tiny$\pm$0.11} & 18.53{\tiny$\pm$0.28} & 55.95{\tiny$\pm$1.70} & 38k{\tiny$\pm$0.4k} \\
    $\lambda = 12500$                    &                                    & 0                 & $-$0.791{\tiny$\pm$0.065} & 28.66{\tiny$\pm$0.33} & 18.51{\tiny$\pm$0.11} & 18.53{\tiny$\pm$0.28} & 54.78{\tiny$\pm$2.00} & / 250 \\
    \bottomrule
\end{tabular}
}
\caption{Additional metrics for various fine-tuning methods for text-to-image generation, which complement the ones in \Cref{tab:evaluation_metrics} (both tables correspond to the same runs). The second and third columns show the noise schedules $\sigma(t)$ used for fine-tuning and for inference: $\sigma(t) = \sqrt{2\eta_t}$ corresponds to Memoryless Flow Matching, 
% \eqref{eq:memoryless_FM_sde}, 
and $\sigma(t) = 0$ to the Flow Matching ODE \eqref{eq:FM_ode}.
}
\label{table:metrics_multiprompt_alternative}
\end{table}

\begin{table}[h!]
\centering
{\footnotesize
\begin{tabular}{lllccccccc}
    \toprule
    \multirow{2}{*}{$w$} & Fine-tuning & $\#$iter. & Fine-tun. & Sampl.  & \multirow{2}{*}{ImageReward$\, \uparrow$} & \multirow{2}{*}{ClipScore$\, \uparrow$} & \multirow{2}{*}{PickScore$\, \uparrow$} & \multirow{2}{*}{HPS v2$\, \uparrow$} & DreamSim \\
     & loss & / $\lambda$ & $\sigma(t)$ & $\sigma(t)$ &  &  &  &  & diversity$\, \uparrow$ \\
    \midrule
    \multirow{2}{*}{0.0} & \multirow{2}{*}{None} &  \multirow{2}{*}{N/A}        & \multirow{2}{*}{N/A}          & $\sqrt{2 \eta_t}$         & $-$1.384{\tiny$\pm$0.040} & 24.15{\tiny$\pm$0.26} & 17.25{\tiny$\pm$0.06} & 16.19{\tiny$\pm$0.17} & 53.60{\tiny$\pm$1.37} \\
     &  &       &                & 0                 & $-$0.920{\tiny$\pm$0.042} & 28.32{\tiny$\pm$0.22} & 18.15{\tiny$\pm$0.07} & 17.89{\tiny$\pm$0.16} & \textbf{56.53{\tiny$\pm$1.52}}       
                                    \\
    \midrule
    \multirow{7}{*}{0.0} & \multirow{7}{*}{DRaFT-1} &  \multirow{2}{*}{1000}       & $\sqrt{2 \eta_t}$ & $\sqrt{2 \eta_t}$ & 0.913{\tiny$\pm$0.068} & 29.80{\tiny$\pm$0.22} & 19.16{\tiny$\pm$0.06} & 23.63{\tiny$\pm$0.16} & 35.21{\tiny$\pm$1.93} \\
    & &  & 0                 & 0                 & 0.626{\tiny$\pm$0.195} & 30.48{\tiny$\pm$0.32} & 18.91{\tiny$\pm$0.34} & 21.92{\tiny$\pm$1.63} & 38.52{\tiny$\pm$2.01} \\
    % \midrule
    \addlinespace
     &  &  \multirow{2}{*}{2000}       & $\sqrt{2 \eta_t}$ & $\sqrt{2 \eta_t}$ & 1.204{\tiny$\pm$0.046} & 29.90{\tiny$\pm$0.43} & 19.29{\tiny$\pm$0.12} & 24.40{\tiny$\pm$0.27} & 28.51{\tiny$\pm$1.68} \\
                                   & &  & 0                 & 0                 & 1.052{\tiny$\pm$0.088} & 30.65{\tiny$\pm$0.24} & 19.27{\tiny$\pm$0.11} & 23.81{\tiny$\pm$0.44} & 32.11{\tiny$\pm$2.37} \\
    % \midrule
    \addlinespace
    & &  \multirow{2}{*}{3000}       & $\sqrt{2 \eta_t}$ & $\sqrt{2 \eta_t}$ & \textbf{1.307{\tiny$\pm$0.041}} & 29.96{\tiny$\pm$0.22} & 19.31{\tiny$\pm$0.06} & 24.42{\tiny$\pm$0.13} & 26.57{\tiny$\pm$1.32} \\
                                   & &  & 0                 & 0                 & 1.173{\tiny$\pm$0.058} & 30.86{\tiny$\pm$0.25} & 19.37{\tiny$\pm$0.06} & 24.17{\tiny$\pm$0.23} & 29.69{\tiny$\pm$1.30}
    \\
    \addlinespace
    & &  \multirow{2}{*}{4000}       & $\sqrt{2 \eta_t}$ & $\sqrt{2 \eta_t}$ & \textbf{1.357{\tiny$\pm$0.039}} & 30.18{\tiny$\pm$0.24} & 19.38{\tiny$\pm$0.08} & 24.61{\tiny$\pm$0.17} & 25.54{\tiny$\pm$0.99} \\
                                   & &  & 0                 & 0                 & 1.251{\tiny$\pm$0.040} & 30.95{\tiny$\pm$0.28} & 19.37{\tiny$\pm$0.06} & 24.37{\tiny$\pm$0.17} & 27.39{\tiny$\pm$1.14}
    \\
    \midrule
    \multirow{7}{*}{0.0} & \multirow{7}{*}{Adj.-Match.} & \multirow{2}{*}{1000} & $\sqrt{2 \eta_t}$ & $\sqrt{2 \eta_t}$  & 0.550{\tiny$\pm$0.043} & 30.36{\tiny$\pm$0.22} & 19.29{\tiny$\pm$0.08} & 24.12{\tiny$\pm$0.17} & 40.89{\tiny$\pm$1.50} \\
     & &  & 0                 & 0                 & 0.454{\tiny$\pm$0.055} &  31.41{\tiny$\pm$0.22} & 19.57{\tiny$\pm$0.09} & 23.29{\tiny$\pm$0.18} & 43.10{\tiny$\pm$1.76} \\
    \addlinespace
    & & \multirow{2}{*}{2500} & $\sqrt{2 \eta_t}$ & $\sqrt{2 \eta_t}$ & 0.755{\tiny$\pm$0.040} & 30.59{\tiny$\pm$0.40} & 19.49{\tiny$\pm$0.10} & 24.85{\tiny$\pm$0.23} & 37.07{\tiny$\pm$1.47} \\
    & &  & 0 & 0 & 0.671{\tiny$\pm$0.047} & 31.64{\tiny$\pm$0.21} & 19.71{\tiny$\pm$0.09} & 24.12{\tiny$\pm$0.27} & 39.88{\tiny$\pm$1.59} \\
    \addlinespace
    & & \multirow{2}{*}{12500} & $\sqrt{2 \eta_t}$ & $\sqrt{2 \eta_t}$ & 0.882{\tiny$\pm$0.058} & 30.62{\tiny$\pm$0.30} & 19.50{\tiny$\pm$0.09} & 24.95{\tiny$\pm$0.28} & 34.50{\tiny$\pm$1.33} \\
    & &  & 0 & 0 & 0.778{\tiny$\pm$0.050} & 31.65{\tiny$\pm$0.19} & 19.76{\tiny$\pm$0.08} & 24.49{\tiny$\pm$0.27} & 37.24{\tiny$\pm$1.57} \\
    \midrule
    \multirow{2}{*}{1.0} & \multirow{2}{*}{None} &  \multirow{2}{*}{N/A}        & \multirow{2}{*}{N/A}          & $\sqrt{2 \eta_t}$         & $-$0.269{\tiny$\pm$0.050} & 30.41{\tiny$\pm$0.22} & 18.74{\tiny$\pm$0.07} & 20.47{\tiny$\pm$0.18} & 43.82{\tiny$\pm$1.24} \\
     &  &       &                & 0                 & $-$0.123{\tiny$\pm$0.041} & 31.83{\tiny$\pm$0.17} & 19.28{\tiny$\pm$0.07} & 20.95{\tiny$\pm$0.16} & 42.59{\tiny$\pm$1.23}                             
                                    \\
    \midrule
    \multirow{6}{*}{1.0} & \multirow{6}{*}{DRaFT-1} &  \multirow{2}{*}{1000}        & $\sqrt{2 \eta_t}$          & $\sqrt{2 \eta_t}$         & 1.123{\tiny$\pm$0.051} & 32.06{\tiny$\pm$0.19} & 19.69{\tiny$\pm$0.06} & 24.56{\tiny$\pm$0.17} & 28.25{\tiny$\pm$1.55} \\
     &  &       & 0                 & 0                 & 0.856{\tiny$\pm$0.167} & 32.32{\tiny$\pm$0.25} & 19.38{\tiny$\pm$0.34} & 22.88{\tiny$\pm$1.54} & 29.98{\tiny$\pm$1.86} \\
    % \midrule
    \addlinespace
     &  &  2000        & 0                 & 0                 & 1.177{\tiny$\pm$0.053} & 32.36{\tiny$\pm$0.18} & 19.67{\tiny$\pm$0.08} & 24.48{\tiny$\pm$0.28} & 25.09{\tiny$\pm$1.82} \\
    % \midrule
    \addlinespace
    & & 3000 & 0                 & 0                 & 1.255{\tiny$\pm$0.038} & 32.36{\tiny$\pm$0.19} & 19.70{\tiny$\pm$0.06} & 24.64{\tiny$\pm$0.17} & 23.24{\tiny$\pm$1.19} \\
    \addlinespace
    & &  4000  & 0 & 0 & \textbf{1.296{\tiny$\pm$0.033}} & 32.30{\tiny$\pm$0.19} & 19.68{\tiny$\pm$0.06} & 24.71{\tiny$\pm$0.14} & 21.54{\tiny$\pm$0.96} \\
    \midrule
    \multirow{5}{*}{1.0} & \multirow{5}{*}{Adj.-Match.} & 1000 & 0                 & 0                 & 0.782{\tiny$\pm$0.044} & 33.05{\tiny$\pm$0.22} & 20.20{\tiny$\pm$0.09} & 24.81{\tiny$\pm$0.18} & 32.67{\tiny$\pm$1.26} \\
    \addlinespace
    & & \multirow{2}{*}{2500} & $\sqrt{2 \eta_t}$ & $\sqrt{2 \eta_t}$ & 1.027{\tiny$\pm$0.038} & 32.85{\tiny$\pm$0.21} & 20.08{\tiny$\pm$0.08} & \textbf{25.88{\tiny$\pm$0.20}} & 29.83{\tiny$\pm$1.00} \\
    & &  & 0 & 0 & 0.910{\tiny$\pm$0.040} & 33.20{\tiny$\pm$0.17} & 20.29{\tiny$\pm$0.09} & 25.39{\tiny$\pm$0.24} & 30.34{\tiny$\pm$1.51} \\
    \addlinespace
    & & 12500 & 0 & 0 & 0.985{\tiny$\pm$0.041} & 33.10{\tiny$\pm$0.18} & 20.28{\tiny$\pm$0.08} & \textbf{25.61{\tiny$\pm$0.27}} & 28.86{\tiny$\pm$1.37} \\
    \midrule
    \multirow{2}{*}{4.0} & \multirow{2}{*}{None} &  \multirow{2}{*}{N/A}        & \multirow{2}{*}{N/A}          & $\sqrt{2 \eta_t}$         & 0.277{\tiny$\pm$0.043} & 32.68{\tiny$\pm$0.18} & 19.50{\tiny$\pm$0.07} & 22.29{\tiny$\pm$0.16} & 35.12{\tiny$\pm$0.92} \\
     &  &       &                & 0                 & 0.209{\tiny$\pm$0.046} & 32.83{\tiny$\pm$0.17} & 19.79{\tiny$\pm$0.07} & 22.30{\tiny$\pm$0.17} & 32.05{\tiny$\pm$1.05}    
                                    \\
    \midrule
    \multirow{6}{*}{4.0} & \multirow{6}{*}{DRaFT-1} &  \multirow{2}{*}{1000}        & $\sqrt{2 \eta_t}$          & $\sqrt{2 \eta_t}$         & 1.062{\tiny$\pm$0.045} & 32.29{\tiny$\pm$0.16} & 19.48{\tiny$\pm$0.06} & 23.67{\tiny$\pm$0.13} & 25.03{\tiny$\pm$1.32} \\
     &  &       & 0                 & 0                 & 0.604{\tiny$\pm$0.395} & 31.80{\tiny$\pm$0.86} & 19.09{\tiny$\pm$0.53} & 21.69{\tiny$\pm$2.10} & 25.92{\tiny$\pm$2.57} \\
    % \midrule
    \addlinespace
     &  &  2000        & 0                 & 0                 & 1.112{\tiny$\pm$0.046} & 32.29{\tiny$\pm$0.20} & 19.34{\tiny$\pm$0.11} & 23.31{\tiny$\pm$0.22} & 21.02{\tiny$\pm$1.67} \\
    % \midrule
    \addlinespace
    & & 3000 & 0                 & 0                 & 1.151{\tiny$\pm$0.036} & 32.31{\tiny$\pm$0.21} & 19.36{\tiny$\pm$0.06} & 23.29{\tiny$\pm$0.14} & 19.53{\tiny$\pm$1.24} \\
    \addlinespace
    & &  4000  & 0 & 0 & 1.172{\tiny$\pm$0.040} & 32.20{\tiny$\pm$0.22} & 19.30{\tiny$\pm$0.07} & 23.20{\tiny$\pm$0.15} & 18.45{\tiny$\pm$1.06} \\
    \midrule
    \multirow{5}{*}{4.0} & \multirow{5}{*}{Adj.-Match.} & 1000 & 0                 & 0                 & 0.852{\tiny$\pm$0.046} & \textbf{33.50{\tiny$\pm$0.22}} & 20.31{\tiny$\pm$0.08} & 24.97{\tiny$\pm$0.19} & 25.83{\tiny$\pm$0.82} \\
    \addlinespace
    & & \multirow{2}{*}{2500} & $\sqrt{2 \eta_t}$ & $\sqrt{2 \eta_t}$ & 1.052{\tiny$\pm$0.039} & \textbf{33.51}{\tiny$\pm$0.19} & 20.15{\tiny$\pm$0.07} & \textbf{25.56{\tiny$\pm$0.18}} & 26.21{\tiny$\pm$0.73} \\
    & &  & 0 & 0 & 0.942{\tiny$\pm$0.042} & \textbf{33.61}{\tiny$\pm$0.19} & \textbf{20.35}{\tiny$\pm$0.08} & 25.34{\tiny$\pm$0.21} & 24.30{\tiny$\pm$0.86} \\
    \addlinespace
    & & 12500 & 0 & 0 & 1.007{\tiny$\pm$0.052} & \textbf{33.48{\tiny$\pm$0.20}} & 20.29{\tiny$\pm$0.08} & \textbf{25.50{\tiny$\pm$0.29}} & 23.48{\tiny$\pm$0.81} \\
    \bottomrule
\end{tabular}
}
\caption{ 
Evaluation metrics when using classifier-free guidance (CFG; \citet{ho2022classifier}).
}
\end{table}

\begin{table}[h!]
\centering
{\footnotesize
\begin{tabular}{llccccccc}
    \toprule
    LR / & Fine-tuning & Fine-tun. & Generat. & \multirow{2}{*}{ImageReward$\, \uparrow$} & \multirow{2}{*}{ClipScore$\, \uparrow$} & \multirow{2}{*}{PickScore$\, \uparrow$} & \multirow{2}{*}{HPS v2$\, \uparrow$} & DreamSim \\
    Adam $\beta_1$ & loss & $\sigma(t)$ & $\sigma(t)$ & & &  &  & diversity$\, \uparrow$ \\
    \midrule
    \num{3e-5} & DRaFT-1          & $\sqrt{2 \eta_t}$ & $\sqrt{2 \eta_t}$ & 1.467{\tiny$\pm$0.029} & 30.28{\tiny$\pm$0.56} & 19.37{\tiny$\pm$0.09} & 24.70{\tiny$\pm$0.15} & 21.20{\tiny$\pm$0.93} \\
               % &  &  & $\num{3e-5} / 0.97$ & $- \pm -$ & $- \pm -$ & $- \pm -$ & $- \pm -$ \\
    \addlinespace
    / \num{0.97} & Adj.-Match.  & \multirow{2}{*}{$\sqrt{2 \eta_t}$} & \multirow{2}{*}{$\sqrt{2 \eta_t}$} & \multirow{2}{*}{1.130{\tiny$\pm$0.034}} & \multirow{2}{*}{31.01{\tiny$\pm$0.27}} & \multirow{2}{*}{19.60{\tiny$\pm$0.08}} & \multirow{2}{*}{25.01{\tiny$\pm$0.25}} & \multirow{2}{*}{26.73{\tiny$\pm$0.88}} \\
    & $\lambda = 12500$                     &                                    &                 &  &  &  &  &  \\
    \midrule
    \num{2e-5} & Disc. Adj.  & $\sqrt{2 \eta_t}$ & $\sqrt{2 \eta_t}$ & $-$1.186{\tiny$\pm$0.553} & 21.95{\tiny$\pm$4.29} & 16.94{\tiny$\pm$0.95} & 12.34{\tiny$\pm$4.40} & 28.33{\tiny$\pm$10.26} \\
     / \num{0.95} & $\lambda = 12500$                     &  0                                  &  0               & $-$0.961{\tiny$\pm$0.653} & 24.07{\tiny$\pm$4.71} & 17.86{\tiny$\pm$1.17} & 15.93{\tiny$\pm$5.80} & 33.62{\tiny$\pm$7.80} \\
    \bottomrule
\end{tabular}
}
\caption{Metrics for alternative optimization hyperparameters (learning rate and Adam $\beta_1$). 
% learning rate $=\num{3e-5}$ and Adam $\beta_1=$ 0.97, in contrast to the values $\num{2e-5}$ and 0.95 used elsewhere.
}
\label{table:alternative_hyperparameters}
\end{table}

\begin{table}[h!]
\centering
{\footnotesize
\begin{tabular}{lccccccc}
    \toprule
    Fine-tuning & Fine-tuning & Generative & \multirow{2}{*}{ImageReward$\, \uparrow$} & \multirow{2}{*}{ClipScore$\, \uparrow$} & \multirow{2}{*}{PickScore$\, \uparrow$} & \multirow{2}{*}{HPS v2$\, \uparrow$} & DreamSim \\
    loss & $\sigma(t)$ & $\sigma(t)$ & & &  &  & diversity$\, \uparrow$ \\
    \midrule
    Adj.-Matching  & \multirow{2}{*}{1} & 1 & 0.009{\tiny$\pm$0.077} & 29.18{\tiny$\pm$0.51} & 18.66{\tiny$\pm$0.09} & 20.75{\tiny$\pm$0.32} & 41.33{\tiny$\pm$1.24} \\
    $\lambda = 12500$                     &                                    &  0               & 0.454{\tiny$\pm$0.055} & 31.41{\tiny$\pm$0.22} & 19.57{\tiny$\pm$0.09} & 23.29{\tiny$\pm$0.18} & 43.10{\tiny$\pm$1.76} \\
    \addlinespace
    Adj.-Matching  & \multirow{2}{*}{$\sqrt{2 \eta_t}$} & $\sqrt{2 \eta_t}$ & 0.882{\tiny$\pm$0.058} & 30.62{\tiny$\pm$0.30} & 19.50{\tiny$\pm$0.09} & 24.95{\tiny$\pm$0.28} & 34.50{\tiny$\pm$1.33}  \\
    $\lambda = 12500$                    &                                    & 0                 & 0.778{\tiny$\pm$0.050} & 31.65{\tiny$\pm$0.19} & 19.76{\tiny$\pm$0.08} & 24.49{\tiny$\pm$0.27} & 37.24{\tiny$\pm$1.57}  \\
    \bottomrule
\end{tabular}
}
\caption{Comparison with an alternative fine-tuning noise schedule $\sigma(t)=1$. We see that the initial value function bias (\Cref{sec:value_function_bias_problem}) results in the model not having a high reward function (ImageReward is the reward function used for fine-tuning). Its performance on other metrics are also lower than when fine-tuning with the memoryless noise schedule, except for diversity.}
\label{table:alternative_noise_schedule}
\end{table}

\begin{table}[h!]
\centering
{\footnotesize
\begin{tabular}{llccccccc}
    \toprule
    $\#$sampl. & Fine-tuning & Fine-tun. & Sampl. & \multirow{2}{*}{ImageReward$\, \uparrow$} & \multirow{2}{*}{ClipScore$\, \uparrow$} & \multirow{2}{*}{PickScore$\, \uparrow$} & \multirow{2}{*}{HPS v2$\, \uparrow$} & DreamSim \\
    timesteps & loss & $\sigma(t)$ & $\sigma(t)$ &  &  &  &  & diversity$\, \uparrow$ \\
    \midrule
    \multirow{6}{*}{$10$} & \multirow{2}{*}{None (Base)} & \multirow{2}{*}{N/A} & $\sqrt{2 \eta_t}$ & $-$2.279{\tiny$\pm$0.001} & 13.99{\tiny$\pm$0.12} & 14.98{\tiny$\pm$0.05} & 7.37{\tiny$\pm$0.10} & 5.07{\tiny$\pm$0.13} \\
                            &     &                     & 0                 & $-$1.386{\tiny$\pm$0.040} & 26.26{\tiny$\pm$0.24} & 17.64{\tiny$\pm$0.07} & 14.92{\tiny$\pm$0.17} & 51.26{\tiny$\pm$1.38} \\
    \addlinespace
     & \multirow{2}{*}{DRaFT-1}           & $\sqrt{2 \eta_t}$ & $\sqrt{2 \eta_t}$ & 1.033{\tiny$\pm$0.051} & 25.98{\tiny$\pm$0.25} & 18.28{\tiny$\pm$0.07} & 22.08{\tiny$\pm$0.18} & 14.47{\tiny$\pm$0.67} \\
                                    &   & 0                 & 0                 & 1.236{\tiny$\pm$0.038} & 31.54{\tiny$\pm$0.27} & 19.53{\tiny$\pm$0.07} & 24.47{\tiny$\pm$0.19} & 24.78{\tiny$\pm$0.88} \\
    \addlinespace
     & Adj.-Match.  & \multirow{2}{*}{$\sqrt{2 \eta_t}$} & $\sqrt{2 \eta_t}$ & $-$2.104{\tiny$\pm$0.074} & 17.12{\tiny$\pm$0.56} & 15.76{\tiny$\pm$0.20} & 11.48{\tiny$\pm$1.03} & 9.88{\tiny$\pm$0.81} \\
     & $\lambda = 12500$                     &                                    & 0                 & 0.607{\tiny$\pm$0.055} & 31.36{\tiny$\pm$0.20} & 19.56{\tiny$\pm$0.08} & 23.23{\tiny$\pm$0.28} & 33.75{\tiny$\pm$1.48} \\
    \midrule
    \multirow{6}{*}{$20$} & \multirow{2}{*}{None (Base)} & \multirow{2}{*}{N/A} & $\sqrt{2 \eta_t}$ & $-$2.275{\tiny$\pm$0.002} & 14.58{\tiny$\pm$0.13} & 15.07{\tiny$\pm$0.05} & 7.47{\tiny$\pm$0.10} & 11.27{\tiny$\pm$0.33} \\
                         &        &                     & 0                 &  $-$1.017{\tiny$\pm$0.055} & 27.92{\tiny$\pm$0.19} & 18.01{\tiny$\pm$0.07} & 17.17{\tiny$\pm$0.15} & 54.69{\tiny$\pm$1.45} \\
    \addlinespace
     & \multirow{2}{*}{DRaFT-1}           & $\sqrt{2 \eta_t}$ & $\sqrt{2 \eta_t}$ & \textbf{1.301{\tiny$\pm$0.039}} & 27.09{\tiny$\pm$0.24} & 18.93{\tiny$\pm$0.07} & 23.78{\tiny$\pm$0.20} & 21.05{\tiny$\pm$1.12} \\
                                    &   & 0                 & 0                 & 1.255{\tiny$\pm$0.038} & 31.14{\tiny$\pm$0.25} & 19.43{\tiny$\pm$0.06} & 24.52{\tiny$\pm$0.16} & 26.15{\tiny$\pm$1.11} \\
    \addlinespace
     & Adj.-Match.  & \multirow{2}{*}{$\sqrt{2 \eta_t}$} & $\sqrt{2 \eta_t}$ & $-$0.032{\tiny$\pm$0.072} & 25.07{\tiny$\pm$0.27} & 18.01{\tiny$\pm$0.07} & 20.75{\tiny$\pm$0.23} & 29.06{\tiny$\pm$2.34} \\
     & $\lambda = 12500$                     &                                    & 0                 & 0.768{\tiny$\pm$0.048} & \textbf{31.70{\tiny$\pm$0.17}} & \textbf{19.73{\tiny$\pm$0.08}} & 24.30{\tiny$\pm$0.26} & 35.90{\tiny$\pm$1.52} \\
     \midrule
    \multirow{6}{*}{$40$} & \multirow{2}{*}{None (Base)} & \multirow{2}{*}{N/A} & $\sqrt{2 \eta_t}$ & $-$1.384{\tiny$\pm$0.040} & 24.15{\tiny$\pm$0.26} & 17.25{\tiny$\pm$0.06} & 16.19{\tiny$\pm$0.17} & 53.60{\tiny$\pm$1.37} \\
                         &        &                     & 0                 & $-$0.920{\tiny$\pm$0.042} & 28.32{\tiny$\pm$0.22} & 18.15{\tiny$\pm$0.07} & 17.89{\tiny$\pm$0.16} & \textbf{56.53{\tiny$\pm$1.52}} \\
    \addlinespace
     & \multirow{2}{*}{DRaFT-1}           & $\sqrt{2 \eta_t}$ & $\sqrt{2 \eta_t}$ & \textbf{1.357{\tiny$\pm$0.039}} & 30.18{\tiny$\pm$0.24} & 19.38{\tiny$\pm$0.08} & 24.61{\tiny$\pm$0.17} & 25.54{\tiny$\pm$0.99} \\
                                    &   & 0                 & 0                 & 1.251{\tiny$\pm$0.040} & 30.95{\tiny$\pm$0.28} & 19.37{\tiny$\pm$0.06} & 24.37{\tiny$\pm$0.17} & 27.39{\tiny$\pm$1.14} \\
    \addlinespace
     & Adj.-Match.  & \multirow{2}{*}{$\sqrt{2 \eta_t}$} & $\sqrt{2 \eta_t}$ & 0.882{\tiny$\pm$0.058} & 30.62{\tiny$\pm$0.30} & 19.50{\tiny$\pm$0.09} & \textbf{24.95{\tiny$\pm$0.28}} & 34.50{\tiny$\pm$1.33} \\
     & $\lambda = 12500$                     &                                    & 0                 & 0.778{\tiny$\pm$0.050} & \textbf{31.65{\tiny$\pm$0.19}} & \textbf{19.76{\tiny$\pm$0.08}} & 24.49{\tiny$\pm$0.27} & 37.24{\tiny$\pm$1.57} \\
    \midrule
    \multirow{6}{*}{$100$} & \multirow{2}{*}{None (Base)} & \multirow{2}{*}{N/A} & $\sqrt{2 \eta_t}$ & $-$0.881{\tiny$\pm$0.041} & 27.83{\tiny$\pm$0.19} & 18.10{\tiny$\pm$0.07} & 18.43{\tiny$\pm$0.17} & \textbf{57.21{\tiny$\pm$1.50}} \\
                         &        &                     & 0                 & $-$0.881{\tiny$\pm$0.036} & 28.65{\tiny$\pm$0.18} & 18.22{\tiny$\pm$0.06} & 18.20{\tiny$\pm$0.17} & \textbf{57.73{\tiny$\pm$1.68}} \\
    \addlinespace
     & \multirow{2}{*}{DRaFT-1}           & $\sqrt{2 \eta_t}$ & $\sqrt{2 \eta_t}$ & \textbf{1.343{\tiny$\pm$0.040}} & 30.64{\tiny$\pm$0.20} & 19.38{\tiny$\pm$0.08} & 24.37{\tiny$\pm$0.17} & 25.51{\tiny$\pm$1.10} \\
                                    &   & 0                 & 0                 & 1.239{\tiny$\pm$0.037} & 30.74{\tiny$\pm$0.28} & 19.33{\tiny$\pm$0.06} & 24.24{\tiny$\pm$0.17} & 28.70{\tiny$\pm$1.11} \\
    \addlinespace
     & Adj.-Match.  & \multirow{2}{*}{$\sqrt{2 \eta_t}$} & $\sqrt{2 \eta_t}$ & 0.892{\tiny$\pm$0.044} & 31.23{\tiny$\pm$0.23} & \textbf{19.65{\tiny$\pm$0.08}} & \textbf{24.92{\tiny$\pm$0.23}} & 35.13{\tiny$\pm$1.40} \\
     & $\lambda = 12500$                     &                                    & 0                 & 0.779{\tiny$\pm$0.048} & \textbf{31.64{\tiny$\pm$0.17}} & \textbf{19.76{\tiny$\pm$0.08}} & 24.57{\tiny$\pm$0.25} & 38.26{\tiny$\pm$1.65} \\
    \midrule
    \multirow{6}{*}{$200$} & \multirow{2}{*}{None (Base)} & \multirow{2}{*}{N/A} & $\sqrt{2 \eta_t}$ & $-$0.848{\tiny$\pm$0.048} & 28.37{\tiny$\pm$0.21} & 18.27{\tiny$\pm$0.08} & 18.56{\tiny$\pm$0.19} & \textbf{58.00{\tiny$\pm$1.58}} \\
                         &        &                     & 0                 & $-$0.871{\tiny$\pm$0.036} & 28.50{\tiny$\pm$0.18} & 18.23{\tiny$\pm$0.06} & 18.25{\tiny$\pm$0.14} & \textbf{57.84{\tiny$\pm$1.60}} \\
    \addlinespace
     & \multirow{2}{*}{DRaFT-1}           & $\sqrt{2 \eta_t}$ & $\sqrt{2 \eta_t}$ & \textbf{1.331{\tiny$\pm$0.044}} & 30.69{\tiny$\pm$0.23} & 19.36{\tiny$\pm$0.07} & 24.21{\tiny$\pm$0.17} & 26.41{\tiny$\pm$1.18} \\
                                    &   & 0                 & 0                 & 1.222{\tiny$\pm$0.042} & 30.77{\tiny$\pm$0.27} & 19.32{\tiny$\pm$0.06} & 24.18{\tiny$\pm$0.16} & 29.09{\tiny$\pm$1.07} \\
    \addlinespace
     & Adj.-Match.  & \multirow{2}{*}{$\sqrt{2 \eta_t}$} & $\sqrt{2 \eta_t}$ & 0.869{\tiny$\pm$0.062} & 31.33{\tiny$\pm$0.21} & \textbf{19.68{\tiny$\pm$0.09}} & \textbf{24.81{\tiny$\pm$0.30}} & 35.90{\tiny$\pm$1.55} \\
     & $\lambda = 12500$                     &                                    & 0                 & 0.766{\tiny$\pm$0.050} & \textbf{31.61{\tiny$\pm$0.16}} & \textbf{19.75{\tiny$\pm$0.08}} & 24.52{\tiny$\pm$0.24} & 38.60{\tiny$\pm$1.38} \\
    \bottomrule
\end{tabular}
}
\caption{Performance metrics for different number of sampling steps. Only the number of sampling steps is ablated; the fine-tuned models used in all cases are the ones fine-tuned using 40 steps.
}
\label{table:metrics_multiprompt_sampling_steps}
\end{table}

\begin{figure}[h!]
    \centering
    \begin{subfigure}[t]{0.49\linewidth}
        \centering
        \rotatebox{90}{\;\;\; $w=0.0$}\,%
        \includegraphics[width=0.24\linewidth]{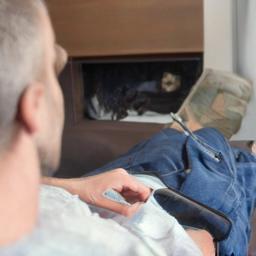}\,%
        \includegraphics[width=0.24\linewidth]{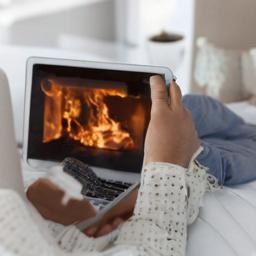}\,%
        \includegraphics[width=0.24\linewidth]{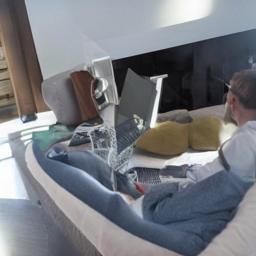}\,%
        \includegraphics[width=0.24\linewidth]{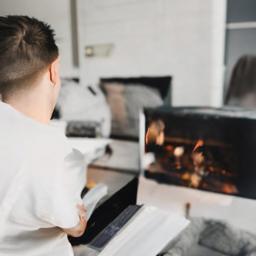}\\
        \rotatebox{90}{\;\;\; $w=1.0$}\,%
        \includegraphics[width=0.24\linewidth]{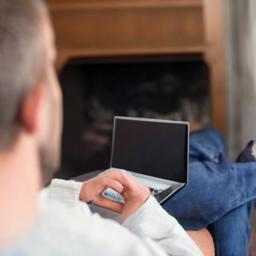}\,%
        \includegraphics[width=0.24\linewidth]{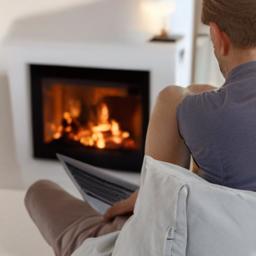}\,%
        \includegraphics[width=0.24\linewidth]{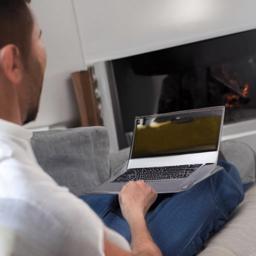}\,%
        \includegraphics[width=0.24\linewidth]{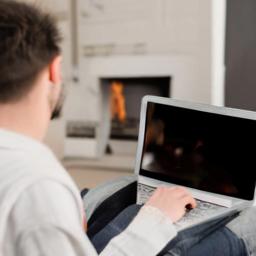}\\
        \rotatebox{90}{\;\;\; $w=4.0$}\,%
        \includegraphics[width=0.24\linewidth]{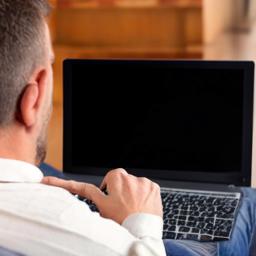}\,%
        \includegraphics[width=0.24\linewidth]{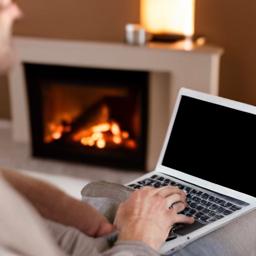}\,%
        \includegraphics[width=0.24\linewidth]{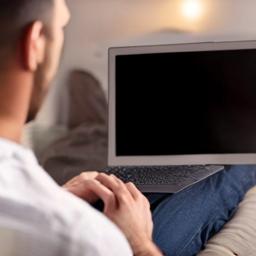}\,%
        \includegraphics[width=0.24\linewidth]{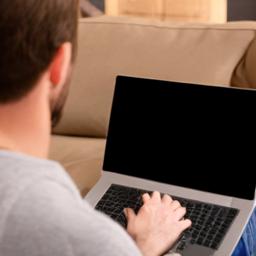}
        \caption*{Text prompt: ``\textit{Man sitting on sofa at home in front of fireplace and using laptop computer, rear view}''}
    \end{subfigure}\hfill
    \begin{subfigure}[t]{0.49\linewidth}
        \centering
        \includegraphics[width=0.24\linewidth]{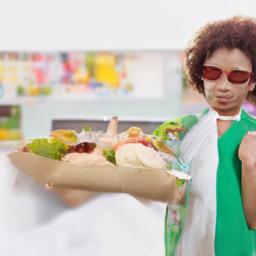}\,%
        \includegraphics[width=0.24\linewidth]{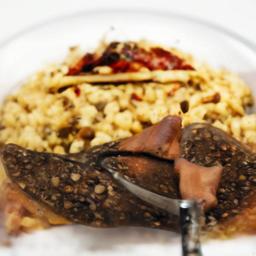}\,%
        \includegraphics[width=0.24\linewidth]{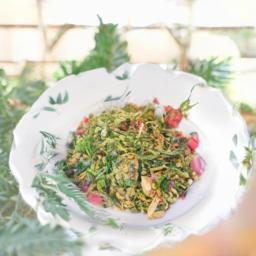}\,%
        \includegraphics[width=0.24\linewidth]{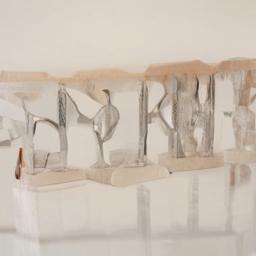}\\
        \includegraphics[width=0.24\linewidth]{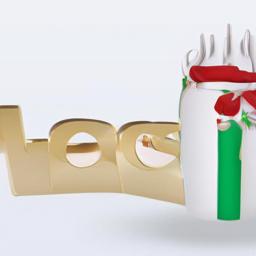}\,%
        \includegraphics[width=0.24\linewidth]{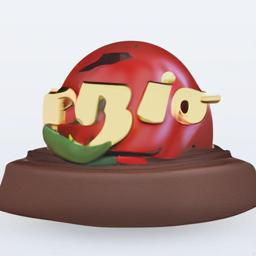}\,%
        \includegraphics[width=0.24\linewidth]{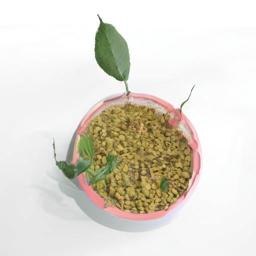}\,%
        \includegraphics[width=0.24\linewidth]{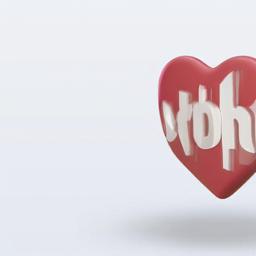}\\
        \includegraphics[width=0.24\linewidth]{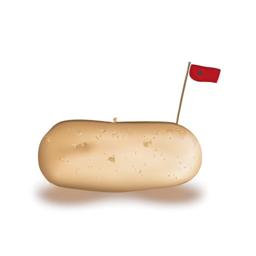}\,%
        \includegraphics[width=0.24\linewidth]{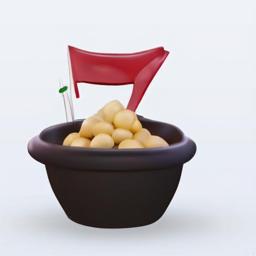}\,%
        \includegraphics[width=0.24\linewidth]{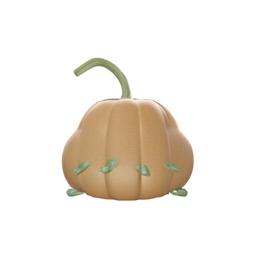}\,%
        \includegraphics[width=0.24\linewidth]{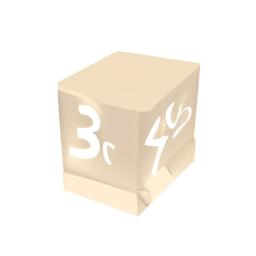}
        \caption*{Text prompt: ``\textit{3D World Food Day Morocco}''}
    \end{subfigure}
    \caption{
    Generated samples from varying classifier-free guidance weights, from the pre-trained Flow Matching model. 
    Corresponding samples from the fine-tuned model can be found in \Cref{fig:ablation_tradeoff_cfg}. 
    }
    \label{fig:ablation_tradeoff_cfg_base}
\end{figure}

\begin{figure}[h!]
    \centering
    \begin{subfigure}[t]{0.32\linewidth}
    \centering
    \caption*{Base Flow Matching model}
	\includegraphics[width=0.320\linewidth]{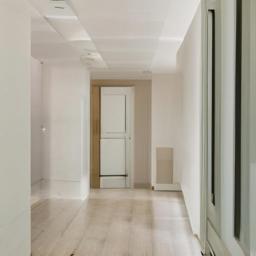}\;%
	\includegraphics[width=0.320\linewidth]{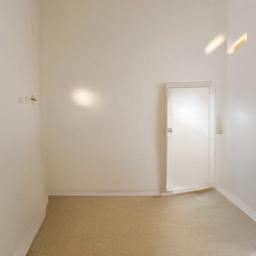}\;%
	\includegraphics[width=0.320\linewidth]{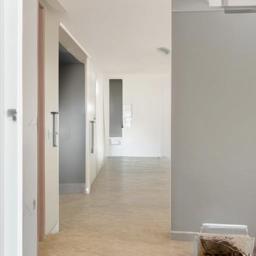}\\ 
	\includegraphics[width=0.320\linewidth]{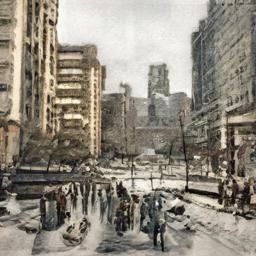}\;%
	\includegraphics[width=0.320\linewidth]{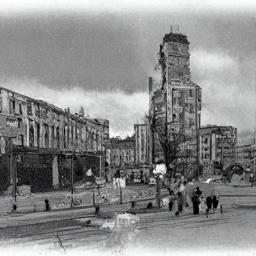}\;%
	\includegraphics[width=0.320\linewidth]{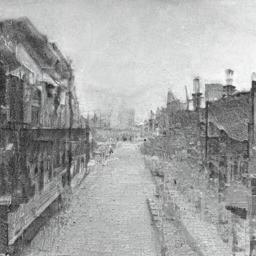}\\ 
	\includegraphics[width=0.320\linewidth]{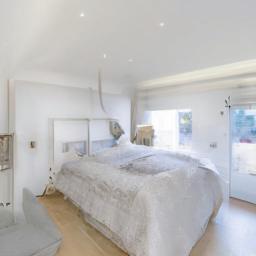}\;%
	\includegraphics[width=0.320\linewidth]{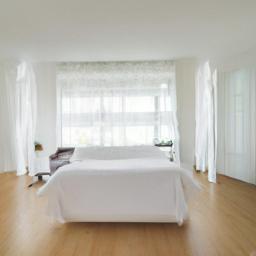}\;%
	\includegraphics[width=0.320\linewidth]{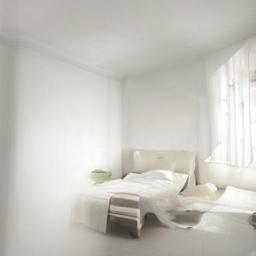}\\ 
	\includegraphics[width=0.320\linewidth]{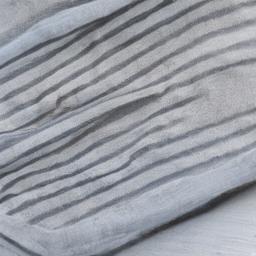}\;%
	\includegraphics[width=0.320\linewidth]{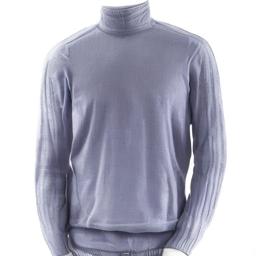}\;%
	\includegraphics[width=0.320\linewidth]{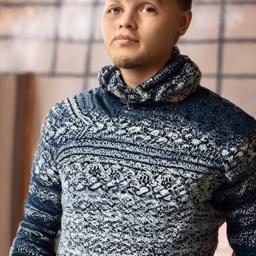}\\ 
	\includegraphics[width=0.320\linewidth]{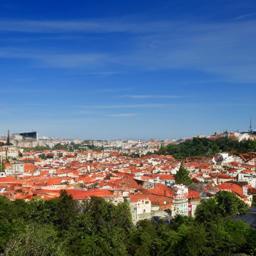}\;%
	\includegraphics[width=0.320\linewidth]{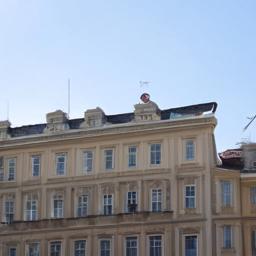}\;%
	\includegraphics[width=0.320\linewidth]{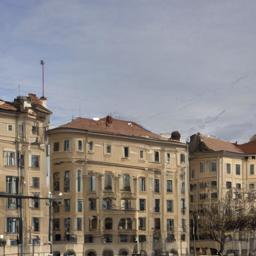}\\ 
	\includegraphics[width=0.320\linewidth]{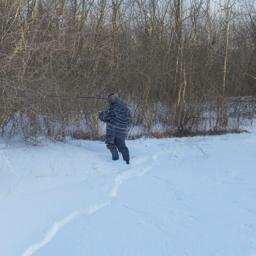}\;%
	\includegraphics[width=0.320\linewidth]{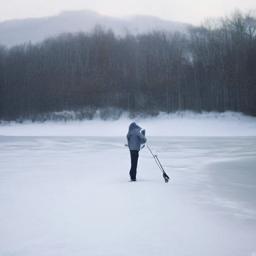}\;%
	\includegraphics[width=0.320\linewidth]{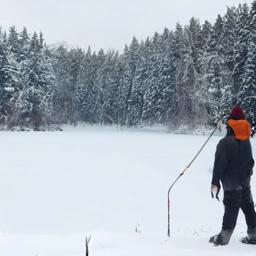}\\ 
	\includegraphics[width=0.320\linewidth]{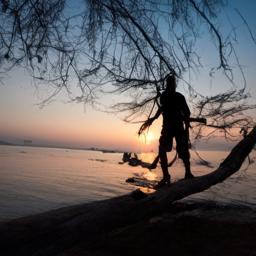}\;%
	\includegraphics[width=0.320\linewidth]{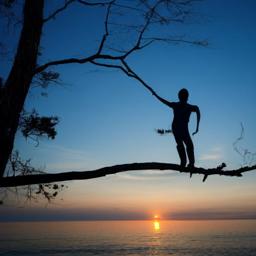}\;%
	\includegraphics[width=0.320\linewidth]{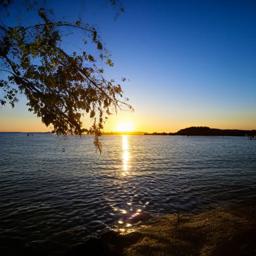}\\ 
	\includegraphics[width=0.320\linewidth]{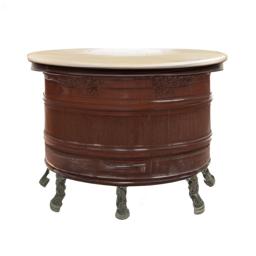}\;%
	\includegraphics[width=0.320\linewidth]{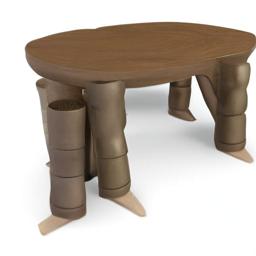}\;%
	\includegraphics[width=0.320\linewidth]{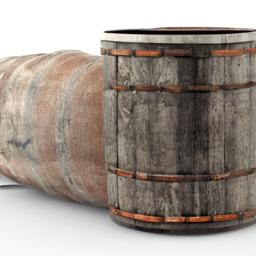}\\ 
	\includegraphics[width=0.320\linewidth]{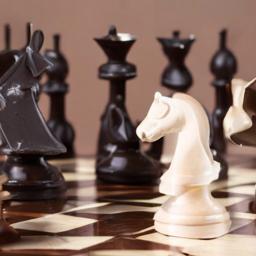}\;%
	\includegraphics[width=0.320\linewidth]{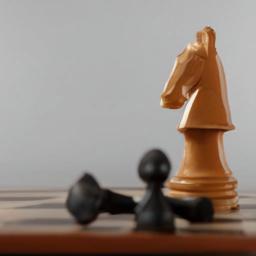}\;%
	\includegraphics[width=0.320\linewidth]{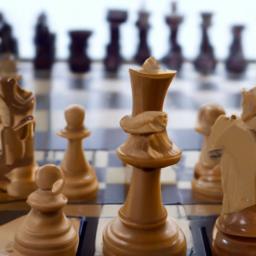}\\ 
	\includegraphics[width=0.320\linewidth]{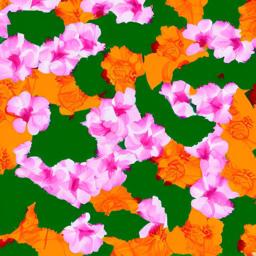}\;%
	\includegraphics[width=0.320\linewidth]{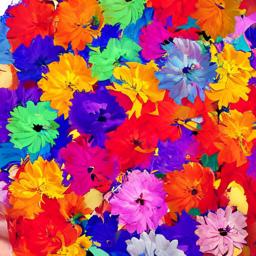}\;%
	\includegraphics[width=0.320\linewidth]{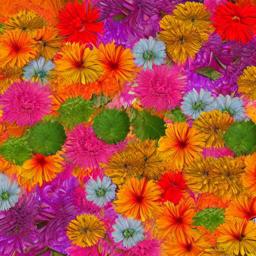}
    \end{subfigure}\hfill
    \begin{subfigure}[t]{0.32\linewidth}
    \centering
    \caption*{Adjoint Matching (Ours)}
	\includegraphics[width=0.320\linewidth]{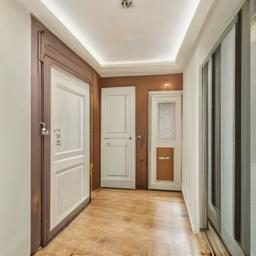}\;%
	\includegraphics[width=0.320\linewidth]{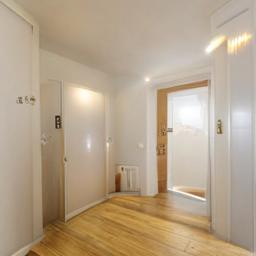}\;%
	\includegraphics[width=0.320\linewidth]{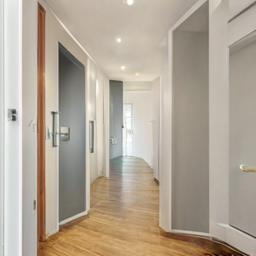}\\ 
	\includegraphics[width=0.320\linewidth]{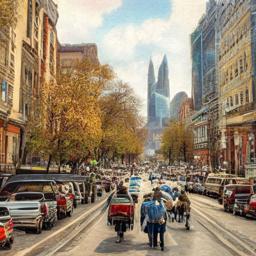}\;%
	\includegraphics[width=0.320\linewidth]{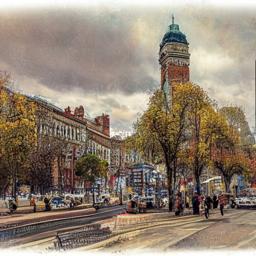}\;%
	\includegraphics[width=0.320\linewidth]{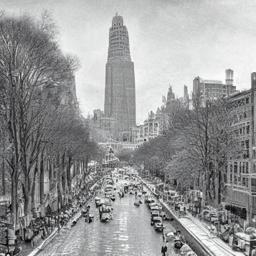}\\ 
	\includegraphics[width=0.320\linewidth]{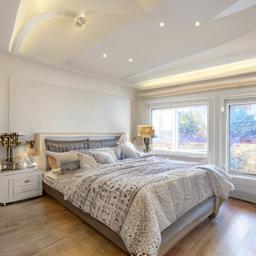}\;%
	\includegraphics[width=0.320\linewidth]{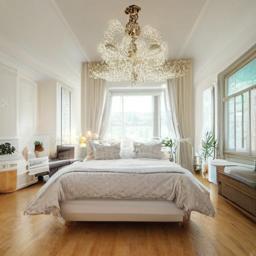}\;%
	\includegraphics[width=0.320\linewidth]{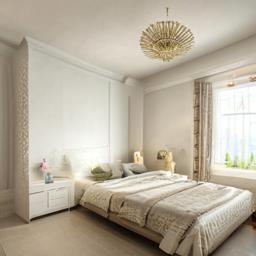}\\ 
	\includegraphics[width=0.320\linewidth]{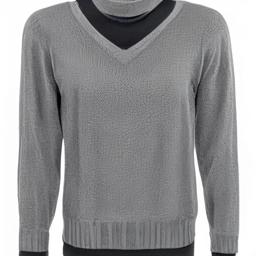}\;%
	\includegraphics[width=0.320\linewidth]{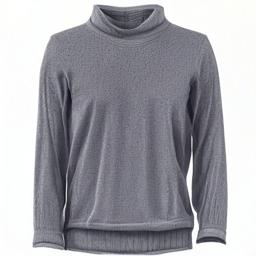}\;%
	\includegraphics[width=0.320\linewidth]{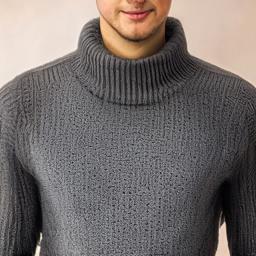}\\ 
	\includegraphics[width=0.320\linewidth]{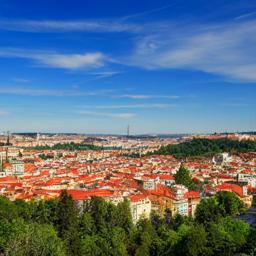}\;%
	\includegraphics[width=0.320\linewidth]{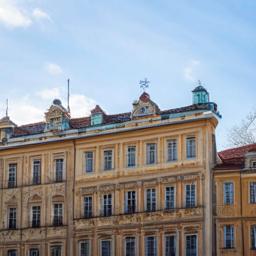}\;%
	\includegraphics[width=0.320\linewidth]{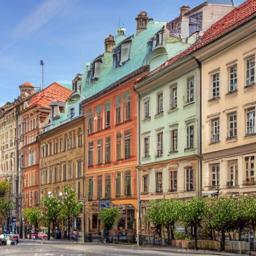}\\ 
	\includegraphics[width=0.320\linewidth]{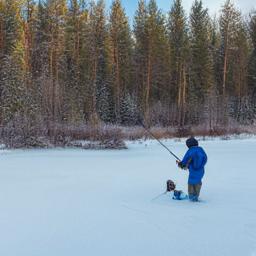}\;%
	\includegraphics[width=0.320\linewidth]{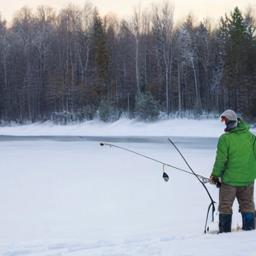}\;%
	\includegraphics[width=0.320\linewidth]{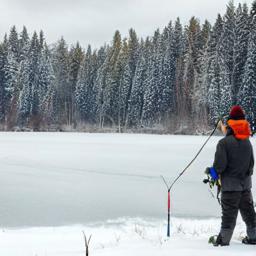}\\ 
	\includegraphics[width=0.320\linewidth]{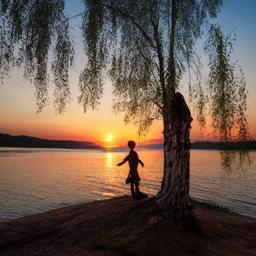}\;%
	\includegraphics[width=0.320\linewidth]{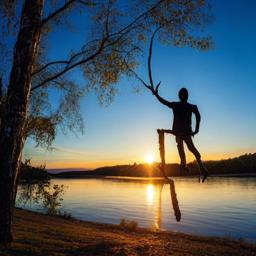}\;%
	\includegraphics[width=0.320\linewidth]{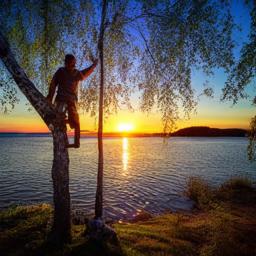}\\ 
	\includegraphics[width=0.320\linewidth]{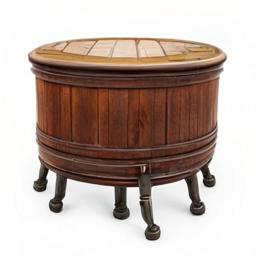}\;%
	\includegraphics[width=0.320\linewidth]{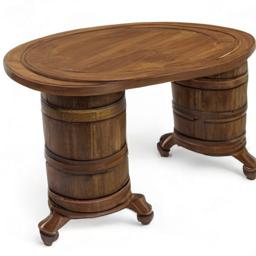}\;%
	\includegraphics[width=0.320\linewidth]{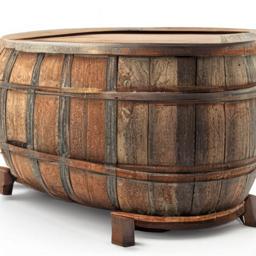}\\ 
	\includegraphics[width=0.320\linewidth]{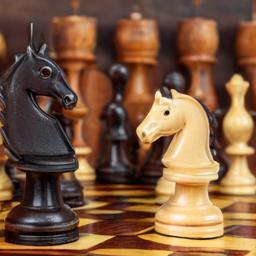}\;%
	\includegraphics[width=0.320\linewidth]{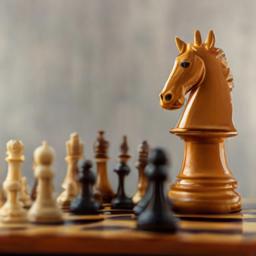}\;%
	\includegraphics[width=0.320\linewidth]{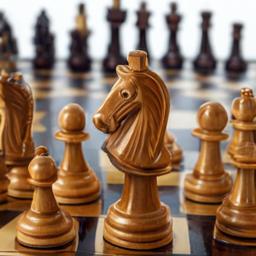}\\ 
	\includegraphics[width=0.320\linewidth]{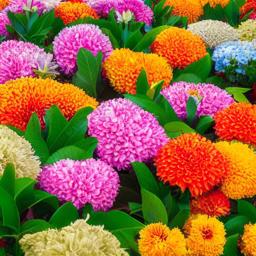}\;%
	\includegraphics[width=0.320\linewidth]{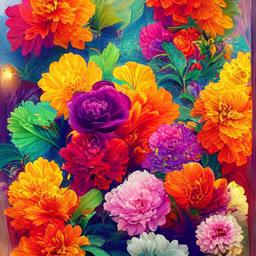}\;%
	\includegraphics[width=0.320\linewidth]{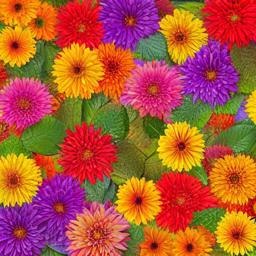}
    \end{subfigure}\hfill
    \begin{subfigure}[t]{0.32\linewidth}
    \centering
    \caption*{DRaFT-1}
	\includegraphics[width=0.320\linewidth]{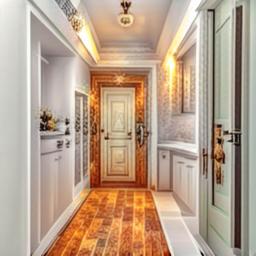}\;%
	\includegraphics[width=0.320\linewidth]{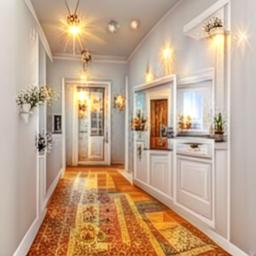}\;%
	\includegraphics[width=0.320\linewidth]{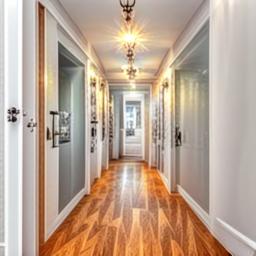}\\ 
	\includegraphics[width=0.320\linewidth]{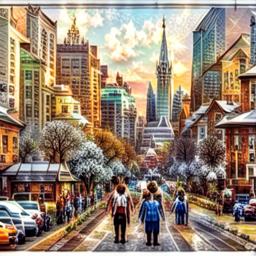}\;%
	\includegraphics[width=0.320\linewidth]{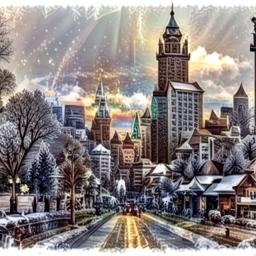}\;%
	\includegraphics[width=0.320\linewidth]{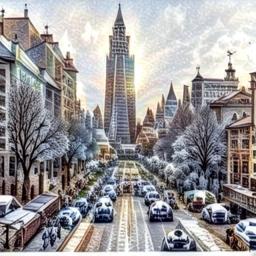}\\ 
	\includegraphics[width=0.320\linewidth]{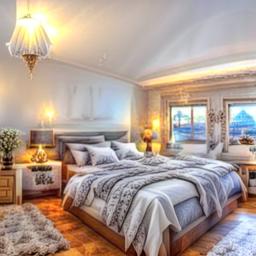}\;%
	\includegraphics[width=0.320\linewidth]{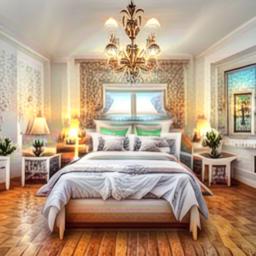}\;%
	\includegraphics[width=0.320\linewidth]{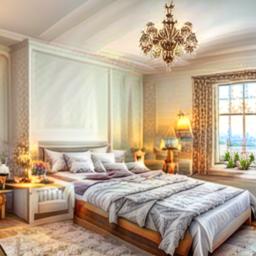}\\ 
	\includegraphics[width=0.320\linewidth]{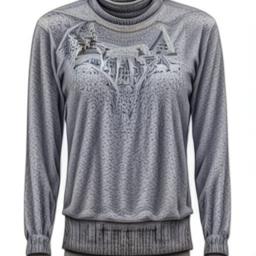}\;%
	\includegraphics[width=0.320\linewidth]{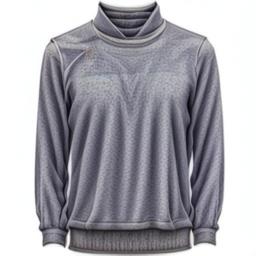}\;%
	\includegraphics[width=0.320\linewidth]{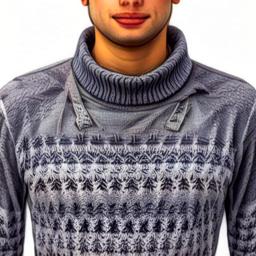}\\ 
	\includegraphics[width=0.320\linewidth]{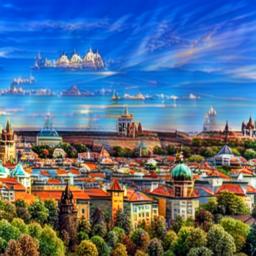}\;%
	\includegraphics[width=0.320\linewidth]{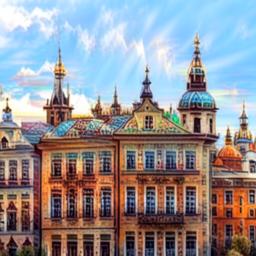}\;%
	\includegraphics[width=0.320\linewidth]{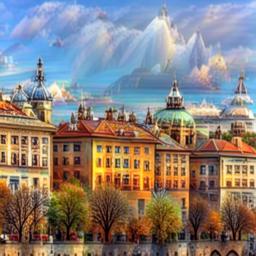}\\ 
	\includegraphics[width=0.320\linewidth]{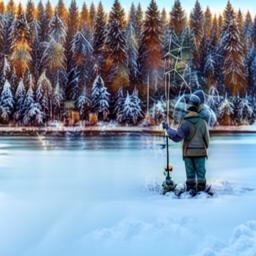}\;%
	\includegraphics[width=0.320\linewidth]{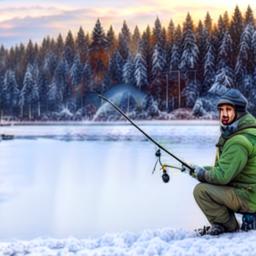}\;%
	\includegraphics[width=0.320\linewidth]{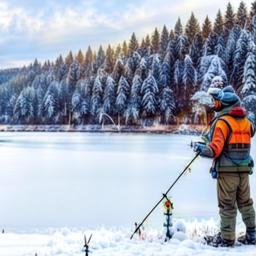}\\ 
	\includegraphics[width=0.320\linewidth]{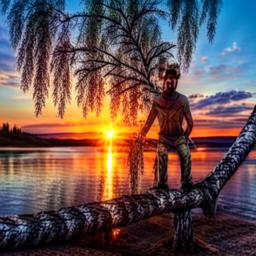}\;%
	\includegraphics[width=0.320\linewidth]{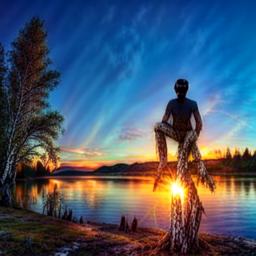}\;%
	\includegraphics[width=0.320\linewidth]{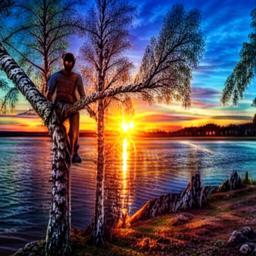}\\ 
	\includegraphics[width=0.320\linewidth]{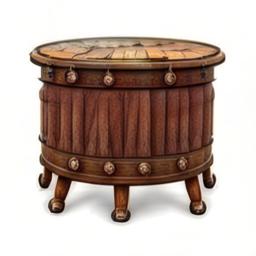}\;%
	\includegraphics[width=0.320\linewidth]{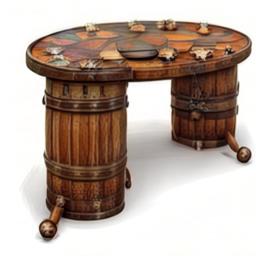}\;%
	\includegraphics[width=0.320\linewidth]{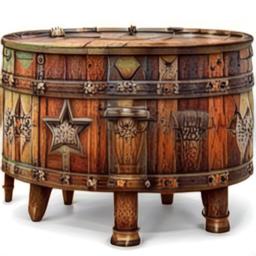}\\ 
	\includegraphics[width=0.320\linewidth]{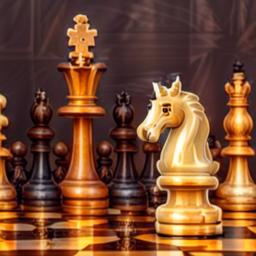}\;%
	\includegraphics[width=0.320\linewidth]{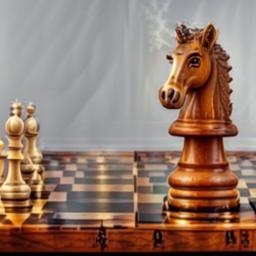}\;%
	\includegraphics[width=0.320\linewidth]{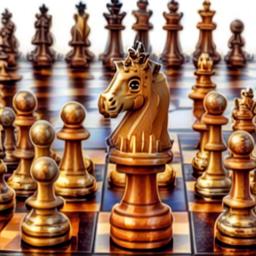}\\ 
	\includegraphics[width=0.320\linewidth]{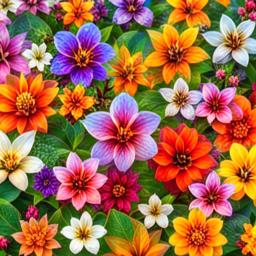}\;%
	\includegraphics[width=0.320\linewidth]{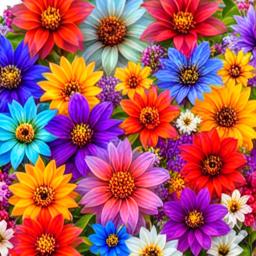}\;%
	\includegraphics[width=0.320\linewidth]{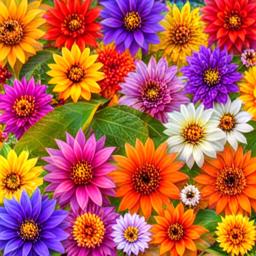}
    \end{subfigure}
    \caption{
    Generated samples with classifier-free guidance ($w=1$) and $\sigma(t)=0$ across ten selected prompts.  Each row corresponds to a different prompt and each image corresponds to a different random seed consistent across models.
    }
    \label{fig:image_comparison_v2}
\end{figure}

\begin{figure}[h!]
\centering
    \begin{subfigure}[t]{0.32\linewidth}
    \centering
    \caption*{Base Flow Matching model}
    	\includegraphics[width=0.32\linewidth]{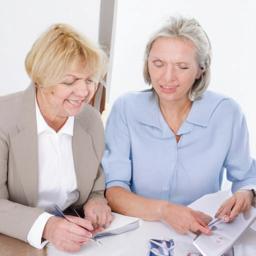}\;%
    	\includegraphics[width=0.32\linewidth]{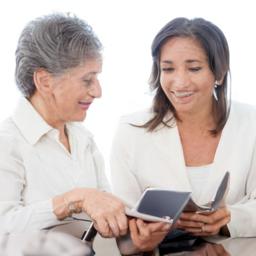}\;%
    	\includegraphics[width=0.32\linewidth]{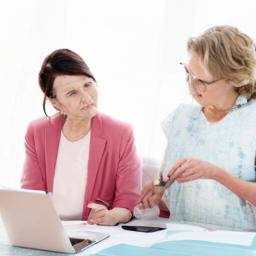}\\ 
    	\includegraphics[width=0.32\linewidth]{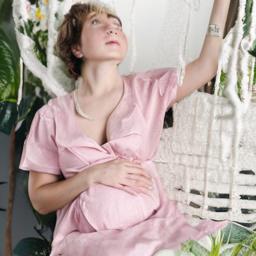}\;%
    	\includegraphics[width=0.32\linewidth]{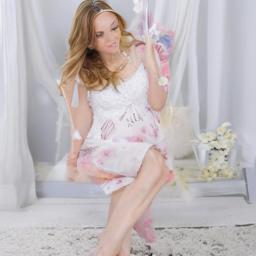}\;%
    	\includegraphics[width=0.32\linewidth]{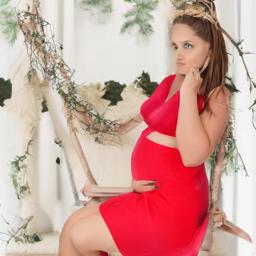}\\ 
    	\includegraphics[width=0.32\linewidth]{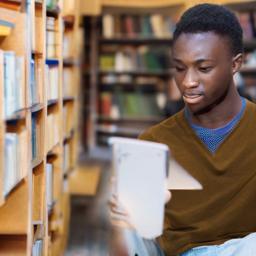}\;%
    	\includegraphics[width=0.32\linewidth]{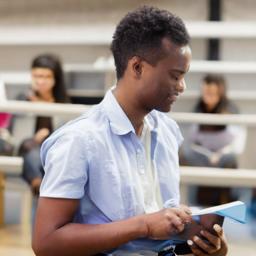}\;%
    	\includegraphics[width=0.32\linewidth]{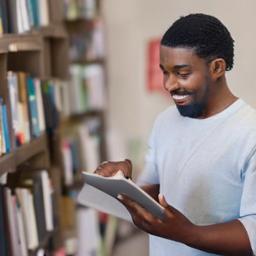}\\ 
    	\includegraphics[width=0.32\linewidth]{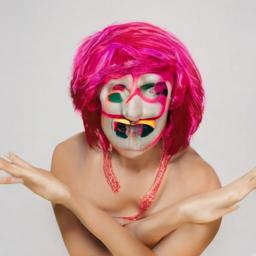}\;%
    	\includegraphics[width=0.32\linewidth]{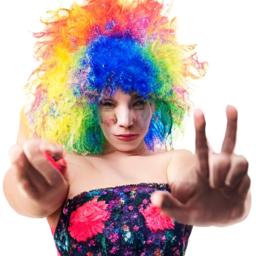}\;%
    	\includegraphics[width=0.32\linewidth]{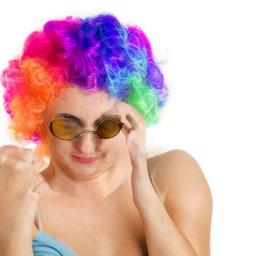}\\ 
    	\includegraphics[width=0.32\linewidth]{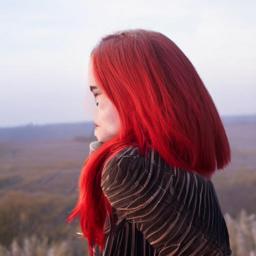}\;%
    	\includegraphics[width=0.32\linewidth]{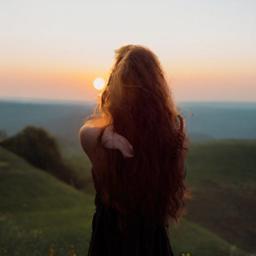}\;%
    	\includegraphics[width=0.32\linewidth]{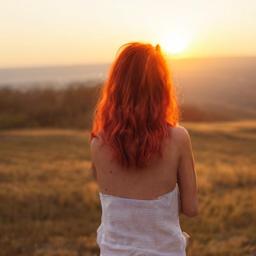}\\ 
    	\includegraphics[width=0.32\linewidth]{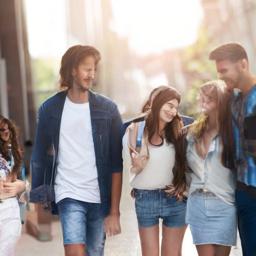}\;%
    	\includegraphics[width=0.32\linewidth]{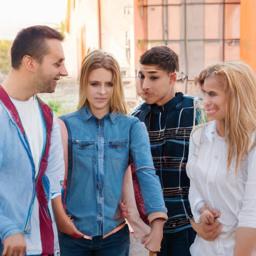}\;%
    	\includegraphics[width=0.32\linewidth]{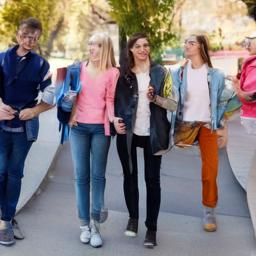}\\ 
    	\includegraphics[width=0.32\linewidth]{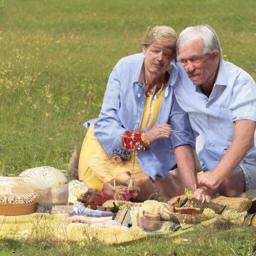}\;%
    	\includegraphics[width=0.32\linewidth]{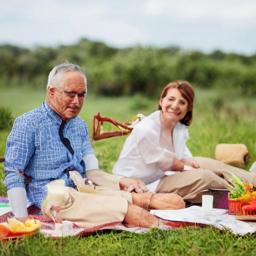}\;%
    	\includegraphics[width=0.32\linewidth]{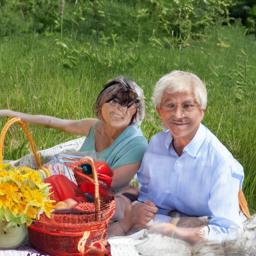}\\ 
    	\includegraphics[width=0.32\linewidth]{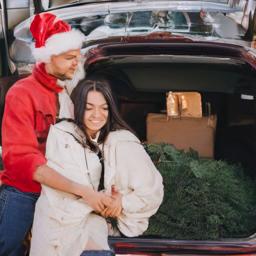}\;%
    	\includegraphics[width=0.32\linewidth]{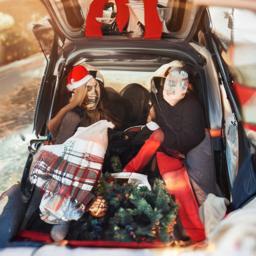}\;%
    	\includegraphics[width=0.32\linewidth]{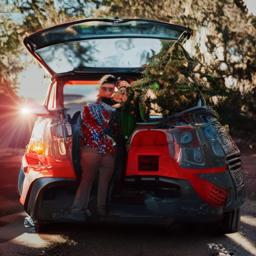}\\ 
    	\includegraphics[width=0.32\linewidth]{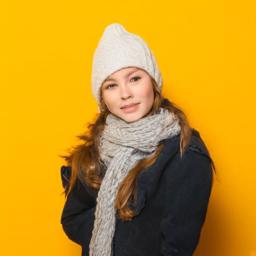}\;%
    	\includegraphics[width=0.32\linewidth]{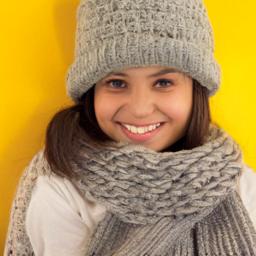}\;%
    	\includegraphics[width=0.32\linewidth]{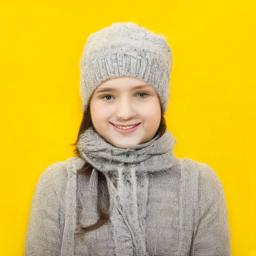}\\ 
    	\includegraphics[width=0.32\linewidth]{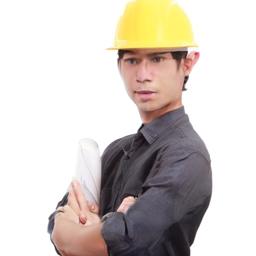}\;%
    	\includegraphics[width=0.32\linewidth]{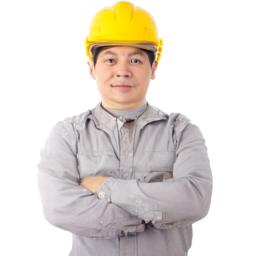}\;%
    	\includegraphics[width=0.32\linewidth]{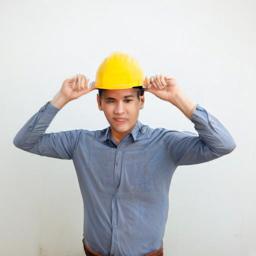}
    \end{subfigure}\hfill
    \begin{subfigure}[t]{0.32\linewidth}
    \centering
    \caption*{Adjoint Matching (Ours)}
    	\includegraphics[width=0.32\linewidth]{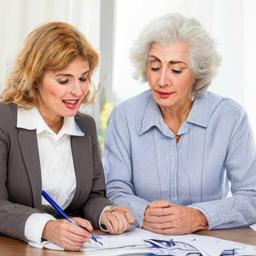}\;%
    	\includegraphics[width=0.32\linewidth]{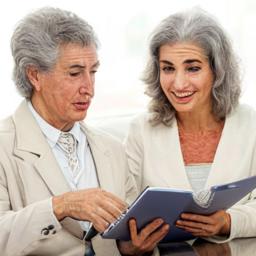}\;%
    	\includegraphics[width=0.32\linewidth]{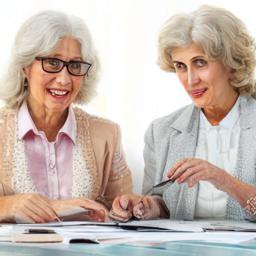}\\ 
    	\includegraphics[width=0.32\linewidth]{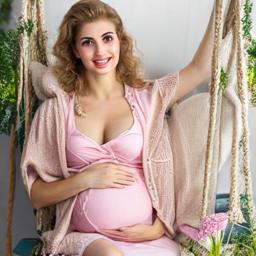}\;%
    	\includegraphics[width=0.32\linewidth]{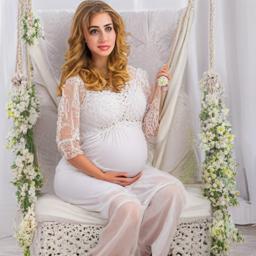}\;%
    	\includegraphics[width=0.32\linewidth]{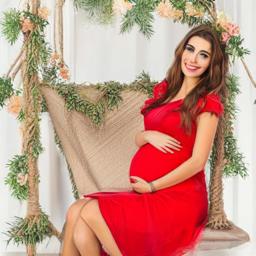}\\ 
    	\includegraphics[width=0.32\linewidth]{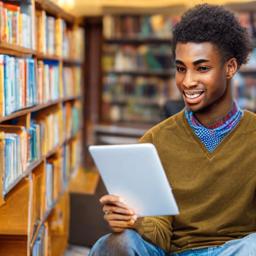}\;%
    	\includegraphics[width=0.32\linewidth]{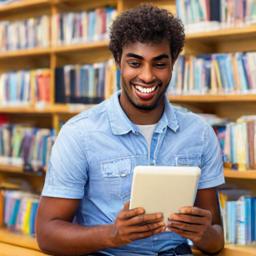}\;%
    	\includegraphics[width=0.32\linewidth]{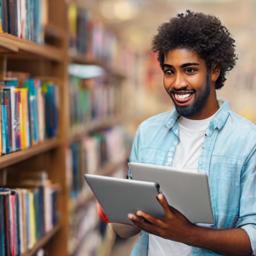}\\ 
    	\includegraphics[width=0.32\linewidth]{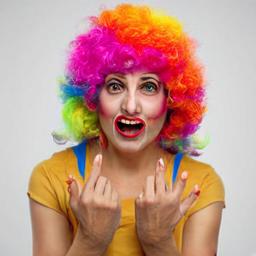}\;%
    	\includegraphics[width=0.32\linewidth]{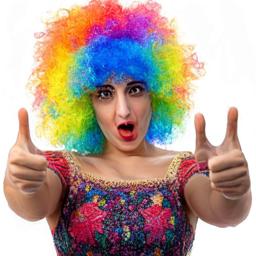}\;%
    	\includegraphics[width=0.32\linewidth]{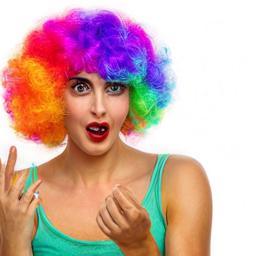}\\ 
    	\includegraphics[width=0.32\linewidth]{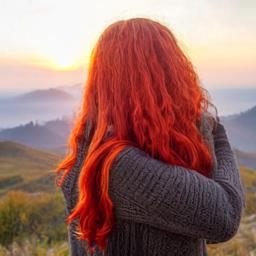}\;%
    	\includegraphics[width=0.32\linewidth]{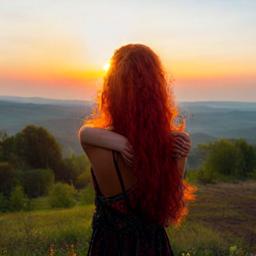}\;%
    	\includegraphics[width=0.32\linewidth]{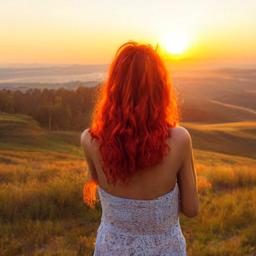}\\ 
    	\includegraphics[width=0.32\linewidth]{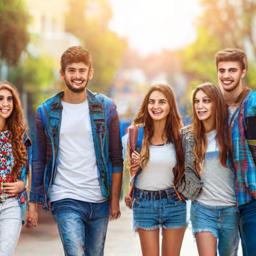}\;%
    	\includegraphics[width=0.32\linewidth]{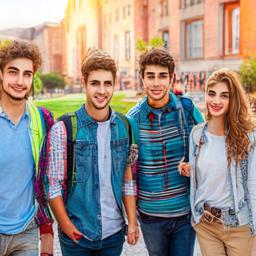}\;%
    	\includegraphics[width=0.32\linewidth]{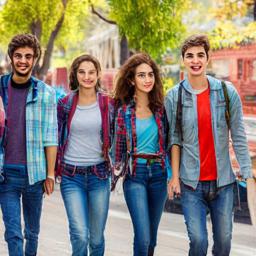}\\ 
    	\includegraphics[width=0.32\linewidth]{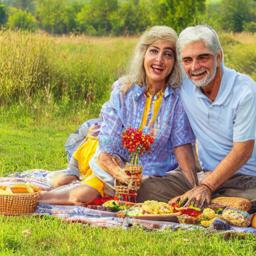}\;%
    	\includegraphics[width=0.32\linewidth]{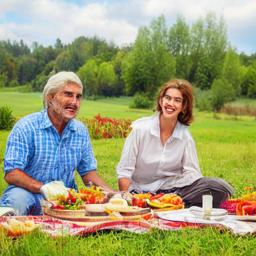}\;%
    	\includegraphics[width=0.32\linewidth]{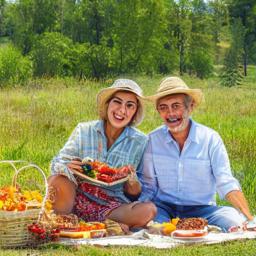}\\ 
    	\includegraphics[width=0.32\linewidth]{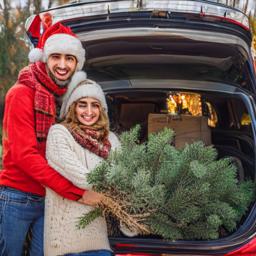}\;%
    	\includegraphics[width=0.32\linewidth]{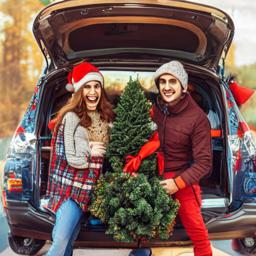}\;%
    	\includegraphics[width=0.32\linewidth]{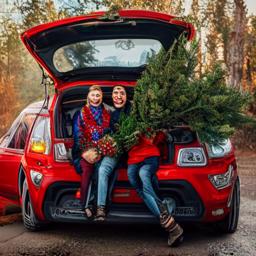}\\ 
    	\includegraphics[width=0.32\linewidth]{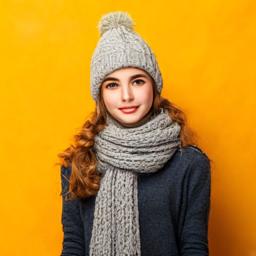}\;%
    	\includegraphics[width=0.32\linewidth]{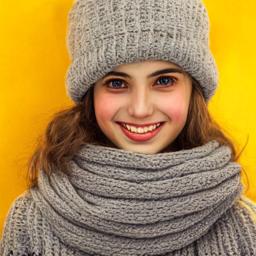}\;%
    	\includegraphics[width=0.32\linewidth]{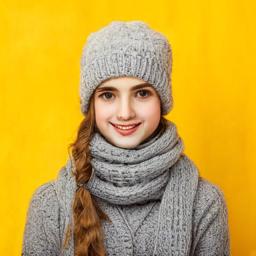}\\ 
    	\includegraphics[width=0.32\linewidth]{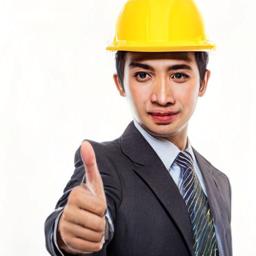}\;%
    	\includegraphics[width=0.32\linewidth]{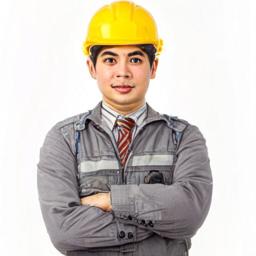}\;%
    	\includegraphics[width=0.32\linewidth]{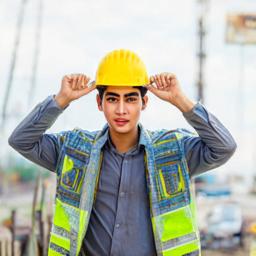}
    \end{subfigure}\hfill
    \begin{subfigure}[t]{0.32\linewidth}
    \centering
    \caption*{DRaFT-1}
    	\includegraphics[width=0.32\linewidth]{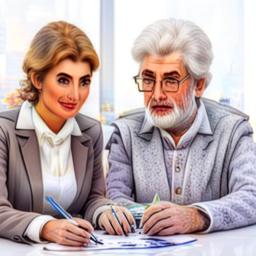}\;%
    	\includegraphics[width=0.32\linewidth]{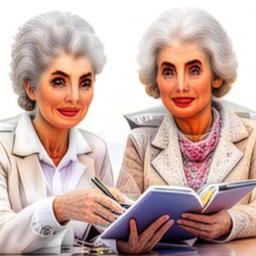}\;%
    	\includegraphics[width=0.32\linewidth]{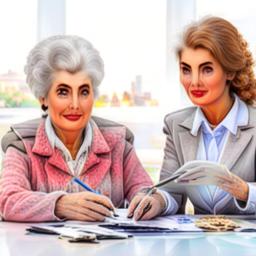}\\ 
    	\includegraphics[width=0.32\linewidth]{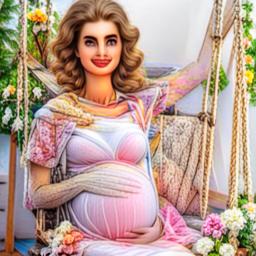}\;%
    	\includegraphics[width=0.32\linewidth]{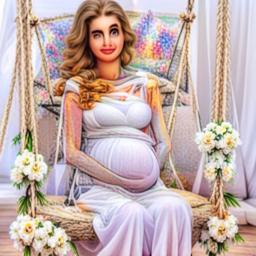}\;%
    	\includegraphics[width=0.32\linewidth]{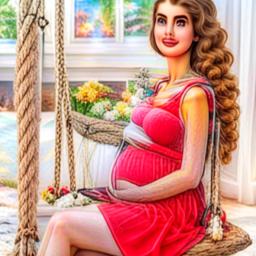}\\ 
    	\includegraphics[width=0.32\linewidth]{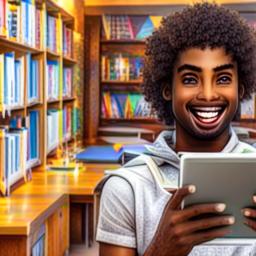}\;%
    	\includegraphics[width=0.32\linewidth]{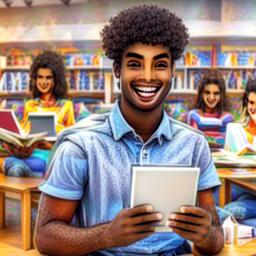}\;%
    	\includegraphics[width=0.32\linewidth]{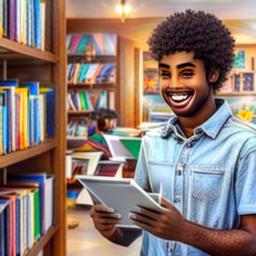}\\ 
    	\includegraphics[width=0.32\linewidth]{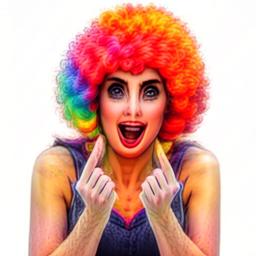}\;%
    	\includegraphics[width=0.32\linewidth]{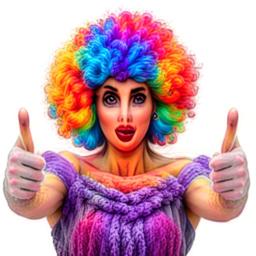}\;%
    	\includegraphics[width=0.32\linewidth]{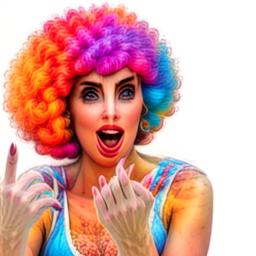}\\ 
    	\includegraphics[width=0.32\linewidth]{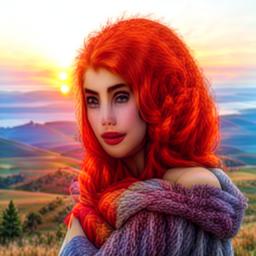}\;%
    	\includegraphics[width=0.32\linewidth]{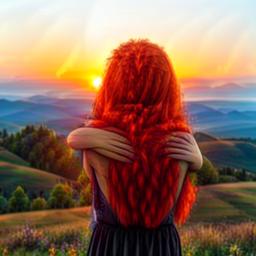}\;%
    	\includegraphics[width=0.32\linewidth]{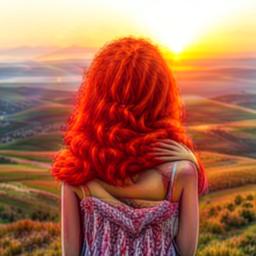}\\ 
    	\includegraphics[width=0.32\linewidth]{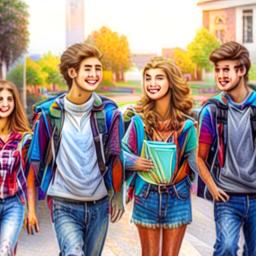}\;%
    	\includegraphics[width=0.32\linewidth]{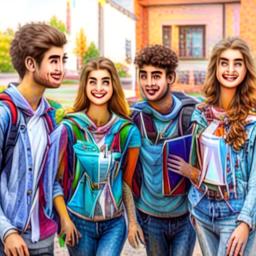}\;%
    	\includegraphics[width=0.32\linewidth]{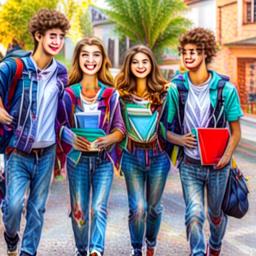}\\ 
    	\includegraphics[width=0.32\linewidth]{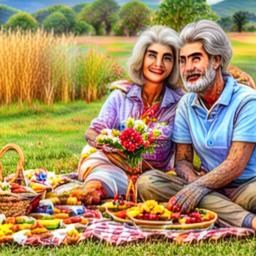}\;%
    	\includegraphics[width=0.32\linewidth]{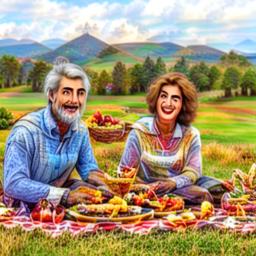}\;%
    	\includegraphics[width=0.32\linewidth]{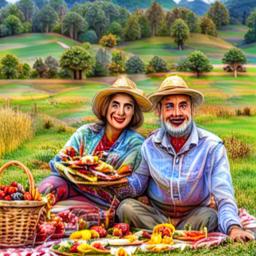}\\ 
    	\includegraphics[width=0.32\linewidth]{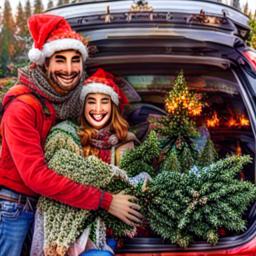}\;%
    	\includegraphics[width=0.32\linewidth]{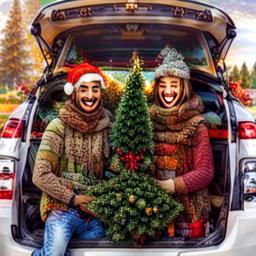}\;%
    	\includegraphics[width=0.32\linewidth]{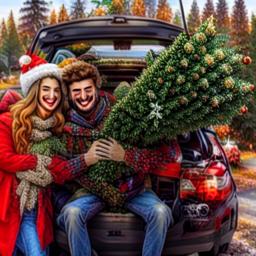}\\ 
    	\includegraphics[width=0.32\linewidth]{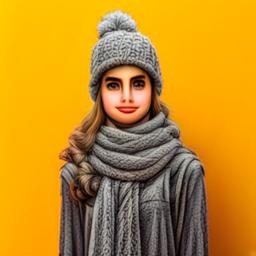}\;%
    	\includegraphics[width=0.32\linewidth]{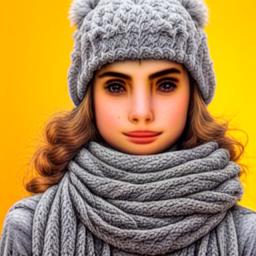}\;%
    	\includegraphics[width=0.32\linewidth]{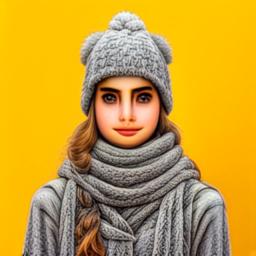}\\ 
    	\includegraphics[width=0.32\linewidth]{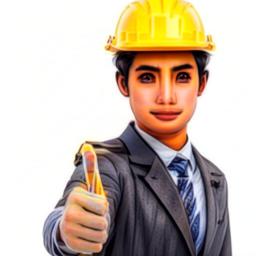}\;%
    	\includegraphics[width=0.32\linewidth]{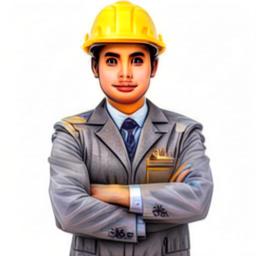}\;%
    	\includegraphics[width=0.32\linewidth]{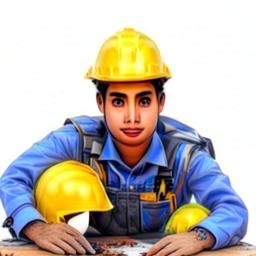}
    \end{subfigure}
    \caption{
    Generated samples with classifier-free guidance ($w=1$) and $\sigma(t) = 0$ across ten selected prompts with people.  Each row corresponds to a different prompt and each image corresponds to a different random seed consistent across models.
    }
    \label{fig:image_comparison_people}
\end{figure}

\begin{figure}[!htb]
    \centering
    \begin{tabular}{>{\centering\arraybackslash}m{0.3cm} m{0.96\linewidth}}
        \rotatebox{90}{\parbox{2cm}{\centering\footnotesize None (Base)}} &
        \includegraphics[width=0.135\linewidth]{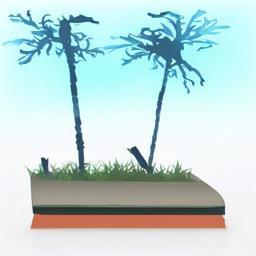}
    	\includegraphics[width=0.135\linewidth]{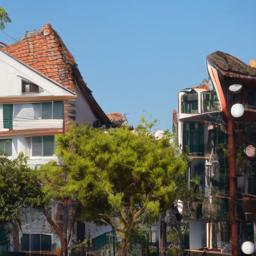}
    	\includegraphics[width=0.135\linewidth]{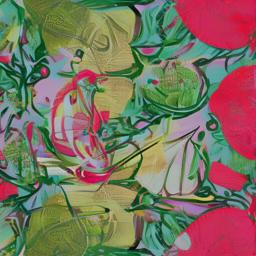}
    	\includegraphics[width=0.135\linewidth]{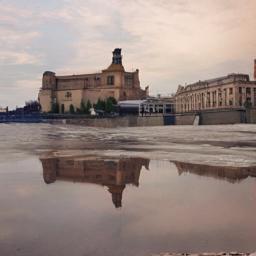}
    	\includegraphics[width=0.135\linewidth]{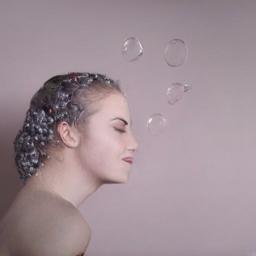}
    	\includegraphics[width=0.135\linewidth]{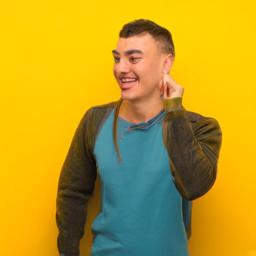}
    	\includegraphics[width=0.135\linewidth]{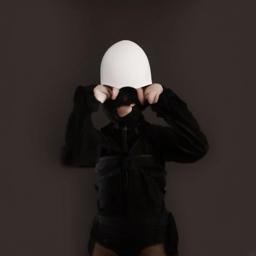} \\

        \rotatebox{90}{\parbox{2cm}{\centering\footnotesize DRaFT-1}} &
        \includegraphics[width=0.135\linewidth]{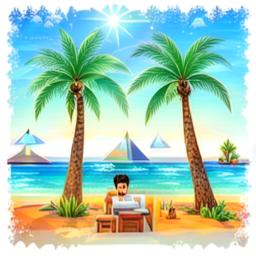}
    	\includegraphics[width=0.135\linewidth]{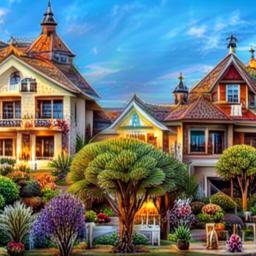}
    	\includegraphics[width=0.135\linewidth]{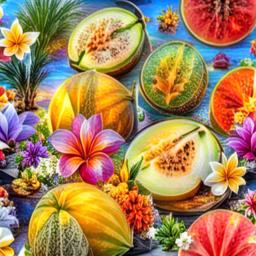}
    	\includegraphics[width=0.135\linewidth]{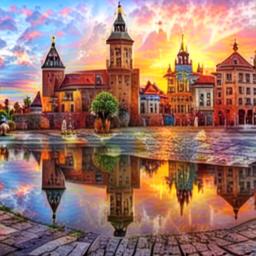}
    	\includegraphics[width=0.135\linewidth]{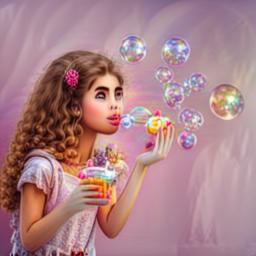}
    	\includegraphics[width=0.135\linewidth]{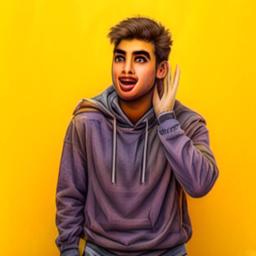}
    	\includegraphics[width=0.135\linewidth]{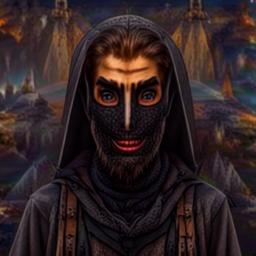} \\

        \rotatebox{90}{\parbox{2cm}{\centering\footnotesize DRaFT-40}} &
        \includegraphics[width=0.135\linewidth]{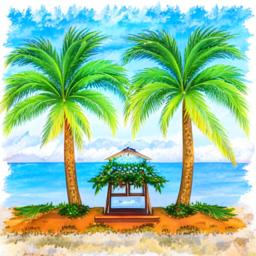}
    	\includegraphics[width=0.135\linewidth]{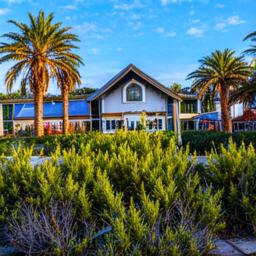}
    	\includegraphics[width=0.135\linewidth]{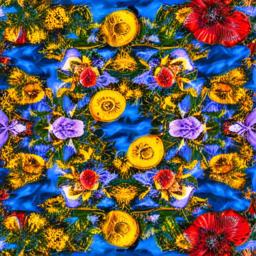}
    	\includegraphics[width=0.135\linewidth]{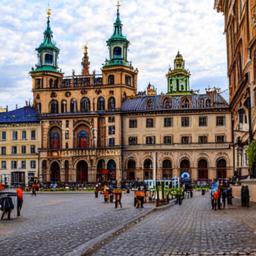}
    	\includegraphics[width=0.135\linewidth]{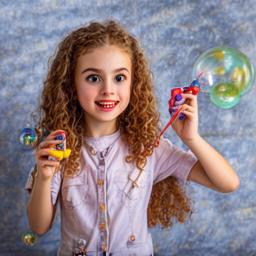}
    	\includegraphics[width=0.135\linewidth]{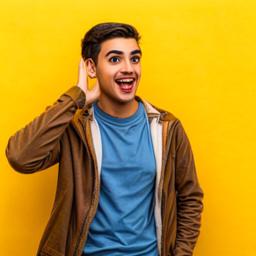}
    	\includegraphics[width=0.135\linewidth]{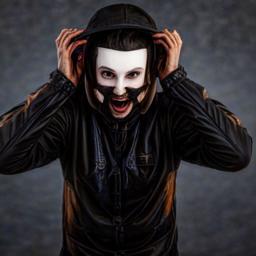} \\

        \rotatebox{90}{\parbox{2cm}{\centering\footnotesize ReFL}} &
        \includegraphics[width=0.135\linewidth]{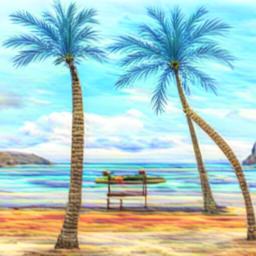}
    	\includegraphics[width=0.135\linewidth]{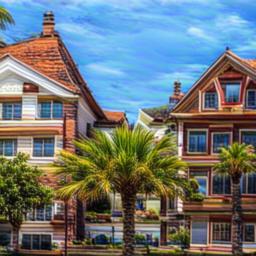}
    	\includegraphics[width=0.135\linewidth]{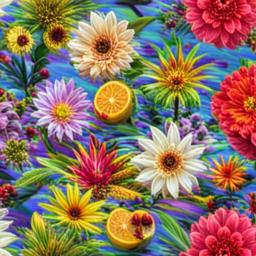}
    	\includegraphics[width=0.135\linewidth]{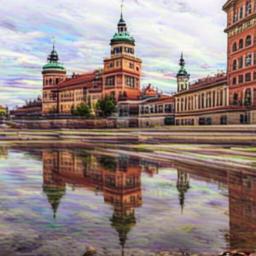}
    	\includegraphics[width=0.135\linewidth]{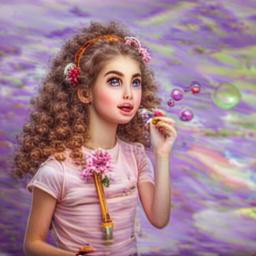}
    	\includegraphics[width=0.135\linewidth]{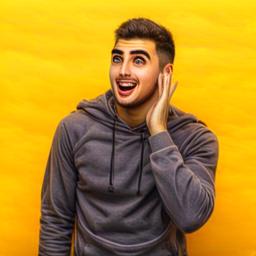}
    	\includegraphics[width=0.135\linewidth]{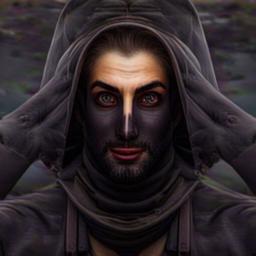} \\

        \rotatebox{90}{\parbox{2cm}{\centering\footnotesize Cont. Adj.\\$\lambda = 12500$}} &
        \includegraphics[width=0.135\linewidth]{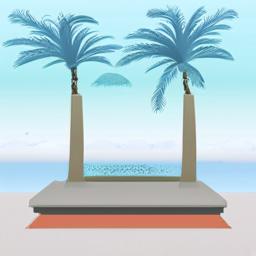}
    	\includegraphics[width=0.135\linewidth]{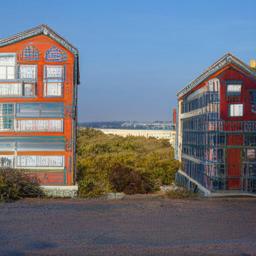}
    	\includegraphics[width=0.135\linewidth]{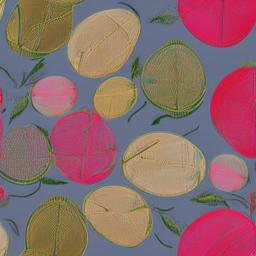}
    	\includegraphics[width=0.135\linewidth]{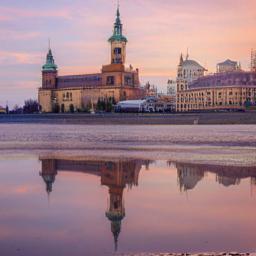}
    	\includegraphics[width=0.135\linewidth]{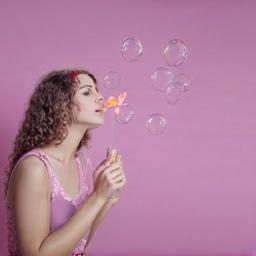}
    	\includegraphics[width=0.135\linewidth]{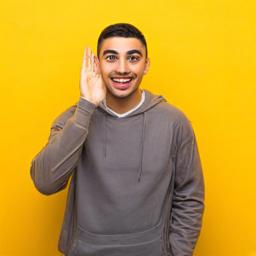}
    	\includegraphics[width=0.135\linewidth]{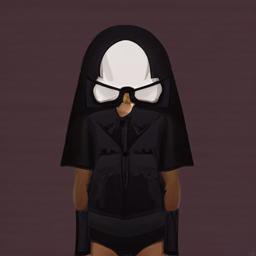} \\

        \rotatebox{90}{\parbox{2cm}{\centering\footnotesize Disc. Adj.\\$\lambda = 12500$}} &
        \includegraphics[width=0.135\linewidth]{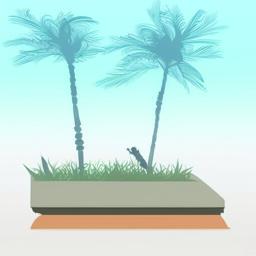}
    	\includegraphics[width=0.135\linewidth]{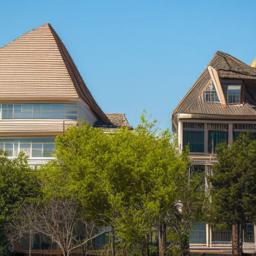}
    	\includegraphics[width=0.135\linewidth]{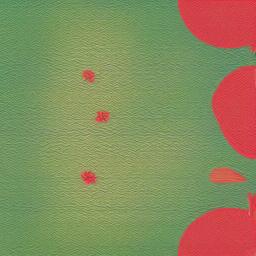}
    	\includegraphics[width=0.135\linewidth]{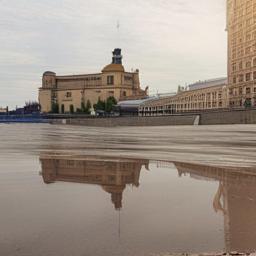}
    	\includegraphics[width=0.135\linewidth]{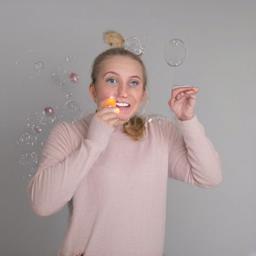}
    	\includegraphics[width=0.135\linewidth]{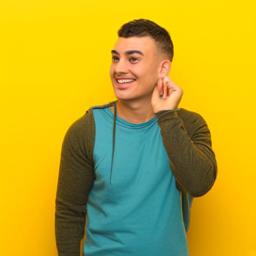}
    	\includegraphics[width=0.135\linewidth]{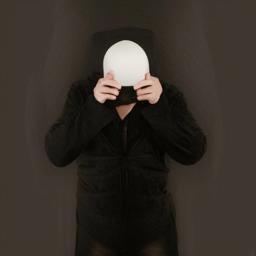} \\
     
        \rotatebox{90}{\parbox{2cm}{\centering\footnotesize Adj. match.\\$\lambda = 1000$}} &
        \includegraphics[width=0.135\linewidth]{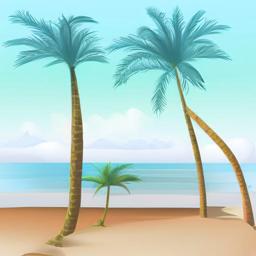}
        \includegraphics[width=0.135\linewidth]{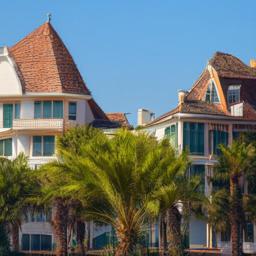}
        \includegraphics[width=0.135\linewidth]{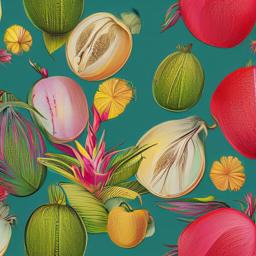}
        \includegraphics[width=0.135\linewidth]{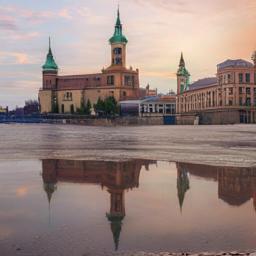}
        \includegraphics[width=0.135\linewidth]{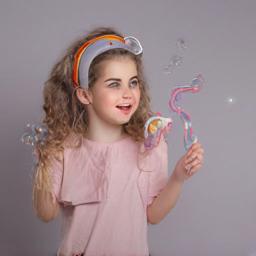}
        \includegraphics[width=0.135\linewidth]{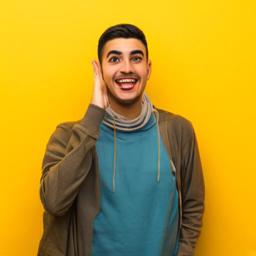}
        \includegraphics[width=0.135\linewidth]{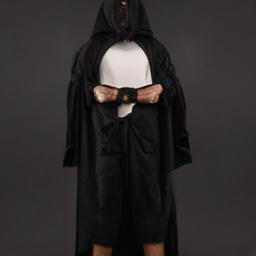} \\

        \rotatebox{90}{\parbox{2cm}{\centering\footnotesize Adj. match.\\$\lambda = 2500$}} &
        \includegraphics[width=0.135\linewidth]{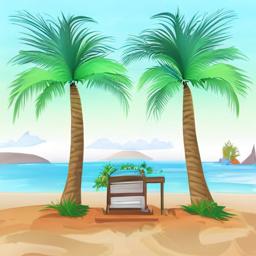}
    	\includegraphics[width=0.135\linewidth]{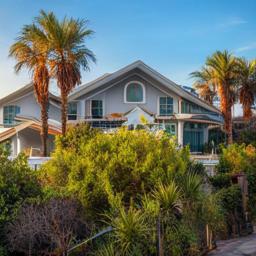}
    	\includegraphics[width=0.135\linewidth]{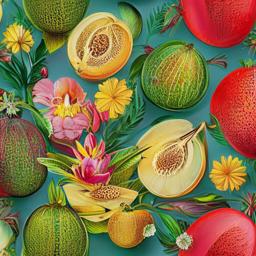}
    	\includegraphics[width=0.135\linewidth]{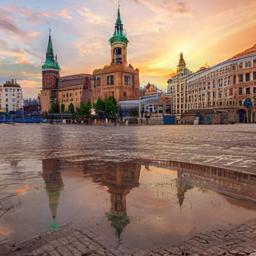}
    	\includegraphics[width=0.135\linewidth]{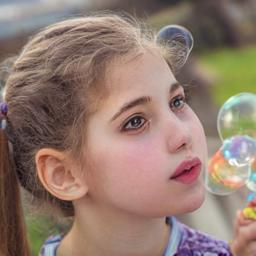}
    	\includegraphics[width=0.135\linewidth]{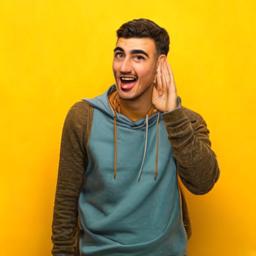}
    	\includegraphics[width=0.135\linewidth]{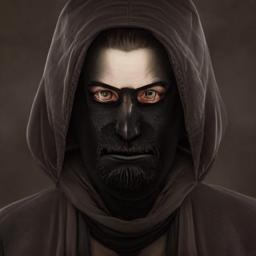} \\

        \rotatebox{90}{\parbox{2cm}{\centering\footnotesize Adj. match.\\$\lambda = 12500$}} &
        \includegraphics[width=0.135\linewidth]{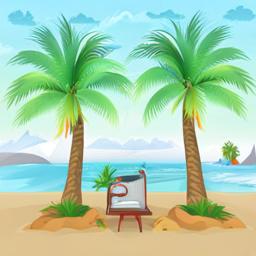}
    	\includegraphics[width=0.135\linewidth]{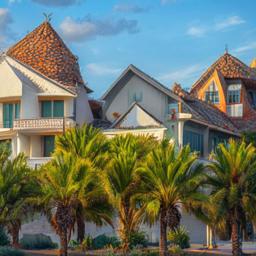}
    	\includegraphics[width=0.135\linewidth]{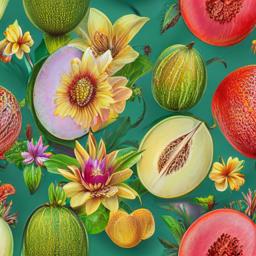}
    	\includegraphics[width=0.135\linewidth]{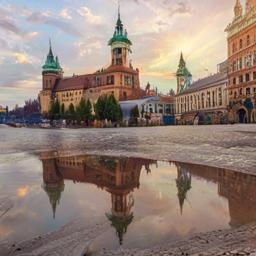}
    	\includegraphics[width=0.135\linewidth]{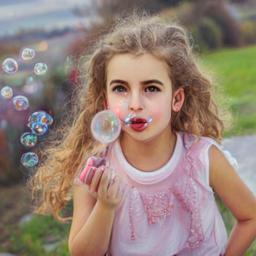}
    	\includegraphics[width=0.135\linewidth]{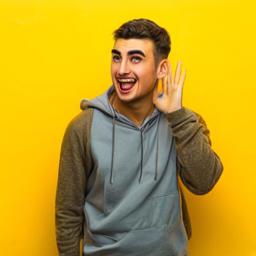}
    	\includegraphics[width=0.135\linewidth]{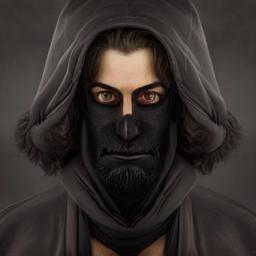} 
     \end{tabular}
     \caption{Generated samples without guidance ($w=0$) and $\sigma(t) = 0$ across seven selected prompts. Each row corresponds to a different finetuning algorithm. Prompts: ``\textit{Seaside view poster with palm trees vector image}'', ``\textit{Cayucos Beach Inn}'', ``\textit{Happy Summer Life- Aloha Flowers and Melon - Pattern Metal Print}'', ``\textit{Castle Square, Warsaw Old Town}'', ``\textit{Funny girl blowing soap bubbles. High quality photo}'', ``\textit{Colombian man with sweatshirt over yellow wall listening to something by putting hand on the ear}'', ``\textit{man in the hood black mask masquerade}''.}
    \label{fig:one_image_per_prompt_ode}
\end{figure}

\begin{figure}[!htb]
    \centering
    \begin{tabular}{>{\centering\arraybackslash}m{0.3cm} m{0.96\linewidth}}
        \rotatebox{90}{\parbox{2cm}{\centering\footnotesize None (Base)}} &
        \includegraphics[width=0.135\linewidth]{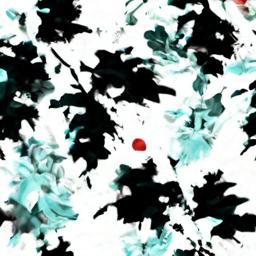}
    	\includegraphics[width=0.135\linewidth]{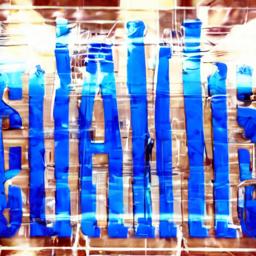}
    	\includegraphics[width=0.135\linewidth]{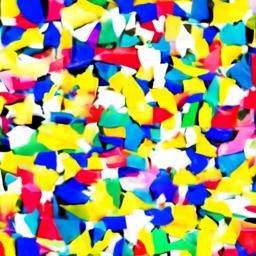}
    	\includegraphics[width=0.135\linewidth]{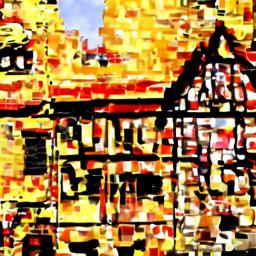}
    	\includegraphics[width=0.135\linewidth]{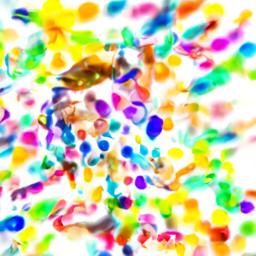}
    	\includegraphics[width=0.135\linewidth]{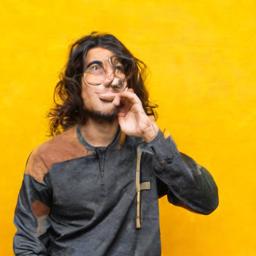}
    	\includegraphics[width=0.135\linewidth]{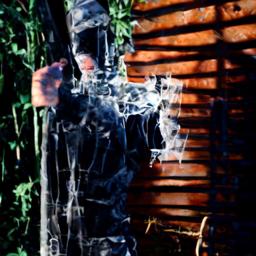} \\

        \rotatebox{90}{\parbox{2cm}{\centering\footnotesize DRaFT-1}} &
        \includegraphics[width=0.135\linewidth]{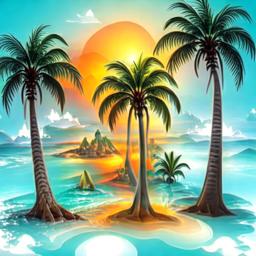}
    	\includegraphics[width=0.135\linewidth]{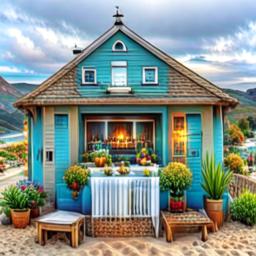}
    	\includegraphics[width=0.135\linewidth]{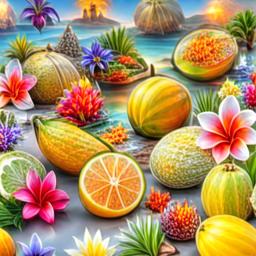}
    	\includegraphics[width=0.135\linewidth]{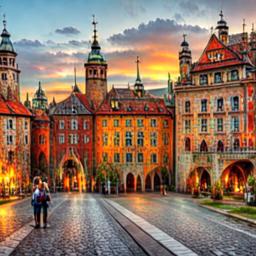}
    	\includegraphics[width=0.135\linewidth]{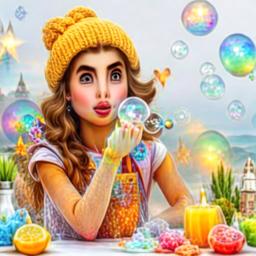}
    	\includegraphics[width=0.135\linewidth]{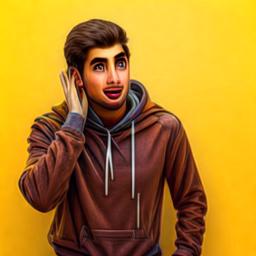}
    	\includegraphics[width=0.135\linewidth]{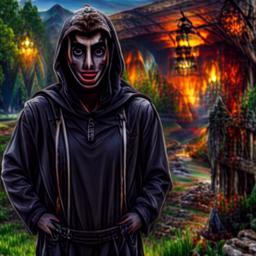} \\

        \rotatebox{90}{\parbox{2cm}{\centering\footnotesize DRaFT-40}} &
        \includegraphics[width=0.135\linewidth]{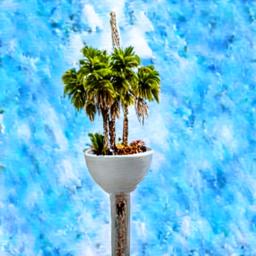}
    	\includegraphics[width=0.135\linewidth]{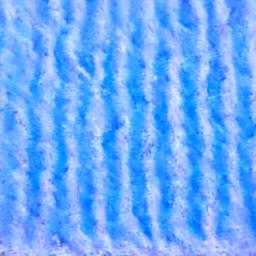}
    	\includegraphics[width=0.135\linewidth]{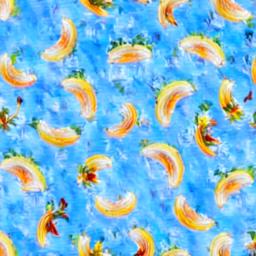}
    	\includegraphics[width=0.135\linewidth]{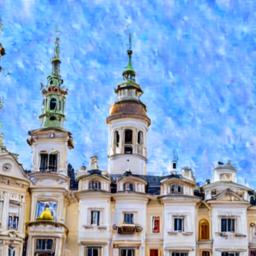}
    	\includegraphics[width=0.135\linewidth]{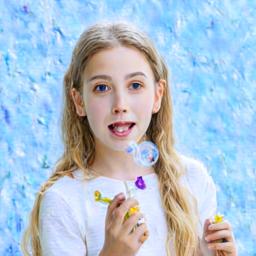}
    	\includegraphics[width=0.135\linewidth]{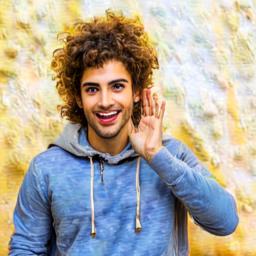}
    	\includegraphics[width=0.135\linewidth]{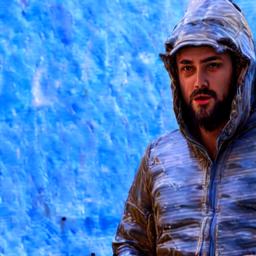}\\

        \rotatebox{90}{\parbox{2cm}{\centering\footnotesize ReFL}} &
        \includegraphics[width=0.135\linewidth]{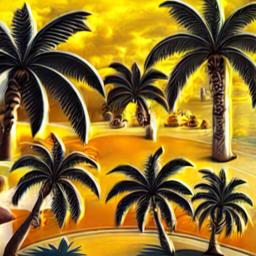}
    	\includegraphics[width=0.135\linewidth]{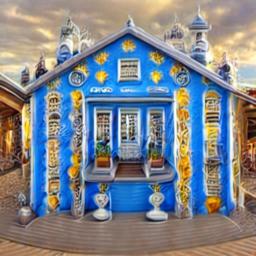}
    	\includegraphics[width=0.135\linewidth]{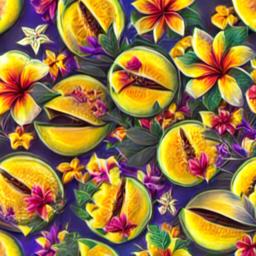}
    	\includegraphics[width=0.135\linewidth]{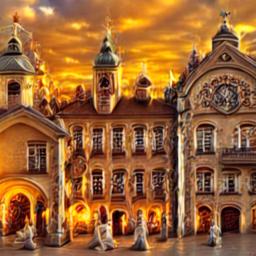}
    	\includegraphics[width=0.135\linewidth]{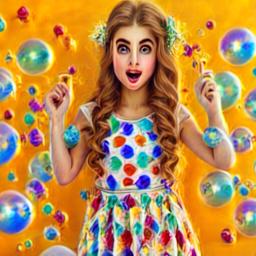}
    	\includegraphics[width=0.135\linewidth]{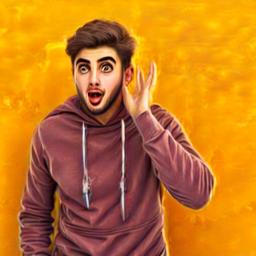}
    	\includegraphics[width=0.135\linewidth]{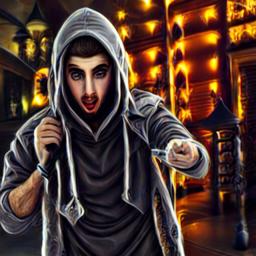} \\

        \rotatebox{90}{\parbox{2cm}{\centering\footnotesize Cont. Adj.\\$\lambda = 12500$}} &
        \includegraphics[width=0.135\linewidth]{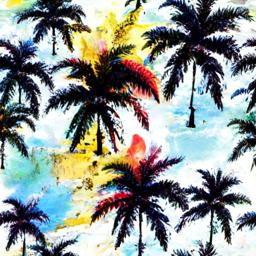}
    	\includegraphics[width=0.135\linewidth]{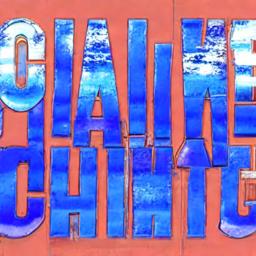}
    	\includegraphics[width=0.135\linewidth]{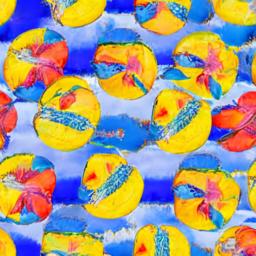}
    	\includegraphics[width=0.135\linewidth]{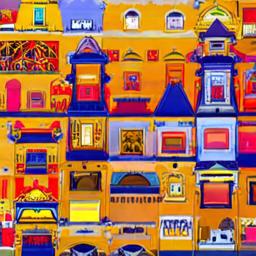}
    	\includegraphics[width=0.135\linewidth]{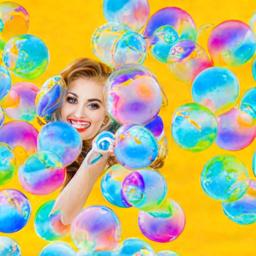}
    	\includegraphics[width=0.135\linewidth]{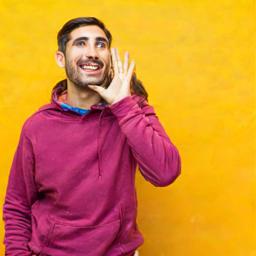}
    	\includegraphics[width=0.135\linewidth]{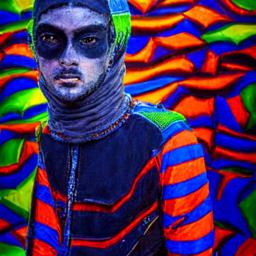} \\

        \rotatebox{90}{\parbox{2cm}{\centering\footnotesize Disc. Adj.\\$\lambda = 12500$}} &
        \includegraphics[width=0.135\linewidth]{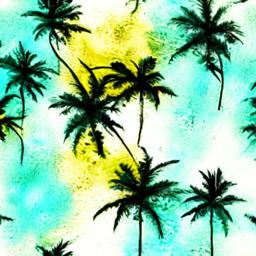}
    	\includegraphics[width=0.135\linewidth]{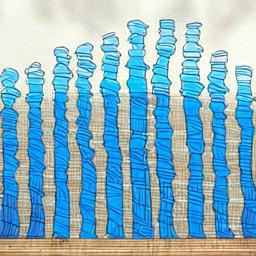}
    	\includegraphics[width=0.135\linewidth]{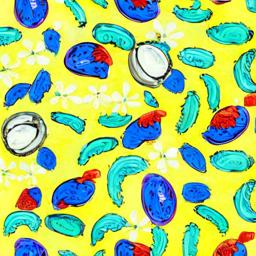}
    	\includegraphics[width=0.135\linewidth]{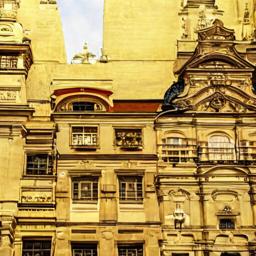}
    	\includegraphics[width=0.135\linewidth]{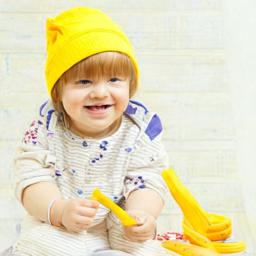}
    	\includegraphics[width=0.135\linewidth]{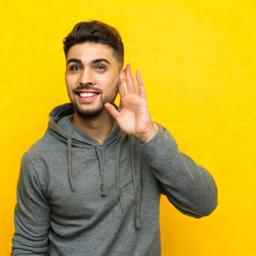}
    	\includegraphics[width=0.135\linewidth]{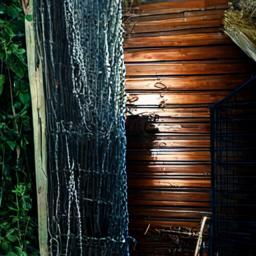} \\

        \rotatebox{90}{\parbox{2cm}{\centering\footnotesize Adj. match.\\$\lambda = 1000$}} &
        \includegraphics[width=0.135\linewidth]{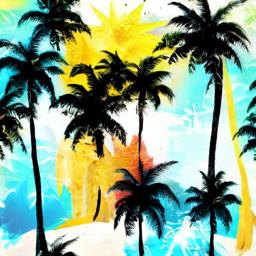}
        \includegraphics[width=0.135\linewidth]{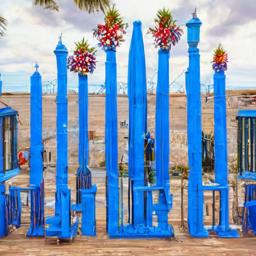}
        \includegraphics[width=0.135\linewidth]{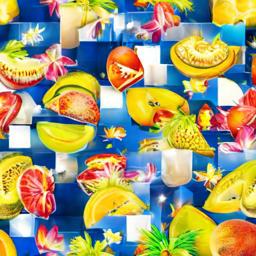}
        \includegraphics[width=0.135\linewidth]{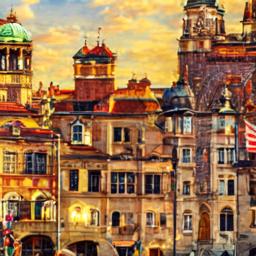}
        \includegraphics[width=0.135\linewidth]{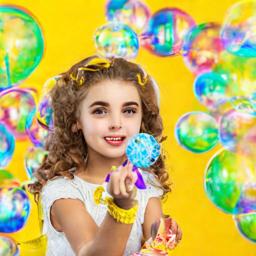}
        \includegraphics[width=0.135\linewidth]{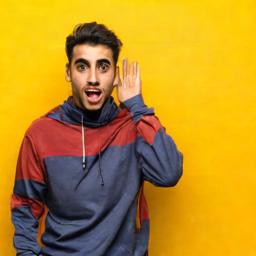}
        \includegraphics[width=0.135\linewidth]{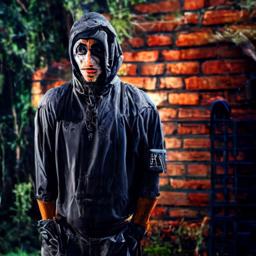} \\

        \rotatebox{90}{\parbox{2cm}{\centering\footnotesize Adj. match.\\$\lambda = 2500$}} &
        \includegraphics[width=0.135\linewidth]{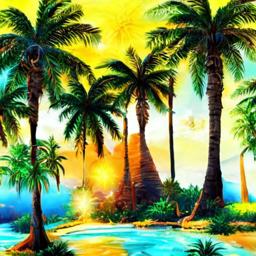}
    	\includegraphics[width=0.135\linewidth]{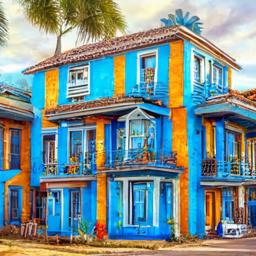}
    	\includegraphics[width=0.135\linewidth]{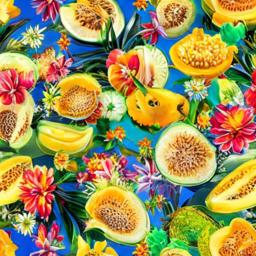}
    	\includegraphics[width=0.135\linewidth]{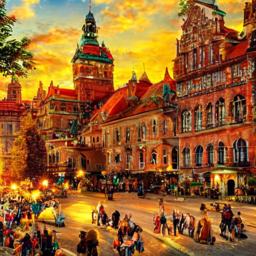}
    	\includegraphics[width=0.135\linewidth]{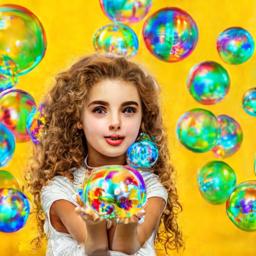}
    	\includegraphics[width=0.135\linewidth]{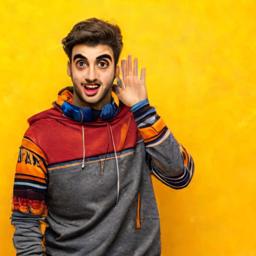}
    	\includegraphics[width=0.135\linewidth]{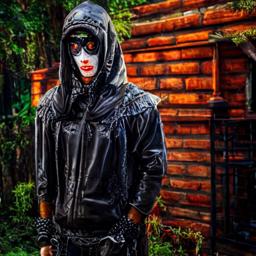} \\

        \rotatebox{90}{\parbox{2cm}{\centering\footnotesize Adj. match.\\$\lambda = 12500$}} &
        \includegraphics[width=0.135\linewidth]{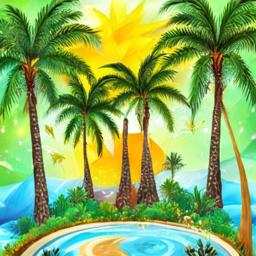}
    	\includegraphics[width=0.135\linewidth]{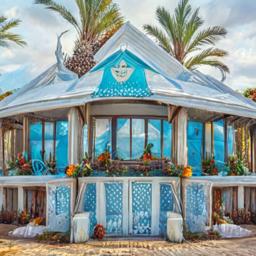}
    	\includegraphics[width=0.135\linewidth]{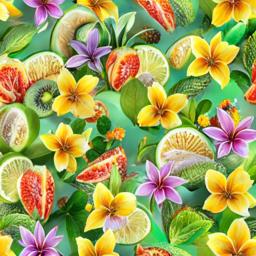}
    	\includegraphics[width=0.135\linewidth]{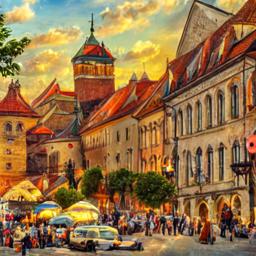}
    	\includegraphics[width=0.135\linewidth]{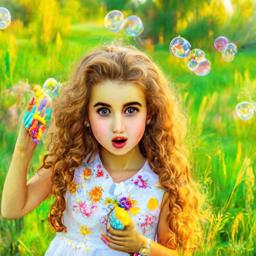}
    	\includegraphics[width=0.135\linewidth]{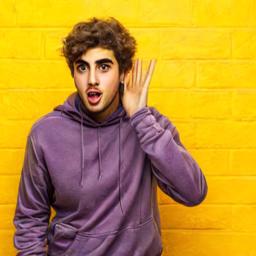}
    	\includegraphics[width=0.135\linewidth]{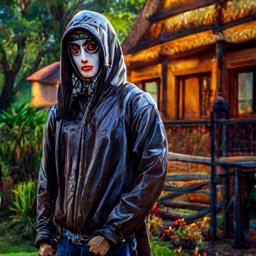} 
     \end{tabular}
     \caption{Generated samples without guidance ($w=0$) and $\sigma(t) = \sqrt{2 \eta_t}$ across seven selected prompts. Each row corresponds to a different finetuning algorithm. The prompts are the same as in \Cref{fig:one_image_per_prompt_ode}.}
    \label{fig:one_image_per_prompt_sde}
\end{figure}

\clearpage
\newpage

\section{Results on DDIM and Flow Matching}
\label{sec:proofs_background}

\subsection{The continuous-time limit of DDIM} \label{subsec:continuous_DDIM}
% \ricky{TODO: switch notation to $\bar{\alpha}$}
The DDIM inference update \citep[Eq.~12]{song2021denoising} is 
\begin{talign} \label{eq:DDIM_original}
    x_{k+1} = \sqrt{\bar{\alpha}_{k+1}} \big( \frac{x_{k} - \sqrt{1-\bar{\alpha}_k} \epsilon(x_k,k)}{\sqrt{\bar{\alpha}_k}} \big) + \sqrt{1-\bar{\alpha}_{k+1} - \sigma_{k}^2} \epsilon(x_k,k) + \sigma_{k} \epsilon_{k}, \qquad x_K \sim N(0,I).
\end{talign}
If we let 
% $\Delta \bar{\alpha}_{k} = \bar{\alpha}_{k} - \bar{\alpha}_{k-1}$,
$\Delta \bar{\alpha}_{k} = \bar{\alpha}_{k+1} - \bar{\alpha}_{k}$,
we have that
\begin{talign}
    \sqrt{\frac{\bar{\alpha}_{k+1}}{\bar{\alpha}_k}} = \sqrt{\frac{\bar{\alpha}_k + \bar{\alpha}_{k+1} - \bar{\alpha}_k}{\bar{\alpha}_k}} = \sqrt{1 + \frac{\bar{\alpha}_{k+1} - \bar{\alpha}_{k}}{\bar{\alpha}_k}} = \sqrt{1 + \frac{\Delta \bar{\alpha}_k}{\bar{\alpha}_k}} \approx 1 + \frac{\Delta \bar{\alpha}_k}{2\bar{\alpha}_k}, 
\end{talign}
where we used the first-order Taylor approximation of $\sqrt{1+x}$. And
\begin{talign}
\begin{split}
    &- \sqrt{\frac{\bar{\alpha}_{k+1}}{\bar{\alpha}_k} (1-\bar{\alpha}_k)} + \sqrt{1-\bar{\alpha}_{k+1} - \sigma_{k}^2} = - \sqrt{\big( 1 + \frac{\Delta \bar{\alpha}_k}{\bar{\alpha}_k} \big) (1-\bar{\alpha}_k)} + \sqrt{1-\bar{\alpha}_{k+1} - \sigma_{k}^2} \\ &= - \sqrt{1 + \frac{\Delta \bar{\alpha}_k}{\bar{\alpha}_k} -\bar{\alpha}_k - \Delta \bar{\alpha}_k} + \sqrt{1-\bar{\alpha}_{k+1} - \sigma_{k}^2} = - \sqrt{1 -\bar{\alpha}_{k+1} + \frac{\Delta \bar{\alpha}_k}{\bar{\alpha}_k}} + \sqrt{1-\bar{\alpha}_{k+1} - \sigma_{k}^2} \\ &= \sqrt{1- \bar{\alpha}_{k+1}} \big( - \sqrt{1 + \frac{\Delta \bar{\alpha}_k}{\bar{\alpha}_k(1-\bar{\alpha}_{k+1})}} + \sqrt{1 - \frac{\sigma_{k}^2}{1-\bar{\alpha}_{k+1}}} \big) \approx \sqrt{1- \bar{\alpha}_{k+1}} \big( - \big( 1+\frac{\Delta \bar{\alpha}_k}{2\bar{\alpha}_k(1-\bar{\alpha}_{k+1})} \big) + 1 - \frac{\sigma_{k}^2}{2(1-\bar{\alpha}_{k+1})} \big) \\ &= %\sqrt{1- \bar{\alpha}_{k-1}} \big( 
    % \frac{\Delta \bar{\alpha}_k}{2\bar{\alpha}_k\sqrt{1-\bar{\alpha}_{k-1}}} - \frac{\sigma_{k}^2}{2\sqrt{1-\bar{\alpha}_{k-1}}} 
    % \big)
    - \big( \frac{\Delta \bar{\alpha}_k}{2\bar{\alpha}_k} + \frac{\sigma_{k}^2}{2} \big) \frac{1}{\sqrt{1-\bar{\alpha}_{k+1}}},
\end{split}
\end{talign}
where we used the same first-order Taylor approximation. Thus, up to first-order approximations, \eqref{eq:DDIM_original} is equivalent to
\begin{talign} \label{eq:DDIM_approx}
    x_{k-1} = \big( 1 + \frac{\Delta \bar{\alpha}_k}{2\bar{\alpha}_k} \big) x_k - \big( \frac{\Delta \bar{\alpha}_k}{2\bar{\alpha}_k} + \frac{\sigma_{k}^2}{2} \big) \frac{\epsilon(x_k,k)}{\sqrt{1-\bar{\alpha}_{k+1}}} + \sigma_k \epsilon_k, \qquad x_K \sim N(0,I).
\end{talign}
If we modify our notation slightly, we can rewrite this as
\begin{talign} \label{eq:DDIM_approx_continuous}
    X_{(k+1)h} = \big( 1 - \frac{h \dot{\bar{\alpha}}_{kh}}{2\bar{\alpha}_{kh}} \big) X_{kh} + \big( \frac{h \dot{\bar{\alpha}}_{kh}}{2\bar{\alpha}_{kh}} - \frac{h \sigma(kh)^2}{2} \big) \frac{\epsilon(X_{kh},kh)}{\sqrt{1-\bar{\alpha}_{kh}}} + \sqrt{h} \sigma(kh) \epsilon_k, \qquad X_{0} \sim N(0,I).
\end{talign}
To go from \eqref{eq:DDIM_approx} to \eqref{eq:DDIM_approx_continuous}, we introduced a continuous time variable and a stepsize $h = 1/K$, and we regard the increment $h \bar{\alpha}_k$ as approximately equal to $h$ times the derivative of $\bar{\alpha}$. We also 
% flipped the time axis for the iterate variables, and
identified $\sigma_k$ with $\sqrt{h} \sigma(kh)$, where $\sigma(kh)$ plays the role of a diffusion coefficient. Note that equation \eqref{eq:DDIM_approx_continuous} can be reverse-engineered as the Euler-Maruyama discretization of the SDE
\begin{talign}
    \mathrm{d}X_t = \big( - \frac{\dot{\bar{\alpha}}_{t}}{2\bar{\alpha}_{t}} + \big( \frac{\dot{\bar{\alpha}}_{t}}{2\bar{\alpha}_{t}} - \frac{\sigma(t)^2}{2} \big) \frac{\epsilon(X_{t},t)}{\sqrt{1-\bar{\alpha}_{t}}} \big) \mathrm{d}t + \sigma(t) \mathrm{d}B_t, \qquad X_{0} \sim N(0,I).
\end{talign}
% Next, we prove equation \eqref{eq:DDPM_discrete}. When $\sigma_k = \sqrt{1-\frac{\bar{\alpha}_k}{\bar{\alpha}_{k+1}}}$, we have that
% \begin{talign}
% \begin{split}
%     &- \sqrt{\frac{\bar{\alpha}_{k+1}}{\bar{\alpha}_k} (1-\bar{\alpha}_k)} + \sqrt{1 - \bar{\alpha}_{k+1} - \sigma_k^2} = - \sqrt{\frac{\bar{\alpha}_{k+1}}{\bar{\alpha}_k} (1-\bar{\alpha}_k)} + \sqrt{1 - \bar{\alpha}_{k+1} - \big( 1 - \frac{\bar{\alpha}_k}{\bar{\alpha}_{k+1}} \big)} \\ &= - \sqrt{\frac{\bar{\alpha}_{k+1}}{\bar{\alpha}_k} - \bar{\alpha}_{k+1}} + \sqrt{\frac{\bar{\alpha}_k}{\bar{\alpha}_{k+1}} - \bar{\alpha}_{k+1}} = - \sqrt{1  - \bar{\alpha}_{k+1} + \frac{\Delta \bar{\alpha}_{k}}{\bar{\alpha}_k}} + \sqrt{1 - \bar{\alpha}_{k+1} - \frac{\Delta \bar{\alpha}_k}{\bar{\alpha}_{k+1}}} \\ &= \sqrt{1 - \bar{\alpha}_{k+1}} \big( -\sqrt{1 + \frac{\Delta \bar{\alpha}_k}{\bar{\alpha}_k (1-\bar{\alpha}_{k+1})}} + \sqrt{1 - \frac{\Delta \bar{\alpha}_k}{\bar{\alpha}_{k+1} (1-\bar{\alpha}_{k+1})}} \big) \\ &\approx \sqrt{1 - \bar{\alpha}_{k+1}} \big( -\big( 1 + \frac{\Delta \bar{\alpha}_k}{2\bar{\alpha}_k (1-\bar{\alpha}_{k+1})} \big) + \big( 1 - \frac{\Delta \bar{\alpha}_k}{2 \bar{\alpha}_{k-1} (1-\bar{\alpha}_{k-1})} \big) \big) = - \frac{\Delta \bar{\alpha}_k}{2 \sqrt{1-\bar{\alpha}_{k+1}}} \big( \frac{1}{\bar{\alpha}_{k}} + \frac{1}{\bar{\alpha}_{k+1}} \big) \\ &\approx - \frac{\Delta \bar{\alpha}_k}{\bar{\alpha}_{k+1} \sqrt{1-\bar{\alpha}_{k+1}}} = -\frac{1 - \frac{\bar{\alpha}_k}{\bar{\alpha}_{k+1}}}{\sqrt{1-\bar{\alpha}_{k+1}}}.
% \end{split}
% \end{talign}

\subsection{Forward and backward stochastic differential equations}

Let ${(\kappa_t)}_{t \in [0,1]}$ and ${(\eta_t)}_{t \in [0,1]}$ such that
\begin{talign} \label{eq:kappa_eta_conditions}
    \forall t \in [0,1], \quad \eta_t \geq 0, \qquad\qquad \int_0^1 \kappa_{1-s} \, \mathrm{d}s = +\infty, \qquad\qquad 2 \int_0^1 \eta_{1-t'} \exp \big( - 2 \int_{t'}^{t} \kappa_{1-s} \, \mathrm{d}s \big) \, \mathrm{d}t' = 1.
\end{talign}

As shown in \Cref{tab:coefficients}, DDIM corresponds to $\kappa_t = \frac{\dot{\bar{\alpha}}_{t}}{2\bar{\alpha}_{t}}$, $\eta_t = \frac{\dot{\bar{\alpha}}_{t}}{2\bar{\alpha}_{t}}$, and Flow Matching corresponds to $\kappa_t = \frac{\dot{\alpha}_{t}}{\alpha_{t}}$, $\eta_t = \beta_t \big( \frac{\dot{\alpha}_{t}}{\alpha_{t}} \beta_t - \dot{\beta}_t \big)$. 
\begin{lemma}[DDIM and Flow Matching fulfill the conditions \eqref{eq:kappa_eta_conditions}] \label{lem:DDIM_FM_conditions}
The choices of $(\kappa_t)_{t \in [0,1]}$ and $(\eta_t)_{t \in [0,1]}$ for DDIM and Flow Matching fulfill the conditions \eqref{eq:kappa_eta_conditions}. For DDIM, we have that
\begin{talign}
\begin{split}
    \int_0^t \kappa_{1-s} \, \mathrm{d}s = -\frac{1}{2} \log \bar{\alpha}_{1-t} &\implies \int_0^1 \kappa_{1-s} \, \mathrm{d}s = +\infty, \\
    2 \int_0^t \eta_{t'} \exp \big( - 2 \int_{t'}^{t} \kappa_s \, \mathrm{d}s \big) \, \mathrm{d}t' = 1-\bar{\alpha}_{1-t} &\implies 2 \int_0^1 \eta_{t'} \exp \big( - 2 \int_{t'}^{t} \kappa_s \, \mathrm{d}s \big) \, \mathrm{d}t' = 1.
\end{split}
\end{talign}
For Flow Matching, 
\begin{talign}
    \int_0^t \kappa_{1-s} \, \mathrm{d}s = - \log \alpha_{1-t} &\implies \int_0^1 \kappa_{1-s} \, \mathrm{d}s = +\infty, \\
    2 \int_0^t \eta_{t'} \exp \big( - 2 \int_{t'}^{t} \kappa_s \, \mathrm{d}s \big) \, \mathrm{d}t' = \beta^2_{1-t} &\implies 2 \int_0^1 \eta_{t'} \exp \big( - 2 \int_{t'}^{t} \kappa_s \, \mathrm{d}s \big) \, \mathrm{d}t' = 1.
\end{talign}
\end{lemma}

\paragraph{Forward and backward SDEs} Consider the forward and backward SDEs
\begin{talign} \label{eq:forward_generic}
    % \mathrm{d} \vec{X}_{t} &= - \frac{\dot{\alpha}_{1-t}}{2\alpha_{1-t}} \vec{X}_{t} \mathrm{d}t + \sqrt{\frac{\dot{\alpha}_{1-t}}{\alpha_{1-t}}} \mathrm{d}B_{t}, \qquad \vec{X}_0 \sim p_{\mathrm{data}}, \\ \mathrm{d} X_{t} &= \big( \frac{\dot{\alpha}_{t}}{2\alpha_{t}} X_{t} + \frac{\dot{\alpha}_{t}}{\alpha_{t}} \nabla \log \vec{p}_{1-t}(X_{t}) \big) \mathrm{d}t + \sqrt{\frac{\dot{\alpha}_{t}}{\alpha_{t}}} \mathrm{d}B_{t}, \qquad X_0 \sim N(0,I),
    \mathrm{d}\vec{X}_t &= - \kappa_{1-t} \vec{X}_t \, \mathrm{d}t + \sqrt{2 \eta_{1-t}} \, \mathrm{d}B_t, \qquad \vec{X}_0 \sim p_{\mathrm{data}}, \\
    \mathrm{d}X_t &= \big( \kappa_t X_t + 2 \eta_t \mathfrak{s}(X_t,t) \big) \mathrm{d}t + \sqrt{2 \eta_t} \, \mathrm{d}B_t, \qquad X_0 \sim N(0,I),
    \label{eq:backward_generic}
\end{talign}
where we let $\vec{p}_t$ be the density of $\vec{X}_t$, and we define the score function as $\mathfrak{s}(x,t) := \nabla \log \vec{p}_{1-t}(x)$. Similarly, we let $p_t$ be the density of $X_t$. $\vec{p}_t$ and $p_t$ solve the Fokker-Planck equations:
\begin{talign} \label{eq:FP_forward}
    \partial_t \vec{p}_t &= \nabla \cdot \big( \kappa_{1-t} x \vec{p}_t \big) + \eta_{1-t} \Delta \vec{p}_t, \qquad \vec{p}_0 = p_{\mathrm{data}}, \\
    \partial_t p_t &= \nabla \cdot \big( \big(- \kappa_t x - 2 \eta_t \nabla \log \vec{p}_{1-t}(X_{t}) \big) p_t \big) + \eta_t \Delta p_t, \qquad p_0 = N(0,I). \label{eq:FP_backward}
\end{talign}
\begin{lemma}[Solution of the forward SDE] \label{lem:OU_process}
    Let $(\kappa_t)_{t \geq 0}$, $(\eta_t)_{t \geq 0}$ with $\eta_t \geq 0$, and $(\xi_t)_{t \geq 0}$ be arbitrary. The solution $\vec{X}_t$ of the SDE 
    \begin{talign}
        \mathrm{d}\vec{X}_t &= \big( - \kappa_{1-t} \vec{X}_t + \xi_t \big) \, \mathrm{d}t + \sqrt{2 \eta_{1-t}} \, \mathrm{d}B_t, \qquad \vec{X}_0 \sim p_{\mathrm{data}}
    \end{talign}
    is
    \begin{talign}
        \vec{X}_t = \vec{X}_0 \exp \big( - \int_0^t \kappa_{1-s} \, \mathrm{d}s \big) + \int_0^t \exp \big( - \int_{t'}^{t} \kappa_{1-s} \, \mathrm{d}s \big) \xi_{1-t'} \, \mathrm{d}t' + \int_0^t \sqrt{2 \eta_{1-t'}} \exp \big( - \int_{t'}^{t} \kappa_{1-s} \, \mathrm{d}s \big) \, \mathrm{d}B_{t'},
    \end{talign}
    which has the same distribution as the random variable
    \begin{talign} 
    \begin{split} \label{eq:hat_X_t_OU_process}
        \hat{X}_t &= \vec{X}_0 \exp \big( - \int_0^t \kappa_{1-s} \, \mathrm{d}s \big) + \int_0^t \exp \big( - \int_{t'}^{t} \kappa_{1-s} \, \mathrm{d}s \big) \xi_{1-t'} \, \mathrm{d}t' + \sqrt{2 \int_0^t \eta_{1-t'} \exp \big( - 2 \int_{t'}^{t} \kappa_{1-s} \, \mathrm{d}s \big) \, \mathrm{d}t'} \epsilon, \\ \epsilon &\sim N(0,I).
    \end{split}
    \end{talign}
\end{lemma}
Applying \Cref{lem:OU_process} with $\xi_t \equiv 0$, we obtain that $\vec{p}_1$ is also the distribution of
\begin{talign} \label{eq:hat_X_1}
    \hat{X}_1 = \vec{X}_0 \exp \big( - \int_0^t \kappa_{1-s} \, \mathrm{d}s \big) + \sqrt{2 \int_0^t \eta_{1-t'} \exp \big( - 2 \int_{t'}^{t} \kappa_{1-s} \, \mathrm{d}s \big) \, \mathrm{d}t'} \epsilon = \epsilon,
\end{talign}
where $\epsilon \sim N(0,I)$. The third equality in \eqref{eq:hat_X_1} holds by \eqref{eq:kappa_eta_conditions}. Hence we obtain that $\vec{p}_1 = N(0,I)$. Note also that
\begin{talign}
    \partial_t \vec{p}_{1-t} &= -\nabla \cdot \big( \kappa_{t} x \vec{p}_{1-t} \big) - \eta_{t} \Delta \vec{p}_{1-t} = -\nabla \cdot \big( \big( - \kappa_{t} x - 2 \eta_t \nabla \log \vec{p}_{1-t}(x) \big) \vec{p}_{1-t} \big) + \eta_{t} \Delta \vec{p}_{1-t}
\end{talign}
Thus, $\vec{p}_{1-t}$ is a solution of the backward Fokker-Planck equation \eqref{eq:FP_backward}, which proves the following: 
\begin{proposition}[Equality of marginal distributions] \label{prop:equality_marginals}
For any time $t \in [0,1]$, the densities of the solutions $\vec{X}_t$, $X_t$ of the forward and backward SDEs are equal up to a time flip: $p_t = \vec{p}_{1-t}$.
\end{proposition}

\paragraph{Forward and backward SDEs with arbitrary noise schedule} Next, we look at the following pair of forward-backward SDEs:
\begin{talign} \label{eq:forward_arbitrary}
    % \mathrm{d} \vec{X}_{t} &= - \frac{\dot{\alpha}_{1-t}}{2\alpha_{1-t}} \vec{X}_{t} \mathrm{d}t + \sqrt{\frac{\dot{\alpha}_{1-t}}{\alpha_{1-t}}} \mathrm{d}B_{t}, \qquad \vec{X}_0 \sim p_{\mathrm{data}}, \\ \mathrm{d} X_{t} &= \big( \frac{\dot{\alpha}_{t}}{2\alpha_{t}} X_{t} + \frac{\dot{\alpha}_{t}}{\alpha_{t}} \nabla \log \vec{p}_{1-t}(X_{t}) \big) \mathrm{d}t + \sqrt{\frac{\dot{\alpha}_{t}}{\alpha_{t}}} \mathrm{d}B_{t}, \qquad X_0 \sim N(0,I),
    \mathrm{d}\vec{X}_t &= \big( - \kappa_{1-t} \vec{X}_t + \big( \frac{\sigma(1-t)^2}{2} - \eta_{1-t} \big) \mathfrak{s}(\vec{X}_t,1-t) \big) \, \mathrm{d}t + \sigma(1-t) \, \mathrm{d}B_t, \qquad \vec{X}_0 \sim p_{\mathrm{data}}, \\
    \mathrm{d}X_t &= \big( \kappa_t X_t + \big( \frac{\sigma(t)^2}{2} + \eta_t \big) \mathfrak{s}(X_t,t) \big) \mathrm{d}t + \sigma(t) \, \mathrm{d}B_t, \qquad X_0 \sim N(0,I), \label{eq:backward_arbitrary}
\end{talign}
Here, the score function $\mathfrak{s}$ is the same vector field as in \eqref{eq:backward_arbitrary}.
Remark that equations \eqref{eq:forward_generic}-\eqref{eq:backward_generic} are a particular case of \eqref{eq:forward_arbitrary}-\eqref{eq:backward_arbitrary} for which $\sigma(t) = \sqrt{2 \eta_t}$. The Fokker-Planck equations for \eqref{eq:forward_arbitrary}-\eqref{eq:backward_arbitrary} are:
\begin{talign} \label{eq:FP_forward_arbitrary}
    \partial_t \vec{p}_t &= \nabla \cdot \big( \big( \kappa_{1-t} x + \big( - \frac{\sigma(1-t)^2}{2} + \eta_{1-t} \big) \mathfrak{s}(X_{t},t) \big) \vec{p}_t \big) + \eta_{1-t} \Delta \vec{p}_t, \qquad \vec{p}_0 = p_{\mathrm{data}}, \\
    \partial_t p_t &= \nabla \cdot \big( \big(- \kappa_t x - \big( \frac{\sigma(t)^2}{2} + \eta_t \big) \mathfrak{s}(X_{t},t) \big) p_t \big) + \frac{\sigma(t)^2}{2} \Delta p_t, \qquad p_0 = N(0,I). \label{eq:FP_backward_arbitrary}
\end{talign}
It is straight-forward to see that for any $\sigma$, the solutions $\vec{p}_t$ and $p_t$ of \eqref{eq:FP_forward_arbitrary}-\eqref{eq:FP_backward_arbitrary} are also solutions of \eqref{eq:FP_forward}-\eqref{eq:FP_backward}. Hence, the marginals $\vec{X}_t$ and $X_t$ are equally distributed for all noise schedules $\sigma$, and they are equal to each other up to a time flip.

\paragraph{Equality of distributions over trajectories} The result in \Cref{prop:equality_marginals} can be made even stronger:
\begin{proposition}[Equality of distributions over trajectories] \label{lem:equal_process_distributions}
    Let $\vec{\bm{X}}$, $\bm{X}$ be the solutions of the SDEs \eqref{eq:forward_arbitrary}-\eqref{eq:backward_arbitrary} with arbitrary noise schedule. For any sequence of times $(t_i)_{0 \leq i \leq I}$, the joint distribution of $(\vec{X}_{t_i})_{0 \leq i \leq I}$ is equal to the joint distribution of $(X_{1-t_i})_{0 \leq i \leq I}$, or equivalently, that the probability measures $\vec{\mathbb{P}}$, $\mathbb{P}$ of the forward and backward processes $\vec{\bm{X}}$, $\bm{X}$ are equal, up to a flip in the time direction.
\end{proposition}
This result states that sampling trajectories from the backward process is equivalent to sampling them from the forward process and then flipping their order.

% \paragraph{DDIM and Flow Matching fulfill the conditions \eqref{eq:kappa_eta_conditions}} 
\subsubsection{Proof of \Cref{lem:DDIM_FM_conditions}}
As shown in \Cref{tab:coefficients}, DDIM corresponds to $\kappa_t = \frac{\dot{\bar{\alpha}}_{t}}{2\bar{\alpha}_{t}}$, $\eta_t = \frac{\dot{\bar{\alpha}}_{t}}{2\bar{\alpha}_{t}}$. Thus, $\eta_t \geq 0$ because $\bar{\alpha}_t$ is increasing, and
\begin{talign}
    \begin{split}
        % &\exp \big( 
        &\int_0^t \kappa_{1-s} \, \mathrm{d}s 
        % \big)
         = 
        % \exp \big( 
        \int_0^t \frac{\dot{\bar{\alpha}}_{1-s}}{2\bar{\alpha}_{1-s}} \, \mathrm{d}s 
        % \big) 
         =  
        % \sqrt{\exp \big( 
        - \frac{1}{2} \int_0^t \partial_s \log \bar{\alpha}_{1-s} \, \mathrm{d}s 
        % \big)} 
         =  
        % \sqrt{\exp \big( 
        - \frac{1}{2} (\log \bar{\alpha}_{1-t} - \log \bar{\alpha}_{1}) 
        % \big)} 
         = - \frac{1}{2} \log \bar{\alpha}_{1-t}, \\
        % &\implies \exp \big( \! - \! \int_0^1 \kappa_s \, \mathrm{d}s \big) \! = \! \sqrt{\alpha_{0}} \! = \! 0, 
        &\implies \int_0^1 \kappa_{1-s} \, \mathrm{d}s = - \frac{1}{2} \log \bar{\alpha}_{0} = +\infty
    \end{split} \\
    \begin{split}
        & %\sqrt{
        2 \int_0^t \eta_{t'} \exp \big( - 2 \int_{t'}^{t} \kappa_s \, \mathrm{d}s \big) \, \mathrm{d}t'
        % } 
        = %\sqrt{ 
        \int_0^t \frac{\dot{\bar{\alpha}}_{1-t'}}{\bar{\alpha}_{1-t'}} \exp \big( - \int_{t'}^{t} \frac{\dot{\bar{\alpha}}_{1-s}}{\bar{\alpha}_{1-s}} \, \mathrm{d}s \big) \, \mathrm{d}t'
        % } 
        \\ &= 
        % \sqrt{
        \int_0^t \frac{\dot{\bar{\alpha}}_{1-t'}}{\bar{\alpha}_{1-t'}} \frac{\bar{\alpha}_{1-t}}{\bar{\alpha}_{1-t'}} \, \mathrm{d}t'
        % } 
        = 
        % \sqrt{
        \bar{\alpha}_{1-t} \int_0^t \partial_{t'} \big( \frac{1}{\bar{\alpha}_{1-t'}} \big) \, \mathrm{d}t'
        % } 
        = 
        % \sqrt{
        \bar{\alpha}_{1-t} \big( \frac{1}{\bar{\alpha}_{1-t}} - \frac{1}{\bar{\alpha}_{1}} \big)
        % } 
        = 
        % \sqrt{
        1 - \bar{\alpha}_{1-t},
        % }, 
        \\
        &\implies %\sqrt{
        2 \int_0^1 \eta_{t'} \exp \big( - 2 \int_{t'}^{t} \kappa_s \, \mathrm{d}s \big) \, \mathrm{d}t'
        % }
        = %\sqrt{
        1 - \bar{\alpha}_{0}
        % } 
        = 1.
    \end{split}
    \end{talign}
    where we used that $\bar{\alpha}_1 = 1$ and $\bar{\alpha}_0 = 0$. And Flow Matching corresponds to $\kappa_t = \frac{\dot{\alpha}_{t}}{\alpha_{t}}$, $\eta_t = \beta_t \big( \frac{\dot{\alpha}_{t}}{\alpha_{t}} \beta_t - \dot{\beta}_t \big)$. We have that $\eta_t \geq 0$ because $\alpha_t$ is increasing and $\beta_t$ is decreasing, and
    \begin{talign}
        \begin{split}
        &\int_0^t \kappa_{1-s} \, \mathrm{d}s = 
        \int_0^t \frac{\dot{\alpha}_{1-s}}{\alpha_{1-s}} \, \mathrm{d}s 
         =  
        - \int_0^t \partial_s \log \alpha_{1-s} \, \mathrm{d}s 
         =  
        - (\log \alpha_{1-t} - \log \alpha_{1}) 
         = - \log \alpha_{1-t}, \\
        &\implies \int_0^1 \kappa_{1-s} \, \mathrm{d}s = - \log \alpha_{0} = +\infty,
        \end{split}
    \end{talign}
    and
    \begin{talign}
    \begin{split} \label{eq:second_term_FM}
        &2 \int_0^t \eta_{1-t'} \exp \big( - 2 \int_{t'}^{t} \kappa_{1-s} \, \mathrm{d}s \big) \, \mathrm{d}t' = 2 \int_0^t  \beta_{1-t'} \big( \frac{\dot{\alpha}_{1-t'}}{\alpha_{1-t'}} \beta_{1-t'} - \dot{\beta}_{1-t'} \big) \exp \big( - 2 \int_{t'}^{t} \frac{\dot{\alpha}_{1-s}}{\alpha_{1-s}} \, \mathrm{d}s \big) \, \mathrm{d}t' \\ &= 2 \int_0^t  \beta_{1-t'} \big( \frac{\dot{\alpha}_{1-t'}}{\alpha_{1-t'}} \beta_{1-t'} - \dot{\beta}_{1-t'} \big) \big( \frac{\alpha_{1-t}}{\alpha_{1-t'}} \big)^2 \, \mathrm{d}t',
    \end{split}
    \end{talign}
    To develop the right-hand side, note that by integration by parts,
    \begin{talign}
    \begin{split}
        &\int_0^t \dot{\beta}_{1-t'} \beta_{1-t'} \big( \frac{\alpha_{1-t}}{\alpha_{1-t'}} \big)^2 \, \mathrm{d}t' = - \int_0^t \partial_{t'} \big( \frac{\beta_{1-t'}^2}{2} \big) \big( \frac{\alpha_{1-t}}{\alpha_{1-t'}} \big)^2 \, \mathrm{d}t' \\ &= - \big[ \frac{\beta_{1-t'}^2}{2} \big( \frac{\alpha_{1-t}}{\alpha_{1-t'}} \big)^2 \big]_{0}^{1} + \int_0^t \frac{\beta_{1-t'}^2}{2} \partial_{t'} \big( \frac{\alpha_{1-t}}{\alpha_{1-t'}} \big)^2 \, \mathrm{d}t' 
        = - \big[ \frac{\beta_{1-t'}^2}{2} \big( \frac{\alpha_{1-t}}{\alpha_{1-t'}} \big)^2 \big]_{0}^{t} + \int_0^t \beta_{1-t'}^2 \frac{\alpha_{1-t}^2 \dot{\alpha}_{1-t'}}{\alpha_{1-t'}^3} \, \mathrm{d}t'. 
    \end{split}
    \end{talign}
    And if we plug this into the right-hand side of \eqref{eq:second_term_FM}, we obtain
    \begin{talign}
        &2 \int_0^t \eta_{1-t'} \exp \big( - 2 \int_{t'}^{t} \kappa_{1-s} \, \mathrm{d}s \big) \, \mathrm{d}t' = \big[ \beta_{1-t'}^2 \big( \frac{\alpha_{1-t}}{\alpha_{1-t'}} \big)^2 \big]_{0}^{t} = \beta_{1-t}^2 - \beta_{1}^2 \big( \frac{\alpha_{1-t}}{\alpha_{1}} \big)^2 = \beta_{1-t}^2, \\
        &\implies 2 \int_0^1 \eta_{1-t'} \exp \big( - 2 \int_{t'}^{t} \kappa_{1-s} \, \mathrm{d}s \big) \, \mathrm{d}t' = \beta_{1}^2 = 1.
    \end{talign}
    where we used that $\beta_{1} = 0$, $\alpha_1 = 1$.
    % \begin{talign}
    % \begin{split}
    %     &2 \int_0^t (1-\alpha_{1-t'}) \big( \frac{\dot{\alpha}_{1-t'}}{\alpha_{1-t'}} (1-\alpha_{1-t'}) + \dot{\alpha}_{1-t'} \big) \big( \frac{\alpha_{1-t}}{\alpha_{1-t'}} \big)^2 \, \mathrm{d}t' = - 2 \alpha_{1-t}^2 \int_0^t \partial_{t'} \big( \frac{1}{\alpha_{1-t'}} - \frac{1}{2\alpha_{1-t'}^2} \big)  \, \mathrm{d}t' \\ &= - 2 \alpha_{1-t}^2 \big( \frac{1}{\alpha_{1-t}} - \frac{1}{2\alpha_{1-t}^2} - \frac{1}{\alpha_{1}} + \frac{1}{2\alpha_{1}^2} \big) = - 2 \alpha_{1-t} + 1 + \alpha_{1-t}^2 = (1-\alpha_{1-t})^2
    % \end{split}
    % \end{talign}
    % \begin{talign}
    %     \int \frac{(1-x)^2}{x^3} \, \mathrm{d}x = \log x + \frac{2}{x} - \frac{1}{2x^2}, \qquad \int \frac{1-x}{x^2} \, \mathrm{d}x = - \log x - \frac{1}{x}, \qquad \int \big( \frac{(1-x)^2}{x^3} + \frac{1-x}{x^2} \big) \, \mathrm{d}x = \frac{1}{x} - \frac{1}{2x^2}
    % \end{talign}

\subsubsection{Proof of \Cref{lem:OU_process}}
    We can solve this equation by variation of parameters. To simplify the notation, we replace $\kappa_{1-s}$, $\eta_{1-s}$ and $\xi_{1-s}$ by $\kappa_{s}$, $\eta_{s}$ and $\xi_{s}$. Defining $f(\vec{X}_t,t) = \vec{X}_t \exp \big( \int_0^t \kappa_{1-s} \, \mathrm{d}s \big)$, we get that
    \begin{talign}
    \begin{split}
        df(\vec{X}_t,t) &= \kappa_{1-t} \vec{X}_t \exp \big( \int_0^t \kappa_{1-s} \, \mathrm{d}s \big) \, \mathrm{d}t + \exp \big( \int_0^t \kappa_{1-s} \, \mathrm{d}s \big) \mathrm{d}\vec{X}_t \\ &= \kappa_{1-t} \vec{X}_t \exp \big( \int_0^t \kappa_{1-s} \, \mathrm{d}s \big) \, \mathrm{d}t + \exp \big( \int_0^t \kappa_{1-s} \, \mathrm{d}s \big) \big( (- \kappa_{1-t} \vec{X}_t + \xi_{1-t}) \, \mathrm{d}t + \sqrt{2 \eta_{1-t}} \, \mathrm{d}B_t \big) \\ &= \exp \big( \int_0^t \kappa_{1-s} \, \mathrm{d}s \big) \xi_{1-t} \, \mathrm{d}t + \sqrt{2 \eta_t} \exp \big( \int_0^t \kappa_{1-s} \, \mathrm{d}s \big) \, \mathrm{d}B_t.
    \end{split}
    \end{talign}
    Integrating from 0 to $t$, we get that
    \begin{talign}
        &\vec{X}_t \exp \big( \int_0^t \kappa_{1-s} \, \mathrm{d}s \big) = \vec{X}_0 + \int_0^t \exp \big( \int_0^{t'} \kappa_{1-s} \, \mathrm{d}s \big) \xi_{1-t'} \, \mathrm{d}t' + \int_0^t \sqrt{2 \eta_{1-t'}} \exp \big( \int_0^{t'} \kappa_{1-s} \, \mathrm{d}s \big) \, \mathrm{d}B_{t'}, \\
        &\iff \vec{X}_t = \vec{X}_0 \exp \big( - \int_0^t \kappa_{1-s} \, \mathrm{d}s \big) + \int_0^t \exp \big( - \int_{t'}^{t} \kappa_{1-s} \, \mathrm{d}s \big) \xi_{1-t'} \, \mathrm{d}t' + \int_0^t \sqrt{2 \eta_{1-t'}} \exp \big( - \int_{t'}^{t} \kappa_{1-s} \, \mathrm{d}s \big) \, \mathrm{d}B_{t'}. \label{eq:alternative_forward_process}
    \end{talign}
    Since 
    \begin{talign}
    \mathbb{E} \big[ \big( \int_0^t \sqrt{2 \eta_{1-t'}} \exp \big( - \int_{t'}^{t} \kappa_{1-s} \, \mathrm{d}s \big) \, \mathrm{d}B_{t'} \big)^2 \big] = 2 \int_0^t \eta_{1-t'} \exp \big( - 2 \int_{t'}^{t} \kappa_{1-s} \, \mathrm{d}s \big) \, \mathrm{d}t',
    \end{talign}
    we obtain that $\int_0^t \sqrt{2 \eta_{1-t'}} \exp \big( - \int_{t'}^{t} \kappa_{1-s} \, \mathrm{d}s \big) \, \mathrm{d}B_{t'}$ has the same distribution as $\sqrt{2 \int_0^t \eta_{1-t'} \exp \big( - 2 \int_{t'}^{t} \kappa_{1-s} \, \mathrm{d}s \big) \, \mathrm{d}t'} \epsilon$, where $\epsilon \sim N(0,1)$. 
% \end{proof}

\subsubsection{Proof of \Cref{lem:equal_process_distributions}}

This is a result that has been used by previous works, e.g. \cite[Sec.~2.1]{debortoli2021diffusion}, but their derivation lacks rigor as it uses some unexplained approximations. While natural, the result is not common knowledge in the area. We provide a derivation which is still in discrete time, and hence not completely formal, but that corrects the gaps in the proof of \cite{debortoli2021diffusion}.

We introduce the short-hand
\begin{talign}
    \vec{b}(x,t) &= - \kappa_{1-t} x + \big( \frac{\sigma(1-t)^2}{2} - \eta_{1-t} \big) \mathfrak{s}(x,1-t), \\
    b(x,t) &= \kappa_t X_t + \big( \frac{\sigma(t)^2}{2} + \eta_t \big) \mathfrak{s}(X_t,t), \\
    \vec{\sigma}(t) &= \sigma(1-t).
\end{talign}
Remark that $b(x,t) = - \vec{b}(x,1-t) + \sigma(t)^2 \mathfrak{s}(X_t,t)$.

Suppose that we discretize the forward process $\vec{X}$ using $K+1$ equispaced timesteps:
\begin{talign} \label{eq:forward_discretized}
    x_{k+1} = x_{k} + h \vec{b}(x_k,kh) + \sqrt{h} \vec{\sigma}(kh) \epsilon_k, \qquad  \text{with} \ \epsilon_k \sim N(0,1).
\end{talign}
It is important to remark that $x_{k+1} - x_k = O(h^{1/2})$. Throughout the proof we will keep track of all terms up to linear order in $h$, while neglecting terms of order $O(h^{3/2})$ and higher.
The distribution of the discretized forward process is:
\begin{talign}
    \vec{p}(x_{0:K}) = \vec{p}_0(x_0) \prod_{k=0}^{K-1} \vec{p}_{k+1|k}(x_{k+1}|x_{k}), \qquad \text{where} \qquad
    \vec{p}_{k+1|k}(x_{k+1}|x_{k}) = \frac{\exp \big( - \frac{\|x_{k+1} - x_{k} - h \vec{b}(x_k,kh) \|^2}{2 h \vec{\sigma}(kh)^2} \big)}{(2\pi h \vec{\sigma}(kh)^2)^{d/2}}
\end{talign}
% \begin{talign}
%     \vec{p}_{k+1|k}(x_{k+1}|x_{k}) = \frac{\exp \big( - \frac{\|x^{(k+1)} - (1+h) x^{(k)} - 2h s^{(k)}(x^{(k)})\|^2}{4 h} \big)}{(4\pi h)^{d/2}}
% \end{talign}
Using telescoping products, we have that
\begin{talign}
\begin{split} \label{eq:telescoping}
    \vec{p}(x_{0:K}) &= \vec{p}_{K}(x_{K}) \prod_{k=0}^{K-1} \vec{p}_{k+1|k}(x_{k+1}|x_{k}) \frac{\vec{p}_{k}(x_{k})}{\vec{p}_{k+1}(x_{k+1})} \\ &= \vec{p}_{K}(x_{K}) \prod_{k=0}^{K-1} \vec{p}_{k+1|k}(x_{k+1}|x_{k}) \exp \big( \log(\vec{p}_{k}(x_{k})) - \log(\vec{p}_{k+1}(x_{k+1})) \big) 
    % \\ &= \vec{p}_{K}(x_{K}) \prod_{k=0}^{K-1} \vec{p}_{k+1|k}(x_{k+1}|x_{k}) \exp \big( \log(\vec{p}_{k}(x_{k})) - \log(\vec{p}_{k+1}(x_{k})) \\ &\qquad\qquad\qquad\qquad\qquad\qquad\qquad\qquad + \log(\vec{p}_{k+1}(x_{k})) - \log(\vec{p}_{k+1}(x_{k+1})) \big) 
\end{split}
\end{talign}
We can use a discrete time version of Ito's lemma:
\begin{talign} \label{eq:update_log_p_ito}
    \log \vec{p}(x_{k+1},(k+1)h) &\approx \log \vec{p}(x_{k},kh) + h \big( \partial_t \log \vec{p}(x_{k},kh) + %\langle \nabla \log p(X_k,kh), \mu_k \rangle + 
    \frac{\vec{\sigma}(kh)^2}{2} \Delta \log \vec{p}(x_k,kh) 
    \big) \\ &\qquad %+ \sqrt{2h} 
    + \langle \nabla \log \vec{p}(x_k,kh), x_{k+1} - x_k \rangle + O(h^{3/2}).
\end{talign}
Using equation \eqref{eq:forward_discretized} and a Taylor approximation, observe that
\begin{talign}
\begin{split} \label{eq:nabla_log_diff_x}
    &\langle \nabla \log p(x_k,kh), x_{k+1} - x_k \rangle \\ &= \langle \nabla \log p(x_{k+1},(k+1)h) -\nabla^2 \log p(x_{k+1},(k+1)h) (x_{k+1} - x_k), x_{k+1} - x_k \rangle + O(h^{3/2}) \\ &= \langle \nabla \log p(x_{k+1},(k+1)h), x_{k+1} - x_k \rangle %- 2 h \Delta \log p(x_{k+1},(k+1)h).
    \\ &\qquad - \langle h \vec{b}(x_k,kh) + \sqrt{h} \vec{\sigma}(kh) \epsilon_k, \nabla^2 \log p(x_{k+1},(k+1)h) \big( h \vec{b}(x_k,kh) + \sqrt{h} \vec{\sigma}(kh) \epsilon_k \big) \rangle  + O(h^{3/2}) \\ &= \langle \nabla \log p(x_{k+1},(k+1)h), x_{k+1} - x_k \rangle - h \vec{\sigma}(kh)^2 \Delta \log p(x_{k+1},(k+1)h)  + O(h^{3/2}).
    % \\ &\qquad - \langle h \vec{b}(x_k,kh) + \sqrt{h} \sigma(kh) \epsilon_k, \nabla^2 \log p(x_{k+1},(k+1)h) \big( h \vec{b}(x_k,kh) + \sqrt{h} \sigma(kh) \epsilon_k \big) \rangle.
\end{split}
\end{talign}
% We can use the mean value theorem:
% \begin{talign}
%     &\exp \big( \log(\vec{p}_{k+1}(x_{k})) - \log(\vec{p}_{k+1}(x_{k+1})) \big) = \exp \big( \big\langle \nabla \log(\vec{p}_{k+1}(\tilde{x}_{k+1})), x_{k} - x_{k+1} \big\rangle \big) \\ &\approx \exp \big( \big\langle \nabla \log(\vec{p}_{k+1}(x_{k+1})), x_{k} - x_{k+1} \big\rangle \big),
% \end{talign}
% where $\tilde{x}_{k+1}$ is a convex combination of $x_k$ and $x_{k+1}$. And
% \begin{talign}
% &\exp \big( \log(\vec{p}_{k}(x_{k})) - \log(\vec{p}_{k+1}(x_{k})) = \exp \big( \log(\vec{p}(x_{k},kh)) - \log(\vec{p}(x_{k},(k+1)h)) \big) \\ &= \exp \big( - h \partial_t \! \log(\vec{p}(x_{k},\tilde{t})) \big) \approx \exp \big( - h \partial_t \! \log(\vec{p}(x_{k},kh)) \big)
% \end{talign}
% We observe that
% \begin{talign}
% \begin{split}
%     &\vec{p}_{k+1|k}(x_{k+1}|x_{k}) \exp \big( \log(\vec{p}_{k+1}(x_{k})) - \log(\vec{p}_{k+1}(x_{k+1})) \\ &\approx \frac{\exp \big( - \frac{\|x_{k} - (1+h) x_{k+1}\|^2 - \langle 4 h \nabla \log(\vec{p}_{k+1}(x_{k+1})), x_{k} - x_{k+1} \rangle}{4 h} \big)}{(4\pi h)^{d/2}} \\ &= \frac{\exp \big( - \frac{\|x_{k} - (1+h) x_{k+1} - 2 h \nabla \log(\vec{p}_{k+1}(x_{k+1})) \|^2 + h^2 (\| x_{k+1} \|^2 - \| x_{k+1} + 2 \nabla \log(\vec{p}_{k+1}(x_{k+1})) \|^2)}{4 h} \big)}{(4\pi h)^{d/2}} \\ &= \frac{\exp \big( - \frac{\|x_{k} - (1+h) x_{k+1} - 2 h \nabla \log(\vec{p}_{k+1}(x_{k+1})) \|^2}{4 h} \big)}{(4\pi h)^{d/2}} \times \exp \big( - \frac{h (\| x_{k+1} \|^2 - \| x_{k+1} + 2 \nabla \log(\vec{p}_{k+1}(x_{k+1})) \|^2)}{4} \big)
% \end{split}
% \end{talign}
And since $\vec{p}$ satisfies the Fokker-Planck equation
\begin{talign}
    \partial_t \vec{p}_t = \nabla \cdot \big( (-\vec{b}(x,t) + \frac{\vec{\sigma}(t)^2}{2} \nabla \log \vec{p}_t(x) ) \vec{p}_t \big),
\end{talign}
we have that
\begin{talign}
\begin{split}
    \partial_t \log \vec{p}_t &= \frac{\partial_t \vec{p}_t}{\vec{p}_t} = \frac{\nabla \cdot \big( (-\vec{b}(x,t) + \frac{\vec{\sigma}(t)^2}{2} \nabla \log \vec{p}_t(x) ) \vec{p}_t \big)}{\vec{p}_t} \\ &= 
    % d + \Delta \log \vec{p}_t(x) + \langle x + \nabla \log \vec{p}_t(x), \nabla \log \vec{p}_t(x) \rangle
    - \nabla \cdot \vec{b}(x,t) + \frac{\vec{\sigma}(t)^2}{2} \Delta \log \vec{p}_t(x) + \langle -\vec{b}(x,t) + \frac{\vec{\sigma}(t)^2}{2} \nabla \log \vec{p}_t(x), \nabla \log \vec{p}_t(x) \rangle.
\end{split}
\end{talign}
Hence,
\begin{talign}
\begin{split} \label{eq:partial_t}
    &\partial_t \log p(x_{k},kh) = \partial_t \log p(x_{k+1},(k+1)h) + O(h^{1/2}) \\ &= - \nabla \cdot \vec{b}(x_{k+1},(k+1)h) + \frac{\vec{\sigma}((k+1)h)^2}{2} \Delta \log \vec{p}(x_{k+1},(k+1)h) \\ &\qquad + \langle -\vec{b}(x_{k+1},(k+1)h) + \frac{\vec{\sigma}((k+1)h)^2}{2} \nabla \log \vec{p}(x_{k+1},(k+1)h), \nabla \log \vec{p}(x_{k+1},(k+1)h) \rangle + O(h^{1/2}).
\end{split}
\end{talign}
If we plug \eqref{eq:nabla_log_diff_x} and \eqref{eq:partial_t} into \eqref{eq:update_log_p_ito}, we obtain
\begin{talign}
\begin{split} \label{eq:log_difference}
    &\log p(x_{k+1},(k+1)h) - \log p(x_{k},kh) \\ &= h \big( \! - \! \nabla \cdot \vec{b}(x_{k+1},(k+1)h) \! + \! \langle -\vec{b}(x_{k+1},(k+1)h) \! + \! \frac{\vec{\sigma}((k+1)h)^2}{2} \nabla \log \vec{p}(x_{k+1},(k+1)h), \nabla \log \vec{p}(x_{k+1},(k+1)h) \rangle \big) \\ &\qquad + \langle \nabla \log p(x_{k+1},(k+1)h), x_{k+1} - x_k \rangle + O(h^{3/2}) \\ &= \frac{\langle 2 h \vec{\sigma}(kh)^2 \nabla \log p(x_{k+1},(k+1)h), x_{k+1} - x_k - h \vec{b}(x_{k+1},(k+1)h) \rangle}{2 h \vec{\sigma}(kh)^2} \\ &\qquad + h \big( - \nabla \cdot \vec{b}(x_{k+1},(k+1)h) + \frac{\vec{\sigma}((k+1)h)^2}{2} \| \nabla \log \vec{p}(x_{k+1},(k+1)h) \|^2 \big) + O(h^{3/2}).
\end{split}
\end{talign}
Applying a discrete time version of Ito's lemma again, we have that
\begin{talign}
\begin{split}
    \vec{b}(x_k,kh) 
    % &= \vec{b}(x_{k+1},(k+1)h) 
    &= \vec{b}(x_{k+1},(k+1)h) - h \big( \partial_t \vec{b}(x_{k+1},(k+1)h) + %\langle \nabla \log p(X_k,kh), \mu_k \rangle + 
    \frac{\vec{\sigma}((k+1)h)^2}{2} \Delta \vec{b}(x_{k+1},(k+1)h) 
    \big)
    % \\ &\qquad 
    %+ \sqrt{2h} 
    \\ &\qquad + \nabla \vec{b}(x_{k+1},(k+1)h)^{\top} (x_{k} - x_{k+1}) + O(h^{3/2}) \\ &= \vec{b}(x_{k+1},(k+1)h) + \nabla \vec{b}(x_{k+1},(k+1)h)^{\top} (x_{k} - x_{k+1}) + O(h).
\end{split}
\end{talign}
where $\Delta \vec{b}$ denotes the component-wise Laplacian of $\vec{b}$. Thus,
\begin{talign}
\begin{split} \label{eq:log_transition}
    &\log \vec{p}_{k+1|k}(x_{k+1}|x_{k}) \\ &= - \frac{d}{2} \log \big( 2\pi h \vec{\sigma}(kh)^2 \big) - \frac{\|x_{k+1} - x_{k} - h \vec{b}(x_k,kh) \|^2}{2 h \vec{\sigma}(kh)^2} \\ &= - \frac{d}{2} \log \big( 2\pi h \vec{\sigma}(kh)^2 \big) - \frac{\|x_{k+1} - x_{k} - h (\vec{b}(x_{k+1},(k+1)h) + \nabla \vec{b}(x_{k+1},(k+1)h)^{\top} (x_{k} - x_{k+1})) \|^2}{2 h \vec{\sigma}(kh)^2} + O(h^{3/2}) \\ &= - \frac{d}{2} \log \big( 2\pi h \vec{\sigma}(kh)^2 \big) - \frac{\|x_{k+1} - x_{k} - h \vec{b}(x_{k+1},(k+1)h) \|^2}{2 h \vec{\sigma}(kh)^2} + \frac{\langle x_{k+1} - x_{k}, \nabla \vec{b}(x_{k+1},(k+1)h)^{\top} (x_{k} - x_{k+1}) \rangle}{\vec{\sigma}(kh)^2} + O(h^{3/2}) \\ &= - \frac{d}{2} \log \big( 2\pi h \vec{\sigma}(kh)^2 \big) - \frac{\|x_{k+1} - x_{k} - h \vec{b}(x_{k+1},(k+1)h) \|^2}{ h \vec{\sigma}(kh)^2} - \frac{h \vec{\sigma}(kh)^2 \langle \epsilon_k, \nabla \vec{b}(x_{k+1},(k+1)h)^{\top} \epsilon_k \rangle}{\vec{\sigma}(kh)^2} + O(h^{3/2}) \\
    &= - \frac{d}{2} \log \big( 2\pi h \vec{\sigma}(kh)^2 \big) - \frac{\|x_{k+1} - x_{k} - h \vec{b}(x_{k+1},(k+1)h) \|^2}{ h \vec{\sigma}(kh)^2} - h \Delta \vec{b}(x_{k+1},(k+1)h) + O(h^{3/2}) 
\end{split}
\end{talign}
Combining \eqref{eq:log_difference} and \eqref{eq:log_transition}, we obtain that
\begin{talign}
\begin{split}
    &\log \vec{p}_{k+1|k}(x_{k+1}|x_{k}) - \big( \log p(x_{k+1},(k+1)h) - \log p(x_{k},kh) \big) \\ &= - \frac{d}{2} \log \big( 2\pi h \vec{\sigma}(kh)^2 \big) - \frac{\|x_{k+1} - x_{k} - h \vec{b}(x_{k+1},(k+1)h) + h \vec{\sigma}(kh)^2 \nabla \log p(x_{k+1},(k+1)h) \|^2}{ h \vec{\sigma}(kh)^2} + O(h^{3/2})
    \\ &= - \frac{d}{2} \log \big( 2\pi h \vec{\sigma}((k+1)h)^2 \big) - \frac{\|x_{k+1} - x_{k} - h \vec{b}(x_{k+1},(k+1)h) + h \vec{\sigma}((k+1)h)^2 \nabla \log p(x_{k+1},(k+1)h) \|^2}{ h \vec{\sigma}((k+1)h)^2} + O(h^{3/2}). 
    % \\ &h \Delta \vec{b}(x_{k+1},(k+1)h) 
\end{split}    
\end{talign}
By Bayes rule, and taking the exponential of this equation, we obtain
\begin{talign}
\begin{split} \label{eq:bayes_application}
    \vec{p}_{k+1|k}(x_{k+1}|x_{k}) &:= \vec{p}_{k+1|k}(x_{k+1}|x_{k}) \frac{\vec{p}_{k}(x_{k})}{\vec{p}_{k+1}(x_{k+1})} \\ &= \frac{\exp \big( - \frac{\|x_{k} - x_{k+1} + h \vec{b}(x_{k+1},(k+1)h) - h \vec{\sigma}((k+1)h)^2 \nabla \log p(x_{k+1},(k+1)h) \|^2}{2 h \vec{\sigma}((k+1)h)^2} \big)}{(2\pi h \vec{\sigma}((k+1)h)^2)^{d/2}} + O(h^{3/2}).
\end{split}    
\end{talign}
Up to the $O(h^{3/2})$ term, the right-hand side is the conditional Gaussian corresponding to the update
\begin{talign} \label{eq:cond_gaussian_update}
    x_{k} = x_{k+1} + h \big( - \vec{b}(x_{k+1},(k+1)h) + \vec{\sigma}((k+1)h)^2 \nabla \log p(x_{k+1},(k+1)h) \big) + \sqrt{h} \vec{\sigma}((k+1)h) \epsilon_{k+1}, \ \epsilon_{k+1} \sim N(0,I).
\end{talign}
If we define $y_k = x_{K-k}$, and we use that $b(x,t) = - \vec{b}(x,1-t) + \vec{\sigma}(t)^2 \nabla \log p(x,1-t)$, we can rewrite \eqref{eq:cond_gaussian_update} as
\begin{talign}
\begin{split}
    y_{K-k} &= y_{K-k-1} + h \big( - \vec{b}(y_{K-k-1},(K-k-1)h) + \vec{\sigma}((K-k-1)h)^2 \nabla \log p(y_{K-k-1},(K-k-1)h) \big) \\ &\qquad\qquad\quad + \sqrt{h} \vec{\sigma}((K-k-1)h) \epsilon_{k} = y_{K-k-1} + h b(y_{K-k-1},kh) + \sqrt{h} \sigma(kh) \epsilon_{K-k-1}, \\
    &\implies y_{k+1} = y_{k} + h b(y_{k},kh) + \sqrt{h} \sigma(kh) \epsilon_{k}.
\end{split}
\end{talign}
And this is the Euler-Maruyama discretization of the backward process $\cev{X}$.
If we plug \eqref{eq:bayes_application} into \eqref{eq:telescoping}, we obtain that 
\begin{talign}
    \vec{p}(x_{0:K}) &\approx \vec{p}_{K}(x_{K}) \prod_{k=0}^{K-1} \vec{p}_{k+1|k}(x_{k+1}|x_{k}).
    % \vec{p}_{k+1|k}(x_{k+1}|x_{k}) \frac{\vec{p}_{k}(x_{k})}{\vec{p}_{k+1}(x_{k+1})}
\end{talign}
which concludes the proof, as $\vec{p}_{K}(x_{K})$ is the initial distribution of the backward process, and $\vec{p}_{k+1|k}(x_{k+1}|x_{k})$ are its transition kernels. 
% which means that
% \begin{talign}
% \begin{split}
%     &\vec{p}_{k+1|k}(x_{k+1}|x_{k}) \exp \big( \log(\vec{p}_{k}(x_{k})) - \log(\vec{p}_{k+1}(x_{k+1})) \big) \\ &\approx \frac{1}{(2\pi h \sigma(kh)^2)^{d/2}} \exp \big( - \frac{\langle x_{k+1} - x_{k} - h \vec{b}(x_k,kh), x_{k+1} - x_{k} - h \vec{b}(x_k,kh) \rangle}{2 h \sigma(kh)^2} \big)
% \end{split}
% \end{talign}
% Note that the joint density of the backward process can be written as $\cev{p}(x_{0:K}) = \cev{p}_K(x_K) \prod_{k=0}^{K-1} \vec{p}_{k|k+1}(x_{k}|x_{k+1})$
% \begin{talign}
%     \vec{p}_{k|k+1}(x_{k}|x_{k+1}) = \frac{\exp \big( - \frac{\|x_{k} - (1+h) x_{k} - 2h s_{k}(x_{k})\|^2}{4 h} \big)}{(4\pi h)^{d/2}}.
% \end{talign}
% Thus, we can write
% \begin{talign}
%     \vec{p}(x_{0:K}) \approx \cev{p}(x_{0:K}) Q(x_{0:K}),
% \end{talign}
% where
% \begin{talign}
%     Q(x_{0:K}) &= \prod_{k=0}^{K-1} \exp \big( - \frac{h (\| x_{k+1} \|^2 - \| x_{k+1} + 2 \nabla \log(\vec{p}_{k+1}(x_{k+1})) \|^2)}{4} \\ &\qquad\qquad - h \big( d + \Delta \log \vec{p}_{k+1}(x_{k+1}) + \langle x_{k+1} + \nabla \log \vec{p}_{k+1}(x_{k+1}), \nabla \log \vec{p}_{k+1}(x_{k+1}) \rangle \big) \big) \\ &= \prod_{k=0}^{K-1} \exp \big( - h \big( d + \Delta \log \vec{p}_{k+1}(x_{k+1}) \big) \big)
% \end{talign}

% \subsection{Proof of \Cref{lem:DDPM_DDIM_FM}}
\subsection{The relationship between the noise predictor $\epsilon$ and the score function} \label{subsec:hat_epsilon_score}
    Applying \Cref{lem:OU_process} with the choices of $(\kappa_t)_{t \geq 0}$ and $(\eta_t)_{t \geq 0}$ for DDIM, we obtain that $\vec{X}_t$ has the same distribution as
    \begin{talign} \label{eq:hat_X_t_diffusion}
        \hat{X}_t = \sqrt{\bar{\alpha}_{1-t}} \vec{X}_0 + \sqrt{1 - \bar{\alpha}_{1-t}} \epsilon, \qquad \epsilon \sim N(0,1).
    \end{talign}
    % Since $\hat{X}_t \sim N(0,1)$, then $\vec{p}_1 = N(0,1)$. Thus, $\vec{p}_{1-t}$ is a solution of the Fokker-Planck equation \eqref{eq:FP_backward}, which means that the distribution $p_t$ of $X_t$ is equal to the distribution $\vec{p}_{1-t}$ of $\vec{X}_{1-t}$. 
    
    Since $\vec{X}_t$ and $\hat{X}_t$ have the same distribution, predicting the noise of $\vec{X}_t$ is equivalent to predicting the noise of $\hat{X}_t$. The noise predictor $\epsilon$ can be written as:
    \begin{talign} \label{eq:hat_epsilon_app}
        \epsilon(x,t) := \mathbb{E}[\epsilon|\hat{X}_{1-t} = x] = \mathbb{E} \big[\epsilon| \sqrt{\bar{\alpha}_{t}} \vec{X}_0 + \sqrt{1 - \bar{\alpha}_{t}} \epsilon = x \big] = \mathbb{E} \big[\frac{x - \sqrt{\bar{\alpha}_{t}} \vec{X}_0}{\sqrt{1 - \bar{\alpha}_{t}}} | \sqrt{\bar{\alpha}_{t}} \vec{X}_0 + \sqrt{1 - \bar{\alpha}_{t}} \epsilon = x \big]
    \end{talign}
    And the score function $\mathfrak{s}(x,t) := \nabla \log \vec{p}_{1-t}(x)$ admits the expression
    \begin{talign} \label{eq:score_app}
        \mathfrak{s}(x,t) &:= %\nabla \log p_{t}(x) = 
        \nabla \log \vec{p}_{1-t}(x) = \frac{\nabla \vec{p}_{1-t}(x)}{\vec{p}_{1-t}(x)} = \frac{\nabla \mathbb{E}[\vec{p}_{1-t|0}(x|\vec{X}_0)]}{\vec{p}_{1-t}(x)} = \frac{\mathbb{E}[\nabla \log \vec{p}_{1-t|0}(x|\vec{X}_0) \vec{p}_{1-t|0}(x|\vec{X}_0)]}{\vec{p}_{1-t}(x)}, 
        % \\ &= \mathbb{E}\big[ - \frac{x - \sqrt{\alpha_{t}} \vec{X}_0}{1 - \alpha_{t}} | \sqrt{\alpha_{t}} \vec{X}_0 + \sqrt{1 - \alpha_{t}} \epsilon = x \big].
    \end{talign}
    where
    \begin{talign}
        \vec{p}_{1-t|0}(x|\vec{X}_0) = \frac{\exp(-\|x - \sqrt{\bar{\alpha}_t} Y_1\|^2/(2(1-\bar{\alpha}_{t})))}{(2\pi (1-\bar{\alpha}_{t}))^{d/2}} \implies \nabla \log \vec{p}_{t|1}(x|Y_1) = - \frac{x - \sqrt{\bar{\alpha}_t} Y_1}{1-\bar{\alpha}_{t}}. 
    \end{talign} 
    Plugging this into the right-hand side of \eqref{eq:score_app} and using Bayes' rule, we get
    \begin{talign} \label{eq:score_app_2}
        \mathfrak{s}(x,t) = \mathbb{E}\big[ - \frac{x - \sqrt{\bar{\alpha}_{t}} \vec{X}_0}{1 - \bar{\alpha}_{t}} | \sqrt{\bar{\alpha}_{t}} \vec{X}_0 + \sqrt{1 - \bar{\alpha}_{t}} \epsilon = x \big].
    \end{talign}
    % Here, the third equality holds because $\vec{p}_{1-t}$ is also the distribution of $\hat{X}_{1-t}$. 
    Comparing the right-hand sides of \eqref{eq:hat_epsilon_app} and \eqref{eq:score_app_2}, we obtain that $\mathfrak{s}(x,t) = - \frac{\epsilon(x,t)}{\sqrt{1-\bar{\alpha}_t}}$.

\subsection{The relationship between the vector field $v$ and the score function} \label{subsec:v_score}
    % Consider the flow matching framework \citep{lipman2023flow,albergo2023building,albergo2023stochastic} with %(backward) 
    % interpolation path %$\cev{X}_t = \alpha_t \cev{X}_0 + \beta_t \cev{X}_1$, 
    % % $\cev{Y}_t = \beta_t \cev{Y}_0 + \alpha_t \cev{Y}_1$, 
    % $Y_t = \beta_t Y_0 + \alpha_t Y_1$, 
    % where $Y_0 \sim p_0 = N(0,1)$ and $Y_1 \sim p_1 = p_{\mathrm{data}}$. 
    % % The forward interpolation path reads $\vec{Y}_t = \alpha_t \vec{Y}_0 + \beta_t \vec{Y}_1$, where $\vec{Y}_0 \sim p_1$ and $\vec{Y}_1 \sim p_0$.
    % % $\alpha_{\cdot} : [0,1] \to [0,1]$ is a decreasing function such that $\alpha_0 = 1$, $\alpha_1 = 0$, and $\beta_{\cdot} : [0,1] \to [0,1]$ is an increasing function such that $\beta_0 = 0$, $\beta_1 = 1$.
    % In this framework, a vector field $v_t$ is learned, and the ODEs
    % \begin{talign} %\label{eq:FM_forward}
    %     % \frac{d\vec{X}_t}{dt} &= -v(\vec{X}_t,t), \qquad \vec{X}_0 \sim p_1, \\
    %     % \frac{d\cev{X}_t}{dt} &= v(\cev{X}_t,1-t), \qquad \cev{X}_0 \sim p_0. 
    %     % \frac{d\vec{X}_t}{dt} &= -v(\vec{X}_t,1-t), \qquad \vec{X}_0 \sim p_1, \\
    %     \frac{\mathrm{d}X_t}{dt} &= v(X_t,t), \qquad X_0 \sim p_0.
    %     \label{eq:FM_backward}
    % \end{talign}
    % have solutions such that $p_t := \mathrm{Law}(X_t) = \mathrm{Law}(Y_t)$.
    By construction \citep{lipman2023flow,albergo2023building,albergo2023stochastic}, we have that
    \begin{talign} 
    \begin{split} \label{eq:v}
        v(x,t) &= \mathbb{E}[ \dot{\alpha}_t Y_1 + \dot{\beta}_t Y_0 | x = \alpha_t Y_1 + \beta_t Y_0] \\ &= \mathbb{E}[ \frac{\dot{\alpha}_t(x - \beta_t Y_0)}{\alpha_t} + \dot{\beta}_t Y_0 | x = \alpha_t Y_1 + \beta_t Y_0] \\ &= \frac{\dot{\alpha}_t}{\alpha_t} x + (\dot{\beta}_t - \frac{\dot{\alpha}_t}{\alpha_t} \beta_t) \mathbb{E}[ Y_0 | x = \alpha_t Y_1 + \beta_t Y_0 ], 
    \end{split}
    \end{talign}
    where we used that 
    % $\vec{Y}_0 = (x - \beta_t \vec{Y}_1)/\alpha_t$.
    $Y_1 = (x - \beta_t Y_0)/\alpha_t$.
    Also, we can write the score as follows 
    \begin{talign} \label{eq:score}
        % \mathfrak{s}(x,t) \! := \! \nabla \log p_t(x) \! = \! \frac{\nabla p_t(x)}{p_t(x)} \! = \! \frac{\nabla \mathbb{E}[p_{t|1}(x|X_1)]}{p_t(x)} \! = \! \frac{\mathbb{E}[\nabla p_{t|1}(x|X_1)]}{p_t(x)} \! = \! \frac{\mathbb{E}[p_{t|1}(x|X_1) \nabla \log p_{t|1}(x|X_1)]}{p_t(x)},
        % \mathfrak{s}(x,t) \! := \! \nabla \log \vec{p}_t(x) \! = \! \frac{\nabla \vec{p}_t(x)}{\vec{p}_t(x)} \! = \! \frac{\nabla \mathbb{E}[\vec{p}_{t|0}(x|\vec{Y}_0)]}{\vec{p}_t(x)} \! = \! \frac{\mathbb{E}[\nabla \vec{p}_{t|0}(x|\vec{Y}_0)]}{\vec{p}_t(x)} \! = \! \frac{\mathbb{E}[\vec{p}_{t|0}(x|\vec{Y}_0) \nabla \log \vec{p}_{t|0}(x|\vec{X}_0)]}{\vec{p}_t(x)},
        \mathfrak{s}(x,t) \! := \! \nabla \log p_t(x) \! = \! \frac{\nabla p_t(x)}{p_t(x)} \! = \! \frac{\nabla \mathbb{E}[p_{t|1}(x|Y_1)]}{p_t(x)} \! = \! \frac{\mathbb{E}[\nabla p_{t|1}(x|Y_1)]}{p_t(x)} \! = \! \frac{\mathbb{E}[p_{t|1}(x|Y_1) \nabla \log p_{t|1}(x|Y_1)]}{p_t(x)},
    \end{talign}
    where
    \begin{talign}
        % p_t(x|X_1) = \frac{\exp(-\|x - \beta_t X_1\|^2/(2\alpha^2_t))}{(2\pi \alpha^2_t)^{d/2}} \implies \nabla \log p_t(x|X_1) = - \frac{x - \beta_t X_1}{\alpha^2_t} \\
        % \vec{p}_{t|0}(x|\vec{Y}_0) = \frac{\exp(-\|x - \alpha_t \vec{Y}_0\|^2/(2\beta^2_t))}{(2\pi \beta^2_t)^{d/2}} \implies \nabla \log \vec{p}_{t|0}(x|\vec{Y}_0) = - \frac{x - \alpha_t \vec{Y}_0}{\beta^2_t} 
        p_{t|1}(x|Y_1) = \frac{\exp(-\|x - \alpha_t Y_1\|^2/(2\beta^2_t))}{(2\pi \beta^2_t)^{d/2}} \implies \nabla \log \vec{p}_{t|1}(x|Y_1) = - \frac{x - \alpha_t Y_1}{\beta^2_t} 
    \end{talign} 
    Plugging this back into the right-hand side of \eqref{eq:score}, we obtain
    \begin{talign} 
    \begin{split} \label{eq:score2}
        % \mathfrak{s}(x,t) &= - \frac{\mathbb{E}[ \vec{p}_{t|0}(x|\vec{Y}_0) \frac{x - \alpha_t \vec{Y}_0}{\beta^2_t} ]}{\vec{p}_{t}(x)} = - \frac{\int \vec{p}_{t|0}(x|\vec{Y}_0) \vec{p}_{0}(\vec{Y}_0) \frac{x - \alpha_t \vec{Y}_0}{\beta^2_t} \, d\vec{Y}_0}{\vec{p}_{t}(x)} \\ &= - \int p_{0|t}(\vec{Y}_0|x) \frac{x - \alpha_t \vec{Y}_0}{\beta^2_t} \, d\vec{Y}_0 = - \mathbb{E}[ \frac{x - \alpha_t \vec{Y}_0}{\beta^2_t} | x = \alpha_t \vec{Y}_0 + \beta_t \vec{Y}_1 ] \\ &= - \frac{\mathbb{E}[ \vec{Y}_1 | x = \alpha_t \vec{Y}_0 + \beta_t \vec{Y}_1 ]}{\beta_t}
        \mathfrak{s}(x,t) &= - \frac{\mathbb{E}[ p_{t|1}(x|Y_1) \frac{x - \alpha_t Y_1}{\beta^2_t} ]}{p_{t}(x)} = - \frac{\int \vec{p}_{t|1}(x|Y_1) p_{1}(Y_1) \frac{x - \alpha_t Y_1}{\beta^2_t} \, dY_1}{\vec{p}_{t}(x)} \\ &= - \int p_{1|t}(Y_1|x) \frac{x - \alpha_t Y_1}{\beta^2_t} \, dY_1 = - \mathbb{E}[ \frac{x - \alpha_t Y_1}{\beta^2_t} | x = \alpha_t Y_1 + \beta_t Y_0 ] %\\ &
        = - \frac{\mathbb{E}[ Y_0 | x = \alpha_t Y_1 + \beta_t Y_0 ]}{\beta_t}
    \end{split}
    \end{talign}
    The last equality holds because $(x - \alpha_t Y_1)/\beta_t = Y_0$. Putting together \eqref{eq:v} and \eqref{eq:score2}, we obtain that
    \begin{talign} \label{eq:v_and_s}
        % \mathfrak{s}(x,t) = -\frac{1}{\beta_t(\dot{\beta}_t - \frac{\dot{\alpha}_t}{\alpha_t} \beta_t)} \big( v_t(x) - \frac{\dot{\alpha}_t}{\alpha_t} x \big) \iff v_t(x) = -\beta_t(\dot{\beta}_t - \frac{\dot{\alpha}_t}{\alpha_t} \beta_t) \mathfrak{s}(x,t) + \frac{\dot{\alpha}_t}{\alpha_t} x 
        % \iff 
        % v(x,t) = \beta_t(\dot{\beta}_t - \frac{\dot{\alpha}_t}{\alpha_t} \beta_t) \mathfrak{s}(x,t) - \frac{\dot{\alpha}_t}{\alpha_t} x \iff \mathfrak{s}(x,t) = \frac{1}{\beta_t(\dot{\beta}_t - \frac{\dot{\alpha}_t}{\alpha_t} \beta_t)} \big( v(x,t) + \frac{\dot{\alpha}_t}{\alpha_t} x \big)
        v(x,t) = \frac{\dot{\alpha}_t}{\alpha_t} x + \beta_t(\frac{\dot{\alpha}_t}{\alpha_t} \beta_t - \dot{\beta}_t) \mathfrak{s}(x,t) \iff \mathfrak{s}(x,t) = \frac{1}{\beta_t(\frac{\dot{\alpha}_t}{\alpha_t} \beta_t - \dot{\beta}_t)} \big( v(x,t) - \frac{\dot{\alpha}_t}{\alpha_t} x \big)
    \end{talign}
    Thus, the %forward 
    %backward 
    ODE \eqref{eq:FM_ode} 
    can be rewritten like this:
    \begin{talign} \label{eq:FM_ode_rewritten}
        % \frac{\mathrm{d}\vec{X}_t}{\mathrm{d}t} &= - \frac{\dot{\beta}_t}{\beta_t} \vec{X}_t + \alpha_t(\dot{\alpha}_t - \frac{\dot{\beta}_t}{\beta_t} \alpha_t) \mathfrak{s}(\vec{X}_t,t), \qquad \vec{X}_0 \sim p_1.
        \frac{\mathrm{d}X_t}{\mathrm{d}t} &= \frac{\dot{\alpha}_t}{\alpha_t} X_t + \beta_t(\frac{\dot{\alpha}_t}{\alpha_t} \beta_t - \dot{\beta}_t) \mathfrak{s}(X_t,t), \qquad X_0 \sim p_0.
    \end{talign}
    To allow for an arbitrary diffusion coefficient, we need to add a correction term to the drift:
    \begin{talign} \label{eq:SDE_FM_generic}
        % \mathrm{d}\vec{X}_t &= \big( - \frac{\dot{\beta}_t}{\beta_t} \vec{X}_t + (\alpha_t(\dot{\alpha}_t - \frac{\dot{\beta}_t}{\beta_t} \alpha_t) + \frac{\tilde{\sigma}(t)^2}{2}) \mathfrak{s}(\vec{X}_t,t) \big) \mathrm{d}t + \tilde{\sigma}(t) \mathrm{d}B_t, \qquad \vec{X}_0 \sim p_1.
        \mathrm{d}X_t &= \big(\frac{\dot{\alpha}_t}{\alpha_t} X_t + \big( \frac{\sigma(t)^2}{2} + \beta_t(\frac{\dot{\alpha}_t}{\alpha_t} \beta_t - \dot{\beta}_t) \big) \mathfrak{s}(X_t,t) \big) \mathrm{d}t + \sigma(t) \mathrm{d}B_t, \qquad X_0 \sim p_0.
    \end{talign}
    This can be easily shown by writing down the Fokker-Planck equations for \eqref{eq:FM_ode_rewritten} and \eqref{eq:SDE_FM_generic}, and observing that they are the same up to a cancellation of terms.
    Finally, if we plug the right-hand side of \eqref{eq:v_and_s} into \eqref{eq:SDE_FM_generic}, we obtain the SDE for Flow Matching with arbitrary noise schedule (equation  \eqref{eq:FM_general_diffusion_coeff}).

\section{Stochastic optimal control as maximum entropy RL in continuous space and time}
\label{subsec:max_ent_RL}

In this section, we bridge KL-regularized (or MaxEnt) reinforcement learning and stochastic optimal control. We show that when the action space is Euclidean and the transition probabilities are conditional Gaussians, taking the limit in which the stepsize goes to zero on the KL-regularized RL problem gives rise to the SOC problem.  A consequence of this connection is that all algorithms for KL-regularized RL admit an analog for diffusion fine-tuning. This is not novel, but it may be useful for researchers that are familiar with RL fine-tuning formulations. 

\Cref{subsec:proof_eq_cond_kl} is providing a more direct, rigorous, continuous-time connection between SOC and MaxEnt RL, as it shows that the expected control cost is equal to the KL divergence between the distributions over trajectories, conditioned on the starting points (see equation \eqref{eq:cond_kl}). 

\subsection{Maximum entropy RL}
% \textcolor{red}{Add connection to existing diffusion methods based on maxent rl}
Several diffusion fine-tuning methods \citep{black2024training,uehara2024finetuning} are based on KL-regularized RL, also known as maximum entropy RL, which we review in the following.
In the classical reinforcement learning (RL) setting, we have an agent that, starting from state $s_0 \sim p_0$, iteratively observes a state $s_k$, takes an action $a_k$ according to a policy $\pi(a_k;s_k,k)$ which leads to a new state $s_{k+1}$ according to a fixed transition probability $p(s_{k+1}|a_k,s_k)$, and obtains rewards $r_k(s_k,a_k)$. 
% $r_{k+1}(s_{k+1})$. 
This can be summarized into a trajectory $\tau = ((s_k,a_k))_{k=0}^{K}$. 
% $\tau = (s_k)_{k=0}^{K}$.
The goal is to optimize the policy $\pi$ in order to maximize the expected total reward, i.e. $\max_{\pi} \mathbb{E}_{\tau \sim \pi,p} [\sum_{k=0}^{K} r_k(s_k,a_k)
% r_k(s_k)
]$. 

% Without loss of generality and to simplify the exposition, 
% To simplify the exposition, we can merge the policy $\pi$ and the transition probability $p$ into the transition kernel $\pi(s_{k+1};s_k,k) := \sum_{a_k} \pi(a_k;s_k,k) p(s_{k+1}|a_k,s_k)$ by integrating out the actions. The resulting problem is $\max_{\pi} \mathbb{E}_{\tau \sim \pi} [\sum_{k=0}^{K} r_k(s_k)]$, where the transition kernel $\pi$ may be subject to certain constraints, by construction.

Maximum entropy RL (MaxEnt RL; \cite{ziebart2008maximum}) amounts to adding the entropy $H(\pi)$ of the %transition kernel 
policy $\pi(\cdot;s_k,k)$ to the reward for each step $k$, in order to encourage exploration and improve robustness to changes in the environment: $\max_{\pi} \mathbb{E}_{\tau \sim \pi,p} [\sum_{k=0}^{K} r_k(s_k,a_k)
% r_k(s_k) 
+ \sum_{k=0}^{K-1} H(\pi(\cdot;s_k,k))]$ \footnote{The entropy terms are usually multiplied by a factor to tune their magnitude, but one can equivalently rescale the rewards, which is why we do not add any factor.}. As a generalization, one can regularize using the negative KL divergence between $\pi(\cdot;s_k,k)$ and a base %transition kernel 
policy $\pi_{\mathrm{base}}(\cdot;s_k,k)$: 
\begin{talign} \label{eq:maxent_rl}
\max_{\pi} \mathbb{E}_{\tau \sim \pi,p} [\sum_{k=0}^{K} r_k(s_k,a_k) 
% r_k(s_k)
- \sum_{k=0}^{K-1} \mathrm{KL}(\pi(\cdot;s_k,k)||\pi_{\mathrm{base}}(\cdot;s_k,k))],
\end{talign}
which prevents the learned policy to deviate too much from the base policy. Each %transition kernel 
policy $\pi$ induces a distribution $q(\tau)$ over trajectories $\tau$, and the MaxEnt RL problem \eqref{eq:maxent_rl} can be expressed solely in terms of such distributions (\Cref{lem:KL_equality} in \Cref{sec:proofs_max_ent_RL}):
\begin{talign} \label{eq:distribution_maxent_RL_first}
    \max_{q} \mathbb{E}_{\tau \sim q}[\sum_{k=0}^{K} r_k(s_k,a_k)
    % r_k(s_k)
    ] - \mathrm{KL}(q||q^{\mathrm{base}}),
\end{talign}
where $q^{\mathrm{base}}$ is the distribution induced by the base %transition kernel 
policy $\pi_{\mathrm{base}}$, and the maximization is over all distributions $q$ such that their marginal for $s_0$ is $p_0$. We can further recast this problem as (\Cref{lem:q_star_lemma} in \Cref{sec:proofs_max_ent_RL}):
\begin{talign} \label{eq:distribution_maxent_RL}
    &\min_{q} \mathrm{KL}(q||q^*), \qquad \text{where} \ 
     q^*(\tau) := 
     % \frac{q^{\mathrm{base}}(\tau) \exp\big( \sum_{k=0}^{K} r_k(s_k,a_k) \big)}{\frac{1}{p_0(s_0)} \sum_{\{\tau' | s'_0 = s_0 \}} q^{\mathrm{base}}(\tau') \exp\big( \sum_{k=0}^{K} r_k(s'_k,a'_k) \big)}, 
    q^{\mathrm{base}}(\tau) \exp\big( \sum_{k=0}^{K} r_k(s_k,a_k) 
    % r_k(s_k) 
    - \mathcal{V}(s_0,0) \big), 
    % \\
    % \begin{split}
    % &\qquad\qquad \text{with} \ \mathcal{V}(s_0,0) = \log\big( \mathbb{E}_{\tau \sim \pi_{\mathrm{base}},p}[ \exp\big( \sum_{k=0}^{K} r_k(s_k,a_k) \big) | s_0] \big) \\ &\qquad\qquad\qquad\qquad\quad \ \ = \max_{\pi} \mathbb{E}_{\tau \sim \pi,p} \big[\sum_{k=0}^{K} r_{k}(s_{k},a_{k}) - \sum_{k=0}^{K-1} \mathrm{KL}(\pi(\cdot;s_{k},k)||\pi_{\mathrm{base}}(\cdot;s_{k},k)) | s_0 \big]
    % \end{split}
\end{talign}
where 
\begin{talign}
\begin{split} \label{eq:value_function_maxent_RL}
    \mathcal{V}(s_k,k) &:= \log\big( \mathbb{E}_{\tau \sim \pi_{\mathrm{base}},p}[ \exp\big( \sum_{k'=k}^{K} r_{k'}(s_{k'},a_{k'}) 
    % r_{k'}(s_{k'}) 
    \big) | s_k] \big) \\ &= \max_{\pi} \mathbb{E}_{\tau \sim \pi,p} \big[\sum_{k'=k}^{K} r_{k'}(s_{k'},a_{k'})
    % r_{k'}(s_{k'}) 
    - \sum_{k'=k}^{K-1} \mathrm{KL}(\pi(\cdot;s_{k'},k')||\pi_{\mathrm{base}}(\cdot;s_{k'},k')) | s_{k} \big]
\end{split}
\end{talign}
is the value function.
% which makes apparent 
Problem \eqref{eq:distribution_maxent_RL} directly implies that the distribution induced by the optimal policy $\pi^*$ is the tilted distribution $q^*$ (which has initial marginal $p_0$). 

\subsection{From maximum entropy RL to stochastic optimal control}
The following well-known result, which we prove in \Cref{sec:proofs_max_ent_RL}, shows that in a natural sense, the continuous-time continuous-space version of MaxEnt RL is the SOC
% stochastic optimal control 
framework introduced in \Cref{sec:SOC_formulation}. In particular, when states and actions are vectors in $\R^d$, policies are specified by a vector field $u$ (the control), and transition probabilities are conditional Gaussians, the MaxEnt RL problem becomes an SOC problem when the number of timesteps grows to infinity. 

\begin{proposition} \label{prop:max_ent_stochastic_optimal_control}
    Suppose that 
    \begin{enumerate}[label=(\roman*)]
    % \textit{(i)} 
    \item The state space and the action space are 
    $\R^d$,
    \item Policies $\pi$ are specified as $\pi(a_k;s_k,k) = \delta(a_k - u(s_k,kh))$, where $u : \R^d \times [0,T] \to \R^d$ is a vector field, and $\delta$ denotes the Dirac delta, 
    \item 
    Transition probabilities are conditional Gaussian densities: $p(s_{k+1}|a_k,s_k) = N(s_k + h (b(s_k,kh) + \sigma(kh) a_k), h 
    \sigma(kh)\sigma(kh)^{\top})$, where $h  = T/K$ is the stepsize, and $b$ and $\sigma$ are defined as in \Cref{sec:SOC_formulation}. 
    % \item Transition kernels $\pi$ are restricted to be of the form 
    % \begin{talign}
    % \pi(s_{k+1}|s_k,k) = N(s_k + h (b(s_k,kh) + \sigma(kh) u(s_k,kh)), h \sigma(kh)\sigma(kh)^{\top}),
    % \end{talign}
    % where $h  = T/K$ is the stepsize, and $b$, $u$ and $\sigma$ are fixed and defined as in \Cref{sec:soc_problem}. That is, each vector field $u$ is associated to a transition kernel $\pi$. The base transition kernel is the one for which $u \equiv 0$.
    \end{enumerate}
    Then, in the limit in which the number of steps $K$ grows to infinity, the problem \eqref{eq:maxent_rl} is equivalent to the SOC
    % stochastic optimal control 
    problem \eqref{eq:control_problem_def}-\eqref{eq:controlled_SDE}, identifying
    \begin{itemize}
    \item the sequence of states $(s_k)_{k=0}^{k}$ with the trajectory $\bm{X}^u = (X^u_t)_{t \in [0,1]}$,
    \item the running reward $\sum_{k=0}^{K-1} r_k(s_k,a_k)$ with the negative running cost $- \int_0^T f(X^u_t,t) \, \mathrm{d}t$, 
    \item the terminal reward $r_K(s_K,a_K)$ with the negative terminal cost $- 
    g(X^u_T)$, 
    \item the KL regularization $\mathbb{E}_{\tau \sim \pi,p}[\sum_{k=0}^{K-1} \mathrm{KL}(\pi(\cdot;s_k,k)||\pi_{\mathrm{base}}(\cdot;s_k,k))]$ with $\frac{1}{2}$ times the expected $L^2$ norm of the control $\frac{1}{2} \mathbb{E} \big[ \int_0^T 
   \|u(X^u_t,t)\|^2 \, \mathrm{d}t \big]$,
   \item and the value function $\mathcal{V}(s_k,k)$ defined in \eqref{eq:value_function_maxent_RL} with the negative value function $-V(x,t)$ defined in \Cref{sec:SOC_formulation}. 
   \end{itemize}
\end{proposition}

A first consequence of this result is that every loss function designed for generic MaxEnt RL problems has a corresponding loss function for SOC
% stochastic optimal control 
problems. The geometric structure of the latter allows for additional losses that do not have an analog in the classical MaxEnt RL setting; in particular, we can differentiate the state and terminal costs. 
% We provide a brief overview of relevant existing losses for SOC
% % stochastic optimal control 
% problems in \Cref{subsec:existing_losses} and \Cref{sec:more_losses}. %(\textcolor{red}{check}).

%  If we translate this result to the SOC
% % stochastic optimal control 
% setting,
% \subsection{A first approach to diffusion fine-tuning}

A second consequence of \Cref{prop:max_ent_stochastic_optimal_control} is that the characterization \eqref{eq:distribution_maxent_RL} can be translated to the SOC setting. The analogs of the distributions $q^{*}$, $q^{\mathrm{base}}$ induced by the optimal policy $\pi^{*}$ and the base policy $\pi^{\mathrm{base}}$ are the distributions 
% $\mathbb{P}^{*}$, $\mathbb{P}^{\mathrm{base}}$
$p^*, p^{\mathrm{base}}$ induced by the optimal control $u^*$ and the null control.
% \footnote{
% \textcolor{red}{explain what $\mathbb{P}^{*}$, $\mathbb{P}^{\mathrm{base}}$ mean in plain English}
% $\mathbb{P}^{*}$, $\mathbb{P}^{\mathrm{base}}$ are the limiting objects of the joint distributions $q^{*}$, $q^{\mathrm{base}}$ when the stepsize $h$ goes to zero.
% The statement in 
% }. 
For an arbitrary trajectory $\bm{X} = (X_t)_{t\in[0,T]}$, the 
% ratio between $\mathbb{P}^{*}$ and $\mathbb{P}^{\mathrm{base}}$ is given by:
relation between $\mathbb{P}^{*}$ and $\mathbb{P}^{\mathrm{base}}$ is given by
% \footnote{Equation \eqref{eq:optimal_distribution_SOC} is informal because densities over continuous-time processes are ill-defined; the formal statement is $\frac{\mathrm{d}\mathbb{P}^{*}}{\mathrm{d}\mathbb{P}^{\mathrm{base}}}(\bm{X}) = \exp ( - \int_0^T f(X_t,t) \, \mathrm{d}t - g(X_T) + V(X_0,0))$, where $\frac{\mathrm{d}\mathbb{P}^{*}}{\mathrm{d}\mathbb{P}^{\mathrm{base}}}$ is the Radon-Nikodym derivative. This is the notation used in the proofs.}:
% We obtain that the distribution over trajectories $\mathbb{P}^{*}$ induced by the optimal control $u^*$ can be expressed in terms of the distribution $\mathbb{P}^{\mathrm{base}}$ of the uncontrolled process: %and the state and terminal costs:
\begin{talign}
%     % \frac{\mathrm{d}\mathbb{P}^{*}}{\mathrm{d}\mathbb{P}^{\mathrm{base}}}(X_{0:T})
%     % \mathbb{P}^{*}
%     p^{*}(\bm{X}) = 
%     % \mathbb{P}^{\mathrm{base}}
%     p^{\mathrm{base}}(\bm{X}) \exp \big( - \int_0^T f(X_t,t) \, \mathrm{d}t - g(X_T) + V(X_0,0) \big), 
    \frac{\mathrm{d}\mathbb{P}^{*}}{\mathrm{d}\mathbb{P}^{\mathrm{base}}}(\bm{X}) = \exp ( - \int_0^T f(X_t,t) \, \mathrm{d}t - g(X_T) + V(X_0,0))
\end{talign}
% \textcolor{red}{Add footnote that it is more formal in appendix.}
where $V$ is the value function as defined in \Cref{sec:SOC_formulation}. Note that this matches the statement in \eqref{eq:optimal_distribution_SOC}.
% As we show in \Cref{subsec:diffusion_SOC}, this characterization of $\mathbb{P}^*$ is critical to connect diffusion fine-tuning to stochastic optimal control.

\subsection{Proof of \Cref{prop:max_ent_stochastic_optimal_control}: from MaxEnt RL to SOC} \label{sec:proofs_max_ent_RL}

% \subsection{Proof of \Cref{prop:max_ent_stochastic_optimal_control}} \label{subsec:max_ent}

Since the transition $p(s_{k+1}|a_k,s_k)$ is fixed, for each $\pi$ we can define 
\begin{talign} \label{eq:tilde_pi_def}
\tilde{\pi}(a_k,s_{k+1};s_k,k) = \pi(a_k;s_k,k) p(s_{k+1}|a_k,s_k) \ \text{and} \ \tilde{\pi}_{\mathrm{base}}(a_k,s_{k+1};s_k,k) = \pi_{\mathrm{base}}(a_k;s_k,k) p(s_{k+1}|a_k,s_k), 
\end{talign}
and reexpress \eqref{eq:maxent_rl} as (see \Cref{lem:KL_equality})
\begin{talign} \label{eq:loss_maxent_RL}
    \min_{\tilde{\pi}} \mathbb{E}_{\tau \sim \tilde{\pi}} [\sum_{k=0}^{K} r_k(s_k,a_k) - \sum_{k=0}^{K-1} \mathrm{KL}(\tilde{\pi}(\cdot,\cdot;s_k,k)||\tilde{\pi}_{\mathrm{base}}(\cdot,\cdot;s_k,k))].
\end{talign}
% When the state space is the Euclidean space $\R^d$, it is natural to set actions to be vectors in $\R^d$, and transition probabilities to be conditional Gaussian densities. In other words, we can uniquely specify a policy $\pi(a_k;s_k,k) = \delta(a_k - u(s_k,k\eta))$ through a vector field $u : \R^d \times [0,T] \to \R^d$ ($\delta$ is a Dirac delta), and we can assume that $p(s_{k+1}|a_k,s_k) = N(s_k + \eta (b(s_k,k\eta) + \sigma(k\eta) a_k), \eta \lambda \sigma(k\eta)\sigma(k\eta)^{\top})$, where $\eta  = T/K$ is the stepsize, and $b$ and $\sigma$ are defined as in \Cref{sec:soc_problem}. Thus, 
Using the hypothesis of the proposition, we can write
\begin{talign}
\begin{split}
\tilde{\pi}(a_k,s_{k+1};s_k,k) &= \delta(a_k - u(s_k,k\eta)) N(s_k + \eta (b(s_k,k\eta) + \sigma(k\eta) a_k), \eta \sigma(k\eta)\sigma(k\eta)^{\top}) \\ &= \delta(a_k - u(s_k,k\eta)) \tilde{\pi}(s_{k+1};s_k,k), 
\end{split}
\end{talign}
where $\tilde{\pi}(s_{k+1};s_k,k) = N(s_k + \eta (b(s_k,k\eta) + \sigma(k\eta) u(s_k,k\eta)), \eta \sigma(k\eta)\sigma(k\eta)^{\top})$ is the state transition kernel. We set the base policy as $\pi_{\mathrm{base}}(a_k;s_k,k) = \delta(a_k)$, and we obtain analogously that $\tilde{\pi}(a_k,s_{k+1};s_k,k) = \delta(a_k) \tilde{\pi}_{\mathrm{base}}(s_{k+1};s_k,k)$ with $\tilde{\pi}_{\mathrm{base}}(s_{k+1};s_k,k) = N(s_k + \eta b(s_k,k\eta), \eta \sigma(k\eta)\sigma(k\eta)^{\top})$. Now, if we take $K$ large, the trajectory $(s_k)_{k=0}^{K}$ generated by $\tilde{\pi}$ can be regarded as the Euler-Maruyama discretization of a solution $X^u$ of the controlled SDE \eqref{eq:controlled_SDE}, while the trajectory generated by $\tilde{\pi}_{\mathrm{base}}$ is the discretization of the uncontrolled process $X^0$ obtained by setting $u = 0$. As a consequence
\begin{talign}
\begin{split}
    &\lim_{K \to \infty} \mathbb{E}_{\tau \sim \tilde{\pi}} [\sum_{k=0}^{K-1}\mathrm{KL}(\tilde{\pi}(\cdot,\cdot;s_k,k)||\tilde{\pi}_{\mathrm{base}}(\cdot,\cdot;s_k,k))] \\ &= \lim_{K \to \infty} \mathbb{E}_{\tau \sim \tilde{\pi}} [\sum_{k=0}^{K-1}\mathrm{KL}(\tilde{\pi}(\cdot;s_k,k)||\tilde{\pi}_{\mathrm{base}}(\cdot;s_k,k))] = \mathbb{E}_{X^u \sim \mathbb{P}^{u}} [ \log \frac{\mathrm{d}\mathbb{P}^{u}}{\mathrm{d}\mathbb{P}^{0}}(X^u)],
\end{split}
\end{talign}
where $\mathbb{P}^{u}$ and $\mathbb{P}^{0}$ are the measures of the processes $X^u$ and $X^0$, respectively. The Girsanov theorem (\Cref{cor:girsanov_sdes}) implies that $\log \frac{\mathrm{d}\mathbb{P}^{u}}{\mathrm{d}\mathbb{P}^{0}}(X^u) = - \int_0^T \langle u(X^u_t,t), \mathrm{d}B_t \rangle - \frac{1}{2} \int_0^T \|u(X^u_t,t)\|^2 \, \mathrm{d}t$, which implies that $\mathbb{E}_{X^u \sim \mathbb{P}^{u}}[\log \frac{\mathrm{d}\mathbb{P}^{u}}{\mathrm{d}\mathbb{P}^{0}}(X^u)] = - \frac{1}{2} \mathbb{E}_{X^u \sim \mathbb{P}^{u}}[\int_0^T \|u(X^u_t,t)\|^2 \, \mathrm{d}t]$. Setting the rewards $r_k(a_k,s_k) = \eta f(s_k, k\eta)$ for $k \in \{0,\dots,K-1\}$ and $r_{K}(a_K,s_K) = \eta g(s_k)$, where $f$ and $g$ are as in \Cref{sec:SOC_formulation}, yields the following limiting object: 
\begin{talign}
\lim_{K \to \infty} \mathbb{E}_{\tau \sim \tilde{\pi}} [\sum_{k=0}^{K} r_k(s_k,a_k)] = \mathbb{E}_{X^u \sim \mathbb{P}^{u}}[\int_0^T f(X^u_t,t) \, dt + g(X^u_T)].
\end{talign}
Hence, the limit of the MaxEnt RL loss \eqref{eq:loss_maxent_RL} is the %stochastic optimal control 
SOC loss \eqref{eq:control_problem_def}. 

\begin{lemma} \label{lem:KL_equality}
Let $\tilde{\pi}(a_k,s_{k+1};s_k,k)$ and $\tilde{\pi}_{\mathrm{base}}(a_k,s_{k+1};s_k,k)$ be as defined in \eqref{eq:tilde_pi_def}. $\mathrm{KL}(\tilde{\pi}(\cdot,\cdot;s_k,k)||\tilde{\pi}_{\mathrm{base}}(\cdot,\cdot;s_k,k))]$ and $\mathrm{KL}(\pi(\cdot;s_k,k)||\pi_{\mathrm{base}}(\cdot;s_k,k))]$ are equal. Moreover, if $q$, $q^{\mathrm{base}}$ denote the distributions over trajectories induced by $\pi$, $\pi_{\mathrm{base}}$, we have that
\begin{talign} \label{eq:equality_KL}
    \mathrm{KL}(q||q^{\mathrm{base}}) = \mathbb{E}[\sum_{k=0}^{K-1} \mathrm{KL}(\pi(\cdot;s_k,k)||\pi_{\mathrm{base}}(\cdot;s_k,k))].
\end{talign}
\end{lemma}
\begin{proof}
    We have that 
    \begin{talign}
    \begin{split}
        &\mathrm{KL}(\tilde{\pi}(\cdot,\cdot;s_k,k)||\tilde{\pi}_{\mathrm{base}}(\cdot,\cdot;s_k,k))] = \sum_{a_k,s_{k+1}} \tilde{\pi}(a_k,s_{k+1};s_k,k) \log \frac{\tilde{\pi}(a_k,s_{k+1};s_k,k)}{\tilde{\pi}_{\mathrm{base}}(a_k,s_{k+1};s_k,k)} \\ &= \sum_{a_k,s_{k+1}} \pi(a_k;s_k,k) p(s_{k+1}|a_k,s_k) \log \frac{\pi(a_k;s_k,k) p(s_{k+1}|a_k,s_k)}{\pi_{\mathrm{base}}(a_k;s_k,k) p(s_{k+1}|a_k,s_k)} \\ &= \sum_{a_k,s_{k+1}} \pi(a_k;s_k,k) p(s_{k+1}|a_k,s_k) \log \frac{\pi(a_k;s_k,k)}{\pi_{\mathrm{base}}(a_k;s_k,k)} \\ &= \sum_{a_k} \pi(a_k;s_k,k) \big( \sum_{s_{k+1}} p(s_{k+1}|a_k,s_k) \big) \log \frac{\pi(a_k;s_k,k)}{\pi_{\mathrm{base}}(a_k;s_k,k)} \\ &= \sum_{a_k} \pi(a_k;s_k,k) \log \frac{\pi(a_k;s_k,k)}{\pi_{\mathrm{base}}(a_k;s_k,k)} = \mathrm{KL}(\pi(\cdot;s_k,k)||\pi_{\mathrm{base}}(\cdot;s_k,k))].
    \end{split}
    \end{talign}
    To prove \eqref{eq:equality_KL}, by construction we can write
    \begin{talign}
        q(\tau) = p_0(s_0) \prod_{k=0}^{K-1} \tilde{\pi}(a_k,s_{k+1};s_k,k), \qquad\qquad q^{\mathrm{base}}(\tau) = p_0(s_0) \prod_{k=0}^{K-1} \tilde{\pi}_{\mathrm{base}}(a_k,s_{k+1};s_k,k),
    \end{talign}
    which means that 
    \begin{talign}
    \begin{split}
        \mathrm{KL}(q||q^{\mathrm{base}}) &= \mathbb{E}_{\tau \sim q}[\log \frac{q(\tau)}{q^{\mathrm{base}}(\tau)}] = \mathbb{E}_{\tau \sim q}[\sum_{k=0}^{K-1} \log \frac{\tilde{\pi}(a_k,s_{k+1};s_k,k)}{\tilde{\pi}_{\mathrm{base}}(a_k,s_{k+1};s_k,k)}]
        \\ &= \sum_{k=0}^{K-1} \mathbb{E}_{\tau \sim q^{0:(k+1)}}[\log \frac{\tilde{\pi}(a_k,s_{k+1};s_k,k)}{\tilde{\pi}_{\mathrm{base}}(a_k,s_{k+1};s_k,k)}] \\ &= \sum_{k=0}^{K-1} \mathbb{E}_{\tau \sim q^{0:k}}[ \sum_{a_k, s_{k+1}} \tilde{\pi}(a_k,s_{k+1};s_k,k) \log \frac{\tilde{\pi}(a_k,s_{k+1};s_k,k)}{\tilde{\pi}_{\mathrm{base}}(a_k,s_{k+1};s_k,k)}] \\ &= \sum_{k=0}^{K-1} \mathbb{E}_{\tau \sim q^{0:k}}[ \mathrm{KL}(\tilde{\pi}(\cdot,\cdot;s_k,k)||\tilde{\pi}_{\mathrm{base}}(\cdot,\cdot;s_k,k))]
        \\ &= \sum_{k=0}^{K-1} \mathbb{E}_{\tau \sim q^{0:k}}[ \mathrm{KL}(\pi(\cdot;s_k,k)||\pi_{\mathrm{base}}(\cdot;s_k,k))]
        \\ &= \mathbb{E}_{\tau \sim q^{0:k}}[\sum_{k=0}^{K-1} \mathrm{KL}(\pi(\cdot;s_k,k)||\pi_{\mathrm{base}}(\cdot;s_k,k))]
    \end{split}
    \end{talign}
    Here, the notation $q^{0:k}$ denotes the trajectory $q$ up to the state $s_{k}$.
\end{proof}

\begin{lemma} \label{lem:q_star_lemma}
    The distribution-based MaxEnt RL formulation in \eqref{eq:distribution_maxent_RL_first} is equivalent to the the following problem:
    \begin{talign} \label{eq:equivalent_max_ent}
        \min_{q} \mathrm{KL}(q||q^*), \qquad \text{where} \ q^*(\tau) := \frac{q^{\mathrm{base}}(\tau) \exp\big( \sum_{k=0}^{K} r_k(s_k,a_k) 
        % r_k(s_k) 
        \big)}{\frac{1}{p_0(s_0)} \sum_{\{\tau' | s'_0 = s_0 \}} q^{\mathrm{base}}(\tau') \exp\big( \sum_{k=0}^{K} r_k(s'_k,a'_k) 
        % r_k(s'_k) 
        \big)},
    \end{talign}
    where the minimization is over $q$ with marginal $p_0$ at step zero. The optimum of the problem is $q^*$, which satisfies the marginal constraint. The following alternative characterization of $q^*$ holds:
    \begin{talign}
        q^*(\tau) &= q^{\mathrm{base}}(\tau) \exp\big( \sum_{k=0}^{K} r_k(s_k,a_k) 
        % r_k(s_k) 
        - \mathcal{V}(s_0,0) \big), \\
        \text{where} \ \mathcal{V}(x,k) &= \max_{\pi} \mathbb{E}_{\tau \sim \pi,p} \big[\sum_{k'=k}^{K} r_{k'}(s_{k'},a_{k'})
        % r_{k'}(s_{k'}) 
        - \sum_{k'=k}^{K-1} \mathrm{KL}(\pi(\cdot;s_{k'},k')||\pi_{\mathrm{base}}(\cdot;s_{k'},k')) | s_k = x \big].
    \end{talign}
\end{lemma}
\begin{proof}
    Let us expand $\mathrm{KL}(q||q^*)$:
    \begin{talign}
    \begin{split} \label{eq:KL_q_q_star}
        \mathrm{KL}(q||q^*) &= \mathbb{E}_{\tau \sim q} \big[ \log \frac{q(\tau)}{q^*(\tau)} \big] \\ &= \mathbb{E}_{\tau \sim q} \big[ \log q(\tau) - \log q^{\mathrm{base}}(\tau) - \sum_{k=0}^{K} 
        r_k(s_k,a_k)
        % r_k(s_k) 
        \\ &\qquad\qquad + \log \big( \frac{1}{p_0(s_0)} \sum_{\{\tau' | s'_0 = s_0 \}} q^{\mathrm{base}}(\tau') \exp\big( \sum_{k=0}^{K} r_k(s'_k,a'_k)
        % r_k(s'_k) 
        \big) \big) \big] \\
        &= \mathrm{KL}(q || q^{\mathrm{base}}) - \mathbb{E}_{\tau \sim q} \big[ \sum_{k=0}^{K} r_k(s_k,a_k) 
        % r_k(s_k) 
        \big] \\ &\qquad + \mathbb{E}_{s_0 \sim p_0} \big[ \log \big( \frac{1}{p_0(s_0)} \sum_{\{\tau' | s'_0 = s_0 \}} q^{\mathrm{base}}(\tau') \exp\big( \sum_{k=0}^{K} r_k(s'_k,a'_k)
        % r_k(s'_k) 
        \big) \big) \big],
    \end{split}
    \end{talign}
    where the third equality holds because the marginal of $q$ at step zero is $p_0$ by hypothesis. Since the third term in the right-hand side is independent of $q$, this proves the equivalence between \eqref{eq:distribution_maxent_RL_first} and \eqref{eq:equivalent_max_ent}.
    
    Next, we prove that the marginal of $q^*$ at step zero is $p_0$:
    \begin{talign}
        \sum_{\{\tau | s_0 = x \}} q^*(\tau) := \sum_{\{\tau | s_0 = x \}} \frac{q^{\mathrm{base}}(\tau) \exp\big( \sum_{k=0}^{K} r_k(s_k,a_k) 
        % r_k(s_k) 
        \big)}{ \frac{1}{p_0(x)} \sum_{\{\tau' | s'_0 = x \}} q^{\mathrm{base}}(\tau') \exp\big( \sum_{k=0}^{K} r_k(s'_k,a'_k)
        % r_k(s'_k) 
        \big)} = p_0(x).
        % \sum_{\{\tau | s_0 = x \}} q^*(\tau) := \sum_{\{\tau | s_0 = x \}} \frac{q^{\mathrm{base}}(\tau) \exp\big( \sum_{k=0}^{K} r_k(s_k,a_k) \big)}{\sum_{\tau'} q^{\mathrm{base}}(\tau') \exp\big( \sum_{k=0}^{K} r_k(s'_k,a'_k) \big)}
    \end{talign}

    Now, for an arbitrary $s_0$, let $q_{s_0}$, $q^*_{s_0}$ be the distributions $q$, $q^*$ conditioned on the initial state being $s_0$. We can write an analog to equation \eqref{eq:KL_q_q_star} for $q_{s_0}$, $q^*_{s_0}$:
    \begin{talign}
    \begin{split}
        \mathrm{KL}(q_{s_0}||q^*_{s_0}) &= \mathbb{E}_{\tau \sim q_{s_0}} \big[ \log \frac{q_{s_0}(\tau)}{q^*_{s_0}(\tau)} \big] \\ &= \mathbb{E}_{\tau \sim q_{s_0}} \big[ \log q_{s_0}(\tau) - \log q^{\mathrm{base}}_{s_0}(\tau) - \sum_{k=0}^{K} r_k(s_k,a_k)
        % r_k(s_k) 
        \\ &\qquad\qquad + \log \big( \frac{1}{p_0(s_0)} \sum_{\{\tau' | s'_0 = s_0 \}} q^{\mathrm{base}}_{s_0}(\tau') \exp\big( \sum_{k=0}^{K} r_k(s'_k,a'_k)
        % r_k(s'_k) 
        \big) \big) \big] \\
        &= \mathrm{KL}(q_{s_0} || q^{\mathrm{base}}_{s_0}) - \mathbb{E}_{\tau \sim q_{s_0}} \big[ 
        \sum_{k=0}^{K} r_k(s_k,a_k)
        % r_k(s_k) 
        \big] \\ &\qquad + 
        % \mathbb{E}_{s_0 \sim p_0} \big[ 
        \log \big( \frac{1}{p_0(s_0)} \sum_{\{\tau' | s'_0 = s_0 \}} q^{\mathrm{base}}(\tau') \exp\big( \sum_{k=0}^{K} r_k(s'_k,a'_k)
        % r_k(s'_k) 
        \big) \big), 
        % \big],
    \end{split}
    \end{talign}
    Hence,
    \begin{talign}
    \begin{split}
        0 = \min_{q_{s_0}} \mathrm{KL}(q_{s_0}||q^*_{s_0}) &= - \max_{q_{s_0}} \{ \mathbb{E}_{\tau \sim q_{s_0}} \big[ 
        \sum_{k=0}^{K} r_k(s_k,a_k) \big] - \mathrm{KL}(q_{s_0} || q^{\mathrm{base}}_{s_0}) \} \\ &+ \log \big( \frac{1}{p_0(s_0)} \sum_{\{\tau' | s'_0 = s_0 \}} q^{\mathrm{base}}(\tau') \exp\big( \sum_{k=0}^{K} r_k(s'_k,a'_k) \big) \big).
    \end{split}
    \end{talign}
    And applying \eqref{eq:equality_KL} from \eqref{eq:equality_KL}, we obtain that
    \begin{talign}
    \begin{split}
        &\log \big( \frac{1}{p_0(s_0)} \sum_{\{\tau' | s'_0 = s_0 \}} q^{\mathrm{base}}(\tau') \exp\big( \sum_{k=0}^{K} r_k(s'_k,a'_k) \big) \big) \\ &= \max_{\pi} \mathbb{E}_{\tau \sim \pi,p} \big[\sum_{k=0}^{K} r_{k}(s_{k},a_{k}) - \sum_{k=0}^{K-1} \mathrm{KL}(\pi(\cdot;s_{k},k)||\pi_{\mathrm{base}}(\cdot;s_{k},k)) | s_0 \big] = \mathcal{V}(s_0,0),
    \end{split}
    \end{talign}
    which concludes the proof.
\end{proof}

\subsection{Proof of equation \eqref{eq:cond_kl}: the control cost is a KL regularizer}
\label{subsec:proof_eq_cond_kl}

\begin{theorem}[Girsanov theorem for SDEs] 
\label{cor:girsanov_sdes}
    If the two SDEs
    \begin{talign}
    \mathrm{d}X_{t} &= b_1 (X_{t},t) \, \mathrm{d}t + \sigma (X_{t},t) \, \mathrm{d}B_{t}, \qquad X_0 = x_{\mathrm{init}}  \\
    dY_{t} &= (b_1 (Y_{t},t) + b_2 (Y_{t},t)) \, \mathrm{d}t + \sigma (Y_{t},t) \, \mathrm{d}B_{t}, \qquad Y_0 = x_{\mathrm{init}}
    \end{talign}
    admit unique strong solutions on $[0,T]$, then for any bounded continuous functional $\Phi$ on $C([0,T])$, we have that
    \begin{talign} 
    \begin{split} \label{eq:X_to_Y}
        \mathbb{E}[\Phi(\bm{X})] &= \mathbb{E}\big[ \Phi(\bm{Y}) \exp \big( - \int_0^T \sigma(Y_{t},t)^{-1} b_2 (Y_{t},t) \, \mathrm{d}B_t - \frac{1}{2} \int_0^T \|\sigma(Y_{t},t)^{-1} b_2 (Y_{t},t)\|^2 \, \mathrm{d}t \big) \big] \\ &= \mathbb{E}\big[ \Phi(\bm{Y}) \exp \big( - \int_0^T \sigma(Y_{t},t)^{-1} b_2 (Y_{t},t) \, d\tilde{B}_t + \frac{1}{2} \int_0^T \|\sigma(Y_{t},t)^{-1} b_2 (Y_{t},t)\|^2 \, \mathrm{d}t \big) \big], %\label{eq:X_to_Y_2}
    \end{split}
    \end{talign}
    where $\tilde{B}_t = B_t + \int_0^t \sigma(Y_{s},s)^{-1} b_2 (Y_{s},s) \, \mathrm{d}s$. More generally, $b_1$ and $b_2$ can be random processes that are adapted to filtration of $\bm{B}$.
\end{theorem}

Consider the SDEs
\begin{talign} \label{eq:X_app}
    \mathrm{d}X_t &=  b(X_t,t) \, \mathrm{d}t + 
    \sigma(t) \mathrm{d}B_t, \qquad &X_0 = x_0, \\
    \mathrm{d}X^u_t &=  \left( b(X^u_t,t) + \sigma(t) u(X^u_t,t) \right) \, \mathrm{d}t + 
    \sigma(t) \mathrm{d}B_t, \qquad &X^u_0 = x_0.
    \label{eq:X_u_app}
\end{talign}
If we let $\mathbb{P} \rvert_{x_0}$, $\mathbb{P}^u \rvert_{x_0}$ be the probability measures of the solutions of \eqref{eq:X_app} and \eqref{eq:X_u_app}, \Cref{cor:girsanov_sdes} implies that
\begin{talign}
    \log \frac{\mathrm{d}\mathbb{P} \rvert_{x_0}}{\mathrm{d}\mathbb{P}^u \rvert_{x_0}} (\bm{X}^u) = - \int_0^1 u(X^u_t,t) \, \mathrm{d}B_t - \frac{1}{2} \int_0^1 \| u(X^u_t,t) \|^2 \, \mathrm{d}t.
\end{talign}
Hence, 
\begin{talign}
\begin{split}
    \infdiv*{\mathbb{P}^u \rvert_{x_0}}{\mathbb{P} \rvert_{x_0}} &= \mathbb{E} \big[\log \frac{\mathrm{d}\mathbb{P}^u \rvert_{x_0}}{\mathrm{d}\mathbb{P} \rvert_{x_0}}(\bm{X}^u) | X^u_0 = x_0 \big] = - \mathbb{E} \big[\log \frac{\mathrm{d}\mathbb{P} \rvert_{x_0}}{\mathrm{d}\mathbb{P}^u \rvert_{x_0}} (\bm{X}^u) | X^u_0 = x_0 \big] \\ &= \mathbb{E} \big[ \int_0^1 u(X^u_t,t) \, \mathrm{d}B_t + \frac{1}{2} \int_0^1 \| u(X^u_t,t) \|^2 \, \mathrm{d}t | X^u_0 = x_0 \big] = \mathbb{E} \big[ \frac{1}{2} \int_0^1 \| u(X^u_t,t) \|^2 \, \mathrm{d}t | X^u_0 = x_0 \big],
\end{split}
\end{talign}
where we used that stochastic integrals are martingales.

% \section{Proofs of \Cref{subsec:diffusion_SOC}}
\section{Proofs of \Cref{sec:memoryless_schedule}: memoryless noise schedule and fine-tuning recipe}

\subsection{Proof of %\Cref{prop:memorylessness_property}: 
\Cref{prop:memorylessness_noise_schedule}:
the memoryless noise schedule} \label{subsec:proof_memoryless}
% \begin{proposition}
%     Suppose that $b(x,t) = \kappa_t x + \big( \eta_t - \frac{\tilde{\sigma}(t)^2}{2} \big) \mathfrak{s}(x,t)$, and that $\sigma(t) = \sqrt{2 \eta_t}$. Then, the forward process $\vec{X}$ satisfies that $\vec{X}_0$ and $\vec{X}_1$ are independent. Similarly, the backward process $\cev{X}$ satisfies that $\cev{X}_0$ and $\cev{X}_1$ are independent.
% \end{proposition}
% \begin{proof}
    % Under this choice of $\sigma$, the forward process reads
    % The pair of forward and backward SDEs introduced in \Cref{subsec:hat_epsilon_score} can be generalized to arbitrary $\kappa_t$ and $\eta_t$:
    % Consider the pair of forward and backward SDEs:
    % \begin{talign}
    %     \mathrm{d}\vec{X}_t &= - \kappa_t \vec{X}_t \, \mathrm{d}t + \sqrt{2 \eta_t} \, \mathrm{d}B_t, \qquad \vec{X}_0 \sim p_{\mathrm{data}}, \\
    %     \mathrm{d}X_t &= \big( \kappa_t X_t + \big( \frac{\sigma(t)^2}{2} + \eta_t \big) \mathfrak{s}(X_t,t) \big) \mathrm{d}t + \sqrt{2 \eta_t} \, \mathrm{d}B_t, \qquad X_0 \sim N(0,I),
    % \end{talign}
    We consider the forward-backward SDEs \eqref{eq:forward_arbitrary}-\eqref{eq:backward_arbitrary} with arbitrary noise schedule. By \Cref{lem:equal_process_distributions}, the trajectories $\vec{\bm{X}}$, $\bm{X}$ of these two processes are equally distributed up to a time flip, which also means that their marginals satisfy $\vec{p}_{t} = p_{1-t}$, for all $t \in [0,1]$.
    First, we develop an explicit expression for the score function $s(x,t) = \nabla \log p_t(x)$. By the properties of flow matching, we know that $p_t$ is the distribution of the interpolation variable $\bar{X}_t = \beta_t \bar{X}_0 + \alpha_t \bar{X}_1$, where $\bar{X}_0  \sim N(0,I), \bar{X}_1 \sim p^{\mathrm{data}}$ are independent. Thus, $\frac{\bar{X}_t - \alpha_t \bar{X}_1}{\beta_t} \sim N(0,\mathrm{I})$, which means that we can express the density $p_t$ as
    \begin{talign}
        p_t(x) = \int_{\mathbb{R}^d} \frac{\exp\big(-\frac{\|x-\alpha_t y\|^2}{2\beta_t^2}\big)}{(2\pi \beta_t^2)^{d/2}} p^{\mathrm{data}}(y) \, \mathrm{d}y. 
    \end{talign}
    Thus,
    \begin{talign}
        s(x,t) = \nabla \log p_t(x) = - \frac{x}{\beta_t^2} + \frac{\alpha_t}{\beta_t^2} \frac{\int_{\mathbb{R}^d} y \exp\big( - \frac{\|x-\alpha_t y\|^2}{2\beta_t^2}\big) p^{\mathrm{data}}(y) \, \mathrm{d}y}{\int_{\mathbb{R}^d} \exp\big( - \frac{\|x-\alpha_t y\|^2}{2\beta_t^2} \big) p^{\mathrm{data}}(y) \, \mathrm{d}y} := - \frac{x - \alpha_t \xi_t(x)}{\beta_t^2},
    \end{talign}
    where we defined
    \begin{talign}
        \xi_t(x) = \frac{\int_{\mathbb{R}^d} y \exp\big( - \frac{\|x-\alpha_t y\|^2}{2\beta_t^2}\big) p^{\mathrm{data}}(y) \, \mathrm{d}y}{\int_{\mathbb{R}^d} \exp\big( - \frac{\|x-\alpha_t y\|^2}{2\beta_t^2} \big) p^{\mathrm{data}}(y) \, \mathrm{d}y}.
    \end{talign}
    Hence, we can rewrite the forward SDE \eqref{eq:forward_arbitrary} as
    \begin{talign}
        \mathrm{d}\vec{X}_t &= \big( - \kappa_{1-t} \vec{X}_t - \big( \frac{\sigma(1-t)^2}{2} - \eta_{1-t} \big) \frac{\vec{X}_t - \alpha_{1-t} \xi_{1-t}(\vec{X}_{t})}{\beta_{1-t}^2} \big) \, \mathrm{d}t + \sigma(1-t) \, \mathrm{d}B_t, \qquad \vec{X}_0 \sim p_{\mathrm{data}}
    \end{talign}
    Hence, if we substitute $\kappa_{1-t} \gets \kappa_{1-t} + \frac{\sigma(1-t)^2 - 2\eta_{1-t}}{2\beta_{1-t}^2}$, $\xi_{1-t} \gets \frac{\alpha_{1-t}(\sigma(1-t)^2 - 2\eta_{1-t})}{2\beta_{1-t}^2} \xi_{1-t}(\vec{X}_t)$ (where we ignore the dependency on $\vec{X}_t$), $\sqrt{2 \eta_{1-t}} \gets \sigma(1-t)$, we can apply \Cref{lem:OU_process}, which yields
    \begin{talign}
    \begin{split} \label{eq:memoryless_solution}
        \vec{X}_t &= \vec{X}_0 \exp \big( - \int_0^t \big( \kappa_{1-s} + \frac{\sigma(1-s)^2 - 2\eta_{1-s}}{2\beta_{1-s}^2} \big) \, \mathrm{d}s \big) \\ &\qquad + \int_0^t \exp \big( - \int_{t'}^{t} \big( \kappa_{1-s} + \frac{\sigma(1-s)^2 - 2\eta_{1-s}}{2\beta_{1-s}^2} \big) \, \mathrm{d}s \big) \frac{\alpha_{1-t'}(\sigma(1-t')^2 - 2\eta_{1-t'})}{2\beta_{1-t'}^2} \xi_{1-t'}(\vec{X}_{t'}) \, \mathrm{d}t' \\ &\qquad + \int_0^t \sigma(1-t') \exp \big( - \int_{t'}^{t} \big( \kappa_{1-s} + \frac{\sigma(1-s)^2 - 2\eta_{1-s}}{2\beta_{1-s}^2} \big) \, \mathrm{d}s \big) \, \mathrm{d}B_{t'}.
    \end{split}
    \end{talign}
    We simplify the recurring expression: 
    \begin{talign}
        \kappa_{1-s} + \frac{\sigma(1-s)^2 - 2\eta_{1-s}}{2\beta_{1-s}^2} = \frac{\dot{\alpha}_{1-s}}{\alpha_{1-s}} + \frac{\sigma(1-s)^2 - 2 \beta_{1-s} \big( \frac{\dot{\alpha}_{1-s}}{\alpha_{1-s}} \beta_{1-s} - \dot{\beta}_{1-s} \big)}{2 \beta_{1-s}^2} = \frac{\sigma(1-s)^2}{2 \beta_{1-s}^2} + \frac{\dot{\beta}_{1-s}}{\beta_{1-s}}
    \end{talign}
    Thus,
    \begin{talign}
        \int_{t'}^t \big( \kappa_{1-s} + \frac{\sigma(1-s)^2 - 2\eta_{1-s}}{2\beta_{1-s}^2} \big) \, \mathrm{d}s = \int_{t'}^t \big( \frac{\sigma(1-s)^2}{2 \beta_{1-s}^2} - \partial_s \log \beta_{1-s} \big) \, \mathrm{d}s = \int_{t'}^t \frac{\sigma(1-s)^2}{2 \beta_{1-s}^2} \, \mathrm{d}s - \big( \log \beta_{1-t} - \log \beta_{1-t'} \big),
    \end{talign}
    which means that
    \begin{talign} \label{eq:simplification_1a}
        \exp \big( - \int_{t'}^{t} \big( \kappa_{1-s} + \frac{\sigma(1-s)^2 - 2\eta_{1-s}}{2\beta_{1-s}^2} \big) \, \mathrm{d}s \big) &= \exp \big( - \int_{t'}^t \frac{\sigma(1-s)^2}{2 \beta_{1-s}^2} \, \mathrm{d}s \big) \frac{\beta_{1-t}}{\beta_{1-t'}}, \\
        \frac{\alpha_{1-t'}(\sigma(1-t')^2 - 2\eta_{1-t'})}{2\beta_{1-t'}^2} \xi_{1-t'}(\vec{X}_{t'}) &= \alpha_{1-t'} \big( \frac{\sigma(1-t')^2}{2 \beta_{1-t'}^2} + \frac{\dot{\beta}_{1-t'}}{\beta_{1-t'}} - \frac{\dot{\alpha}_{1-t'}}{\alpha_{1-t'}} \big) \xi_{1-t'}(\vec{X}_{t'}).
        \label{eq:simplification_1b}
    \end{talign}
    If we define $\chi(1-s)$ such that $\sigma^2(1-s) = 2 \beta_{1-s} \big( \frac{\dot{\alpha}_{1-s}}{\alpha_{1-s}} \beta_{1-s} - \dot{\beta}_{1-s} \big) + \chi(1-s)$, we obtain that
    \begin{talign} \label{eq:simplification_2a}
    \begin{split}
        &\exp \big( - \int_{t'}^t \frac{\sigma(1-s)^2}{2 \beta_{1-s}^2} \, \mathrm{d}s \big) \frac{\beta_{1-t}}{\beta_{1-t'}} = \exp \big( - \int_{t'}^t \big( \frac{\dot{\alpha}_{1-s}}{\alpha_{1-s}} - \frac{\dot{\beta}_{1-s}}{\beta_{1-s}} + \frac{\chi(1-s)}{2 \beta_{1-s}^2}\big) \, \mathrm{d}s \big) \frac{\beta_{1-t}}{\beta_{1-t'}} \\ &= \exp \big( \int_{t'}^t \big( \partial_s \log \alpha_{1-s} - \partial_s \log \beta_{1-s} - \frac{\chi(1-s)}{2 \beta_{1-s}^2}\big) \, \mathrm{d}s \big) \frac{\beta_{1-t}}{\beta_{1-t'}} = \exp \big( - \int_{t'}^t \frac{\chi(1-s)}{2 \beta_{1-s}^2} \, \mathrm{d}s \big) \frac{\alpha_{1-t}}{\alpha_{1-t'}},
    \end{split} \\
        &\alpha_{1-t'} \big( \frac{\sigma(1-t')^2}{2 \beta_{1-t'}^2} + \frac{\dot{\beta}_{1-t'}}{\beta_{1-t'}} - \frac{\dot{\alpha}_{1-t'}}{\alpha_{1-t'}} \big) \xi_{1-t'}(\vec{X}_{t'}) = \frac{\alpha_{1-t'} \chi(1-t')}{2 \beta_{1-t'}^2} \xi_{1-t'}(\vec{X}_{t'})
        \label{eq:simplification_2b}
    \end{talign}
    If we plug equations \eqref{eq:simplification_2a}-\eqref{eq:simplification_2b} into \eqref{eq:simplification_1a}-\eqref{eq:simplification_1b}, and then those into \eqref{eq:memoryless_solution}, we obtain that
    \begin{talign}
    \begin{split} \label{eq:memoryless_solution_2}
        \vec{X}_t &= \vec{X}_0 \exp \big( - \int_{0}^t \frac{\chi(1-s)}{2 \beta_{1-s}^2} \, \mathrm{d}s \big) \frac{\alpha_{1-t}}{\alpha_{1}} + \alpha_{1-t} \int_0^t \exp \big( - \int_{t'}^t \frac{\chi(1-s)}{2 \beta_{1-s}^2} \, \mathrm{d}s \big) 
        %\frac{\alpha_{1-t}}{\alpha_{1-t'}} 
        \frac{\chi(1-t')}{2 \beta_{1-t'}^2} \xi_{1-t'}(\vec{X}_{t'}) \, \mathrm{d}t' \\ &\qquad + \int_0^t \big( 2 \beta_{1-t'} \big( \frac{\dot{\alpha}_{1-t'}}{\alpha_{1-t'}} \beta_{1-t'} - \dot{\beta}_{1-t'} \big) + \chi(1-t') \big) \exp \big( - \int_{t'}^t \frac{\chi(1-s)}{2 \beta_{1-s}^2} \, \mathrm{d}s \big) \frac{\alpha_{1-t}}{\alpha_{1-t'}} \, \mathrm{d}B_{t'}.
    \end{split}
    \end{talign}
    and if we take the limit $t \to 1^-$ and use that $\alpha_1 = 1$,
    \begin{talign}
    \begin{split} \label{eq:memoryless_solution_3}
        \vec{X}_1 &= \vec{X}_0 \big( \lim_{t \to 1^-} \exp \big( - \int_{0}^t \frac{\chi(1-s)}{2 \beta_{1-s}^2} \, \mathrm{d}s \big) \alpha_{1-t} \big) + \lim_{t \to 1^-} \alpha_{1-t} \int_0^t \exp \big( - \int_{t'}^t \frac{\chi(1-s)}{2 \beta_{1-s}^2} \, \mathrm{d}s \big) 
        %\frac{\alpha_{1-t}}{\alpha_{1-t'}} 
        \frac{\chi(1-t')}{2 \beta_{1-t'}^2} \xi_{1-t'}(\vec{X}_{t'}) \, \mathrm{d}t' \\ &\qquad + \lim_{t \to 1^-} \int_0^t \big( 2 \beta_{1-t'} \big( \frac{\dot{\alpha}_{1-t'}}{\alpha_{1-t'}} \beta_{1-t'} - \dot{\beta}_{1-t'} \big) + \chi(1-t') \big) \exp \big( - \int_{t'}^t \frac{\chi(1-s)}{2 \beta_{1-s}^2} \, \mathrm{d}s \big) \frac{\alpha_{1-t}}{\alpha_{1-t'}} \, \mathrm{d}B_{t'}.
    \end{split}
    \end{talign}
    The assumption on $\chi$ in \eqref{eq:chi_condition} is equivalent, up to a rearrangement of the notation and a flip in the time variable, to the statement that for all $t' \in [0,1)$,
    \begin{talign} \label{eq:coefficient_1}
        \lim_{t \to 1^-} \exp \big( - \int_{t'}^t \frac{\chi(1-s)}{2 \beta_{1-s}^2} \, \mathrm{d}s \big) \alpha_{1-t} = 0.
    \end{talign}
    Hence, under assumption \eqref{eq:chi_condition}, the factor accompanying $\vec{X}_0$ in equation \eqref{eq:memoryless_solution_3} is zero. Moreover, this assumption also implies that
    \begin{talign}
    \begin{split} \label{eq:coefficient_2}
        &\lim_{t \to 1^-} \alpha_{1-t} \int_0^t \exp \big( - \int_{t'}^t \frac{\chi(1-s)}{2 \beta_{1-s}^2} \, \mathrm{d}s \big) 
        %\frac{\alpha_{1-t}}{\alpha_{1-t'}} 
        \frac{\chi(1-t')}{2 \beta_{1-t'}^2} \xi_{1-t'}(\vec{X}_{t'}) \, \mathrm{d}t' \\ &= \int_0^1 \big( \lim_{t \to 1^-} \exp \big( - \int_{t'}^t \frac{\chi(1-s)}{2 \beta_{1-s}^2} \, \mathrm{d}s \big) \alpha_{1-t} \big) 
        %\frac{1}{\alpha_{1-t'}} 
        \frac{\chi(1-t')}{2 \beta_{1-t'}^2} \xi_{1-t'}(\vec{X}_{t'}) \, \mathrm{d}t' = 0.
    \end{split}
    \end{talign}
    If we plug \eqref{eq:coefficient_1} and \eqref{eq:coefficient_2} into \eqref{eq:memoryless_solution_3}, we obtain that
    \begin{talign}
        \vec{X}_1 = \lim_{t \to 1^-} \int_0^t \big( 2 \beta_{1-t'} \big( \frac{\dot{\alpha}_{1-t'}}{\alpha_{1-t'}} \beta_{1-t'} - \dot{\beta}_{1-t'} \big) + \chi(1-t') \big) \exp \big( - \int_{t'}^t \frac{\chi(1-s)}{2 \beta_{1-s}^2} \, \mathrm{d}s \big) \frac{\alpha_{1-t}}{\alpha_{1-t'}} \, \mathrm{d}B_{t'},
    \end{talign}
    which shows that $\vec{X}_1$ is independent of $\vec{X}_0$. Next, we leverage that $\vec{\bm{X}}$ and $\bm{X}$ have equal distributions over trajectories (\Cref{lem:equal_process_distributions}). In particular, the joint distribution of $(\vec{X}_0,\vec{X}_1)$ is equal to the joint distribution of $(X_1,X_0)$. We conclude that $X_1$ and $X_0$ are independent, which is the definition of the memorylessness property.
    Hence, the assumption \eqref{eq:chi_condition} is sufficient for memorylessness to hold.
    
    It remains to prove that the assumption \eqref{eq:chi_condition} is necessary. Looking at equation \eqref{eq:memoryless_solution_2} we deduce that generally, for any $t \in [0,1)$, $\vec{X}_0$ and $\vec{X}_t$ are not independent, because the first two terms in \eqref{eq:memoryless_solution_2} are different from zero. Thus, if there existed a $t' \in [0,1)$ such that the limit \eqref{eq:coefficient_1} is different from zero, then $\vec{X}_1$ would not be independent from $\vec{X}_{t'}$, which means that in general it would not be independent of $\vec{X}_0$ either.

\subsection{Proof of \Cref{thm:general_fine-tuning}: fine-tuning recipe for general noise schedules} \label{subsec:proof_prop_diff_finetuning}

The proof of this result relies heavily on the properties of the Hamilton-Jacobi-Bellman equation:
\begin{theorem}[Hamilton-Jacobi-Bellman equation] \label{thm:HJB}
    If we define the infinitesimal generator
    \begin{talign}
    \mathcal{L} := \frac{1}{2} \sum_{i,j=1}^{d} (\sigma \sigma^{\top})_{ij} (t) \partial_{x_i} \partial_{x_j} + \sum_{i=1}^{d} b_i(x,t) \partial_{x_i},
    \end{talign}
    the value function $V$ for the SOC
    % stochastic optimal control 
    problem \eqref{eq:control_problem_def}-\eqref{eq:controlled_SDE} solves the following Hamilton-Jacobi-Bellman (HJB) partial differential equation:
    \begin{talign}
    \begin{split} \label{eq:HJB_setup}
        % &(\partial_t + L) V(x,t) - \frac{1}{2} \| (\sigma^{\top} \nabla V) (x,t) \|^2 + f(x,t) = 0, \\
        &\partial_t V(x,t) = - \mathcal{L} V(x,t) + \frac{1}{2} \| (\sigma^{\top} \nabla V) (x,t) \|^2 - f(x,t), \\
        &V(x,T) = g(x).
    \end{split}
    \end{talign}
\end{theorem}

Consider forward SDEs like \eqref{eq:forward_arbitrary}, starting from the distributions $p^{\mathrm{base}}$ and $p^*$, where $p^*(x) \propto p^{\mathrm{base}}(x) \exp(r(x))$.
\begin{talign}
\label{eq:forward_base}
    \mathrm{d}\vec{X}_t = \vec{b}(\vec{X}_t,t) \, \mathrm{d}t + \sigma(t) \, \mathrm{d}B_t, \qquad \vec{X}_0 \sim p^{\mathrm{base}}, \\
    \mathrm{d}\vec{X}^*_t = \vec{b}^*(\vec{X}^*_t,t) \, \mathrm{d}t + \sigma(t) \, \mathrm{d}B_t, \qquad \vec{X}_0 \sim p^*. \label{eq:forward_star}
\end{talign}
where the drifts are defined as
%By hypothesis we have that
\begin{talign}
\begin{split} \label{eq:b_b_star}
    \vec{b}(x,t) &= - \kappa_{1-t} x +
     \big( \frac{\sigma(1-t)^2}{2} - \eta_{1-t} \big) \mathfrak{s}(x,1-t) = - \kappa_{1-t} x +
    \big( \frac{\sigma(1-t)^2}{2} - \eta_{1-t} \big) \nabla \log \vec{p}_t(x), \\
    \vec{b}^*(x,t) &= - \kappa_{1-t} x +
    \big( \frac{\sigma(1-t)^2}{2} - \eta_{1-t} \big)
    \mathfrak{s}^*(x,1-t) = - \kappa_{1-t} x +
    \big( \frac{\sigma(1-t)^2}{2} - \eta_{1-t} \big)
    \nabla \log \vec{p}_t^*(x),
\end{split}
\end{talign}
and $\vec{p}_t$, $\vec{p}_t^*$ are the densities of $X_t$, $\vec{X}_t$, respectively. $\vec{p}_t$, $\vec{p}_t^*$ satisfy Fokker-Planck equations:
\begin{talign}
\begin{split} \label{eq:fokker_planck_1}
    \partial_t \vec{p}_t = \nabla \cdot (\vec{b}(x,t) \vec{p}_t) + \nabla \cdot ( %\frac{\sigma(t) \sigma(t)^{\top}}{2} 
    \frac{\sigma(1-t)^2}{2} \nabla \vec{p}_t), \qquad \vec{p}_0 = p^{\mathrm{base}}, \\
    \partial_t \vec{p}^*_t = \nabla \cdot (\vec{b}^*(x,t) \vec{p}^*_t) + \nabla \cdot ( %\frac{\sigma(t) \sigma(t)^{\top}}{2}
    \frac{\sigma(1-t)^2}{2}
    \nabla \vec{p}^*_t), \qquad \vec{p}_0 = p^*.
\end{split}
\end{talign}
Plugging \eqref{eq:b_b_star} into \eqref{eq:fokker_planck_1}, we obtain 
\begin{talign}
\begin{split}
    \partial_t \vec{p}_t = \nabla \cdot (\kappa_{1-t} x \vec{p}_t) + \nabla \cdot \big( 
    % \big( \frac{\sigma(t) \sigma(t)^{\top}}{2} + \eta_t \mathrm{I} \big) 
    % \big( \frac{\tilde{\sigma}(t)^2}{2} + \eta_t \big)
    \eta_{1-t}
    \nabla \vec{p}_t \big), \qquad \vec{p}_0 = p^{\mathrm{base}}, \\
    \partial_t \vec{p}^*_t = \nabla \cdot (\kappa_{1-t} x \vec{p}^*_t) + \nabla \cdot \big( 
    % \big( \frac{\sigma(t) \sigma(t)^{\top}}{2} + \eta_t \mathrm{I} \big) 
    % \big( \frac{\tilde{\sigma}(t)^2}{2} + \eta_t \big)
    \eta_{1-t}
    \nabla \vec{p}^*_t \big), \qquad \vec{p}_0 = p^*.
\end{split}
\end{talign}
We apply the Hopf-Cole transformation to obtain PDEs for $-\log \vec{p}_t$ (and $-\log \vec{p}^*_t$ analogously):
\begin{talign}
\begin{split}
    - \partial_t (- \log \vec{p}_t) &= \frac{\partial_t p_t}{p_t} = \frac{\nabla \cdot (\kappa_{1-t} x \vec{p}_t) + \nabla \cdot \big( 
    % \big( \frac{\tilde{\sigma}(t)^2}{2} + \eta_t \big)
    \eta_{1-t}
    \nabla \vec{p}_t \big)}{p_t} \\ &= \kappa_{1-t} \nabla \cdot x + \kappa_{1-t} \langle x, \nabla \log \vec{p}_t \rangle +
    % \big( \frac{\tilde{\sigma}(t)^2}{2} + \eta_t \big)
    \eta_{1-t}
    \frac{\nabla \cdot (\nabla \log \vec{p}_t \exp(\log p_t))}{p_t} \\ &= \kappa_{1-t} d + \kappa_{1-t} \langle x, \nabla \log \vec{p}_t \rangle + 
    % \big( \frac{\tilde{\sigma}(t)^2}{2} + \eta_t \big) 
    \eta_{1-t}
    \big( \Delta \log \vec{p}_t + \|\nabla \log \vec{p}_t \|^2 \big).
\end{split}
\end{talign}
Hence, if we define $\mathscr{V}(x,t) = - \log \vec{p}_t(x)$, $\mathscr{V}^*(x,t) = - \log \vec{p}^*_t(x)$, then $\mathscr{V}$ and $\mathscr{V}^*$ satisfy the following Hamilton-Jacobi-Bellman equations:
\begin{talign} \label{eq:V_HJB}
    -\partial_t \mathscr{V} &= \kappa_{1-t} d - \kappa_{1-t} \langle x, \nabla \mathscr{V} \rangle + 
    % \big( \frac{\tilde{\sigma}(t)^2}{2} + \eta_t \big) 
    \eta_{1-t}
    \big( - \Delta \mathscr{V} + \|\nabla \mathscr{V} \|^2 \big), \qquad \mathscr{V}(x,0) = - \log p^{\mathrm{base}}(x), \\
    -\partial_t \mathscr{V}^* &= \kappa_{1-t} d - \kappa_{1-t} \langle x, \nabla \mathscr{V}^* \rangle + 
    % \big( \frac{\tilde{\sigma}(t)^2}{2} + \eta_t \big)
    \eta_{1-t}
    \big( - \Delta \mathscr{V}^* + \|\nabla \mathscr{V}^* \|^2 \big), \qquad \mathscr{V}^*(x,0) = - \log p^*(x). \label{eq:V_star_HJB}
\end{talign}
% We can also apply the Hopf-Cole transformation directly on the original Fokker-Planck equations \eqref{eq:fokker_planck_1}, which yields:
% \begin{talign}
%     -\partial_t \mathscr{V} &= \nabla \cdot b(x,t) - \langle b(x,t), \nabla \mathscr{V} \rangle + \frac{\tilde{\sigma}(t)^2}{2} \big( - \Delta \mathscr{V} + \|\nabla \mathscr{V} \|^2 \big), \qquad \mathscr{V}(x,0) = - \log p_1(x), \\
%     -\partial_t \mathscr{V}^* &= \nabla \cdot b(x,t) - \langle b(x,t), \nabla \mathscr{V}^* \rangle + \frac{\tilde{\sigma}(t)^2}{2} \big( - \Delta \mathscr{V}^* + \|\nabla \mathscr{V}^* \|^2 \big), \qquad \mathscr{V}(x,0) = - \log p^*(x).
% \end{talign}
Now, define $\hat{\mathscr{V}}(x,t) = \mathscr{V}^*(x,t) - \mathscr{V}(x,t)$. Subtracting \eqref{eq:V_star_HJB} from \eqref{eq:V_HJB}, we obtain
\begin{talign}
\begin{split}
    -\partial_t \hat{\mathscr{V}} &= - \kappa_{1-t} \langle x, \nabla \hat{\mathscr{V}} \rangle + 
    % \big( \frac{\tilde{\sigma}(t)^2}{2} + \eta_t \big)
    \eta_{1-t}
    \big( - \Delta \hat{\mathscr{V}} + \|\nabla \mathscr{V}^* \|^2 - \|\nabla \mathscr{V} \|^2 \big)
    \\ &= - \kappa_{1-t} \langle x, \nabla \hat{\mathscr{V}} \rangle + 
    % \big( \frac{\tilde{\sigma}(t)^2}{2} + \eta_t \big)
    \eta_{1-t}
    \big( - \Delta \hat{\mathscr{V}} + \|\nabla (\hat{\mathscr{V}} + \mathscr{V}) \|^2 - \|\nabla \mathscr{V} \|^2 \big)
    \\ &= - \kappa_{1-t} \langle x, \nabla \hat{\mathscr{V}} \rangle + 
    % \big( \frac{\tilde{\sigma}(t)^2}{2} + \eta_t \big) 
    \eta_{1-t}
    \big( - \Delta \hat{\mathscr{V}} + \|\nabla \hat{\mathscr{V}}\|^2 + 2\langle \nabla \mathscr{V}, \nabla \hat{\mathscr{V}} \rangle \big) \\ &= \langle - \kappa_{1-t} x + 
    % \big( \tilde{\sigma}(t)^2 + 2 \eta_t \big)
    2 \eta_{1-t}
    \nabla \mathscr{V}, \nabla \hat{\mathscr{V}} \rangle + 
    % \big( \frac{\tilde{\sigma}(t)^2}{2} + \eta_t \big) 
    \eta_{1-t}
    \big( - \Delta \hat{\mathscr{V}} + \|\nabla \hat{\mathscr{V}}\|^2 \big)
    \\ &= \langle - \kappa_{1-t} x - 
    2 \eta_{1-t}
    \mathfrak{s}(x,1-t), \nabla \hat{\mathscr{V}} \rangle + 
    \eta_{1-t}
    \big( - \Delta \hat{\mathscr{V}} + \|\nabla \hat{\mathscr{V}}\|^2 \big),
    \\
    \hat{\mathscr{V}}(x,0) &= - \log p^*(x) + \log p^{\mathrm{base}}(x) = - r(x) + \log \big( \int p^{\mathrm{base}}(y) \exp(r(y)) \, \mathrm{d}y \big). 
\end{split}
\end{talign}
Hence, $\hat{\mathscr{V}}$ also satisfies a Hamilton-Jacobi-Bellman equation. If we define $V$ such that $\hat{\mathscr{V}}(x,t) = V(x,1-t)$, we have that
\begin{talign}
    \partial_t V = \langle - \kappa_{t} x - 
    2 \eta_{t}
    \mathfrak{s}(x,t), \nabla V \rangle + 
    \eta_{t}
    \big( - \Delta V + \|\nabla V\|^2 \big), \qquad 
    V(x,1) = r(x) - \log \big( \int p^{\mathrm{base}}(y) \exp(r(y)) \, \mathrm{d}y \big). 
\end{talign}
Using \Cref{thm:HJB}, we can reverse-engineer 
% $\hat{\mathscr{V}}$ as $\hat{\mathscr{V}}(x,t) = V(x,1-t)$, where 
$V$ as the value function of the following SOC
%stochastic optimal control 
problem:
\begin{talign} \label{eq:control_problem_def_diff}
    &\min\limits_{u \in \mathcal{U}} \mathbb{E} \big[ \frac{1}{2} \int_0^1 \|u(X^u_t,t)\|^2 \, \mathrm{d}t \! - \! r(x) \! + \! \log \big( \int p^{\mathrm{base}}(y) \exp(r(y)) \, \mathrm{d}y \big) \big], \\
    \begin{split}
    \text{s.t.}~ \mathrm{d}X^u_t \! = \! \big(%\kappa_t X^u_t - \big( \tilde{\sigma}(t)^2 + 2 \eta_t \big) \nabla \mathscr{V}, \nabla \hat{\mathscr{V}} \rangle 
    % \hat{b}(X^u_t,1-t) 
    \kappa_t x + 2 \eta_t \mathfrak{s}(x,t)
    \! + \! %\hat{\sigma}(1-t) 
    \sqrt{2 \eta_{t}} 
    u(X^u_t,t) \big) \, \mathrm{d}t \! + \! 
    % \hat{\sigma}(1-t) 
    \sqrt{2 \eta_{t}}
    \mathrm{d}B_t, \qquad X^u_0 \sim p_0.
    \end{split} 
    \label{eq:controlled_SDE_diff}
\end{talign}
Note that this SOC problem is equal to the problem \eqref{eq:control_problem_def}-\eqref{eq:controlled_SDE} with the choices $f=0$, $g=-r$, and $\sigma(t) = \sqrt{2 \eta_t}$.
% where %$\frac{\hat{\sigma}(t)^2}{2} = \frac{\tilde{\sigma}(t)^2}{2} + \eta_t$
% \begin{talign} 
% % \label{eq:hat_sigma_hat_b_app}
% % \hat{\sigma}(t) = \sqrt{\tilde{\sigma}(t)^2 + 2\eta_t}, \qquad\quad 
% \hat{b}(x,t) = \kappa_t x - %\big( \tilde{\sigma}(t)^2 + 2 \eta_t \big) 
% 2 \eta_t
% \nabla \mathscr{V}(x,t) = \kappa_t x + 
% % \big( \tilde{\sigma}(t)^2 + 2 \eta_t \big) 
% 2 \eta_t
% \mathfrak{s}(x,t).
% \end{talign}
By equation \eqref{eq:optimal_control}, %and recalling that $\hat{\sigma}$ is scalar-valued, 
the optimal control of the problem \eqref{eq:control_problem_def_diff}-\eqref{eq:controlled_SDE_diff} is of the form:
\begin{talign} 
\begin{split} \label{eq:optimal_control_HJB}
u^*(x,t) &= - %\hat{\sigma}(1-t) 
\sqrt{2 \eta_{t}}
\nabla V(x,t) = - %\hat{\sigma}(1-t) 
\sqrt{2 \eta_{t}}
\nabla \hat{\mathscr{V}}(x, 1-t) = - 
% \hat{\sigma}(1-t)
\sqrt{2 \eta_{t}}
\big( \nabla \mathscr{V}^*(x, 1-t) - \nabla \mathscr{V}(x, 1-t) \big) \\ &= - \sqrt{2 \eta_{t}} \big( - \nabla \log \vec{p}^*_{1-t}(x) + \nabla \log \vec{p}_{1-t}(x) \big) = %\hat{\sigma}(1-t) 
\sqrt{2 \eta_{t}}
\big( \mathfrak{s}^*(x,t) - \mathfrak{s}(x,t) \big), 
\end{split} \\
&\iff \mathfrak{s}^*(x,t) = \mathfrak{s}(x,t) + u^*(x,t) / %\hat{\sigma}(1-t).
\sqrt{2 \eta_{t}}.
\label{eq:s_star_u_star}
\end{talign}
As in \eqref{eq:backward_arbitrary}, the backward SDEs corresponding to the forward SDEs \eqref{eq:forward_star} take the following form:
\begin{talign}
    % \mathrm{d}X_t &= \big( \kappa_t X_t + \big( \frac{\sigma(t)^2}{2} + \eta_t \big) \mathfrak{s}(X_t,t) \big) \mathrm{d}t + \sigma(t) \, \mathrm{d}B_t, \qquad X_0 \sim N(0,I), \\
    \mathrm{d}X^*_t &= \big( \kappa_t X^*_t + \big( \frac{\sigma(t)^2}{2} + \eta_t \big) \mathfrak{s}^*(X^*_t,t) \big) \mathrm{d}t + \sigma(t) \, \mathrm{d}B_t, \qquad X^*_0 \sim N(0,I).
\end{talign}
If we plug \eqref{eq:s_star_u_star} into this equation, we obtain
\begin{talign}
    \mathrm{d}X^*_t &= \big( \kappa_t X^*_t + \big( \frac{\sigma(t)^2}{2} + \eta_t \big) \big( \mathfrak{s}(X^*_t,t) + \frac{u^*(X^*_t,t)}{\sqrt{2 \eta_{t}}} \big) \big) \mathrm{d}t + \sigma(t) \, \mathrm{d}B_t, \qquad X^*_0 \sim N(0,I), \\
    \iff \mathrm{d}X^*_t &= \big( b(X^*_t,t) + \frac{\frac{\sigma(t)^2}{2} + \eta_t}{\sqrt{2 \eta_{t}}} u^*(X^*_t,t) \big) \mathrm{d}t + \sigma(t) \, \mathrm{d}B_t, \qquad X^*_0 \sim N(0,I).
\end{talign}
where we used that $b(x,t) = \kappa_t x + \big( \frac{\sigma(t)^2}{2} + \eta_t \big) \mathfrak{s}(x,t)$ by definition in equation \eqref{eq:gen_process_2}. 

% \subsection{Fine-tuned inference SDEs for DDIM and Flow Matching} \label{subsec:finetuned_inference_SDE}
\paragraph{The fine-tuned inference SDE for DDIM}
Now, for DDIM, we have that $u^*(x,t) = 
- \sqrt{\frac{\dot{\alpha}_t}{\alpha_t(1-\alpha_t)}} (\epsilon^*(x,t) - \epsilon^{\mathrm{base}}(x,t))$ by \eqref{eq:conversion_DDPM}. Hence,
\begin{talign}
    &\frac{\frac{\sigma(t)^2}{2} + \eta_t}{\sqrt{2 \eta_{t}}} u^*(x,t) = - \frac{\frac{\sigma(t)^2}{2} + \frac{\dot{\alpha}_t}{2\alpha_t}}{\sqrt{\frac{\dot{\alpha}_t}{\alpha_t}}} \sqrt{\frac{\dot{\alpha}_t}{\alpha_t(1-\alpha_t)}} (\epsilon^*(x,t) - \epsilon^{\mathrm{base}}(x,t)) = - \frac{\frac{\sigma(t)^2}{2} + \frac{\dot{\alpha}_t}{2\alpha_t}}{\sqrt{1-\alpha_t}} (\epsilon^*(x,t) - \epsilon^{\mathrm{base}}(x,t)), \\
    \begin{split}
    &\implies b(x,t) + \frac{\frac{\sigma(t)^2}{2} + \eta_t}{\sqrt{2 \eta_{t}}} u^*(x,t) = 
    % \frac{\dot{\alpha}_{t}}{2\alpha_{t}} x - \frac{\dot{\alpha}_{t}}{\alpha_{t}} \frac{\epsilon^{\mathrm{base}}(x,t)}{\sqrt{1-\alpha_{t}}} 
    \frac{\dot{\alpha}_{t}}{2\alpha_{t}} X_t - \big( \frac{\dot{\alpha}_{t}}{2\alpha_{t}} + \frac{\sigma(t)^2}{2} \big) \frac{\epsilon^{\mathrm{base}}(X_{t},t)}{\sqrt{1-\alpha_{t}}}
    - \frac{\frac{\sigma(t)^2}{2} + \frac{\dot{\alpha}_t}{2\alpha_t}}{\sqrt{1-\alpha_t}} (\epsilon^*(x,t) - \epsilon^{\mathrm{base}}(x,t)) \\ &\qquad\qquad\qquad\qquad\qquad = 
    \frac{\dot{\alpha}_{t}}{2\alpha_{t}} X_t - \big( \frac{\dot{\alpha}_{t}}{2\alpha_{t}} + \frac{\sigma(t)^2}{2} \big) \frac{\epsilon^{*}(X_{t},t)}{\sqrt{1-\alpha_{t}}}.
    % \frac{\dot{\alpha}_{t}}{2\alpha_{t}} X_t + \frac{\frac{\sigma(t)^2}{2} - \frac{\dot{\alpha}_t}{2\alpha_t}}{\sqrt{1-\alpha_t}} \epsilon^{\mathrm{base}}(x,t) - \frac{\frac{\sigma(t)^2}{2} + \frac{\dot{\alpha}_t}{2\alpha_t}}{\sqrt{1-\alpha_t}} \epsilon^*(x,t)
    \end{split}
\end{talign}
We obtain that the fine-tuned inference SDE for DDIM is 
\begin{talign}
    \mathrm{d}X^*_t &= \big( \frac{\dot{\alpha}_{t}}{2\alpha_{t}} X^*_t - \big( \frac{\dot{\alpha}_{t}}{2\alpha_{t}} + \frac{\sigma(t)^2}{2} \big) \frac{\epsilon^{*}(X^*_{t},t)}{\sqrt{1-\alpha_{t}}} \big) \mathrm{d}t + \sigma(t) \, \mathrm{d}B_t, \qquad X^*_0 \sim N(0,I),
\end{talign}
which is matches the SDE \eqref{eq:euler_maruyama_DDIM} with the choice $\epsilon = \epsilon^*$.

\paragraph{The fine-tuned inference SDE for Flow Matching}
For Flow Matching, we have that $u^*(x,t) = 
\sqrt{\frac{2}{\beta_{t}(\frac{\dot{\alpha}_{t}}{\alpha_{t}} \beta_{t} - \dot{\beta}_{t})}} (v^*(x,t) - v^{\mathrm{base}}(x,t))$ by \eqref{eq:conversion_MFM}. Hence,
\begin{talign}
    \begin{split}
    &\frac{\frac{\sigma(t)^2}{2} + \eta_t}{\sqrt{2 \eta_{t}}} u^*(x,t) = \frac{\frac{\sigma(t)^2}{2} + \beta_t(\frac{\dot{\alpha}_t}{\alpha_t} \beta_t - \dot{\beta}_t)}{\sqrt{2 \beta_t(\frac{\dot{\alpha}_t}{\alpha_t} \beta_t - \dot{\beta}_t)}} \sqrt{\frac{2}{\beta_{t}(\frac{\dot{\alpha}_{t}}{\alpha_{t}} \beta_{t} - \dot{\beta}_{t})}} (v^*(x,t) - v^{\mathrm{base}}(x,t)) \\ &\qquad\qquad\qquad \ = \big( 1 + \frac{\sigma(t)^2}{2 \beta_t(\frac{\dot{\alpha}_t}{\alpha_t} \beta_t - \dot{\beta}_t)} \big) (v^*(x,t) - v^{\mathrm{base}}(x,t)). 
    \end{split}
    \\
    \begin{split}
    &\implies b(x,t) + \frac{\frac{\sigma(t)^2}{2} + \eta_t}{\sqrt{2 \eta_{t}}} u^*(x,t) =  
    % - \frac{\dot{\alpha}_{t}}{\alpha_{t}} x + 2v^{\mathrm{base}}(x,t) 
    v^{\mathrm{base}}(x,t) + \frac{\sigma(t)^2}{2\beta_{t}(\frac{\dot{\alpha}_{t}}{\alpha_{t}} \beta_{t} -\dot{\beta}_{t})} \big( v^{\mathrm{base}}(x,t) - \frac{\dot{\alpha}_{t}}{\alpha_{t}} x \big) 
    \\ &\qquad\qquad\qquad\qquad\qquad\qquad\quad + \big( 1 + \frac{\sigma(t)^2}{2 \beta_t(\frac{\dot{\alpha}_t}{\alpha_t} \beta_t - \dot{\beta}_t)} \big) (v^*(x,t) - v^{\mathrm{base}}(x,t)) \\ &\qquad\qquad\qquad\qquad\qquad\qquad\quad = 
    % - \frac{\dot{\alpha}_{t}}{\alpha_{t}} x + \big( 1 -  \frac{\sigma(t)^2}{2 \beta_t(\frac{\dot{\alpha}_t}{\alpha_t} \beta_t - \dot{\beta}_t)} \big) v^{\mathrm{base}}(x,t) + \big( 1 + \frac{\sigma(t)^2}{2 \beta_t(\frac{\dot{\alpha}_t}{\alpha_t} \beta_t - \dot{\beta}_t)} \big) v^*(x,t).
    v^{*}(x,t) + \frac{\sigma(t)^2}{2\beta_{t}(\frac{\dot{\alpha}_{t}}{\alpha_{t}} \beta_{t} -\dot{\beta}_{t})} \big( v^{*}(x,t) - \frac{\dot{\alpha}_{t}}{\alpha_{t}} x \big).
    \end{split}
\end{talign}
We obtain that the fine-tuned inference SDE for Flow Matching is 
\begin{talign}
    \mathrm{d}X^*_t = \big( v(X^*_t,t) + \frac{\sigma(t)^2}{2\beta_{t}(\frac{\dot{\alpha}_{t}}{\alpha_{t}} \beta_{t} -\dot{\beta}_{t})} \big( v^*(X^*_t,t) - \frac{\dot{\alpha}_{t}}{\alpha_{t}} X^*_t \big) \big) \, \mathrm{d}t + \sigma(t) \, \mathrm{d}B_t, \qquad X^*_0 \sim N(0,I),
\end{talign}
which matches equation \eqref{eq:FM_general_diffusion_coeff} with the choice $v = v^*$.

\section{Loss function derivations}
\label{sec:proof_existing_losses}

\subsection{Derivation of the Continuous Adjoint method} \label{subsec:derivation_cont_adj_method}
\begin{proposition} \label{prop:cont_adjoint_method}
    The gradient $\frac{\mathrm{d} \mathcal{L}}{\mathrm{d} \theta}$ of the adjoint loss $\mathcal{L}(u ; \fX)$ defined in \eqref{eq:L_RE} with respect to the parameters $\theta$ of the control can be expressed as in \eqref{eq:continuous_adjoint_grads}.
\end{proposition}
\begin{proof}
    First, note that we can write 
    \begin{talign}
    \begin{split} \label{eq:first_eq_cont_adjoint}
        &\nabla_{\theta} \mathbb{E} \big[ \int_0^T \big(\frac{1}{2} \|u_{\theta}(X^{u_{\theta}}_t,t)\|^2 \! + \! f(X^{u_{\theta}}_t,t) \big) \, \mathrm{d}t \! + \! g(X^{u_{\theta}}_T) \big] \\ &= \mathbb{E} \big[ \int_0^T \nabla_{\theta} u_{\theta}(X^{u_{\theta}}_t,t) u_{\theta}(X^{u_{\theta}}_t,t) \, \mathrm{d}t \big] 
        % \\ &\qquad 
        + \nabla_{\theta} \mathbb{E} \big[ \int_0^T \big(\frac{1}{2} \|v(X^{u_{\theta}}_t,t)\|^2 \! + \! f(X^{u_{\theta}}_t,t) \big) \, \mathrm{d}t \! + \! g(X^{u_{\theta}}_T)
        % \textcolor{blue}{+ \, \int_0^T v(X^{u_{\theta}}_t,t) \, \mathrm{d}B_t}
        \big] \rvert_{v = \mathrm{stopgrad}(u_{\theta})}.
    \end{split}
    \end{talign}
    To develop the second term, we apply \Cref{lem:adjoint_state_properties}. Namely, by the Leibniz rule and equation \eqref{eq:nabla_theta_cost}, we have that
    \begin{talign}
    \begin{split}
        &\nabla_{\theta} \mathbb{E} \big[ \int_0^T \big(\frac{1}{2} \|v(X^{u_{\theta}}_t,t)\|^2 \! + \! f(X^{u_{\theta}}_t,t) \big) \, \mathrm{d}t \! + \! g(X^{u_{\theta}}_T) \big] \rvert_{v = \mathrm{stopgrad}(u_{\theta})} \\ &= \mathbb{E} \big[ \nabla_{\theta} \big( \int_0^T \big(\frac{1}{2} \|v(X^{u_{\theta}}_t,t)\|^2 \! + \! f(X^{u_{\theta}}_t,t) \big) \, \mathrm{d}t \! + \! g(X^{u_{\theta}}_T) 
        \big) \rvert_{v = \mathrm{stopgrad}(u_{\theta})} \big] \\ &= \mathbb{E} \big[ \int_0^T (\nabla_{\theta} u_{\theta})(X^{u_{\theta}}_t(\omega),t)^{\top} \sigma(t)^{\top} a_t(\omega) \, \mathrm{d}t \big].
    \end{split}
    \end{talign}
    Plugging the right-hand side of this equation into \eqref{eq:first_eq_cont_adjoint} concludes the proof.
\end{proof}

\begin{lemma} \label{lem:adjoint_state_properties}
    Let $v$ be an arbitrary fixed vector field.
    The unique solution of the ODE 
    \begin{talign} 
    \begin{split} \label{eq:cont_adjoint_1_app}
        \frac{\mathrm{d}}{\mathrm{d}t} a(t;\fX^u,u)  &=  - \left[ \left(\nabla_{X^u_t} (b (X^u_t,t) + \sigma(t) u(X^u_t,t))\right)\tran{} a(t;\fX^u,u) 
        + \nabla_{X^u_t} \left( f(X^u_t,t) + \frac{1}{2}\|v(X^u_t,t)\|^2 \right) \right],
    \end{split}
        \\ a(1;\fX^u,u) &= \nabla g(X^u_1), \label{eq:cont_adjoint_2_app}
    \end{talign}
    satisfies:  
    \begin{talign} 
    \begin{split} \label{eq:adjoint_grad_x}
        &a(t ; \fX^u, u) := \nabla_{X^u_t} \big(\int_t^1 \big(\frac{1}{2} \|u(X_{t'}^u,t')\|^2 \! + \! f(X_{t'}^u,t') \big) \, \mathrm{d}t' \! + \! g(X^u_1) \big), \\
        &\text{where } \fX^u \text{ solves } \mathrm{d}X^u_t =  \left( b(X^u_t,t) + \sigma(t) u(X^u_t,t) \right) \, \mathrm{d}t + 
        \sigma(t) \mathrm{d}B_t.
    \end{split}
    \end{talign}
    Moreover, when $u = u_{\theta}$ is parameterized by $\theta$ we have that 
    \begin{talign}
    \begin{split} \label{eq:nabla_theta_cost}
        \nabla_{\theta} \big( \int_0^T \big(\frac{1}{2} \|v(X^{u_{\theta}}_t,t)\|^2 \! + \! f(X^{u_{\theta}}_t,t) \big) \, \mathrm{d}t \! + \! g(X^{u_{\theta}}_T) 
        \big) = \int_0^T (\nabla_{\theta} u_{\theta})(X^{u_{\theta}}_t(\omega),t) \sigma(t)^{\top} a_t(\omega) \, \mathrm{d}t.
    \end{split}
    \end{talign}
\end{lemma}
\begin{proof}
    We use an approach based on Lagrange multipliers which mirrors and extends the derivation of the adjoint ODE \citep[Lemma~8]{domingoenrich2023stochastic}. For shortness, we use the notation $\tilde{b}_{\theta}(x,t) := b(x,t) + \sigma(t) u_{\theta}(x,t)$. Define a process $a : \Omega \times [0,T] \to \R^d$ such that for any $\omega \in \Omega$, $a(\omega,\cdot)$ is differentiable. For a given $\omega \in \Omega$, we can write
    \begin{talign}
    \begin{split}
        % &\int_0^T f(X_t(\omega),t) \, \mathrm{d}t + \int_0^T \langle h(X_t(\omega),t), \, \mathrm{d}B_t \rangle + g(X_T(\omega)) 
        &\int_0^T \big(\frac{1}{2} \|v(X^{u_{\theta}}_t,t)\|^2 \! + \! f(X^{u_{\theta}}_t,t) \big) \, \mathrm{d}t \! + \! g(X^{u_{\theta}}_T)
        % \textcolor{blue}{+ \, \int_0^T v(X^{u_{\theta}}_t,t) \, \mathrm{d}B_t}
        \\ &= 
        \int_0^T \big(\frac{1}{2} \|v(X^{u_{\theta}}_t,t)\|^2 \! + \! f(X^{u_{\theta}}_t,t) \big) \, \mathrm{d}t \! + \! g(X^{u_{\theta}}_T)
        % \textcolor{blue}{+ \, \int_0^T v(X^{u_{\theta}}_t,t) \, \mathrm{d}B_t} 
        \\ &\qquad - \int_0^T \langle a_t(\omega), (dX^{u_{\theta}}_t(\omega) - \tilde{b}_{\theta}(X^{u_{\theta}}_t(\omega),t) \, \mathrm{d}t - \sigma(t) \, \mathrm{d}B_t) \rangle. 
    \end{split}
    \end{talign}
    By stochastic integration by parts \citep[Lemma~9]{domingoenrich2023stochastic}, we have that
    \begin{talign}
    \begin{split}
        \int_0^T \langle a_t(\omega), dX^{u_{\theta}}_t(\omega) \rangle = \langle a_T(\omega), X^{u_{\theta}}_T(\omega) \rangle - \langle a_0(\omega), X^{u_{\theta}}_0(\omega) \rangle - \int_0^T \langle X^{u_{\theta}}_t(\omega), \frac{da_t}{dt}(\omega) \rangle \, \mathrm{d}t.
    \end{split}
    \end{talign}
    Hence, if $X^{u_{\theta}}_0 = x_0$ is the initial condition, we have that\footnote{Unlike \citep[Lemma~8]{domingoenrich2023stochastic}, we use the convention that a Jacobian matrix $J = \nabla_x v(x)$ is defined as $J_{ij} = \frac{\partial v_i(x)}{\partial x_j}$. Their definition of $\nabla_x v$ is the transpose of ours.}
    \begin{talign}
    \begin{split}
        &\nabla_{x_0} \big( \int_0^T \big(\frac{1}{2} \|v(X^{u_{\theta}}_t,t)\|^2 \! + \! f(X^{u_{\theta}}_t,t) \big) \, \mathrm{d}t \! + \! g(X^{u_{\theta}}_T) 
        % \textcolor{blue}{+ \, \int_0^T v(X^{u_{\theta}}_t,t) \, \mathrm{d}B_t} 
        \big) \\
        &= \nabla_{x_0} \big( 
        \int_0^T \big(\frac{1}{2} \|v(X^{u_{\theta}}_t,t)\|^2 \! + \! f(X^{u_{\theta}}_t,t) \big) \, \mathrm{d}t \! + \! g(X^{u_{\theta}}_T)
        % \textcolor{blue}{+ \, \int_0^T v(X^{u_{\theta}}_t,t) \, \mathrm{d}B_t} 
        \\ &\qquad\quad - \langle a_T(\omega), X^{u_{\theta}}_T(\omega) \rangle + \langle a_0(\omega), X^{u_{\theta}}_0(\omega) \rangle + \int_0^T \big( \langle a_t(\omega), \tilde{b}_{\theta}(X^{u_{\theta}}_t(\omega),t) \rangle + \langle \frac{da_t}{dt}(\omega), X^{u_{\theta}}_t(\omega) \rangle \big) \, \mathrm{d}t \\ &\qquad\quad + \int_0^T \langle a_t(\omega), \sigma(t) \, \mathrm{d}B_t \rangle \big) \\
        &=  
        \int_0^T \nabla_{x_0} X^{u_{\theta}}_t(\omega)^{\top} \nabla_x \big( \frac{1}{2} \|v(X^{u_{\theta}}_t,t)\|^2 \! + \! f(X^{u_{\theta}}_t(\omega),t) \big) \, \mathrm{d}t + \nabla_{x_0} X^{u_{\theta}}_T(\omega)^{\top} \nabla_x g(X^{u_{\theta}}_T(\omega)) \\ &\qquad\quad 
        % \textcolor{blue}{+ \, \int_0^T \nabla_{\theta} X^{u_{\theta}}_t(\omega) \nabla_x v(X^{u_{\theta}}_t(\omega),t)) \, \mathrm{d}B_t} 
        - \nabla_{x_0} X^{u_{\theta}}_T(\omega)^{\top} a_T(\omega) + \nabla_{x_0} X^{u_{\theta}}_0(\omega)^{\top} a_0(\omega) \\ &\qquad\quad + \int_0^T \big( \nabla_{x_0} X^{u_{\theta}}_t(\omega)^{\top} \nabla_{x} \tilde{b}_{\theta}(X^{u_{\theta}}_t(\omega),t)^{\top} a_t(\omega)
        + \nabla_{x_0} X^{u_{\theta}}_t(\omega)^{\top} \frac{da_t}{dt}(\omega) \big) \, \mathrm{d}t
        \\ &= \int_0^T \nabla_{x_0} X^{u_{\theta}}_t(\omega)^{\top} \big( \nabla_x \big( \frac{1}{2} \|v(X^{u_{\theta}}_t,t)\|^2 \! + \! f(X^{u_{\theta}}_t(\omega),t) \big) + \nabla_{x} \tilde{b}_{\theta}(X^{u_{\theta}}_t(\omega),t)^{\top} a_t(\omega) + \frac{da_t}{dt}(\omega) \big) \, \mathrm{d}t \\ &\qquad\quad + \nabla_{x_0} X^{u_{\theta}}_T(\omega)^{\top} \big( \nabla_x g(X^{u_{\theta}}_T(\omega)) - a_T(\omega) \big) + a_0(\omega).
    \end{split}    
    \end{talign}
    In the last line we used that $\nabla_{x_0} X^{u_{\theta}}_0(\omega) = \nabla_{x_0} x_0 = \mathrm{I}$.
    If choose $a$ such that
    \begin{talign}
    \begin{split} \label{eq:a_rewritten}
        da_t(\omega) &= \big( - \nabla_{x} \tilde{b}_{\theta}(X^{u_{\theta}}_t(\omega),t)^{\top} a_t(\omega) - \nabla_x \big( \frac{1}{2} \|v(X^{u_{\theta}}_t,t)\|^2 \! + \! f(X^{u_{\theta}}_t(\omega),t) \big) \big) \, \mathrm{d}t, 
        % \, \textcolor{blue}{- \, \nabla_x v(X^{u_{\theta}}_t(\omega),t)) \, \mathrm{d}B_t}, 
        \\ a_T(\omega) &= \nabla_x g(X^{u_{\theta}}_T(\omega)),
    \end{split}
    \end{talign}
    which is the ODE \eqref{eq:cont_adjoint_1_app}-\eqref{eq:cont_adjoint_2_app}, then we obtain that
    \begin{talign}
        \nabla_{x_0} \big( \int_0^T \big(\frac{1}{2} \|v(X^{u_{\theta}}_t,t)\|^2 \! + \! f(X^{u_{\theta}}_t,t) \big) \, \mathrm{d}t \! + \! g(X^{u_{\theta}}_T) 
        \big) = a_0(\omega)
    \end{talign}
    Without loss of generality, this argument can be extended from $t=0$ to an arbitrary $t \in [0,1]$, which proves the first statement of the lemma.

    To prove \eqref{eq:nabla_theta_cost}, we similarly write
    \begin{talign}
    \begin{split}
        &\nabla_{\theta} \big( \int_0^T \big(\frac{1}{2} \|v(X^{u_{\theta}}_t,t)\|^2 \! + \! f(X^{u_{\theta}}_t,t) \big) \, \mathrm{d}t \! + \! g(X^{u_{\theta}}_T) 
        % \textcolor{blue}{+ \, \int_0^T v(X^{u_{\theta}}_t,t) \, \mathrm{d}B_t} 
        \big) 
        \\ &= \nabla_{\theta} \big( 
        \int_0^T \big(\frac{1}{2} \|v(X^{u_{\theta}}_t,t)\|^2 \! + \! f(X^{u_{\theta}}_t,t) \big) \, \mathrm{d}t \! + \! g(X^{u_{\theta}}_T)
        % \textcolor{blue}{+ \, \int_0^T v(X^{u_{\theta}}_t,t) \, \mathrm{d}B_t} 
        \\ &\qquad\quad - \langle a_T(\omega), X^{u_{\theta}}_T(\omega) \rangle + \langle a_0(\omega), X^{u_{\theta}}_0(\omega) \rangle + \int_0^T \big( \langle a_t(\omega), \tilde{b}_{\theta}(X^{u_{\theta}}_t(\omega),t) \rangle + \langle \frac{da_t}{dt}(\omega), X^{u_{\theta}}_t(\omega) \rangle \big) \, \mathrm{d}t \\ &\qquad\quad + \int_0^T \langle a_t(\omega), \sigma(t) \, \mathrm{d}B_t \rangle \big)
        \\ &=  
        \int_0^T \nabla_{\theta} X^{u_{\theta}}_t(\omega)^{\top} \nabla_x \big( \frac{1}{2} \|v(X^{u_{\theta}}_t,t)\|^2 \! + \! f(X^{u_{\theta}}_t(\omega),t) \big) \, \mathrm{d}t + \nabla_{\theta} X^{u_{\theta}}_T(\omega)^{\top} \nabla_x g(X^{u_{\theta}}_T(\omega)) \\ &\qquad\quad 
        % \textcolor{blue}{+ \, \int_0^T \nabla_{\theta} X^{u_{\theta}}_t(\omega) \nabla_x v(X^{u_{\theta}}_t(\omega),t)) \, \mathrm{d}B_t} 
        - \nabla_{\theta} X^{u_{\theta}}_T(\omega)^{\top} a_T(\omega) + \nabla_{\theta} X^{u_{\theta}}_0(\omega)^{\top} a_0(\omega) \\ &\qquad\quad + \int_0^T \big( \nabla_{\theta} X^{u_{\theta}}_t(\omega)^{\top} \nabla_{x} \tilde{b}_{\theta}(X^{u_{\theta}}_t(\omega),t)^{\top} a_t(\omega) + \nabla_{\theta} \tilde{b}_{\theta}(X^{u_{\theta}}_t(\omega),t)^{\top} a_t(\omega) + \nabla_{\theta} X^{u_{\theta}}_t(\omega)^{\top} \frac{da_t}{dt}(\omega) \big) \, \mathrm{d}t 
        \\ &= \int_0^T \nabla_{\theta} X^{u_{\theta}}_t(\omega)^{\top} \big( \nabla_x \big( \frac{1}{2} \|v(X^{u_{\theta}}_t,t)\|^2 \! + \! f(X^{u_{\theta}}_t(\omega),t) \big) + \nabla_{x} \tilde{b}_{\theta}(X^{u_{\theta}}_t(\omega),t)^{\top} a_t(\omega) + \frac{da_t}{dt}(\omega) \big) \, \mathrm{d}t \\ &\qquad\quad + \nabla_{\theta} X^{u_{\theta}}_T(\omega)^{\top} \big( \nabla_x g(X^{u_{\theta}}_T(\omega)) - a_T(\omega) \big) 
        % \textcolor{blue}{+ \, \int_0^T \nabla_{\theta} X^{u_{\theta}}_t(\omega) \nabla_x v(X^{u_{\theta}}_t(\omega),t)) \, \mathrm{d}B_t} 
        % \\ &\qquad\quad 
        + \int_0^T (\nabla_{\theta} \tilde{b}_{\theta})(X^{u_{\theta}}_t(\omega),t)^{\top} a_t(\omega) \, \mathrm{d}t.
    \end{split}
    \end{talign}
    In the last line we used that $\nabla_{\theta} X^{u_{\theta}}_0(\omega) = \nabla_{\theta} x = 0$.
    % If choose $a$ such that
    % \begin{talign}
    % \begin{split}
    %     da_t(\omega) &= \big( - \nabla_{x} \tilde{b}_{\theta}(X^{u_{\theta}}_t(\omega),t) a_t(\omega) - \nabla_x \big( \frac{1}{2} \|v(X^{u_{\theta}}_t,t)\|^2 \! + \! f(X^{u_{\theta}}_t(\omega),t) \big) \big) \, \mathrm{d}t, 
    %     % \, \textcolor{blue}{- \, \nabla_x v(X^{u_{\theta}}_t(\omega),t)) \, \mathrm{d}B_t}, 
    %     \\ a_T(\omega) &= \nabla_x g(X^{u_{\theta}}_T(\omega)),
    % \end{split}
    % \end{talign}
    % which is the Continuous Adjoint ODE \eqref{eq:cont_adjoint_1}-\eqref{eq:cont_adjoint_2}, 
    When $a$ satisfies \eqref{eq:a_rewritten}, we obtain that
    \begin{talign}
    \begin{split}
        &\nabla_{\theta} \big( \int_0^T \big(\frac{1}{2} \|v(X^{u_{\theta}}_t,t)\|^2 \! + \! f(X^{u_{\theta}}_t,t) \big) \, \mathrm{d}t \! + \! g(X^{u_{\theta}}_T)
        % \textcolor{blue}{+ \, \int_0^T v(X^{u_{\theta}}_t,t) \, \mathrm{d}B_t} 
        \big) \\ &= \int_0^T (\nabla_{\theta} \tilde{b}_{\theta})(X^{u_{\theta}}_t(\omega),t) a_t(\omega) \, \mathrm{d}t = \int_0^T (\nabla_{\theta} u_{\theta})(X^{u_{\theta}}_t(\omega),t)^{\top} \sigma(t)^{\top} a_t(\omega) \, \mathrm{d}t.
    \end{split}
    \end{talign}
    The last equality holds because $\tilde{b}_{\theta}(x,t) := b(x,t) + \sigma(t) u_{\theta}(x,t)$. 
\end{proof}

\subsection{Proof of \Cref{prop:continuous_adjoint_loss_main}: Theoretical guarantees of the basic Adjoint Matching loss}
\label{subsec:derivation_continuous}

Let $\bar{u} = \texttt{stopgrad}(u_{\theta})$. We can rewrite equation \eqref{eq:continuous_adjoint_grads} as:
\begin{talign}\label{eq:continuous_adjoint_grads_app}
    \nabla_{\theta} \mathcal{L}(u_{\theta}; \fX^{\bar{u}}) &=  
    \frac{1}{2}\int_0^1 \nabla_{\theta} \norm{u_{\theta}(X^{\bar{u}}_t, t)}^2 \mathrm{d} t
    +\int_0^1 \nabla_{\theta} u(X^{\bar{u}}_t, t)\tran{} \sigma(t)\tran{} a(t; \fX^{\bar{u}}, \bar{u}) \mathrm{d} t \\ &= \frac{1}{2}\int_0^1 \nabla_{\theta} \norm{u_{\theta}(X^{\bar{u}}_t, t) + \sigma(t)\tran{} a(t; \fX^{\bar{u}}, \bar{u})}^2 \mathrm{d} t = \nabla_{\theta} \mathcal{L}_{\mathrm{Basic-Adj-Match}}(u_{\theta}; \fX^{\bar{u}})
\end{talign}
This proves the first statement of the proposition. To prove that the only critical point of the expected basic Adjoint Matching loss is the optimal control, we first compute the first variation of $\mathbb{E}[\mathcal{L}_{\mathrm{Basic-Adj-Match}}]$. Letting $v : \mathbb{R}^d \times [0,T] \to \mathbb{R}^d$ be arbitrary, we have that
    \begin{talign}
    \begin{split}
        &\frac{\mathrm{d}}{\mathrm{d}\epsilon} \mathbb{E}[\mathcal{L}_{\mathrm{Basic-Adj-Match}} (u + \epsilon v; \fX^{\bar{u}})] = \frac{\mathrm{d}}{\mathrm{d}\epsilon} \mathbb{E} \big[ \frac{1}{2} \int_0^T \| (u+\epsilon v)(X^{\bar{u}}_t,t) + \sigma(t)^{\top} a(t,X^{\bar{u}},\bar{u}) \|^2 \, \mathrm{d}t \big] \\ &= \mathbb{E} \big[ \int_0^T \langle v(X^{\bar{u}}_t,t), u(X^{\bar{u}}_t,t) + \sigma(t)^{\top} a(t,X^{\bar{u}},\bar{u}) \rangle \, \mathrm{d}t \big] \\ &= \mathbb{E} \big[ \int_0^T \langle v(X^{\bar{u}}_t,t), u(X^{\bar{u}}_t,t) + \sigma(t)^{\top} \mathbb{E}\big[a(t,X^{\bar{u}},\bar{u}) | X^{\bar{u}}_t \big] \rangle \, \mathrm{d}t \big] \\
        &\implies \frac{\delta}{\delta u} \mathbb{E}[\mathcal{L}_{\mathrm{Basic-Adj-Match}}(u)(x,t) = u(x,t) + \mathbb{E}\big[a(t,X^{\bar{u}},\bar{u}) | X^{\bar{u}}_t = x \big] 
    \end{split}
    \end{talign}
    Hence, critical points satisfy that
    \begin{talign} 
    \begin{split} \label{eq:critical_point_cont_adj}
        u(x,t) &= -\sigma(t)^{\top} \mathbb{E}[a(t,X^u,u)|X^u_t=x] = - \sigma(t)^{\top} \mathbb{E} \big[ \nabla_{X^v_t} \int_t^T \big(\frac{1}{2} \|v(X^v_t,t)\|^2 \! + \! f(X^v_t,t) \big) \, \mathrm{d}t \! + \! g(X^v_T) | X^v_0 = x \big] \\ &= - \sigma(t)^{\top} \nabla_{x} \mathbb{E} \big[ \int_t^T \big(\frac{1}{2} \|v(X^v_t,t)\|^2 \! + \! f(X^v_t,t) \big) \, \mathrm{d}t \! + \! g(X^v_T) | X^v_0 = x \big] = - \sigma(t)^{\top} \nabla J(u;x,t),
    \end{split}
    \end{talign}
    % where the last equality holds by equation \eqref{eq:expectation_a_cont_adjoint}. 
    In this equation, the second equality holds by equation \eqref{eq:adjoint_grad_x} from \Cref{lem:adjoint_state_properties}, and the third equality holds by the Leibniz rule.
    
    \Cref{eq:lemma_cost_functional} shows that any control $u$ that satisfies \eqref{eq:critical_point_cont_adj} is equal to the optimal control, which concludes the proof. 

\begin{lemma} \label{eq:lemma_cost_functional}
    Suppose that for any $x \in \mathbb{R}^d$, $t \in [0,T]$, $u(x,t) = - \sigma(t)^{\top} \nabla_x J(u;x,t)$. Then, $J(u;\cdot,\cdot)$ satisfies the Hamilton-Jacobi-Bellman equation \eqref{eq:HJB_setup}. By the uniqueness of the solution to the HJB equation, we have that $J(u;x,t) = V(x,t)$ for any $x \in \mathbb{R}^d$, $t \in [0,T]$. Hence, $u(x,t) = - \sigma(t)^{\top} \nabla_x V(x,t)$ is the optimal control.
\end{lemma}
\begin{proof}
    Since $J(u;x,t) = \mathbb{E} \big[ \int_t^T \big( \frac{1}{2}\|u(X^u_t,t)\|^2 + f(X^u_t,t) \big) \, ds + g(X^u_T) | X^u_t = x \big]$, we have that
    \begin{talign}
        J(u;x,t) = \mathbb{E} \big[J(u;X^u_{t+\Delta t},t+\Delta t) | X_t = x \big] + \mathbb{E} \big[ \int_t^{t+\Delta t}
        \big( \frac{1}{2}\|u(X^u_s,s)\|^2 + f(X^u_s,s) \big) \, ds | X_t = x \big],
    \end{talign}
    which means that
    \begin{align} \label{eq:pre_limit}
        0 = \frac{\mathbb{E} [J(u;X^u_{t+\Delta t},t+\Delta t) | X_t = x ] - J(u;x,t)}{\Delta t} + \frac{\mathbb{E} \big[ \int_t^{t+\Delta t}
        \big( \frac{1}{2}\|u(X^u_t,t)\|^2 + f(X^u_t,t) \big) \, ds | X_t = x \big]}{\Delta t}
    \end{align}
    Recall that the generator $\mathcal{T}^u$ of the controlled SDE \eqref{eq:controlled_SDE} takes the form:
    \begin{talign}
    \begin{split}
        \mathcal{T}^u f(x,t) &:= \lim_{\Delta t \to 0} \frac{\mathbb{E} \big[f(X^u_{t+\Delta t},t) | X_t = x \big] - f(x,t)}{\Delta t} \\ &= \partial_t f(x,t) + \langle \nabla f(x,t), b(x,t) + \sigma(t) u(x,t) \rangle + \mathrm{Tr}\big( \frac{\sigma(t) \sigma(t)^{\top}}{2} \nabla^2 f(x,t) \big)
    \end{split}
    \end{talign}
    Hence, if we take the limit $\Delta t \to 0$ on equation \eqref{eq:pre_limit}, we obtain that:
    \begin{talign} 
    \begin{split} \label{eq:HJB_discounted_1}
        0 &= \mathcal{T}^u J(u;x,t) + \frac{1}{2}\|u(x,t)\|^2 + f(x,t) \\ &= \partial_t J(u;x,t) + \langle \nabla J(u;x,t), b(x,t) + \sigma(t) u(x,t) \rangle + \mathrm{Tr}\big( \frac{\sigma(t) \sigma(t)^{\top}}{2} \nabla^2 J(u;x,t) \big) + \frac{1}{2}\|u(x,t)\|^2 + f(x,t).
    \end{split}
    \end{talign}
    % Setting $C(t,x,u) = - \nabla \cdot b_t(x) + \frac{1}{2}\|x\|^2$, equation \eqref{eq:HJB_discounted_1} becomes
    Now using that $u(x,t) = - \sigma(t)^{\top} \nabla_x J(u;x,t)$, we have that 
    \begin{talign}
    \begin{split}
        \langle \nabla J(u;x,t), \sigma(t) u(x,t) \rangle + \frac{1}{2}\|u(x,t)\|^2 &= - \|\sigma(t)^{\top} \nabla_x J(u;x,t) \|^2 + \frac{1}{2} \|\sigma(t)^{\top} \nabla_x J(u;x,t) \|^2 \\ &= - \frac{1}{2} \|\sigma(t)^{\top} \nabla_x J(u;x,t) \|^2.
    \end{split}
    \end{talign}
    Plugging this back into \eqref{eq:HJB_discounted_1}, we obtain that
    \begin{talign} 
    \begin{split} \label{eq:HJB_discounted_2}
        0 &= \partial_t J(u;x,t) + \langle \nabla J(u;x,t), b(x,t) \rangle + \mathrm{Tr}\big( \frac{\sigma(t) \sigma(t)^{\top}}{2} \nabla^2 J(u;x,t) \big) - \frac{1}{2} \|\sigma(t)^{\top} \nabla_x J(u;x,t) \|^2 + f(x,t).
    \end{split}
    \end{talign}
    And since $J(u;x, T) = g(x)$ by construction, we conclude that $J(u;x,t)$ satisfies the HJB equation \eqref{eq:HJB_setup}.
    \end{proof}

% \section{Proofs of \Cref{sec:new_losses}}

\subsection{Theoretical guarantees of the Adjoint Matching loss}
\label{subsec:proof_lean_adjoint}

\begin{proposition}[Theoretical guarantee of the Adjoint Matching loss] \label{prop:lean_adjoint}
    The only critical point of the loss $\mathbb{E}[\mathcal{L}_{\mathrm{Adj-Match}}]$ is the optimal control $u^*$.
\end{proposition}

\begin{proof}
    Let $v$ be an arbitrary control. If $\tilde{a}(t; \mathbf{X}^v)
    % := \tilde{a}(\omega,t)
    $ is the solution of the Lean Adjoint ODE \eqref{eq:lean_adjoint_1}-\eqref{eq:lean_adjoint_2},
    it satisfies the integral equation 
    \begin{talign}
    \tilde{a}(t; \mathbf{X}^v) = \int_t^T \big( \nabla_x b(X^v_s, s)^\top \tilde{a}(s; \mathbf{X}^v) + \nabla_x f(X^v_s, s) \big) \, \mathrm{d}s + \nabla g(X^v_T).
    \end{talign}
    Hence,
    \begin{talign}
    \begin{split} \label{eq:expected_lean_integral}
    \mathbb{E}\big[\tilde{a}(t; \mathbf{X}^v) \big| X^v_t \big] &= \mathbb{E}\big[\int_t^T \big( \nabla_x b(X^v_s, s)^\top \tilde{a}(s; \mathbf{X}^v) + \nabla_x f(X^v_s, s) \big) \, \mathrm{d}s + \nabla g(X^v_T) \big| X^v_t \big] \\
    &= \mathbb{E}\big[\int_t^T \big( \nabla_x b(X^v_s, s)^{\top} \mathbb{E}\big[\tilde{a}(s; \mathbf{X}^v) \big| X^v_s \big] + \nabla_x f(X^v_s, s) \big) \, \mathrm{d}s + \nabla g(X^v_T) \big| X^v_t \big],
    \end{split}
    \end{talign}
    where we used the tower property of conditional expectation in the second equality.

    Similarly, if $a(t; \mathbf{X}^v, v)$ is the solution of the Adjoint ODE \eqref{eq:cont_adjoint_1}-\eqref{eq:cont_adjoint_2}, it satisfies the integral equation
    \begin{talign}
    a(t; \mathbf{X}^v, v) = \int_t^T \big( \nabla_x \big( b(X^v_s, s)^{\top} a(s; \mathbf{X}^v, v) + \sigma(s) v(X^v_s,s) \big) + \nabla_x \big( f(X^v_s, s) + \frac{1}{2} \|v(X^v_s,s)\|^2 \big) \big) \, \mathrm{d}s + \nabla g(X^v_T),
    \end{talign}
    and its expected value satisfies
    \begin{talign}
    \begin{split} \label{eq:expected_full_integral}
    &\mathbb{E}\big[a(t; \mathbf{X}^v, v)\big| X^v_t \big] \\ &= \mathbb{E}\big[\int_t^T \big( \nabla_x \big( b(X^v_s, s) + \sigma(s) v(X^v_s,s) \big)^{\top} a(s; \mathbf{X}^v, v) + \nabla_x \big( f(X^v_s, s) + \frac{1}{2} \|v(X^v_s,s)\|^2 \big) \big) \, \mathrm{d}s + \nabla g(X^v_T)\big| X^v_t \big] \\ &= \! \mathbb{E}\big[\int_t^T \big( \nabla_x \big( b(X^v_s, s) \! + \! \sigma(s) v(X^v_s,s) \big)^{\top} \mathbb{E}\big[ a(s; \mathbf{X}^v, v) \big| X^v_s \big] \! + \! \nabla_x \big( f(X^v_s, s) \! + \! \frac{1}{2} \|v(X^v_s,s)\|^2 \big) \big) \, \mathrm{d}s \! + \! \nabla g(X^v_T)\big| X^v_t \big].
    \end{split}
    \end{talign}
    
    Let us rewrite $\mathbb{E}[\mathcal{L}_{\mathrm{Adj-Match}}]$ as follows:
    \begin{talign}
    \begin{split} \label{eq:lean_adjoint_rewritten}
        \mathbb{E}[\mathcal{L}_{\mathrm{Adj-Match}}(u)] &:= \mathbb{E} \big[\int_0^{T} \big\| u(X^v_t,t)
        + \sigma(t)^{\top} \mathbb{E}\big[ \tilde{a}(t,\mathbf{X}^v) | X^v_{t} \big] \big\|^2 \, \mathrm{d}t \big] \rvert_{v = \mathrm{stopgrad}(u)} \\ &\qquad + \mathbb{E} \big[\int_0^{T} \big\| \sigma(t)^{\top} \big(\mathbb{E}\big[ \tilde{a}(t,\mathbf{X}^v) | X^v_{t} \big] - \tilde{a}(t,\mathbf{X}^v)  \big)\big\|^2 \, \mathrm{d}t \big] \rvert_{v = \mathrm{stopgrad}(u)},
    \end{split}
    \end{talign}
    Now, suppose that $\hat{u}$ is a critical point of $\mathbb{E}[\mathcal{L}_{\mathrm{Adj-Match}}]$. By definition, this implies that the first variation of $\mathbb{E}[\mathcal{L}_{\mathrm{Adj-Match}}]$ is zero. Using \eqref{eq:lean_adjoint_rewritten}, we can write this as follows:
    \begin{talign}
        &0 = \frac{\delta}{\delta u} \mathbb{E}[\mathcal{L}_{\mathrm{Adj-Match}}(\hat{u})](x) = 2 \big(\hat{u}%(X^{\hat{u}}_t,t)
        (x,t)
        + \sigma(t)^{\top} \mathbb{E}[\tilde{a}(t,\mathbf{X}^{\hat{u}}) | X^{\hat{u}}_t = x] \big), \\ &\implies \hat{u}(x,t) = - \sigma(t)^{\top} \mathbb{E}[\tilde{a}(t,\mathbf{X}^{\hat{u}})| X^{\hat{u}}_t = x].
        \label{eq:AM_critical}
    \end{talign}
    Hence, we have
    \begin{talign}
        % - 
        &\nabla_x  \hat{u}(X^{\hat{u}}_t,t)^{\top} \sigma(t)^{\top} \mathbb{E}[\tilde{a}(t,\mathbf{X}^{\hat{u}}) |X^{\hat{u}}_t] 
        % - 
        + \nabla_x {\hat{u}}(X^{\hat{u}}_t,t)^{\top} {\hat{u}}(X^{\hat{u}}_t,t) = 0, \\
        &\implies \mathbb{E}\big[\int_t^T \big( \nabla_x \big( \sigma(s) \hat{u}(X^{\hat{u}}_s,s) \big)^{\top} \mathbb{E}\big[ \tilde{a}(s; \mathbf{X}^{\hat{u}}) \big| X^{\hat{u}}_s \big] \! + \! \nabla_x \big( \frac{1}{2} \|{\hat{u}}(X^{\hat{u}}_s,s)\|^2 \big) \big) \, \mathrm{d}s \big| X^{\hat{u}}_t \big] = 0.
        \label{eq:zero_equality}
    \end{talign}
    If we set $v = \hat{u}$ in equation \eqref{eq:expected_lean_integral}, and add \eqref{eq:zero_equality} to its right-hand side, we obtain that $\mathbb{E}[\tilde{a}(t,X^{\hat{u}})|X^{\hat{u}}_{t}]$ also solves the integral equation
    \begin{talign}
    \begin{split}
        &\mathbb{E}\big[\tilde{a}(t; \mathbf{X}^{\hat{u}})\big| X^{\hat{u}}_t \big] \\ &= \! \mathbb{E}\big[\int_t^T \big( \nabla_x \big( b(X^{\hat{u}}_s, s) \! + \! \sigma(s) {\hat{u}}(X^{\hat{u}}_s,s) \big)^{\top} \mathbb{E}\big[ \tilde{a}(s; \mathbf{X}^{\hat{u}}) \big| X^{\hat{u}}_s \big] \! + \! \nabla_x \big( f(X^{\hat{u}}_s, s) \! + \! \frac{1}{2} \|{\hat{u}}(X^{\hat{u}}_s,s)\|^2 \big) \big) \, \mathrm{d}s \! + \! \nabla g(X^{\hat{u}}_T)\big| X^{\hat{u}}_t \big].
    \end{split}
    \end{talign}
    Note that this integral equation is the same one as equation \eqref{eq:expected_full_integral} when we set $v = \hat{u}$ in the latter. \Cref{prop:uniqueness_integral} states that the solution of the integral equation is unique, which means that $\mathbb{E}\big[\tilde{a}(t; \mathbf{X}^{\hat{u}})\big| X^{\hat{u}}_t \big] = \mathbb{E}\big[a(t; \mathbf{X}^{\hat{u}},\hat{u})\big| X^{\hat{u}}_t \big]$ for all $t \in [0,T]$.

    Since we can reexpress the basic Adjoint Matching loss as
    \begin{talign}
    \begin{split} \label{eq:cont_adjoint_rewritten}
        \mathbb{E}[\mathcal{L}_{\mathrm{Basic-Adj-Match}}(u)] &:= \mathbb{E} \big[\int_0^{T} \big\| u(X^v_t,t)
        + \sigma(t)^{\top} \mathbb{E}\big[ a(t;\mathbf{X}^v,v) | X^v_{t} \big] \big\|^2 \, \mathrm{d}t \big] \rvert_{v = \mathrm{stopgrad}(u)} \\ &\qquad + \mathbb{E} \big[\int_0^{T} \big\| \sigma(t)^{\top} \big(\mathbb{E}\big[ a(t;\mathbf{X}^v,v) | X^v_{t} \big] - a(t;\mathbf{X}^v,v)  \big)\big\|^2 \, \mathrm{d}t \big] \rvert_{v = \mathrm{stopgrad}(u)},
    \end{split}
    \end{talign}
    we obtain that when ${\hat{u}}$ is a critical point of $\mathbb{E}[\mathcal{L}_{\mathrm{Adj-Match}}]$,
    \begin{talign}
    \begin{split}
        \frac{\mathrm{d}}{\mathrm{d}u} \mathbb{E}[\mathcal{L}_{\mathrm{Basic-Adj-Match}}(\hat{u})](x) &= 2 \big(\hat{u}(x,t)
        + \sigma(t)^{\top} \mathbb{E}[a(t;\mathbf{X}^{\hat{u}},\hat{u}) | X^{\hat{u}}_t = x] \big) \\ &= 2 \big(\hat{u}(x,t)
        + \sigma(t)^{\top} \mathbb{E}[\tilde{a}(t;\mathbf{X}^{\hat{u}}) | X^{\hat{u}}_t = x] \big) = 0,
    \end{split}
    \end{talign}
    where the second equality holds because $\mathbb{E}\big[\tilde{a}(t; \mathbf{X}^{\hat{u}})\big| X^{\hat{u}}_t \big] = \mathbb{E}\big[a(t; \mathbf{X}^{\hat{u}},\hat{u})\big| X^{\hat{u}}_t \big]$, and the third equality holds by equation \eqref{eq:AM_critical}.
    Thus, we deduce that the critical points of $\mathbb{E}[\mathcal{L}_{\mathrm{Adj-Match}}]$ are critical points of $\mathbb{E}[\mathcal{L}_{\mathrm{Basic-Adj-Match}}]$. By \Cref{prop:continuous_adjoint_loss_main}, $\mathbb{E}[\mathcal{L}_{\mathrm{Basic-Adj-Match}}]$ has a single critical point, which is the optimal control $u^*$, which concludes the proof of the statement for $\mathbb{E}[\mathcal{L}_{\mathrm{Adj-Match}}]$.
    \end{proof}

\begin{proposition} \label{prop:uniqueness_integral}
    Let $v$ be an arbitrary control. Consider the integral equation:
    \begin{talign}
         Y_t = 
         % \mathbb{E}\big[\int_t^T \big(Y_s^\top \nabla_x b(X_s^*, s) + \nabla_x f(X_s^*, s) \big) \mathrm{d}s  + \nabla g(X^*_T) \big| X^*_t\big],
         \mathbb{E}\big[\int_t^T \big( \nabla_x \big( b(X^v_s, s) \! + \! \sigma(s) v(X^v_s,s) \big)^{\top} Y_s \! + \! \nabla_x \big( f(X^v_s, s) \! + \! \frac{1}{2} \|v(X^v_s,s)\|^2 \big) \big) \, \mathrm{d}s \! + \! \nabla g(X^v_T)\big| X^v_t \big],
    \end{talign}
    where $t \in [0,T]$. This equation has a unique solution, i.e. if $Y^1$, $Y^2$ are two solutions then $Y_1 = Y_2$.
\end{proposition}
\begin{proof}
    Let $Y^1$, $Y^2$ be two solutions of the integral equation. We have that
    \begin{talign} 
        Y^1_t - Y^2_t = \mathbb{E}\big[\int_t^T \big((Y^1_s - Y^2_s)^\top \nabla_x b(X_s^*, s) \big) \mathrm{d}s \big| X^*_t\big].
    \end{talign}
    Thus,
    \begin{talign}
    \begin{split}
        &\|Y^1_t - Y^2_t\| \\ &\leq \mathbb{E}\big[\big\|\int_t^T \big((Y^1_s - Y^2_s)^\top \nabla_x b(X_s^*, s) \big) \mathrm{d}s \big\| \big| X^*_t\big] \leq \mathbb{E}\big[\int_t^T \big\|\big((Y^1_s - Y^2_s)^\top \nabla_x b(X_s^*, s) \big) \big\| \mathrm{d}s \big| X^*_t\big] \\ &\leq \mathbb{E}\big[\int_t^T \big\|Y^1_s - Y^2_s\big\| \cdot \big\| \nabla_x b(X_s^*, s) \big) \big\| \mathrm{d}s \big| X^*_t\big] = \int_t^T \mathbb{E}\big[\big\|Y^1_s - Y^2_s\big\| \cdot \big\| \nabla_x b(X_s^*, s) \big) \big\| \big| X^*_t\big] \mathrm{d}s \\ &\leq \int_t^T \big(\mathbb{E}\big[\big\|Y^1_s - Y^2_s\big\|^2 \big| X^*_t \big] \big)^{1/2} \cdot \big( \mathbb{E}\big[\big\| \nabla_x b(X_s^*, s) \big\|^2 \big| X^*_t\big] \big)^{1/2} \mathrm{d}s
    \end{split}
    \end{talign}
    And this implies that
    \begin{talign}
    \begin{split}
        &\sup_{t' \in [0,t]} \big(\mathbb{E}[\|Y^1_t - Y^2_t\|^2 | X^*_{t'} ] \big)^{1/2} \\ &\leq \int_t^T \big(\mathbb{E}\big[\big\|Y^1_s - Y^2_s\big\|^2 \big| X^*_t \big] \big)^{1/2} \cdot \big( \mathbb{E}\big[\big\| \nabla_x b(X_s^*, s) \big\|^2 \big| X^*_t\big] \big)^{1/2} \mathrm{d}s \\ &\leq \int_t^T \sup_{t' \in [0,s]} \big(\mathbb{E}\big[\big\|Y^1_s - Y^2_s\big\|^2 \big| X^*_{t'} \big] \big)^{1/2} \cdot \sup_{t' \in [0,s]} \big( \mathbb{E}\big[\big\| \nabla_x b(X_s^*, s) \big\|^2 \big| X^*_{t'}\big] \big)^{1/2} \mathrm{d}s.
    \end{split}
    \end{talign}
    Applying Grönwall's inequality on the function $f(t) = \sup_{t' \in [0,t]} \big(\mathbb{E}[\|Y^1_t - Y^2_t\|^2 | X^*_{t'} ] \big)^{1/2}$, we obtain that $\sup_{t' \in [0,t]} \big(\mathbb{E}[\|Y^1_t - Y^2_t\|^2 | X^*_{t'} ] \big)^{1/2} = 0$ for all $t \in [0,T]$, which means that $Y^1_t = Y^2_t$ almost surely. And since $\|Y^1_t - Y^2_t\| \leq \int_t^T \big(\mathbb{E}\big[\big\|Y^1_s - Y^2_s\big\|^2 | X^*_{t} \big] \big)^{1/2} \cdot \big( \mathbb{E}\big[\big\| \nabla_x b(X_s^*, s) \big\|^2 | X^*_{t} \big] \big)^{1/2} \mathrm{d}s = 0$, we obtain that $Y^1 = Y^2$.
\end{proof}

\subsection{Pseudo-code of Adjoint Matching for DDIM fine-tuning} \label{subsec:pseudocode_DDIM}

\begin{algorithm}[h]
\SetAlgoNoLine % Disable line numbering
\SetAlgoNlRelativeSize{0} %Set number line font to zero
\small{
\KwIn{Pre-trained denoiser $\epsilon^{\mathrm{base}}$, number of fine-tuning iterations $N$.}

Initialize fine-tuned denoiser: $\epsilon^{\mathrm{finetune}} = \epsilon^{\mathrm{base}}$ with parameters $\theta$.

  \For{$n \in \{0,\dots,N-1\}$}{
    Sample $m$ trajectories $\bm{X} = (X_t)_{t\in\{0, \dots, 1\}}$ 
    %with memoryless noise schedule $\sigma(t) = \sqrt{2 \beta_t (\frac{\dot{\alpha}_t}{\alpha_t} \beta_t - \dot{\beta}_t)}$, \eg :
    according to DDPM, \eg :
    \begin{talign} 
    \begin{split} \label{eq:EM_update_box_DDIM}
    % X_{k+1} = \sqrt{\bar{\alpha}_{k+1}}
    % \big( \frac{X_{k} - \sqrt{1-\bar{\alpha}_k} \epsilon(X_k,k)}{\sqrt{\bar{\alpha}_k}} \big)
    % + \sqrt{1-\bar{\alpha}_{k+1} - \sigma_{k}^2} \epsilon(X_k,k) +
    % % \frac{1}{\sqrt{K}} 
    % \sigma_{k} \varepsilon_{k},
    X_{k+1} &= \sqrt{\frac{\bar{\alpha}_{k+1}}{\bar{\alpha}_{k}}} \big( X_k - \frac{1-\bar{\alpha}_{k}/\bar{\alpha}_{k+1}}{\sqrt{1-\bar{\alpha}_k}} \epsilon^{\mathrm{finetune}}(X_k,k) \big) + \sqrt{\frac{1-\bar{\alpha}_{k+1}}{1-\bar{\alpha}_{k}} \big( 1 - \frac{\bar{\alpha}_k}{\bar{\alpha}_{k+1}} \big)} \varepsilon_k,
    \quad \varepsilon_k \sim \mathcal{N}(0,I), \ X_0 \sim \mathcal{N}(0,I),
    \end{split} \\
    \begin{split}
    \text{or } X_{k+1} &= X_{k} + \frac{\bar{\alpha}_{k+1} - \bar{\alpha}_{k}}{2\bar{\alpha}_{k}} X_{k} - \frac{\bar{\alpha}_{k+1}-\bar{\alpha}_{k}}{\bar{\alpha}_{k}\sqrt{1-\bar{\alpha}_k}} \epsilon^{\mathrm{finetune}}(X_k,k) + \sqrt{\frac{\bar{\alpha}_{k+1} - \bar{\alpha}_{k}}{\bar{\alpha}_{k}}} \varepsilon_k.
    \end{split} \label{eq:EM_update_box_DDIM_2}
    \end{talign}
    For each trajectory, solve the \textit{lean adjoint ODE} \eqref{eq:lean_adjoint_1}-\eqref{eq:lean_adjoint_2} backwards in time from $k=K$ to $0$, \eg:
    \begin{talign}
    \begin{split} \label{eq:Euler_lean_adjoint_DDIM}
        \tilde{a}_{k} 
        &= \tilde{a}_{k+1} + \tilde{a}_{k+1}\tran{} \nabla_{X_k} \left(
        \sqrt{\frac{\bar{\alpha}_{k+1}}{\bar{\alpha}_{k}}} \big( X_k - \frac{1-\bar{\alpha}_{k}/\bar{\alpha}_{k+1}}{\sqrt{1-\bar{\alpha}_k}} \epsilon^{\mathrm{base}}(X_k,k) \big) - X_k
        \right), \qquad 
        \tilde{a}_K = \nabla_{X_K} r(X_K),
    \end{split} \\
    \begin{split}
        \text{or } \tilde{a}_{k} 
        &= \tilde{a}_{k+1} + \tilde{a}_{k+1}\tran{} \nabla_{X_t} \left(
        % \frac{\dot{\alpha}_{t}}{2\alpha_{t}}
        \frac{\bar{\alpha}_{k+1} - \bar{\alpha}_{k}}{2\bar{\alpha}_{k}}
        X_k - \frac{\bar{\alpha}_{k+1}-\bar{\alpha}_{k}}{\bar{\alpha}_{k}\sqrt{1-\bar{\alpha}_k}} \epsilon^{\mathrm{base}}(X_k,k)
        \right), \qquad 
        \tilde{a}_K = \nabla_{X_K} r(X_K).
    \end{split} \label{eq:Euler_lean_adjoint_DDIM_2}
    \end{talign}
    Note that $X_k$ and $\tilde{a}_k$ should be computed without gradients, \ie, $X_k = \texttt{stopgrad}(X_k)$, $\tilde{a}_k = \texttt{stopgrad}(\tilde{a}_k)$. \vspace{0.5em}

    For each trajectory, compute the Adjoint Matching objective \eqref{eq:lean_adjoint_matching}: 
    \begin{talign} \begin{split} \label{eq:adj_matching_algorithm_box_DDIM}
    \mathcal{L}_{\mathrm{Adj-Match}}(\theta) &=
        \sum_{k\in\{0, \dots, K - 1\}} \big\|
        % \frac{2}{\sigma(t)} \big(v^{\mathrm{finetune}}_{\theta}(X_t, t) - v^{\mathrm{base}}(X_t, t) \big)
        % \sqrt{\frac{\bar{\alpha}_{k+1} - \bar{\alpha}_{k}}{\bar{\alpha}_k(1-\bar{\alpha}_k)}}
        \sqrt{\frac{\bar{\alpha}_{k+1}
        % (1-\bar{\alpha}_k)
        }{\bar{\alpha}_k(1-\bar{\alpha}_{k+1})} \big( 1 - \frac{\bar{\alpha}_k}{\bar{\alpha}_{k+1}} \big)}
        (\epsilon^{\mathrm{finetune}}(X_k,k) - \epsilon^{\mathrm{base}}(X_k,k))
        \\ &\qquad\qquad\qquad\qquad - 
        % \sqrt{\frac{\bar{\alpha}_{k+1} - \bar{\alpha}_{k}}{\bar{\alpha}_{k}}}
        \sqrt{\frac{1-\bar{\alpha}_{k+1}}{1-\bar{\alpha}_{k}} \big( 1 - \frac{\bar{\alpha}_k}{\bar{\alpha}_{k+1}} \big)}
        \tilde{a}_k \big\|^2,
    \end{split} \\
    \begin{split} \label{eq:adj_matching_algorithm_box_DDIM_2}
    \text{or } \mathcal{L}_{\mathrm{Adj-Match}}(\theta) &=
        \sum_{k\in\{0, \dots, K - 1\}} \big\|
        % \frac{2}{\sigma(t)} \big(v^{\mathrm{finetune}}_{\theta}(X_t, t) - v^{\mathrm{base}}(X_t, t) \big)
        \sqrt{\frac{\bar{\alpha}_{k+1} - \bar{\alpha}_{k}}{\bar{\alpha}_k(1-\bar{\alpha}_k)}} (\epsilon^{\mathrm{finetune}}(X_k,k) - \epsilon^{\mathrm{base}}(X_k,k))
        - \sqrt{\frac{\bar{\alpha}_{k+1} - \bar{\alpha}_{k}}{\bar{\alpha}_{k}}} \tilde{a}_k \big\|^2.
    \end{split}
    \end{talign}
    
    Compute the gradient $\nabla_{\theta} \mathcal{L}(\theta)$ and update $\theta$ using favorite gradient descent algorithm.
  }
\KwOut{Fine-tuned vector field $v^{\mathrm{finetune}}$}}
\caption{Adjoint Matching for fine-tuning DDIM}
\label{alg:adjoint_matching_finetuning_DDIM}
\end{algorithm}

% \textcolor{red}{Finish}
% \begin{talign}\label{eq:euler_maruyama_DDPM_app}
%     \mathrm{d}X_t &= \big( \frac{\dot{\bar{\alpha}}_{t}}{2\bar{\alpha}_{t}} X_t - \frac{\dot{\bar{\alpha}}_{t}}{\bar{\alpha}_{t}} \frac{\epsilon^\text{base}(X_{t},t)}{\sqrt{1-\bar{\alpha}_{t}}} \big) \mathrm{d}t + \sqrt{\frac{\dot{\bar{\alpha}}_{t}}{\bar{\alpha}_{t}}} \mathrm{d}B_t, \qquad X_{0} \sim \mathcal{N}(0,I),
% \end{talign}
Note that for each pair of equations \eqref{eq:EM_update_box_DDIM}-\eqref{eq:EM_update_box_DDIM_2}, \eqref{eq:Euler_lean_adjoint_DDIM}-\eqref{eq:Euler_lean_adjoint_DDIM_2}, 
\eqref{eq:adj_matching_algorithm_box_DDIM}-\eqref{eq:adj_matching_algorithm_box_DDIM_2}, the first equation corresponds to the updates in the DDPM paper, while the second equation is an Euler-Maruyama / Euler discretization of the continuous-time object. 
To check that both discretizations are equal up to first order, remark that
\begin{talign}
    \sqrt{\frac{\bar{\alpha}_{k+1}}{\bar{\alpha}_{k}}} = \sqrt{1 + \frac{\bar{\alpha}_{k+1} - \bar{\alpha}_{k}}{\bar{\alpha}_{k}}} \approx 1 + \frac{\bar{\alpha}_{k+1} - \bar{\alpha}_{k}}{2\bar{\alpha}_{k}} + O((\bar{\alpha}_{k+1} - \bar{\alpha}_{k})^2).
\end{talign}

\section{Adapting diffusion fine-tuning baselines to flow matching}

\subsection{Adapting ReFL \citep{xu2023imagereward} to flow matching}
Reward Feedback Learning (ReFL) is a diffusion fine-tuning algorithm introduced by \cite{xu2023imagereward} which tries to increase the reward on denoised samples. Namely, if $\bm{X} = (X_t)_{t \in [0,1]}$ is the solution of the DDPM SDE \eqref{eq:euler_maruyama_DDPM}, we can denoise $X_t$ as
\begin{talign}
    \hat{X}_1(X_t) = \frac{X_t - \sqrt{1-\bar{\alpha}_t} \epsilon(X_t,t)}{\sqrt{\bar{\alpha}_t}}.
\end{talign}
This equation follows from the stochastic interpolant equation \eqref{eq:reference_flow} if we replace $\bar{X}_0$ with the noise predictor $\epsilon(X_t,t)$. And then, the ReFL optimization update is based on the gradient:
\begin{talign}
    \nabla_{\theta} r(\hat{X}_1(X_t)) = \nabla_{\theta} r\big(\frac{X_t - \sqrt{1-\bar{\alpha}_t} \epsilon_{\theta}(X_t,t)}{\sqrt{\bar{\alpha}_t}} \big),
\end{talign}
where the trajectories have been detached. 

To adapt ReFL to Flow Matching, we need to express the denoiser map in terms of the vector field $v$. We have that
\begin{talign}
\begin{split}
    v(x,t) &= \mathbb{E} \big[\dot{\beta}_t \bar{X}_0 + \dot{\alpha}_t \bar{X}_1 \big| \beta_t \bar{X}_0 + \alpha_t \bar{X}_1 = x \big] = \mathbb{E} \big[ \frac{\dot{\beta}_t}{\beta_t} \big( \beta_t \bar{X}_0 + \alpha_t \bar{X}_1 \big) + \big( \dot{\alpha}_t - \frac{\dot{\beta}_t}{\beta_t} \alpha_t \big) \bar{X}_1 \big| \beta_t \bar{X}_0 + \alpha_t \bar{X}_1 = x \big] \\ &= \frac{\dot{\beta}_t}{\beta_t} x + \big( \dot{\alpha}_t - \frac{\dot{\beta}_t}{\beta_t} \alpha_t \big) \hat{X}_1(x,t).
\end{split}
\end{talign}
where we defined the denoiser map $\hat{X}_1(x,t) := \mathbb{E}\big[\bar{X}_1|\beta_t \bar{X}_0 + \alpha_t \bar{X}_1 = x\big]$.
Hence, 
\begin{talign} \label{eq:FM_denoiser}
\hat{X}_1(x,t) = \frac{v(x,t) - \frac{\dot{\beta}_t}{\beta_t} x}{\dot{\alpha}_t - \frac{\dot{\beta}_t}{\beta_t} \alpha_t}.
\end{talign}
% \begin{talign}
% \end{talign}

\subsection{Adapting Diffusion-DPO \citep{wallace2023diffusion} to flow matching}
The Diffusion-DPO loss assumes access to ranked pairs of generated samples $x_1^w \succ x_1^l$, where $x^w$ and $x^l$ are the winning and losing samples. For DDPM, the loss implemented in practice reads \citep[Eq.~46]{wallace2023diffusion}:
\begin{talign}
\begin{split}
    L_{\mathrm{DPO}}(\theta) &= - \mathbb{E}_{(x_1^w, x_1^l) \sim \mathcal{D}, k \sim U[0,K], x_{kh}^w \sim q(x_{kh}^w|x_1^w), x_t^l \sim q(x_{kh}^l|x_1^l)} \big[ \\ &\qquad \log S \big( - \frac{\tilde{\beta}}{2} 
    % K \omega(\lambda_t)
    \big( \|\varepsilon^w - \epsilon_{\theta}(x_{kh}^w,kh) \|^2 - \|\varepsilon^w - \epsilon_{\mathrm{ref}}(x_{kh}^w,kh) \|^2 \\ &\qquad\qquad\qquad\quad - \big( \|\varepsilon^l - \epsilon_{\theta}(x_{kh}^l,kh) \|^2 - \|\varepsilon^l - \epsilon_{\mathrm{ref}}(x_{kh}^l,kh) \|^2 \big) \big) \big) \big],
\end{split}
\end{talign}
where $S(x) = \frac{1}{1+e^{-x}}$ denotes the sigmoid function, and $q(x_{kh}^{*}|x_1^{*})$ is the conditional distribution of the forward process, i.e. $x_{kh}^{*}$ is sampled as $x_{kh}^* = \sqrt{\gamma_{kh}} x_1^{*} + \sqrt{1-\gamma_{kh}} \epsilon$, $\epsilon \sim N(0,I)$. Following the derivation of the Diffusion-DPO loss in \cite[Sec.~S4]{wallace2023diffusion}, we observe that the term $- \frac{\tilde{\beta}}{2} \|\varepsilon^w - \epsilon_{\theta}(x_{kh}^w,kh) \|^2$ arises from 
\begin{talign}
    -\frac{\tilde{\beta}}{2\frac{1-\gamma_{kh}}{\gamma_{kh}}} \|\hat{x}_1(x_{kh}^w) - x_1^w\|^2, 
\end{talign}
up to a constant term in $\theta$. If we switch to the more general flow matching scheme, the analog of this term is
\begin{talign} \label{eq:denoiser_FM_1}
    -\frac{\tilde{\beta}}{2\frac{\beta^2_{kh}}{\alpha^2_{kh}}} \|\hat{x}_1(x_{kh}^w) - x_1^w\|^2.
\end{talign}
Using the expression of the denoiser map in terms of the vector field $v$ in equation \eqref{eq:FM_denoiser}, we can rewrite \eqref{eq:denoiser_FM_1} as:
\begin{talign} \label{eq:denoiser_FM_2}
    -\frac{\tilde{\beta}}{2\frac{\beta^2_{kh}}{\alpha^2_{kh}}} \big\|\frac{v(x^w_{kh},kh) - \frac{\dot{\beta}_{kh}}{\beta_{kh}} x^w_{kh}}{\dot{\alpha}_{kh} - \frac{\dot{\beta}_{kh}}{\beta_{kh}} \alpha_{kh}} - x_1^w \big\|^2 = -\frac{\tilde{\beta}}{2
    % \frac{\beta^2_{kh}}{\alpha^2_{kh}}
    } \big\|\frac{v(x^w_{kh},kh) - \frac{\dot{\beta}_{kh}}{\beta_{kh}} x^w_{kh}}{\frac{\dot{\alpha}_{kh}}{\alpha_{kh}} \beta_{kh} - \dot{\beta}_{kh}} - \frac{\alpha_{kh}}{\beta_{kh}} x_1^w \big\|^2.
\end{talign}
Thus, the Diffusion-DPO loss for Flow Matching reads
\begin{talign}
\begin{split} \label{eq:diff_dpo_2}
    L_{\mathrm{DPO}}(\theta) &= - \mathbb{E}_{(x_1^w, x_1^l) \sim \mathcal{D}, k \sim U[0,K], x_{kh}^w \sim q(x_{kh}^w|x_1^w), x_t^l \sim q(x_{kh}^l|x_1^l)} \big[ \\ &\qquad \log S \big( - \frac{\tilde{\beta}}{2} 
    \big( \big\|\frac{v_{\theta}(x^w_{kh},kh) - \frac{\dot{\beta}_{kh}}{\beta_{kh}} x^w_{kh}}{\frac{\dot{\alpha}_{kh}}{\alpha_{kh}} \beta_{kh} - \dot{\beta}_{kh}} - \frac{\alpha_{kh}}{\beta_{kh}} x_1^w \big\|^2 - \big\|\frac{v_{\mathrm{ref}}(x^w_{kh},kh) - \frac{\dot{\beta}_{kh}}{\beta_{kh}} x^w_{kh}}{\frac{\dot{\alpha}_{kh}}{\alpha_{kh}} \beta_{kh} - \dot{\beta}_{kh}} - \frac{\alpha_{kh}}{\beta_{kh}} x_1^w \big\|^2 \\ &\qquad\qquad\qquad\quad - \big( \big\|\frac{v_{\theta}(x^l_{kh},kh) - \frac{\dot{\beta}_{kh}}{\beta_{kh}} x^l_{kh}}{\frac{\dot{\alpha}_{kh}}{\alpha_{kh}} \beta_{kh} - \dot{\beta}_{kh}} - \frac{\alpha_{kh}}{\beta_{kh}} x_1^l \big\|^2 - \big\|\frac{v_{\mathrm{ref}}(x^l_{kh},kh) - \frac{\dot{\beta}_{kh}}{\beta_{kh}} x^l_{kh}}{\frac{\dot{\alpha}_{kh}}{\alpha_{kh}} \beta_{kh} - \dot{\beta}_{kh}} - \frac{\alpha_{kh}}{\beta_{kh}} x_1^l \big\|^2 \big) \big) \big) \big],
\end{split}
\end{talign}
\cite[Sec.~5.1]{wallace2023diffusion} claim that $\beta \in [2000,5000]$ yields good performance on Stable Diffusion 1.5 and Stable Diffusion XL-1.0, which if we translate to our notation corresponds to $\tilde{\beta} \in [4000,10000]$.

When we have access to the reward function $r$, instead of a winning sample $x^w_1$ and a losing sample $x^l_1$, we have a pair of samples $(x^a_1,x^b_1)$ with winning weights $S(r(x^a_1) - r(x^b_1)) = \frac{1}{1+\exp \big(r(x^b_1) - r(x^a_1) \big)}$, $S(-(r(x^a_1) - r(x^b_1))) = \frac{1}{1+\exp \big(-(r(x^b_1) - r(x^a_1)) \big)}$. Hence, the loss \eqref{eq:diff_dpo_2} becomes:
\begin{talign}
\begin{split} \label{eq:diff_dpo_3}
    L_{\mathrm{DPO}}(\theta) &= - \mathbb{E}_{(x_1^a, x_1^b) \sim \mathcal{D}, k \sim U[0,K], x_{kh}^a \sim q(x_{kh}^a|x_1^a), x_t^b \sim q(x_{kh}^b|x_1^b)} \bigg[ \sum_{s \in \{ \pm 1\}} S\big( s (r(x^a_1) - r(x^b_1)) \big) \times \\ &\qquad \log S \big( - \frac{s\tilde{\beta}}{2} 
    \big( \big\|\frac{v_{\theta}(x^a_{kh},kh) - \frac{\dot{\beta}_{kh}}{\beta_{kh}} x^a_{kh}}{\frac{\dot{\alpha}_{kh}}{\alpha_{kh}} \beta_{kh} - \dot{\beta}_{kh}} - \frac{\alpha_{kh}}{\beta_{kh}} x_1^a \big\|^2 - \big\|\frac{v_{\mathrm{ref}}(x^a_{kh},kh) - \frac{\dot{\beta}_{kh}}{\beta_{kh}} x^a_{kh}}{\frac{\dot{\alpha}_{kh}}{\alpha_{kh}} \beta_{kh} - \dot{\beta}_{kh}} - \frac{\alpha_{kh}}{\beta_{kh}} x_1^a \big\|^2 \\ &\qquad\qquad\qquad\quad - \big( \big\|\frac{v_{\theta}(x^b_{kh},kh) - \frac{\dot{\beta}_{kh}}{\beta_{kh}} x^b_{kh}}{\frac{\dot{\alpha}_{kh}}{\alpha_{kh}} \beta_{kh} - \dot{\beta}_{kh}} - \frac{\alpha_{kh}}{\beta_{kh}} x_1^b \big\|^2 - \big\|\frac{v_{\mathrm{ref}}(x^b_{kh},kh) - \frac{\dot{\beta}_{kh}}{\beta_{kh}} x^b_{kh}}{\frac{\dot{\alpha}_{kh}}{\alpha_{kh}} \beta_{kh} - \dot{\beta}_{kh}} - \frac{\alpha_{kh}}{\beta_{kh}} x_1^b \big\|^2 \big) \big) \big) \bigg].
\end{split}
\end{talign}
We want to emphasize that despite the similarities, even though the loss $L_{\mathrm{DPO}}$ that we use (equation \eqref{eq:diff_dpo_3}) is very similar to the one implemented by \cite{wallace2023diffusion}, the preference data pairs that we use are very different from theirs. We sample the preference data from the current model, which results in imperfect samples, while they consider off-policy, high-quality, curated preference samples. The reason for this discrepancy is that the starting point of our work is a reward model, not a set of preference data, and we only benchmark against approaches that leverage reward models for an apples-to-apples comparison. Our experimental results on DPO (\Cref{tab:evaluation_metrics}, \Cref{fig:training_figures}, \Cref{table:metrics_multiprompt_diversity}) show that the resulting model performs like the base model, or a bit worse according to some metrics. Hence, we conclude that DPO is not a competitive alternative for on-policy fine-tune when the base model is not already good.
% The Euler-Maruyama discretization of the SDE
% \eqref{eq:memoryless_FM_sde} is:
% \begin{talign}
%     x_{(k+1)h} \sim N\big( x_{kh} - \frac{\alpha_{(k+1)h} - \alpha_{kh}}{\alpha_{kh}} x_{k} + 2 h v(x_{kh},kh), 2 \beta_{kh} \big( \frac{\alpha_{(k+1)h} -\alpha_{kh}}{\alpha_{kh}} \beta_{kh} - (\beta_{(k+1)h}-\beta_{kh}) \big) \big)
% \end{talign}

\section{Experimental details}
\label{sec:experimental_details}

Unless otherwise specified, we used the same hyperparameters across all fine-tuning methods. Namely, we used:
\begin{itemize}
    \item $K=40$ timesteps.
    \item Adam optimizer with learning rate $\num{2e-5}$ and parameters $\beta_1 = 0.95$, $\beta_2 = 0.999$, $\epsilon=\num{1e-8}$, weight decay $\num{1e-2}$, gradient norm clipping value $1$. For Discrete Adjoint, these hyperparameters resulted in fine-tuning instability (see \Cref{table:alternative_hyperparameters}); the results that we report in all other tables for Discrete Adjoint were obtained with learning rate $\num{1e-5}$.
    \item Bfloat16 precision.
    \item Effective batch size 40; for each run we used two 80GB A100 GPUs with batch size 20 each.
    \item A set of 40k fine-tuning prompts taken from a licensed dataset consisting of text and image pairs (note that we disregarded the images). Thus, each epoch lasts 1000 iterations; see the total amount of fine-tuning iterations for each algorithm in \Cref{table:metrics_multiprompt_diversity}. For each of the three runs that we perform for each data point that we report, the set of 40k prompts is sampled independently among a total set of 100k prompts.
\end{itemize}

\subsection{Noise schedule details} \label{subsec:noise_schedule}
Since we use $K=40$ discretization steps, the timesteps are $t \in \{0, 0.025, 0.05, 0.075, 0.1, \dots, 0.95, 0.975\}$. To sample $X_{t+h}$ from $X_t$ we use equation \eqref{eq:EM_update_box}. We use the choices $\alpha_t = t$, $\beta_t = 1-t$, which means that $\sigma(t) = \sqrt{2 \beta_t (\frac{\dot{\alpha}_t}{\alpha_t} \beta_t - \dot{\beta}_t)} = \sqrt{2 (1-t) (\frac{1-t}{t} + 1)} = \sqrt{\frac{2 (1-t)}{t}}$. 

Note that if we plug $t=0$ into this expression, we obtain infinity, and if we plug $t \lessapprox 1$, we obtain $\sigma(t) \approx 0$. For obvious reasons, the former issue requires a fix: we simply add a small offset to the denominator of $\sigma(t)$, replacing $\sqrt{1/t}$ by $\sqrt{1/(t+h)}$ (note that $h:= 1/K = 0.025$). But the latter issue is also not completely satisfactory from a practical standpoint, because looking at the adjoint matching loss \eqref{eq:lean_adjoint_matching}, we observe that $u(X^{\bar{u}}_t,t)$ is trained to approximate the conditional expectation of  $\sigma(t)\tran{} \tilde{a}(t;\bm{X}^{\bar{u}})$. Thus, if we set $\sigma(t)$ very close to zero for $t \lessapprox 1$, we are forcing the control $u$ to be close to zero as well, or equivalently preventing $v^{\mathrm{finetune}}$ from deviating from $v^{\mathrm{base}}$. While this is the right thing to do from a theoretical perspective, we concluded experimentally that setting $\sigma(t)$ just slightly larger results in substantially faster fine-tuning, thanks to the additional leeway provided to $v^{\mathrm{finetune}}$ to deviate from $v^{\mathrm{base}}$. In particular, we added a small offset to the factor $1-t$ in the numerator $1-t$ of $\sigma(t)$: we replaced $1-t$ by $1-t+h$. Thus, the expression that we used to compute the diffusion coefficient in our experiments is
\begin{talign} \label{eq:sigma_empirical}
    \sigma(t) = \sqrt{\frac{2 (1-t + h)}{t + h}}.
\end{talign}
When solving the lean adjoint ODE \eqref{eq:lean_adjoint_1}-\eqref{eq:lean_adjoint_2} backwards in time via the Euler scheme \eqref{eq:Euler_lean_adjoint}, the timesteps we use are $t\in \{1, 0.975, 0.95, 0.925, 0.9, \dots, 0.05, 0.025\}$. We do not actually initialize the adjoint state as $\nabla_x g(X_1)$, but rather as $\nabla_x g(\hat{X}_1)$, where $\hat{X}_1 := X_{1-h} + h v^{\mathrm{base}}(X_{1-h},1-h)$. That is, $\hat{X}_1$ is obtained by performing a final noiseless update, instead of using noise $\sigma(1-h) = \sqrt{4h}$ given by equation \eqref{eq:sigma_empirical}. The reason for this is that the regular final iterate $X_1$ contains some noise that was added in the final step, and that can distort the gradient $\nabla_x g(X_1)$. By setting $\tilde{a}(1;\bm{X}) = \nabla_x g(X_1)$, we get rid of this bias. Note that in the continuous time limit $h \to 0$, $\hat{X}_1 = X_1$, which means that this small trick is consistent.

\subsection{Selection of gradient evaluation timesteps} \label{subsec:timestep_selection}
In \Cref{alg:adjoint_matching_finetuning_FM}, equation \eqref{eq:adj_matching_algorithm_box}, we state that the term $\big\|\frac{2}{\sigma(t)} \big(v^{\mathrm{finetune}}_{\theta}(X_t, t) - v^{\mathrm{base}}(X_t, t) \big) + \sigma(t) \tilde{a}_t \big\|^2$ must be computed for all $K$ steps in $\{0, \dots, 1-h\}$. However, the gradient signal provided by backpropagating through this expression for consecutive times $t$ and $t+h$ is quite similar. In the interest of computational efficiency, we sample a subset $\mathcal{K}$ of timesteps, and we only compute and backpropagate the terms $\big\|\frac{2}{\sigma(t)} \big(v^{\mathrm{finetune}}_{\theta}(X_t, t) - v^{\mathrm{base}}(X_t, t) \big) + \sigma(t) \tilde{a}_t \big\|^2$ for those timesteps. We construct $\mathcal{K}$ by sampling ten timesteps uniformly without repetition among $\{0, \dots, 0.725\}$, and always sampling the last ten timesteps $\{0.75, \dots, 0.975\}$. This is because fine-tuning the last ten steps (25\% of the total) well is critical for good empirical performance, while the initial steps are not as important.

\subsection{Loss function clipping: the $\mathrm{LCT}$ hyperparameter}
Note that the magnitude of $\sigma(t)\tran{} a(t;\bm{X}^{\bar{u}},\bar{u})$ is much larger for times $t \gtrapprox 0$ than for times $t \lessapprox 1$. The reason is two-fold: 
\begin{itemize}
    \item As discussed in \Cref{subsec:noise_schedule}, $\sigma(t)$ is much larger for $t \gtrapprox 0$ than for $t \lessapprox 1$.
    \item The magnitude of the lean adjoint state $\tilde{a}$ grows roughly exponentially as $t$ goes backward in time. In fact, if we assumed that $\nabla_x b(X_t,t)$ is constant in time, this statement would be exact.
\end{itemize}
Observe that when $\sigma(t)\tran{} a(t;\bm{X}^{\bar{u}},\bar{u})$ is large, the gradient $\nabla_{\theta} \big\|\frac{2}{\sigma(t)} \big(v^{\mathrm{finetune}}_{\theta}(X_t, t) - v^{\mathrm{base}}(X_t, t) \big) + \sigma(t) \tilde{a}_t \big\|^2$ also has a high magnitude. Including such terms in our gradient computation decreases the signal to noise ratio of the gradient. Even more so, as discussed in \Cref{subsec:timestep_selection} for good practical performance it is critical to get a good gradient signal from the last 25\% steps. Hence, including the high-magnitude terms for $t \lessapprox 0$ in our gradients can muffle these other important, low-magnitude terms.

To fix this issue, we clip the terms such that $\big\|\frac{2}{\sigma(t)} \big(v^{\mathrm{finetune}}_{\theta}(X_t, t) - v^{\mathrm{base}}(X_t, t) \big) + \sigma(t) \tilde{a}_t \big\|^2 > \mathrm{LCT}$, where $\mathrm{LCT}$ stands for the loss clipping threshold. That is, the adjoint matching loss that we use in our experiments is of the form:
\begin{talign}
    \hat{\mathcal{L}}_{\mathrm{Adj-Match}}(\theta) =
        \sum_{t\in \mathcal{K}} \min\big\{\mathrm{LCT}, \big\|\frac{2}{\sigma(t)} \big(v^{\mathrm{finetune}}_{\theta}(X_t, t) - v^{\mathrm{base}}(X_t, t) \big) + \sigma(t) \tilde{a}_t \big\|^2 \big\},
\end{talign}
where $\mathcal{K}$ is the random timestep subset described in \Cref{subsec:timestep_selection}. 

For adjoint matching, we set $\mathrm{LCT} = 1.6 \times \lambda^2$. Remark that $\mathrm{LCT}$ needs to grow quadratically with $\lambda$, because the magnitude of the lean adjoint $\tilde{a}$ grows quadratically with $\lambda$. We set the constant 1.6 through experimentation; all or almost all of the terms for the last ten timesteps fall below $\mathrm{LCT}$, but only a fraction of the terms ($\approx 25 \%$) for the first ten steps fall below $\mathrm{LCT}$. The constant for $\mathrm{LCT}$ is a relevant hyperparameter that needs to be tuned to obtain a similar behavior.

We also used loss function clipping on the continuous adjoint loss. For that loss we set $\mathrm{LCT} = 1600 \times \lambda^2$. The reason is that the magnitude of the regular adjoint states is significantly larger than the magnitude of the lean adjoint states (which is a big reason why adjoint matching outperforms the continuous adjoint).

\subsection{Computation of evaluation metrics} \label{subsec:evaluation_metrics}
We used the \url{open_clip} library \citep{ilharco2021openclip} to compute ClipScores. We computed ClipScore diversity as the variance of Clip embeddings of 40 generations for a given prompt, averaged across 25 prompts. Namely,
\begin{talign}
    \mathrm{ClipScore\_Diversity} = \frac{1}{40} \sum_{k=1}^{40} \frac{2}{25 \cdot 24} \sum_{1 \leq i < j \leq 25} \|\mathrm{Clip}(g^k_i) - \mathrm{Clip}(g^k_j)\|^2, 
\end{talign}
where $g^k_i$ denotes the $i$-th generation for the $k$-th prompt.

We used the \url{transformers} library to compute the PickScore processor and model \citep{kirstain2023pickapic}. PickScore diversity is computed in analogy with ClipScore diversity.

We used the \url{hps} library to compute values of Human Preference Score v2 \citep{wu2023human}.

To compute Dreamsim diversity we use the \url{dreamsim} library \citep{fu2023learning}. Dreamsim diversity is computed in analogy with ClipScore diversity.

\subsection{Remarks on computational costs} \label{subsec:remarks_computational_cost}
Observe from the figures reported in \Cref{table:metrics_multiprompt_diversity} that the per iteration wall-clock time of Adjoint Matching (156 seconds) is very similar to that of the Discrete Adjoint loss (152 seconds). The reason is that both algorithms perform a similar amount of forward and backward passes on the flow matching model and the reward model. Namely, for each sample in the batch, both algorithms perform $K$ forward passes on the flow model to obtain the trajectories. In order to compute the gradient of the training loss, the Discrete Adjoint loss does $K$ additional forward passes to evaluate the base flow model, one forward and backward pass on the reward model, and $K$ backward passes on the current flow model, which typically use gradient checkpointing to avoid memory overflow. In the case of Adjoint Matching, solving the lean adjoint ODE requires one forward and backward pass on the reward model, and $K$ backward passes on the base flow model. Finally, computing the gradient of the loss takes $K/2$ additional backward passes if we evaluate at only half of the timesteps as we do, although this computation is much quicker because it can be fully parallelized.

Meanwhile, computing the gradient of the Continuous Adjoint loss takes 204 seconds per iteration. With respect to Adjoint Matching, Continuous Adjoint performs additional backward passes to compute the gradients $\nabla_{X_t} \|u(X_t,t)\|^2$ when solving the adjoint ODE. Finally, we observe that models that directly fine-tune the reward are quicker, but that comes with its own set of issues that we discuss throughout the paper.

\subsection{Remarks on number of sampling timesteps} \label{subsec:remarks_sampling_timesteps}

In our experiments and all baselines, we used 40 timesteps in the fine-tuning procedure ($h=1/40$ in \Cref{alg:adjoint_matching_finetuning_FM}). The experiments reported in all tables and figures except for \Cref{table:metrics_multiprompt_sampling_steps} were performed at 40 inference timesteps. In \Cref{table:metrics_multiprompt_sampling_steps} (\Cref{sec:additional_figures_tables}), we show experimental results at 10, 20, 40, 100, and 200 inference timesteps, for the base model and the models fine-tuned with adjoint matching and DRaFT-1. We make the following observations about the results: 
\begin{itemize}
\item The metrics for Adjoint Matching at 100 and 200 timesteps are statistically equal to the ones for 40 timesteps, with slight increases in Dreamsim diversity. This suggests that fine-tuning at large numbers of timesteps is a good idea if we want to perform inference at a large number of timesteps, as otherwise the capabilities of the model are limited by the number of fine-tuning timesteps instead of the inference compute. Also, at 100 and 200 timesteps the difference in performance of Adjoint Matching relative to DRaFT-1 increases.
\item The metrics for Adjoint Matching at 10 and 20 timesteps are worse than at 40 timesteps, especially for 10. The difference in performance between Adjoint Matching and DRaFT-1 vanishes at 10 timesteps for all metrics except for diversity, for which Adjoint Matching is still clearly better.
\end{itemize}

\end{document}